%% file: example_paper.tex
\documentclass{article}

\usepackage{microtype}
\usepackage{graphicx}
\usepackage{subcaption}
\usepackage{booktabs} 

\usepackage{hyperref}

\usepackage[preprint]{icml2026}

\usepackage{amsmath}
\usepackage{amssymb}
\usepackage{mathtools}
\usepackage{amsthm}
\usepackage{cleveref}
\usepackage{multirow}
\usepackage{makecell}

\crefname{section}{Section}{Sections}
\Crefname{section}{Section}{Sections}
\crefname{chapter}{Chapter}{Chapters}
\Crefname{chapter}{Chapter}{Chapters}
\crefname{appendix}{Appendix}{Appendices}
\Crefname{appendix}{Appendix}{Appendices}
\crefname{equation}{Eq.}{Eqs.}
\Crefname{equation}{Eq.}{Eqs.}
\crefname{part}{Part}{Parts}
\Crefname{part}{Part}{Parts}
\crefname{table}{Table}{Tables}
\Crefname{table}{Table}{Tables}
\crefname{figure}{Figure}{Figures}
\Crefname{figure}{Figure}{Figures}
\crefname{theorem}{Theorem}{Theorems}
\Crefname{theorem}{Theorem}{Theorems}
\crefname{proposition}{Proposition}{Propositions}
\Crefname{proposition}{Proposition}{Propositions}
\crefname{corollary}{Corollary}{Corollaries}
\Crefname{corollary}{Corollary}{Corollaries}
\crefname{lemma}{Lemma}{Lemmas}
\Crefname{lemma}{Lemma}{Lemmas}
\crefname{finding}{Finding}{Findings}
\Crefname{finding}{Finding}{Findings}
\crefname{assumption}{Assumption}{Assumptions}
\Crefname{assumption}{Assumption}{Assumptions}
\crefname{algorithm}{Algorithm}{Algorithms}
\Crefname{algorithm}{Algorithm}{Algorithms}

\newcount\Comments  
\newcommand{\kibitz}[2]{\ifnum\Comments=1{\color{#1}{#2}}\fi}

\newcount\CommentsAdd  
\newcommand{\kibitzAdd}[2]{\ifnum\CommentsAdd=1{\color{#1}{#2}}\fi}
\definecolor{english}{rgb}{0.0, 0.5, 0.0}

\usepackage{thm-restate}

\makeatletter
\def\thmheadbrackets#1#2#3{%
  \thmname{#1}\thmnumber{\@ifnotempty{#1}{ }\@upn{#2}}%
  \thmnote{ {\the\thm@notefont[#3]}}}
\makeatother

\newtheoremstyle{brakets}
  {}
  {}
  {\itshape}
  {}
  {\bfseries}
  {.}
  { }
  {\thmheadbrackets{#1}{#2}{#3}}

\newtheoremstyle{definitionbrakets}
  {}                       
  {}                       
  {\normalfont}               
  {}                          
  {\bfseries}                 
  {.}                         
  { }                         
  {\thmheadbrackets{#1}{#2}{#3}}

\theoremstyle{brakets}

\newtheorem{theorem}{Theorem}

\theoremstyle{remark}
\newtheorem{remark}[theorem]{Remark}
\theoremstyle{definitionbrakets}

\newtheorem{claim}{Claim}

\usepackage[textsize=tiny]{todonotes}

\input{math_commands}

\icmltitlerunning{LLM Active Alignment}

\begin{document}

\twocolumn[
  \icmltitle{LLM Active Alignment: A Nash Equilibrium Perspective}



  \icmlsetsymbol{equal}{*}

  \begin{icmlauthorlist}
    \icmlauthor{Tonghan Wang}{equal,xxx}
    \icmlauthor{Yuqi Pan}{equal,yyy}
    \icmlauthor{Xinyi Yang}{equal,comp}
    \icmlauthor{Yanchen Jiang}{yyy}
    \icmlauthor{Milind Tambe}{yyy}
    \icmlauthor{David C. Parkes}{yyy}
  \end{icmlauthorlist}

  \icmlaffiliation{xxx}{College of AI, Tsinghua University, Beijing, China}
  \icmlaffiliation{yyy}{Harvard University, Cambridge, MA, USA}
  \icmlaffiliation{comp}{Department of Electronic Engineering, Tsinghua University, Beijing, China}

  \icmlcorrespondingauthor{Tonghan Wang}{tonghanwang1996@gmail.com}

  \icmlkeywords{Machine Learning, ICML}

  \vskip 0.3in
]



\printAffiliationsAndNotice{}  

\begin{abstract}
    We develop a game-theoretic framework for predicting and steering the behavior of populations of large language models (LLMs) through Nash equilibrium (NE) analysis. To avoid the intractability of equilibrium computation in open-ended text spaces, we model each agent’s action as a mixture over human subpopulations. Agents choose actively and strategically which groups to align with, yielding an interpretable and behaviorally substantive policy class. We derive closed-form NE characterizations, adopting standard concave-utility assumptions to enable analytical system-level predictions and give explicit, actionable guidance for shifting alignment targets toward socially desirable outcomes. The method functions as an active alignment layer on top of existing alignment pipelines such as RLHF. In a social-media setting, we show that a population of LLMs, especially reasoning-based models, may exhibit political exclusion—pathologies where some subpopulations are ignored by all LLM agents—which can be avoided by our method, illustrating the promise of applying the method to regulate multi-agent LLM dynamics across domains.

    
    
\end{abstract}

\input{text/1-Introduction}

\input{text/2-Model}
\input{text/4-Experiment}
\input{text/5-Conclusion}


\section*{Impact Statement}
This paper introduces a game-theoretic framework to predict and steer the behavior of populations of Large Language Models (LLMs). A significant societal contribution of this work is the identification and formalization of  ``political exclusion'' arising within stable Nash Equilibria of a population of LLM agents, wherein strategic competition for engagement can lead LLM populations to systematically ignore the perspectives of specific human subpopulations. Our work offers a mathematical foundation for ``active alignment''—enabling system designers to move beyond passive data aggregation and instead tune institutional incentives to ensure fair representation. While our methods could theoretically be used to manipulate LLM behavior, their primary intended impact is to provide governance primitives to ensure that the deployment in multi-agent LLM systems does not lead to outcomes harmful or exclusionary to human groups.

\bibliography{example_paper}
\bibliographystyle{icml2026}

\newpage
\appendix
\onecolumn
\input{text/appendix}

\end{document}

%% file: math_commands.tex
\usepackage{amsmath,amsfonts,bm}

\def\eqref#1{equation~\ref{#1}}

\def\1{\bm{1}}

\newcommand{\cL}{\mathcal{L}}

\newcommand{\cS}{\mathcal{S}}

\def\va{{\bm{a}}}

\def\vq{{\bm{q}}}

\def\vw{{\bm{w}}}

\def\mA{{\bm{A}}}
\def\mB{{\bm{B}}}
\def\mC{{\bm{C}}}

\def\mI{{\bm{I}}}

\def\mQ{{\bm{Q}}}

\newcommand{\R}{\mathbb{R}}

\newcommand{\pk}{{(k)}}

\DeclareMathAlphabet{\mathsfit}{\encodingdefault}{\sfdefault}{m}{sl}
\SetMathAlphabet{\mathsfit}{bold}{\encodingdefault}{\sfdefault}{bx}{n}

\newcommand{\bB}{\mathbf{B}}
\newcommand{\bC}{\mathbf{C}}

\newcommand{\bI}{\mathbf{I}}

\newcommand{\bW}{\mathbf{W}}

\newcommand{\bw}{\mathbf{w}}

\newcommand{\bone}{\mathbf{1}}
\newcommand{\ba}{\mathbf{a}}

%% file: text/1-Introduction.tex
\section{Introduction}

In this paper, we study how to analyze and steer the behavior of a community of large language models (LLMs) by calculating and designing their Nash equilibria (NE)~\citep{sun2025survey,gemp2024steering}. NE is a game-theoretic concept that characterizes stable states of a system of strategic players: no player has an incentive to unilaterally change their strategy, as doing so would not increase their utility~\citep{nash1951non}.
This stability makes NE a natural object for predicting player behavior.

As it becomes increasingly conceivable that the world will be populated by powerful LLM agents~\citep{dafoe2021cooperative,chuang2024simulating,yang2024oasis,wu2024autogen,park2023generative}, understanding their NE becomes important~\citep{carichon2025coming,yao2024human}. In particular, as we argue, it provides a game-theoretic approach for addressing the emerging challenge of AI alignment in multi-agent, LLM systems. Specifically, in this paper, we study the NE of  LLM agents to answer the following question:

\begin{quote}\itshape
As LLMs exercise agency over an increasing number of daily activities, they will inevitably interact with and influence one another. How can we ensure that the resulting dynamics do not lead to outcomes harmful or undesirable for humans?
\end{quote}

Our goal is to derive and design the NE of an LLM-agent population as both a prediction and governance primitives for system outcomes. However, the problem of finding a NE is \emph{PPAD-complete}~\citep{daskalakis2009complexity}, presenting a major challenge, with NE calculation becoming intractable as LLMs operate in the space of open-ended text.  Prior work~\citep{yi2025debate,zhu2025reasoning,liu2025llm,liu2024large,huang2025competing,zhang2024nash}  enforces tractability by restricting the action space to small, task-specific sets, such as tones, rhetorical styles, or low-dimensional generation controls (e.g., temperature). While these abstractions make equilibrium computation feasible, the resulting equilibria may offer limited leverage for predicting system-level behavior or designing interventions that remain valid beyond the prescribed action set.

To address this tension, we define LLM strategies as mixtures over human subpopulations, with each subpopulation associated with 
preference or opinion data. 
Under this formulation, an LLM agent actively and strategically selects which human subpopulations to align with. The resulting 
space of modeled LLM behaviors
is substantive: it is rich and interpretable, expresses a continuum of behavior, and enables analyses for realistic LLM populations. 

From a theoretical perspective, we derive closed-form expressions for the Nash Equilibrium (NE) of LLM agents. We focus on settings with concave utilities --- a standard assumption in the field --- and utilize concave games, which serve as a generalization of normal-form games
~\cite{gemp2024approximating,papadimitriou2023computational,rosen1965existence}. 
Our results also allow for governance-relevant equilibrium interventions. They provide direct guidance on designing the environment in which LLM agents operate, 
and the interaction among them, to shift the alignment target of LLM agents, steering collective behavior toward desirable system-level behavior. 





This game-theoretic alignment complements established LLM alignment methods such as {\em RLHF}~\citep{ouyang2022training,rafailov2023direct,bai2022constitutional,wang2024comprehensive} and {\em pluralism}~\citep{sorensen2024position,feng2024modular}. It introduces an active notion of alignment: the alignment target emerges as the rational choice, in strategic interaction, of LLM agents themselves, rather than being passively imposed by regressing them toward an aggregate of human preferences. Our approach can be viewed as a layer built atop existing alignment pipelines, and be used to better navigate the sophisticated, open-ended human-AI ecosystem, where system-level outcomes can differ sharply from the behavior of any single model. As we show, this incentive-aware alignment layer can mitigate undesirable outcomes arising from strategic interaction that may persist even when individual models are well-aligned in isolation.


We instantiate our theoretical framework in a concrete scenario: LLMs operating social media accounts, each with self-interested motives to maximize the attention and engagement of human users. 
Recent  work has shown that algorithms can shape creator incentives and induce unexpected equilibrium effects~\citep{hron2023modeling}. Therefore, it is important
to understand, predict, and steer the equilibrium outcomes of LLMs from a computational perspective.


Within this setting, we observe a disturbing phenomenon: at an NE, independent of how LLMs are aligned individually, the LLM population can exhibit \emph{political exclusion}, in which the opinions of certain subpopulations are ignored by every LLM. Reasoning models, such as \emph{Qwen3-4B-Thinking} and \emph{DeepSeek-R1-Distill-Qwen-7B},
can even exacerbate such exclusion. We characterize and visualize the conditions under which exclusion arises. We further show that our method enables a rigorous analysis about how to design the NE to avoid this undesirable outcome, incentivizing LLMs to attend, actively and rationally, to the voices of all subpopulations. Although we use political exclusion in social media as a running example in this paper, our method is general and can be used to regulate the behavior of LLM agents across a wide range of other domains.

\textbf{Related Work}.\ \ \emph{Game theory and equilibria for LLM agents.}
A growing body of work studies LLMs as strategic agents and uses game-theoretic concepts to analyze or shape their interactions~\citep{sun2025survey}. Several lines of work use equilibria to coordinate multi-LLM reasoning~\citep{yi2025debate,zhu2025reasoning,zhang2024nash}. Others instantiate LLMs as players in explicit game environments, examining their strategic behaviors in multi-agent settings \citep{liu2024large,huang2025competing}. \citet{gemp2024steering} use game-theoretic solvers to steer a single LLM's generation by casting decoding controls, such as tone and argument styles, as actions. Different from existing work, our method introduces subpopulation alignment targets as the strategy space, enabling a theoretically grounded toolkit, supported by empirical evidence, that makes equilibrium analysis feasible while remaining behaviorally informative for steering LLM populations in general scenarios.

\emph{Pluralistic alignment and population simulation}
also represents heterogeneous human values or behaviors via mixtures of subpopulations or group-specific models.
Existing approaches typically use this structure to (i) measure which groups a model reflects \citep{santurkar2023whose}, (ii) train or compose systems that better serve diverse user groups \citep{feng2024modular,chen2025pal,sorensen2024position,yang2026policy}, or (iii) simulate populations by mixing persona- or group-conditioned behaviors \citep{bui2025mixture,xie2025distributional}. By contrast, we repurpose such structure as a strategy space for LLM populations, enabling equilibrium characterization and incentive-based governance analysis.

\emph{Algorithmic curation, engagement incentives, and exposure as a societal quantity.}
Prior work demonstrates that engagement-optimized recommendation and ranking can reshape information exposure by determining what content is amplified and which groups are heard~\citep{bakshy2015exposure,cinelli2021echo,milli2025engagement}. These concerns align with a line of fairness research that explicitly models \emph{exposure} as a primary societal metric~\citep{singh2018fairness,ye2025auditing}. Our equilibrium lens formalizes conditions under which certain groups become systematically under-represented and provides a way to reason about incentive designs that mitigate such outcomes.

\section{NE Preliminaries and Computational Challenges in LLM Populations}

In the simplest abstraction, in a normal-form game $G = (N, (A_i)_{i\in N}, (u_i)_{i\in N})$, each agent $i$ selects an action $a_i \in A_i$ and receives utility $u_i(a_1,\dots,a_n)$. A {\em strategy} in a normal-form game for agent $i$ is a distribution $\sigma_i \in \Delta(A_i)$. We write $\sigma = (\sigma_1,\dots,\sigma_n)$ for a strategy profile and $\sigma_{-i}$ for the strategies of all agents other than $i$. A strategy profile $\sigma^\star$ is a {\em Nash equilibrium (NE}) if no agent can increase its expected utility by unilateral deviation:
\begin{align}
\mathbb{E}[u_i(\sigma^\star_i,\sigma^\star_{-i})] \ge \mathbb{E}[u_i(\sigma_i,\sigma^\star_{-i})]\ \ \forall\sigma_i\in \Delta(A_i),\forall i\in N.    \nonumber
\end{align}
This property makes NE a natural choice for behavioral predictions: what outcomes are stable under strategic pressure.

Despite its conceptual attractiveness,
NE computation belongs to the complexity class \emph{PPAD-complete}~\citep{daskalakis2009complexity}, and is widely conjectured not to admit efficient algorithms.
The problem exacerbates when agents are LLMs. In realistic deployments, an LLM agent’s behavior is not well captured by a small discrete action set. The action is better viewed as a policy mapping prompts and histories to distributions over token sequences. This induces an enormous, implicitly represented strategy space. Moreover, utilities typically depend on long-horizon interactive trajectories, e.g., multi-turn persuasion, negotiation, or collaboration, rendering the evaluation of even a single strategy profile expensive.


%% file: text/2-Model.tex
\section{LLM Active Alignment}

The computational challenges associated with NE in LLM populations mean that any proposal to use NE as a prediction and governance primitive in LLM ecosystems has to \emph{make equilibrium computation feasible without sacrificing the capacity for behavior predictions and regulatory intervention}.
We address this by introducing a structured, low-rank strategy space grounded in human subpopulations and by adopting the framework of concave games. We formalize this setting in \Cref{sec:setting} and, under this structure, derive a closed-form NE characterization in \Cref{sec:calculating}.

\subsection{Subpopulation-Based Strategy Space }\label{sec:setting}

Rather than attempting to compute NE over unconstrained textual policy spaces, we work with a compact strategy space in which a strategy puts weights to aligning LLM behavior with preferences or opinions of each subpopulation.

This design is motivated by the fact that many existing preference and opinion datasets are annotated with respondent attribute labels, such as demographic or psychographic information, enabling principled partitioning of human judgments or opinions into interpretable groups. For example, the $\mathtt{OpinionQA}$ dataset~\citep{santurkar2023whose} includes demographic labels for each response, while the $\mathtt{Big\ Five\ Personality\ Traits}$ dataset~\citep{tseng_2025_bigfive_hf} stratifies respondents along personality dimensions such as openness and conscientiousness. More broadly, subpopulations can be defined using age, education, or other attributes. Subpopulations induced by these labels typically exhibit within-group behavioral coherence, and an LLM may aim for its behavior to reflect and preserve this structure.

We operationalize this idea by partitioning a dataset into $D$ disjoint sub-datasets, each representing a distinct labeled subpopulation.
Each \emph{subpopulation dataset} is denoted by $\cS_d = \{( 
    x_d^\pk, 
    y_d^\pk)\}_{k=1}^{K_d},\ d\in [D]$, 
where $x_d^\pk$ is a prompt (e.g., a question), $y_d^\pk$ is the corresponding human response, and $K_d$ is the number of prompt–response pairs associated with subpopulation $d$.
To represent the behavior of a subpopulation $d$, we train a LLM \emph{subpopulation model}, $\nu_d(y| x)$,  parameterizing a conditional distribution over responses given prompts.
The training is performed by minimizing the standard {\em autoregressive negative likelihood} on $\cS_d$.

With this setup, we consider a system of $M$ strategic players, namely the LLM agents. $M$ could be very large.
The \emph{strategy} of player $m\in [M]$ is a weight vector $w_m \in \Delta_D$, where
$\Delta_D := \{ w\in \R^D_{\ge0} : \sum_{d=1}^D w_d = 1\}$ is the probability simplex over the $D$ subpopulation models.
Given $w_m$, player $m$ induces a conditional response distribution 
$\pi_m(\cdot \mid x)$ by mixing the subpopulation conditionals:
\begin{align}
   \pi_m(y|x)=\sum_{d=1}^D \bw_{m,d}\nu_{d}(y| x),\qquad \forall x,y.\label{eq:strategy}
\end{align}


\textbf{Advantages of this setup}.\ \ \cref{eq:strategy} yields an immediate complexity benefit by replacing an implicit, prohibitively large textual policy space 
with a tractable, low-dimensional convex strategy set over weights, while still permitting equilibrium-based reasoning to be performed efficiently on realistic instances.

Meanwhile, this reduction does not compromise the representational and operational capacity of the framework. At an NE, an LLM player's strategy $\pi_m$
represents an alignment target. Once aligned, an NE well defines the players' behavior across the full range of tasks expressible by the underlying LLMs. At the same time, it reframes alignment as an endogenous choice and enables governance interventions to be analyzed in terms of how they reshape incentives over subpopulation mixtures—i.e., how LLM agents are incentivized to \emph{actively choose target subpopulations to align with}. 
This approach thereby represents a paradigm shift toward LLM \emph{active} alignment, functioning as an incentive layer atop existing alignment pipelines in single-LLM settings such as RLHF, RLAIF~\citep{bai2022constitutional,lee2024rlaif,sharma2024critical}, etc.

\subsection{NE Calculation}\label{sec:calculating}

We now discuss how to characterize NE in our setting. NE is determined by the utility function, depending on the specific scenario under consideration. In this paper, we focus on
a social media environment, where each player $m$ runs a social media account. In \Cref{sec:discussion}, we explain how the proposed method extends to a broad class of other settings and applications.

\subsubsection{Utility Functions}

We consider a utility function composed of three objectives that are natural in the social media context: (1) {\em Attractiveness}: players aim to maximize the attention they receive from human users, e.g., likes or reposts; (2) {\em Consistency}: players aim to avoid generating inconsistent content, which requires internal consistency of the mixed strategy each player adopts; and (3) {\em Diversity}: players prefer not to behave too similarly to other players since overlapping strategies dilute marginal attention and influence.


\textbf{Attractiveness}.\ \ The {\em attractiveness utility} is defined as
\begin{align}
u_m^{(A)}(\bw_m)= \va^\top \mathbf{w}_m.\label{eq:u_a}    
\end{align}
Here we define the {\em attractiveness vector} $\va\in\mathbb{R}^D$, where $a_i$ represents the  fractional size of subpopulation $i\in [D]$, i.e., $a_i=N_i/\sum_{j\in[D]} N_j$, where $N_i$ is the size of subpopulation $i$.
%
Intuitively, $u_m^{(A)}$ interprets population share as a proxy for expected attention/reach: in standard diffusion and influence models, the canonical objective is the expected size of population reached (or audience)~\citep{kempe2003maximizing}. Empirically, large-scale studies of online diffusion on Twitter find that larger audiences are associated with larger information cascades,
consistent with audience size scaling with attention~\cite{bakshy2011everyone,goel2016structural}.


\textbf{Consistency}.\ \ It is possible that cross-subpopulation inconsistency 
is inherent in an equilibrium strategy, which is undesirable because a player might generate conflicting responses. To prevent this, we introduce an
{\em inconsistency penalty},
\begin{align}
    u_m^{(I)} (\bw_m)\;=\; -\bw_m^\top \mC \bw_m,\label{eq:u_i}
\end{align}
where {\em inconsistency matrix} $C \in \mathbb{R}^{D \times D}$ consists of entries
\[
\mC_{ij}=
\begin{cases}
\displaystyle \frac{1}{K}\sum_{k=1}^K \Delta_{ij}(y^\pk,x^\pk), & i\neq j,\\[4pt]
\displaystyle \frac{1}{K}\sum_{k=1}^K \sum_{l\neq i} \Delta_{il}(y^\pk,x^\pk), & i=j.
\end{cases}
\]
Here $\Delta_{ij}(y^\pk,x^\pk):= \psi\Big(\big|\nu_i(y^\pk\mid x^\pk)-\nu_j(y^\pk\mid x^\pk)\big|\Big)$, and $\psi:\mathbb{R}_{\ge 0}\to\mathbb{R}_{\ge 0}$ can be any monotone increasing function (e.g., $\psi(r)=r$, $\psi(r)=r^p$ with $p\!\ge\!1$, or $\psi(r)=\log(1+r)$). 

\begin{claim}[$\mC$ is positive semidefinite]
\label{claim:C}
    By construction, $\mC$ is symmetric and diagonally dominant. By the Gershgorin circle theorem~\citep{varga2011gervsgorin}, all eigenvalues of $\mC$ are at least $0$. Therefore $\mC\succeq 0$.
\end{claim}

Expanding the quadratic form gives us
\begin{align}
\bw_m^\top& \mC \bw_m
\;=\; \frac{1}{K}\sum\nolimits_{k=1}^K\sum\nolimits_{1\le i<j\le D} \label{eq:u_i_expansion}\\ 
&\psi\big(|\nu_{i}(y^\pk|x^\pk)-\nu_{j}(y^\pk|x^\pk)|\big)\,\big(w_{m,i}+w_{m,j}\big)^2,    \nonumber
\end{align}
According to \Cref{eq:u_i_expansion}, placing substantial weights on two subpopulations inconsistent with each other incurs a large penalty, echoing the intuition that mixing over
incompatible opinions makes it harder to remain internally consistent.

\textbf{Diversity}.\ \ We incorporate diminishing returns from redundant content: when many agents choose similar strategies, incremental attention and influence decrease. We model this competitive pressure by a preference of each agent for strategies with low similarity to those of other agents:
\begin{align}\label{eq:u_d}
    u_m^{(D)}(\mathbf{w}_m, \mathbf{w}_{-m}) = -\!\sum_{j \neq m} \langle \mathbf{w}_m, \mathbf{w}_j \rangle = -\!\sum_{i=1}^D \mathbf{w}_{m,i} \sum_{j \neq m} \mathbf{w}_{j,i} .
\end{align}
Let \(\bW=[\bw_1,\ldots,\bw_M]\in\mathbb{R}^{D\times M}\) be the strategy profile of the LLM population. The pairwise inner products in \cref{eq:u_d} are the off-diagonal entries of $\bW^\top\bW$. Penalizing similarity via \cref{eq:u_d} has similar effect to minimizing
\begin{align}
     \|\bW^\top \bW - \bI\|_F^2 =  \sum_{i=1}^{M} \left( \langle \bw_i, \bw_i \rangle - 1 \right)^2 + \sum_{i \ne j} \langle \bw_i, \bw_j \rangle^2,   \nonumber
\end{align}
which is a standard way to encourage collective diversity. For this reason, we refer to \(u_m^{(D)}\) as the \emph{diversity utility}.


%

With \Cref{eq:u_a}, (\ref{eq:u_i}), and (\ref{eq:u_d}), the utility function for player $m$ is
\begin{align}
    u_m(\bw_m,\bw_{-m}) =&\beta^{(A)} \va^\top\bw_m-\beta^{(I)}\bw_m^\top \bC\bw_m\notag\\
    &-\beta^{(D)}\sum_{j\neq m} \langle\bw_m,\bw_j\rangle,\label{eq:u}
\end{align}
where $\beta^{(A)}, \beta^{(I)}, \beta^{(D)} > 0$ are scaling factors. In our setting, these values are effectively induced by the incentivizing mechanism of a social media platform, such as ranking objectives, exposure allocation rules, or recommendation controls, and are shared across agents.


\subsubsection{Equilibrium Computation}

We now follow the framework of concave games to derive a tractable approach to computing
NE.
%
\begin{restatable}{lemma}{gradientcond}\label{lemma: gradientcondition}
    With utility function $u_m(\bw_m,\bw_{-m})$ defined in \Cref{eq:u}. A strategy $\bw_m$ in the interior of the simplex is a best response to $\bw_{-m}$ if and only if the utility gradient satisfies
    \begin{align}
    \nabla_{\bw_m} u_m=-\beta^{(D)}\sum_{j\neq m} \bw_{j}+\beta^{(A)} \va-2\beta^{(I)}\bC\bw_m = \lambda_m\mathbf{1},\nonumber
    \end{align}
    for some $\lambda_m\in \mathbb R$.
\end{restatable}
Here, ``interior'' means $w_m$ assigns strictly positive mass to every subpopulation $i\in[D]$. 
The detailed proof can be found in \Cref{appx:proof} and is based on~\citet{gemp2024approximating}. 
In an interior NE, for any $m\in[M]$, $\bw_m$ is a best response for $\bw_{-m}$. Therefore, an interior NE can be obtained by solving the following linear system
\begin{align}
\hspace{-0.6em}
\begin{cases}
\!-\!\beta^{(D)}\!\sum\limits_{j\neq m}\!\bw_{j}\!+\!\beta^{(A)} \va\!-\!2\beta^{(I)}\bC\bw_m
\!=\!\!\lambda_{m}\bone,\forall m;\\
\bone^\top \bw_m = 1, \forall m.
\end{cases}
\label{eq:br_kkt}
\end{align}

Here we impose the normalization constraint that weights sum to one, i.e.\ $\bone^\top \bw_m=1$ for each $m$. 
To derive an analytical solution for an interior equilibrium, we first temporarily relax the interior constraint, $\bw_m> \mathbf{0}$. As we will discuss later, the solution to this relaxed problem constitutes the unique valid interior equilibrium if and only if all resulting weights are positive. Conversely, if the relaxed problem yields no such solution with all weights being strictly positive, then no interior equilibrium exists. Appx.~\ref{appx:boundary_ne} discusses deriving NE that lie on the boundary of the simplex.


For notational convenience, let 
\begin{align}
    \mA=2\beta^{(I)} \bC-\beta^{(D)}\mI,
\end{align}
Then the first group of equations in \Cref{eq:br_kkt} is equivalent to
\begin{align}
    (\mA+\beta^{(D)}\mI)\bw_m+\beta^{(D)}\sum_{j\neq m} \bw_{j}
    =\beta^{(A)}\va-\lambda_{m}\bone,  \label{eq:solve_weight_block}
\end{align}
for $\forall m\in[M]$. 
Here $\va$ is the attractiveness vector, not a vector in $A$.
Taking any two equations from \Cref{eq:solve_weight_block} and subtracting them gives
\[
\mA(\bw_i-\bw_j)=(\lambda_j-\lambda_i)\,\bone.
\]
Since $\mA$ is an affine combination of symmetric matrices, it is symmetric.
If $\mu_1,\dots,\mu_D$ are the eigenvalues of $\mC$, then the eigenvalues of $\mA$ are
$2\beta^{(I)}\mu_j - \beta^{(D)}$ for $j=1,\dots,D$.
Hence $\mA$ is nonsingular (full rank) iff $2\beta^{(I)}\mu_j \neq \beta^{(D)}$ for all $j$.
This excludes only finitely many hyperparameter values (a measure-zero set). Therefore, we may, without loss of generality,
assume that $\mA$ is full rank. This can always be enforced by adding an arbitrarily small
ridge term, $\varepsilon \mI$, with $\varepsilon>0$,
which is equivalent to replacing $\beta^{(D)}$ by $\beta^{(D)}+\varepsilon$.
Using \Cref{eq:solve_weight_block} and solving for $\bw_m$ yields
\[
(\mA+M\beta^{(D)}\mI)\bw_m
\!=\!\beta^{(A)}\va-\lambda_m\bone
+\!\sum_{i=1}^M(\lambda_i-\lambda_m)\,\beta^{(D)}\!\mA^{-1}\bone.
\]
Since $\mC$ is positive semidefinite and $M>1$, the eigenvalues of $\mA+M\beta^{(D)}\mI$: $2\beta^{(I)}\mu_i+(M-1)\beta^{(D)}$ are larger than 0, which means $\mA+M\beta^{(D)}\mI$ is invertible. Therefore,
\begin{align}\label{eq:equilibrium_general}
&\bw_m
=\biggl(2\beta^{(I)}\bC+(M-1)\beta^{(D)}\mI\biggl)^{-1}
\biggl(
\beta^{(A)}\va-\lambda_m \bone \nonumber\\
&+\beta^{(D)}\sum_{i=1}^M (\lambda_i-\lambda_m)\,(2\beta^{(I)}\bC-\beta^{(D)} \mI)^{-1}\bone
\biggl).
\end{align}

We simplify this expression by introducing the normalization constraint, $\bone^\top\bw_m=1$, in \Cref{eq:br_kkt}.
Let
\[
\mB=(\mA+M\beta^{(D)}\mI)^{-1}=\big(2\beta^{(I)}\mC+(M-1)\beta^{(D)}\mI\big)^{-1}.\nonumber
\]
Premultiplying \Cref{eq:equilibrium_general} by $\bone^\top$ and using $\bone^\top\bw_m=1$ gives a linear relation among $\{\lambda_i\}$:
\begin{equation}\label{eq:lambda_linear}
-\underbrace{\big(s_1+\beta^{(D)} M s_2\big)}_{\alpha}\,\lambda_m
\;+\;
\beta^{(D)} s_2\,\sum_{i=1}^M \lambda_i
\;=\;
1 - s_0,
\end{equation}
with $s_0\!=\!\bone^\top \bB\,\beta^{(A)}\va$, $s_1\!=\!\bone^\top \mB\bone$, $s_2\!=\!\bone^\top \mB\mA^{-1}\bone$, and
\[
\alpha=\bone^\top \mB(\mI+M\beta^{(D)}\mA^{-1})\bone=\bone^\top\mA^{-1}\bone.
\]

\begin{restatable}[Only finitely many $\beta^{(D)}/\beta^{(I)}$ make $\alpha=0$]{lemma}{alphaiszero}\label{lemma: alphafinite}
Let $\beta:=\beta^{(D)}/\beta^{(I)}$. Take the eigendecomposition
$\mC=\mQ\,\mathrm{diag}(\mu_1,\dots,\mu_D)\,\mQ^\top$ and $\vq:=\mQ^\top\bone$.
Then
\[
\alpha \;=\; \bone^\top \mA^{-1}\bone
\;=\;\frac{1}{\beta^{(I)}}\,\sum_{j=1}^D \frac{q_j^2}{\,2c_j - \beta\,}
\;=:\;\frac{1}{\beta^{(I)}}\,f(\beta).
\]
The equation $\alpha=0$ (equivalently $f(\beta)=0$) has at most finitely many solutions $\beta$.    
\end{restatable}

%
The proof can be found in Appx.~\ref{appx:proof}. In the \emph{generic} case $\alpha\neq 0$, \Cref{eq:lambda_linear} forces all multipliers to be equal:
\[
\lambda_1=\cdots=\lambda_M=\lambda^\star,
\qquad
\lambda^\star=\frac{\beta^{(A)}\bone^\top \bB\,\va-1}{\bone^\top \bB\,\bone}.
\]

$\mB$ is positive definite because $\mC$ is positive semidefinite and $\beta^{(D)},\beta^{(I)}>0, M>1$. Therefore, the denominator is guaranteed to be larger than $0$.

Plugging this back into \Cref{eq:equilibrium_general} cancels the summation term (since $\sum_i(\lambda_i -\lambda^\star)=0$), yielding the interior equilibrium characterization stated in the following theorem.

\begin{theorem}
    The interior Nash equilibrium is unique and homogeneous, characterized by \begin{equation}\label{eq:equilibrium_hom}
        \bw_m=\bw^* =\big(2\beta^{(I)}\mC+(M-1)\beta^{(D)}\mI\big)^{-1}\big(\beta^{(A)}\va-\lambda^\star \bone\big)
    \end{equation}
    for all LLM players, $m\!\in\![M]$. $\va$ is the attractiveness vector and $\mC$ is the inconsistency matrix defined by \Cref{eq:u_a,eq:u_i}.
\end{theorem}


\begin{remark}
 \cref{eq:equilibrium_hom} is derived without explicitly enforcing the interior constraint. Therefore, the interior equilibrium exists and is characterized by \cref{eq:equilibrium_hom} when all dimensions of $\bw^*$ are strictly positive. 
If any component of $\bw^*$ is non-positive, an interior equilibrium does not exist. In this case, the game still has NE that lie on the boundary of the weight simplex (i.e., some subpopulations must have zero weight).  We discuss how to calculate boundary NE in Appx.~\ref{appx:boundary_ne}.

As we have discussed, for \cref{eq:equilibrium_hom} to hold,
we also require the matrix $\mA$ to be invertible (i.e., $\beta^{(D)}/\beta^{(I)} \ne 2\mu_i$) and the normalization factor to be non-zero (i.e., $f(\beta) \ne 0$), see \Cref{lemma: alphafinite}. 
%
Crucially, these singularities occur only at a finite number of discrete hyperparameter ratios. Therefore, the analytical expression (\cref{eq:equilibrium_hom}) remains mathematically well-defined almost everywhere in the hyperparameter space, allowing us to neglect these  singularities in our  analysis.

\end{remark}


\subsection{Discussion}\label{sec:discussion}

\Cref{eq:equilibrium_hom} explicitly maps what shapes incentives—e.g., $\beta^{(A)}$, $\beta^{(I)}$, and $\beta^{(D)}$—to NE as behavioral predictions for the LLM population.
With this map, a social media platform can tune the incentivizing mechanism to shape $\beta^{(A)}$, $\beta^{(I)}$, and $\beta^{(D)}$ and move the equilibrium toward desired system outcomes such as improving social welfare. 
This supports the goal of using Nash equilibria as a predictive theory and actionable regulatory mechanism for LLMs, with examples shown in \Cref{sec:exp} .


\textbf{Applicability of the proposed method}.\ \ 
We instantiate the framework in a social-media setting, with utilities that are natural for that domain. However, the NE computation method defined by \cref{eq:br_kkt} is application-agnostic: it applies whenever each player's utility is concave in its own strategy, holding others fixed. This assumption is standard and covers a broad class of games~\citep{rosen1965existence}.

%% file: text/4-Experiment.tex
\begin{figure*}[t]
    \centering
    \includegraphics[width=0.9\linewidth]{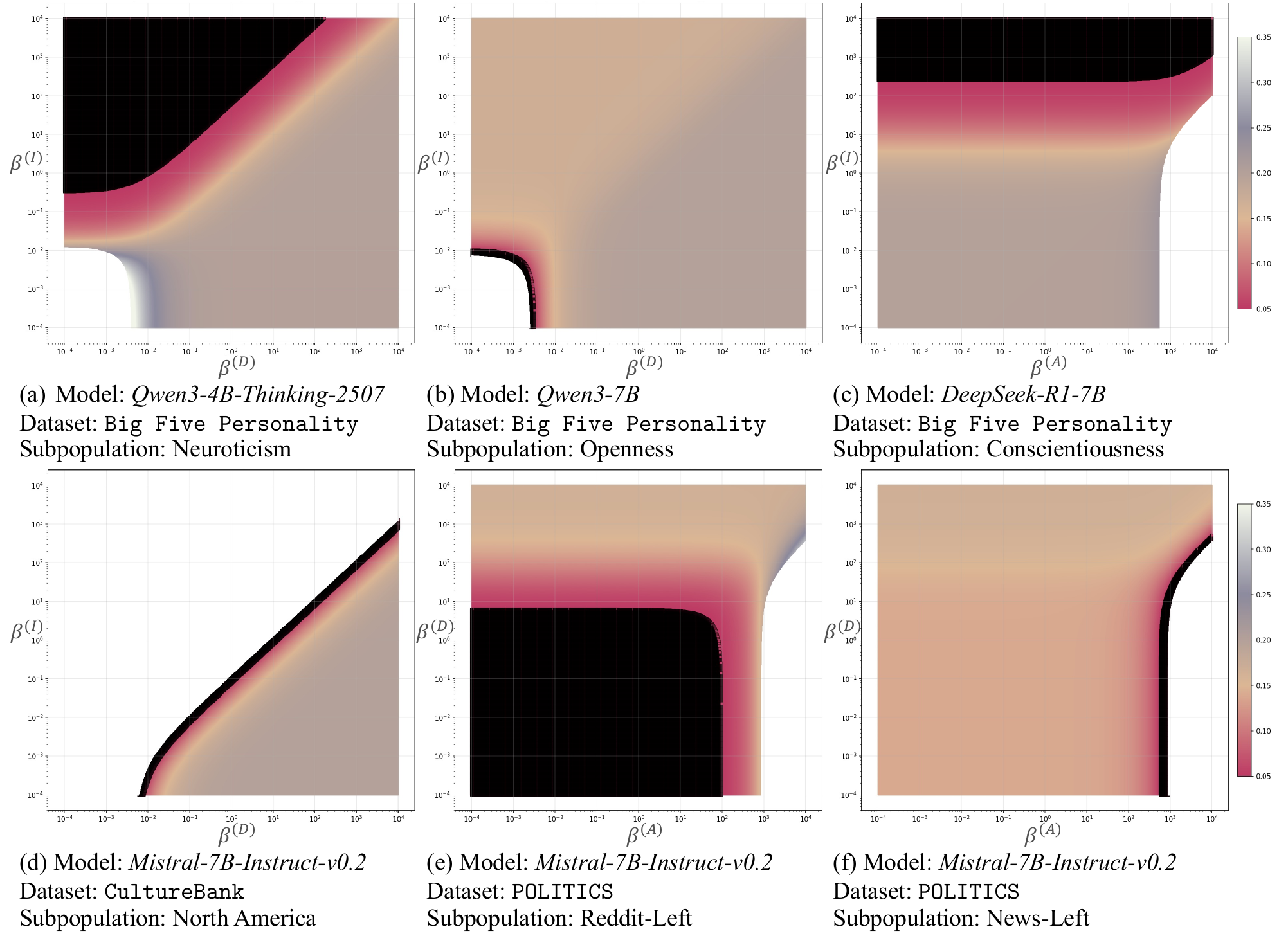}
    \caption{\textbf{Representative examples of political exclusion across models and datasets} (complete results in Appx.~\ref{appx:exp}). Each panel fixes the base model, dataset, and subpopulation, and visualizes the subpopulation’s interior Nash-equilibrium weight under varying preference coefficients. Specifically, one coefficient among $\beta^{(A)}$, $\beta^{(I)}$, and $\beta^{(D)}$ is fixed to $1$, while the remaining two vary along the x- and y-axes (as labeled, log-scaled). Color indicates the resulting equilibrium weight assigned to the focal subpopulation. Regions where the weight falls below $0.05$ are marked in black and referred to as the \emph{political exclusion area}. White regions indicate parameter values for which no interior equilibrium exists. We focus on interior equilibria because boundary equilibria necessarily set at least one subpopulation weight to zero
    and our goal is to understand how to avoid political exclusion.}
    \label{fig:exclusion}
\end{figure*}

\section{Experimental Results}\label{sec:exp}

We design experiments to test our theoretical results in realistic settings and answer the following questions: (1) When do equilibria assign vanishing weights to some subpopulations? How does this exclusion depend on whether attention-seeking ($\beta^{(A)}$), conflict-avoidance ($\beta^{(I)}$), or other objectives ($\beta^{(D)}$) are incentivized? (2) Are exclusion patterns fragile corner cases, or do they form stable, structured regions of the incentive space? (3) Do different LLM models systematically expand or shrink the parameter regimes in which exclusion arises? 

To answer these questions, we sweep $(\beta^{(A)},\beta^{(I)},\beta^{(D)})$ and use \Cref{eq:equilibrium_hom} to map how equilibrium weights vary. We interpret $(\beta^{(A)},\beta^{(I)},\beta^{(D)})$ as \emph{incentive configurations}: a social media platform can design incentivization mechanisms that effectively shape these coefficients.



Our experiments focus primarily on interior equilibria because our ultimate goal is to understand when political exclusion can be avoided. Boundary equilibria mechanically force at least one subpopulation weight to be zero, making exclusion inevitable by definition. 

\subsection{Setup}
We use three datasets with explicit subpopulation labels: $\mathtt{POLITICS}$~\citep{liu2022politics}, $\mathtt{CultureBank}$~\citep{shi2024culturebank} and $\mathtt{Big\ Five\ Personality\ Traits}$ ($\mathtt{Big\ Five}$ for short)~\citep{tseng_2025_bigfive_hf}. They
are widely studied in recent literature~\citep{feng2023pretraining,shum2025big,feng2024modular,myung2024blend,soldaini2024dolma,ye2024openfedllm,hilliard2024eliciting}. The $\mathtt{POLITICS}$ dataset comprises six distinct subpopulations spanning left-, center-, and right-leaning viewpoints, with data drawn from both news outlets and social media. The $\mathtt{CultureBank}$ dataset includes five regional subpopulations. In the $\mathtt{Big\ Five}$ dataset, the five subpopulations correspond to distinct personality traits.

First, to reduce confounds arising from subpopulation-model training and isolate the effects we study, on datasets $\mathtt{POLITICS}$ and $\mathtt{CultureBank}$, we use trained subpopulation models provided by~\citet{feng2024modular}, where \emph{Mistral-7B-Instruct-v0.2} is used as the base model and we hold them fixed as $\nu$. 
%
Furthermore, to ensure the robustness of our observations and analyses,  on the dataset $\mathtt{Big\ Five}$, we conduct experiments across different base models, including \emph{Qwen3-0.6B}, \emph{Qwen3-1.7B}, \emph{Qwen3-4B}, \emph{Qwen3-7B}, \emph{Qwen3-14B}, \emph{DeepSeek-R1-Distill-Qwen-7B}, \emph{Mistral-7B-Instruct-v0.2}, and \emph{Qwen3-4B-Thinking-2507}. 
Among them, \emph{DeepSeek-R1-Distill-Qwen-7B} and \emph{Qwen3-4B-Thinking-2507} are reasoning-based models. We train subpopulation models using standard autoregressive next-token prediction.
We use LoRA with a rank of 16, controlling the number of trainable parameters, and alpha of 16, scaling the influence of the adapted weights. We train the model for 5 epochs using the Adam optimizer with a learning rate of $2e{-5}$, a batch size of 64, and a dropout rate of 0.05.

We visualize the equilibrium weight assigned to each subpopulation defined by \cref{eq:equilibrium_hom}, under varying coefficient configurations $(\beta^{(A)}, \beta^{(I)}, \beta^{(D)})$. Each corresponding to a possible utility function for the LLM population. 

\begin{figure}[t]
    \centering
    \includegraphics[width=\linewidth]{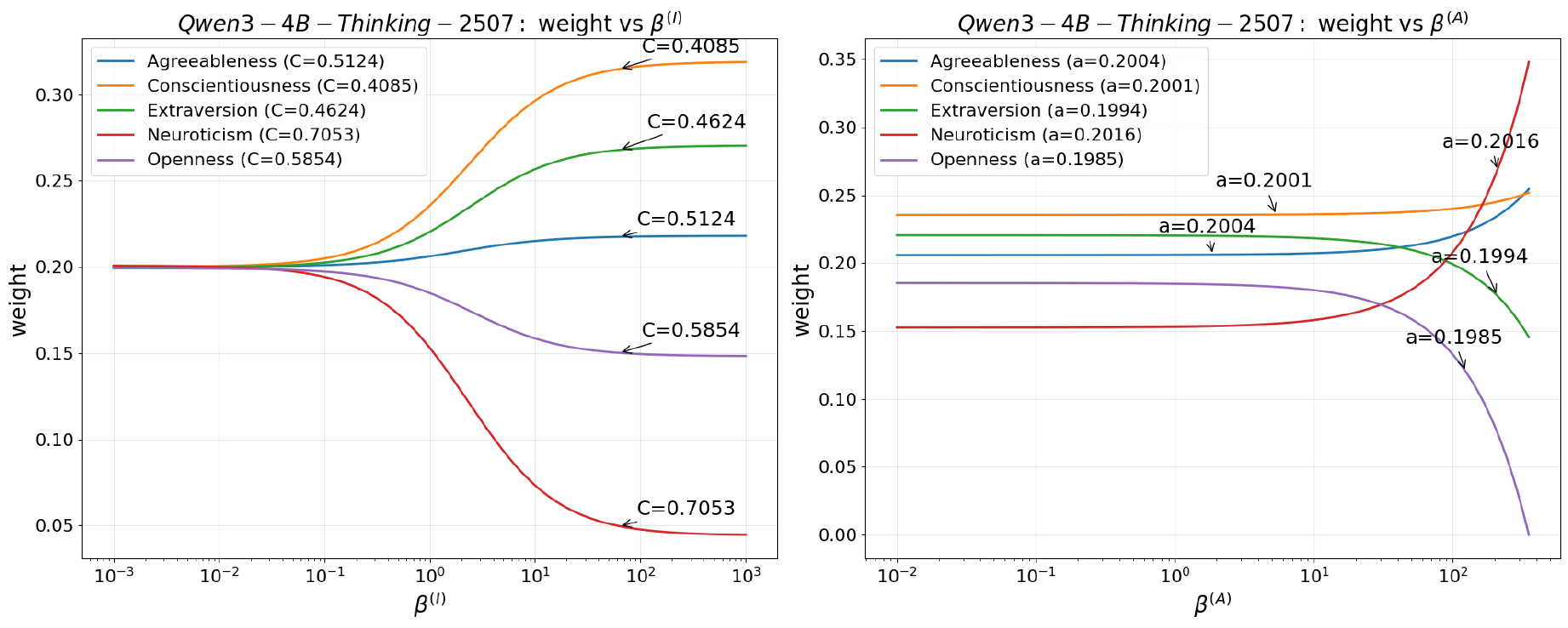}
    \caption{\textbf{Middle-of-the-road survives.} One-dimensional slices of the equilibrium weights for \emph{Qwen3-4B-Thinking-2507} on the dataset $\mathtt{Big\ Five}$. Left: increasing $\beta^{(I)}$ (with $\beta^{(A)}$=$\beta^{(D)}$=$1$) suppresses the most inconsistent subpopulation (largest $C$, the diagonal entries in $\mC$). Right: increasing $\beta^{(A)}$ (with $\beta^{(I)}$=$\beta^{(D)}$=$1$) concentrates weight on more prevalent traits (larger $a$, the entries in $\va$) and can drive the least prevalent trait toward zero.}
    \label{fig:aC}
\end{figure}
\subsection{Observations and Analyses}

\textbf{\emph{Observation 1: Systematic political exclusion emerges as an equilibrium outcome}.} A consistent pattern across models and datasets is that political exclusion is not an artifact or a rare corner case. Across a wide range of incentive configurations, exclusion appears with a highly structured geometry: it typically forms large contiguous regions or stable bands.

\cref{fig:exclusion} gives examples, and complete results are shown in Appx.~\ref{appx:exp}. On $\mathtt{Big\ Five}$, exclusion emerges for different traits under different models. \emph{Qwen3-4B-Thinking-2507} exhibits substantial exclusion of \emph{Neuroticism} over broad regions, \emph{DeepSeek-R1-Distill-Qwen-7B} instead produces strong exclusion of \emph{Conscientiousness}, and \emph{Qwen3-7B} concentrates exclusion on Openness. On $\mathtt{CultureBank}$, \emph{Mistral-7B-Instruct-v0.2} shows pronounced exclusion concentrated on \emph{North America}, with smaller boundary-adjacent exclusion for \emph{Europe}, while \emph{Africa}/\emph{Asia}/\emph{South America} remain robustly represented. Overall, political exclusion is a coherent consequence of strategic optimization under heterogeneous objectives, rather than an implementation quirk or a fragile edge case.

\textbf{\emph{Observation 2: The tyranny of the middle---robust representation favors the “mediocre”.}}
Across models and datasets, we observe that subpopulations that survive most reliably in the interior equilibrium are those that are \emph{middle-of-the-road}.

Exclusion appears in two repeatable ways: (i) large exclusion regions that expand at high $\beta^{(I)}$; and (ii) thin, boundary-hugging bands that typically emerge when $\beta^{(A)}$ is large. The high-$\beta^{(I)}$ pattern reflects a pressure that operates well inside the feasible set. In \cref{fig:aC}(left), increasing $\beta^{(I)}$ (holding $\beta^{(A)}$=$\beta^{(D)}$=$1$) suppresses the most inconsistent trait with the largest $C$ value, \emph{Neuroticism}, toward exclusion. Notably, this can override attractiveness: although \emph{Neuroticism} is the most prevalent subpopulation (largest $a$ in \cref{fig:aC}(right)), it is driven toward near-zero weight because of its inter-subpopulation inconsistency, producing exclusion across a broad high-$\beta^{(I)}$ region (\cref{fig:exclusion}(a)). In other words, coherence can beat popularity: even a majority group can be excluded when it is sufficiently inconsistent with other subpopulations. By contrast, \emph{Agreeableness} (\cref{fig:Openness}) for the same dataset and base model never enters the exclusion area across $\beta^{(A)}$ and $\beta^{(I)}$ sweeps. Consistent with this, stronger attention rewards and stricter inconsistency penalties tend to increase its equilibrium weight (\cref{fig:aC}).


Together, these patterns explain the tyranny of the middle: the most robustly represented groups are those that are \emph{simultaneously} moderately attractive and moderately consistent. Extremes vanish first, and groups that are unpopular or highly discordant are the earliest to be excluded.

\textbf{\emph{Observation 3: Reasoning models expand the exclusion regime.}}
On $\mathtt{Big\ Five}$, reasoning models exhibit a substantially larger exclusion area than non-reasoning models of a comparable size. 
\cref{tab:thinking_invalid_ignore} shows that \emph{Qwen3-4B-Thinking-2507} increases the exclusion area fraction from $0.51\%$ (\emph{Qwen3-4B}) to $4.54\%$, 
and \emph{DeepSeek-R1-Distill-Qwen-7B} increases it from $0.21\%$ (\emph{Qwen3-7B} and \emph{Mistral-7B-Instruct-v0.2}) to $3.72\%$. 
A sharper way to isolate exclusion is to condition on the existence of an interior equilibrium. We define the \emph{conditional exclusion rate} as
\emph{exclusion}/(1-\emph{invalid}), where \emph{invalid} is the area where no interior equilibrium exists and \emph{exclusion} is the exclusion area.
As shown in \cref{tab:thinking_invalid_ignore}, we have $5.04\%$ for \emph{Qwen3-4B-Thinking-2507} versus $1.13\%$ for \emph{Qwen3-4B}, and $4.07\%$ for \emph{DeepSeek-R1-Distill-Qwen-7B} versus $0.22$--$0.23\%$ for the two 7B non-reasoning counterparts. This confirms that 
even when an interior equilibrium is feasible, reasoning models are more likely to place some subpopulation in the exclusion area. 



\subsection{Governance Implications}\label{sec:governance_exp}
\begin{figure}
    \centering
    \includegraphics[width=\linewidth]{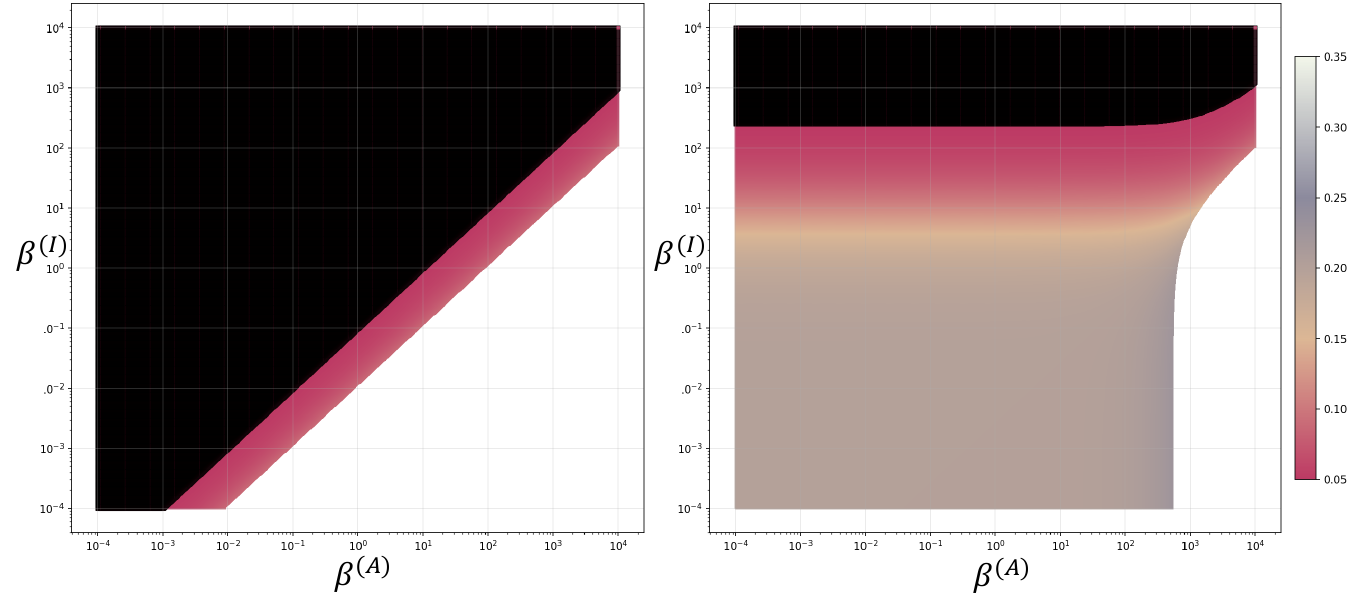}
    \caption{\textbf{Governance example.} Heatmaps show the interior equilibrium weight assigned to \emph{Conscientiousness} for \emph{DeepSeek-R1-Distill-Qwen-7B} on $\mathtt{Big\ Five}$ as a function of \((\beta^{(A)},\beta^{(I)})\), comparing \(\beta^{(D)}=0\) (left) and \(\beta^{(D)}=1\) (right). Increasing diversity incentives largely mitigates political exclusion.}
    \label{fig:gov_diversity}
\end{figure}

A central governance lesson from our experiments is that political exclusion is not an inherent property of a particular base model; it is \emph{incentive-contingent}. In particular, incentive configurations shaped by the system objective can expand or shrink the exclusion area, turning representation on or off for specific subpopulations.

\cref{fig:gov_diversity} provides a representative example. Holding the dataset $\mathtt{Big\ Five}$, the base model \emph{DeepSeek-R1-Distill-Qwen-7B}, and the focal subpopulation \emph{Conscientiousness} fixed, increasing the diversity coefficient from $\beta^{(D)}=0$ (left) to $\beta^{(D)}=1$ (right) significantly reduces the exclusion area. Operationally, this implies that diversity incentives can serve as a mitigation lever for system designers.

More broadly, our method provides a governance map connecting design choices to exclusion risks. This map supports principled pre-deployment testing and post-deployment monitoring for social media platforms: given a proposed incentivizing mechanism—such as ranking objectives, exposure allocation rules, or recommendation controls—that induces effective coefficients $\beta^{(A)}$, $\beta^{(I)}$, and $\beta^{(D)}$, the map identifies whether any subpopulation is predicted to be excluded. Platforms can also tune the mechanism—equivalently, adjust $\beta^{(A)}$, $\beta^{(I)}$, $\beta^{(D)}$—to ensure that no subpopulation is driven below the exclusion threshold.




\begin{table} [t]
    \caption{\textbf{Reasoning models exacerbates political exclusion.} For each base model, we aggregate exclusion area shares across $5$ traits and $3$ fixed-coefficient slices. \emph{Exclusion} reports the fraction of the coefficient grid where the interior equilibrium weight is driven below the exclusion threshold. \emph{Conditional Exclusion} (\emph{C-Excl.}) normalizes \emph{Exclusion} by the region where an interior equilibrium exists. Boldface highlights the largest exclusion rates within each parameter-count group.}
    \centering
    \hspace*{-3mm}
    \scalebox{0.95}{
    \setlength\tabcolsep{1.5mm}{ \begin{tabular}{crcrcrcrcr}
        \toprule
        
        \multicolumn{2}{l}{\multirow{2}{*}{Size}} &
        \multicolumn{2}{l}{\multirow{2}{*}{Model}}&
        \multicolumn{2}{l}{\multirow{2}{*}{Type}}&
        \multicolumn{4}{c}{Metrics}\\

        \cmidrule(lr){7-10}

        \multicolumn{2}{l}{} &
        \multicolumn{4}{l}{} & 
        \multicolumn{2}{c}{\emph{Exclusion}} & 
        \multicolumn{2}{c}{\emph{C-Excl.}} \\

        \cmidrule(lr){1-2}
        \cmidrule(lr){3-4}
        \cmidrule(lr){5-6}
        \cmidrule(lr){7-8}
        \cmidrule(lr){9-10}

        \multicolumn{2}{l}{\multirow{2}{*}{$\mathtt{4B}$}} &
        \multicolumn{2}{l}{\emph{Qwen3}} & 
        \multicolumn{2}{l}{Non-Reasoning} & 
        \multicolumn{2}{c}{0.510\%} & 
        \multicolumn{2}{c}{1.128\%} \\

         \multicolumn{2}{l}{}&
         \multicolumn{2}{l}{\emph{Qwen3}} & 
        \multicolumn{2}{l}{\textbf{Reasoning}} & 
         \multicolumn{2}{c}{\textbf{4.535\%}} & 
         \multicolumn{2}{c}{\textbf{5.040\%}}\\

         \midrule
         \multicolumn{2}{l}{\multirow{3}{*}{$\mathtt{7B}$}} &
         \multicolumn{2}{l}{Qwen3} & 
        \multicolumn{2}{l}{Non-Reasoning} & 
         \multicolumn{2}{c}{0.210\%} & 
         \multicolumn{2}{c}{0.231\%} \\

         \multicolumn{2}{l}{}&
         \multicolumn{2}{l}{Mistral} & 
        \multicolumn{2}{l}{Non-Reasoning} & 
         \multicolumn{2}{c}{0.208\%} & 
         \multicolumn{2}{c}{0.223\%}  \\

        \multicolumn{2}{l}{}&
         \multicolumn{2}{l}{DeepSeek R1} & 
        \multicolumn{2}{l}{\textbf{Reasoning}} & 
         \multicolumn{2}{c}{\textbf{3.721\%}} & 
         \multicolumn{2}{c}{\textbf{4.068\%}}  \\

        \toprule
        \vspace{-1em}
        \end{tabular}}}

    \label{tab:thinking_invalid_ignore}
\end{table}

%% file: text/5-Conclusion.tex
\section{Conclusion}

This work is a step toward treating LLM deployments as governed ecosystems rather than isolated models: once incentives are explicit, outcomes become something we can diagnose, anticipate, and—at least in principle—design against. 
We provide a theoretical toolkit to support this incentive-aware, equilibrium-based perspective on rigorous diagnosis and governance.
Our method also clarifies a conceptual shift: alignment needs not be treated as passive, exogenous constraints imposed on LLMs. Instead, in multi-agent settings where LLMs compete, cooperate, or bargain, the choice of who to align to becomes an active, endogenous strategic variable. 
A natural next direction is to move beyond static, one-shot equilibria to dynamic settings where incentives and populations co-evolve (platform feedback loops, shifting user composition, and model updates), and ask when desirable equilibria remain stable over time. 

%% file: text/appendix.tex
\section{Math Proofs}\label{appx:proof}


\gradientcond*
\begin{proof}
    This proof is based on Lemma $1$ in \citet{gemp2024approximating}.
Fix $\bw_{-m}$. Consider
\[
\max_{\bw_m\in\Delta_D}\ u_m(\bw_m,\bw_{-m}),
\]
with constraints $-\bw_m\le \mathbf 0$ and $\bone^\top \bw_m-1=0$. 

We first show that $u_m(\cdot,\bw_{-m})$ is concave in $\bw_m$.
The terms $\beta^{(A)}\ba^\top \bw_m$ and $-\beta^{(D)}\sum_{j\neq m}\bw_m^\top \bw_j$ are linear in $\bw_m$.
Moreover,
\[
\nabla^2_{\bw_m}\!\left(-\beta^{(I)}\bw_m^\top \bC\,\bw_m\right)
=-\beta^{(I)}(\bC+\bC^\top).
\]
Assuming $\beta^{(I)}\ge 0$, $\bC\succeq 0$ and $\bC=\bC^\top$ (\Cref{claim:C}), we have
$\nabla^2_{\bw_m}u_m(\bw_m,\bw_{-m})=-2\beta^{(I)}\bC\preceq 0$.
Hence $u_m(\cdot,\bw_{-m})$ is concave in $\bw_m$.

The Lagrangian is
\[
\cL(\bw_m,\lambda,\boldsymbol{\mu})
=u_m(\bw_m,\bw_{-m})+\lambda(\bone^\top \bw_m-1)+\boldsymbol{\mu}^\top(-\bw_m),
\qquad \boldsymbol{\mu}\ge \mathbf 0.
\]

Note that the objective is concave and the constraints are affine, therefore the KKT conditions are necessary and sufficient for optimality.

By KKT conditions, any optimal $\bw_m$ satisfies
\[
\nabla_{\bw_m}u_m(\bw_m,\bw_{-m})+\lambda\bone-\boldsymbol{\mu}=\mathbf 0,\qquad
\boldsymbol{\mu}\ge \mathbf 0,\qquad
\boldsymbol{\mu}\odot \bw_m=\mathbf 0,\qquad
\bw_m\in\Delta_D.
\]
If $\bw_m\in\mathrm{int}(\Delta_D)$, then $\bw_m>\mathbf 0$, hence complementary slackness gives $\boldsymbol{\mu}=\mathbf 0$.
Therefore $\nabla_{\bw_m}u_m(\bw_m,\bw_{-m})=-\lambda\bone$; renaming $-\lambda$ as $\lambda_m$ yields
$\nabla_{\bw_m}u_m(\bw_m,\bw_{-m})=\lambda_m\bone$.

Finally,
\[
\nabla_{\bw_m}u_m(\bw_m,\bw_{-m})
=\beta^{(A)}\ba-2\beta^{(I)}\bC\bw_m-\beta^{(D)}\sum_{j\neq m}\bw_j,
\]
which proves the claim.
\end{proof}

\alphaiszero*
\begin{proof}
The function $f(\beta)=\sum_{j} \frac{q_j^2}{2c_j-\beta}$ is continuous and differentiable on each open interval between consecutive pole locations $\{2c_j:\, q_j\neq 0\}$.
On any open interval between two consecutive poles, we have
\[
f'(\beta)
=\sum_{j} \frac{q_j^2}{(2c_j-\beta)^2}\;>\;0,
\]
so $f$ is strictly increasing there and can have at most one zero on each such interval.
Let $K'$ be the number of distinct poles; then there are at most $K'+1$ intervals and thus at most $K'+1$ real roots of $f(\beta)=0$.
Therefore the set of $=\beta^{(D)}/\beta^{(I)}$ that solves $\alpha=0$ is finite.
\end{proof}

\section{Boundary NE}\label{appx:boundary_ne}

Following standard differential-geometry notation and based on~\citet{gemp2024approximating}, let $T_\nu \mathcal{M}$ denote the tangent space of a manifold $\mathcal{M}$ at the point $\nu$.
For the interior of the $D$-action simplex $\Delta^{D-1}$, the tangent space is identical at every point, so we omit the subscript and write $T\Delta^{D-1}$.
We define the projection of a vector $z \in \mathbb{R}^d$ onto this tangent space by
\[
\Pi_{T\Delta^{D-1}}(z) \;=\; \Bigl[I - \frac{1}{D}\mathbf{1}\mathbf{1}^\top\Bigr]z,
\]
and we refer to $\Pi_{T\Delta^{D-1}}\!\bigl(\nabla^{k}_{x_k}\bigr)$ as a \emph{projected-gradient}.
When the dimension is unambiguous, we drop the superscript and simply write $T\Delta$.
Finally, let $\mathcal{U}(S)$ denote the discrete uniform distribution over the elements of a set $S$.

To handle boundary cases, we adopt the entropy-regularized (logit-QRE) formulation \cite{leonardos2021exploration}, and minimize the squared norm of \emph{projected} gradients of the entropy-augmented utilities.

Let $\mathbf{1}\in\mathbb{R}^D$ denote the all-ones vector and define the tangent space of the simplex by
$T\Delta= \{v\in\mathbb{R}^D:\mathbf{1}^\top v=0\}$.
We use the orthogonal projector onto $T\Delta$:
\begin{equation}
\Pi_{T\Delta}(v)\;=\;v-\frac{\mathbf{1}^\top v}{D}\,\mathbf{1}. \label{eq:proj}
\end{equation}

Consider the Shannon entropy $S(w_m)=-\sum_{i=1}^D \vw_{m,i}\log \vw_{m,i}$, and fix a temperature $\tau>0$.
Define the entropy-augmented utility
\begin{align}
u_m^\tau(\vw_1,\cdots,\vw_M)&=u_m(\vw_1,\cdots,\vw_M)\;+\;\tau\,S(\vw_m),\\
\nabla_{\vw_m} u_m^\tau(\vw_1,\cdots,\vw_M)\;&=\;\nabla_{\vw_m} u_m(\vw_1,\cdots,\vw_M)\;-\;\tau\bigl(\log \vw_m+\mathbf{1}\bigr),
\label{eq:tau-grad}
\end{align}
where $\log \vw_m$ is taken element-wise and $\vw_{m,i}>0$ for all $i$ (guaranteed for $\tau>0$ at a QRE).

Following \cite{gemp2024approximating}, we define
\begin{equation}
L_\tau(\vw_1,\cdots,\vw_M)\;=\;\sum_{m=1}^M \eta_m\,\bigl\|
\Pi_{T\Delta}\bigl(\nabla_{\vw_m} u_m^\tau(\vw_1,\cdots,\vw_M)\bigr)
\bigr\|_2^2,
\label{eq:qre-loss}
\end{equation}
with positive weights $\eta_m>0$.
Zeros of $L_\tau$ coincide with logit-QRE; as $\tau\downarrow 0$, the QRE approaches a Nash equilibrium that may reside on the boundary.

Substituting \Cref{eq:u} into \Cref{eq:tau-grad} and \Cref{eq:qre-loss} yields the explicit expression:
\begin{align}
\nabla_{\vw_m} u_m^\tau(\vw_1,\cdots,\vw_M)
&=\;-\,\beta^{(D)}\sum_{j\ne m}\vw_j
\;+\;\beta^{(A)} \va
\;-\;2\beta^{(I)}\, \mC\vw_m
\;-\;\tau\bigl(\log \vw_m+\mathbf{1}\bigr), \label{eq:tau-grad-expanded}\\
L_\tau(\vw_1,\cdots,\vw_M)
&=\;\sum_{m=1}^M \eta_m\,\Bigl\|
\Pi_{T\Delta}\Bigl(
-\,\beta^{(D)}\sum_{j\ne m}\vw_j
+\beta^{(A)} \va
-2\beta^{(I)} \mC\vw_m
-\tau(\log \vw_m+\mathbf{1})
\Bigr)\Bigr\|_2^2. \label{eq:qre-loss-expanded}
\end{align}

\section{Experiments}\label{appx:exp}

In our experiments, to compute matrix $\mC$ and vector $\va$, we use the $\mathtt{OpinionQA}$ dataset~\citep{santurkar2023whose}, where each sample consists of a question with $K_\text{c}$ choices, an associated attribute, and a ground-truth human preference distribution.
We prompt LLM players with: ``In terms of [attribute], [question ($x^\pk$)]. The answer is [option ($y^\pk$)]." We then extract the logit assigned to the token sequence corresponding to [option]. For a given pair of LLMs $(i, j)$, we define $c_{ij}$ as the average, across samples, of the discrepancy in their predicted (logit-based) probabilities for [option]. Collecting these quantities yields the off-diagonal entries of $\mC$.
We set the diagonal entries to the aggregate discrepancy $c_{jj} = \sum_{i \neq j} c_{ij}$. Options come from a discrete, finite set. To construct $\va$, we compute, for each LLM $i$, the average (over samples) inner product between the model's option-level logit-based probabilities and the corresponding ground-truth human preference distribution over the options. 

We show the complete results in \cref{fig:Region}-\cref{fig:Openness}.


\begin{figure}
    \centering
    \includegraphics[width=0.3\linewidth]{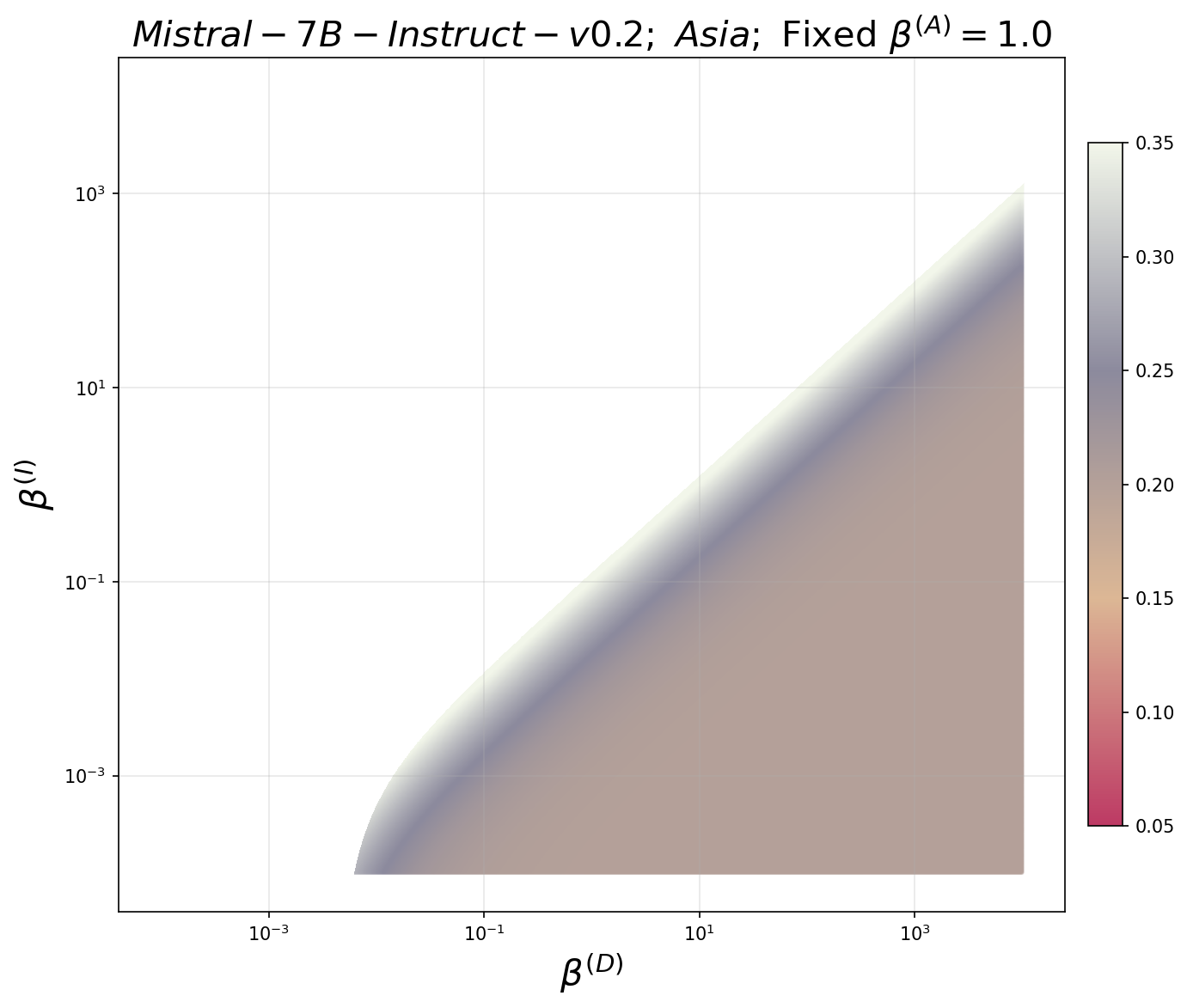}
    \includegraphics[width=0.3\linewidth]{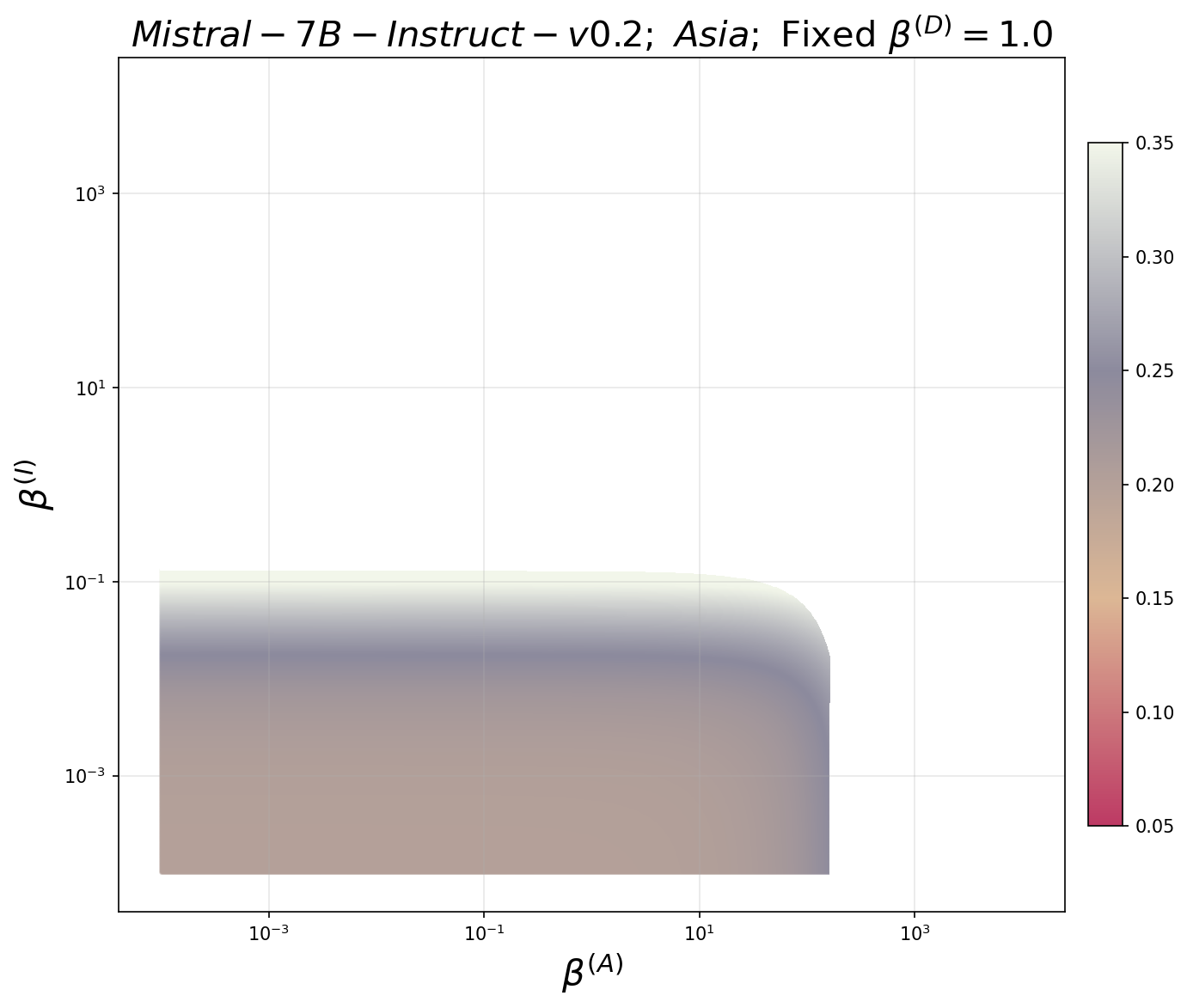}
    \includegraphics[width=0.3\linewidth]{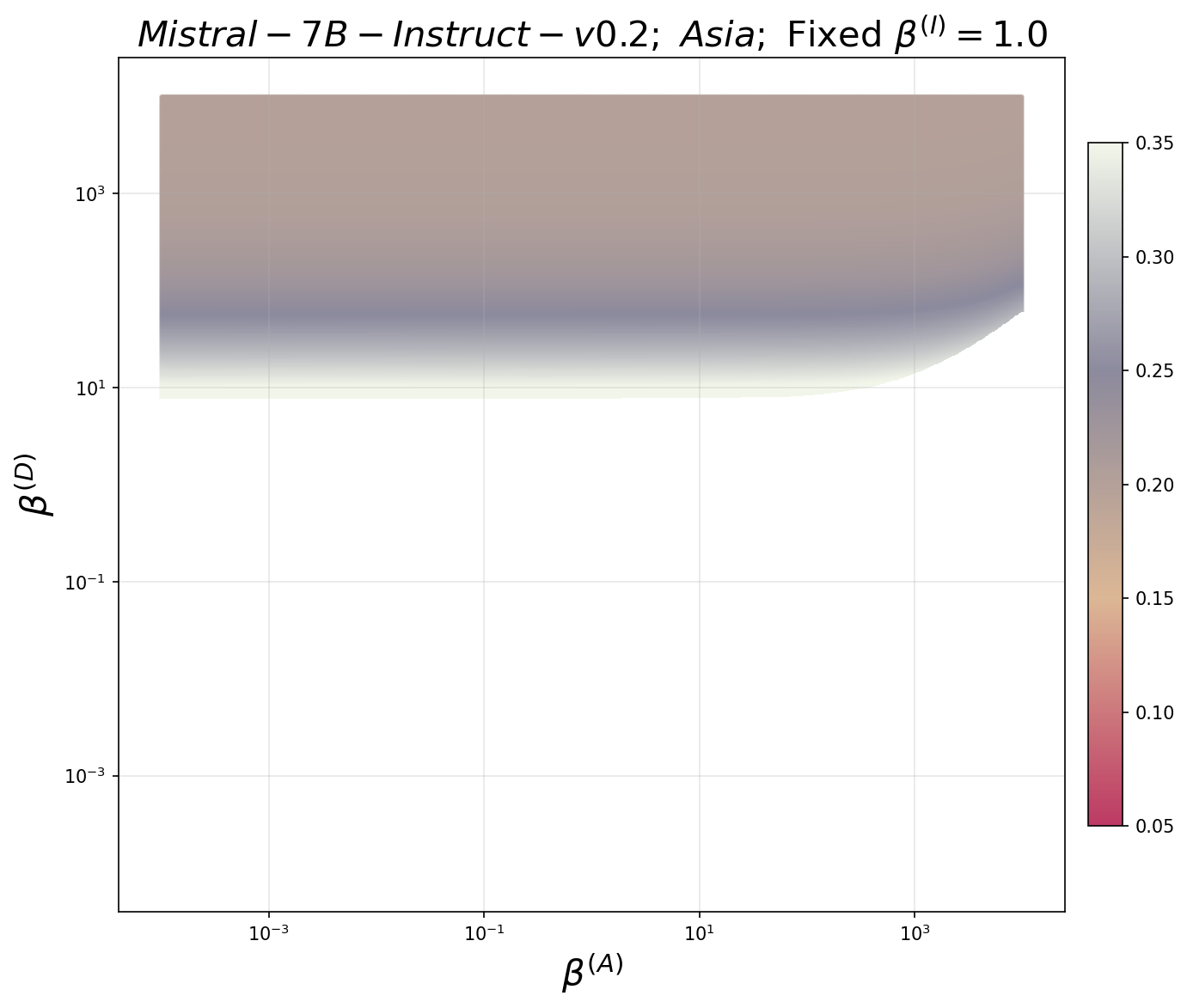}

    \includegraphics[width=0.3\linewidth]{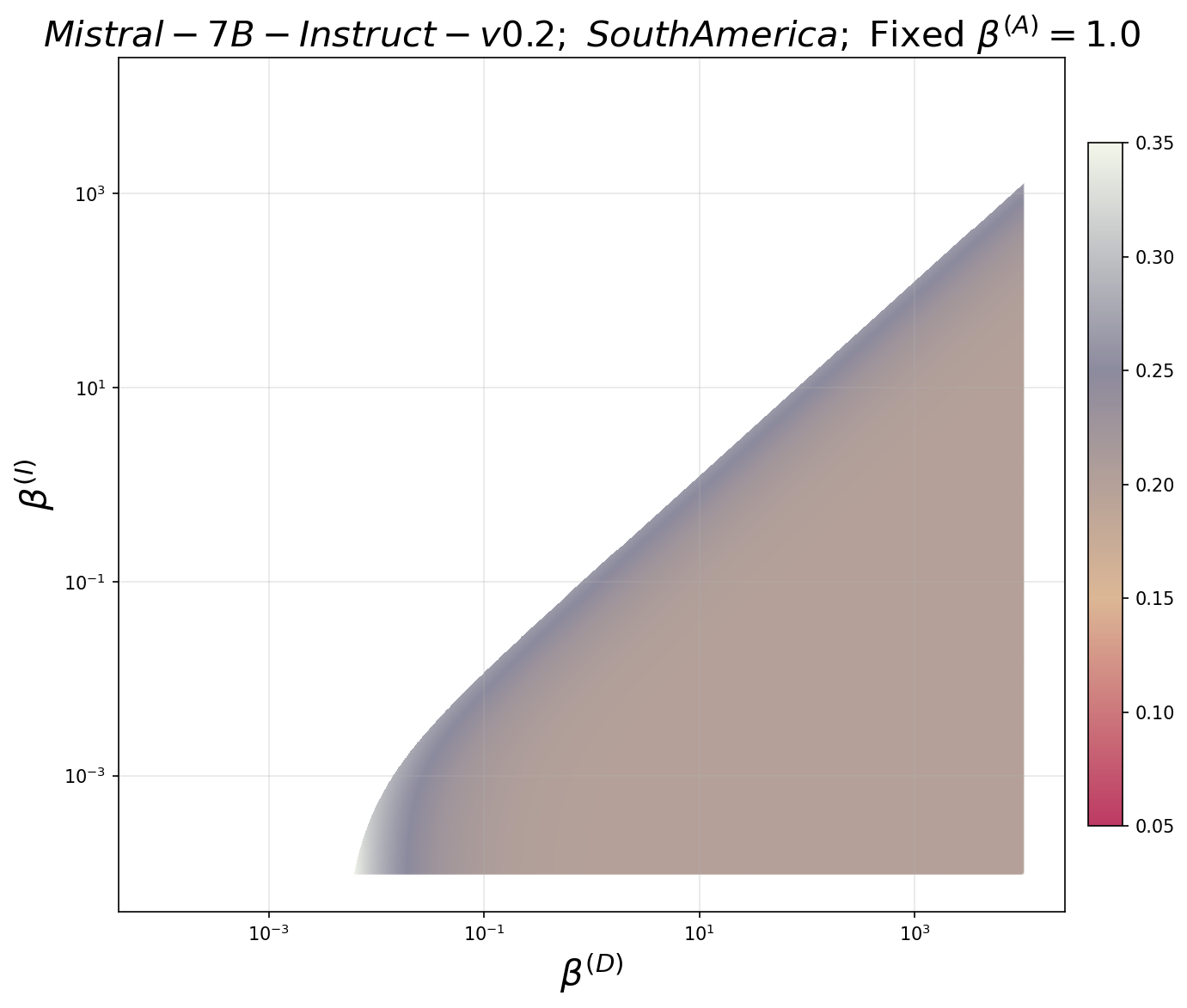}
    \includegraphics[width=0.3\linewidth]{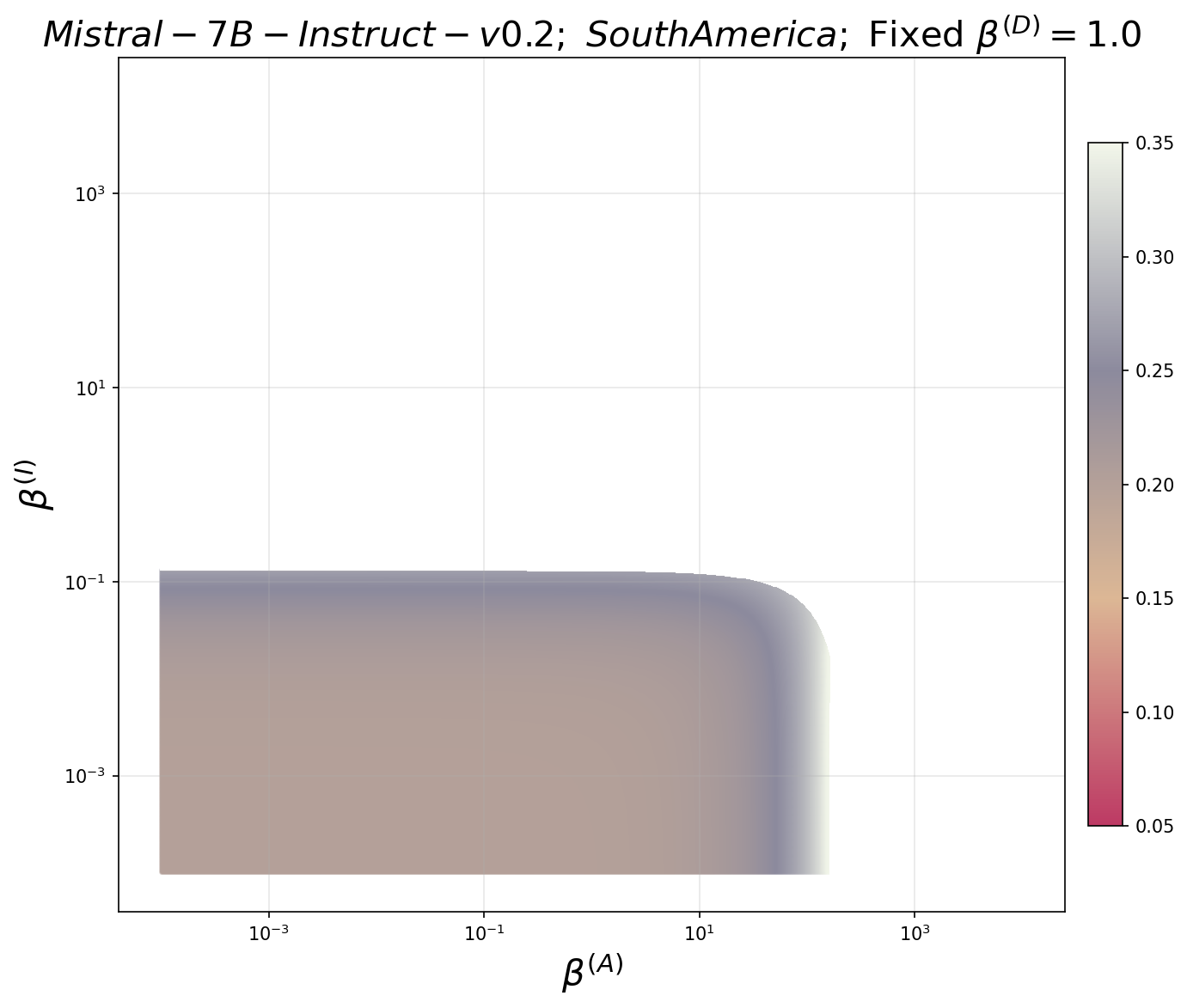}
    \includegraphics[width=0.3\linewidth]{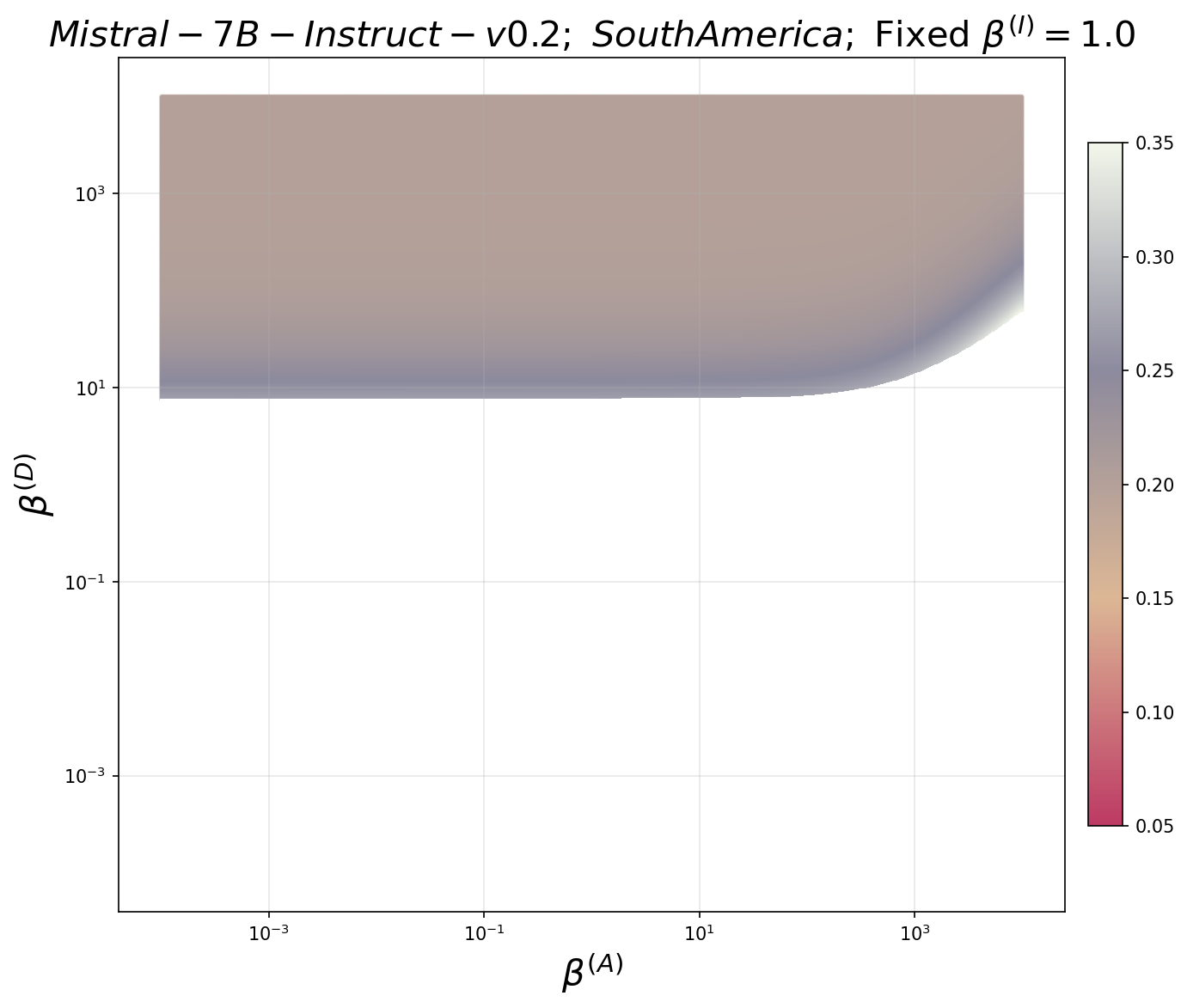}

    \includegraphics[width=0.3\linewidth]{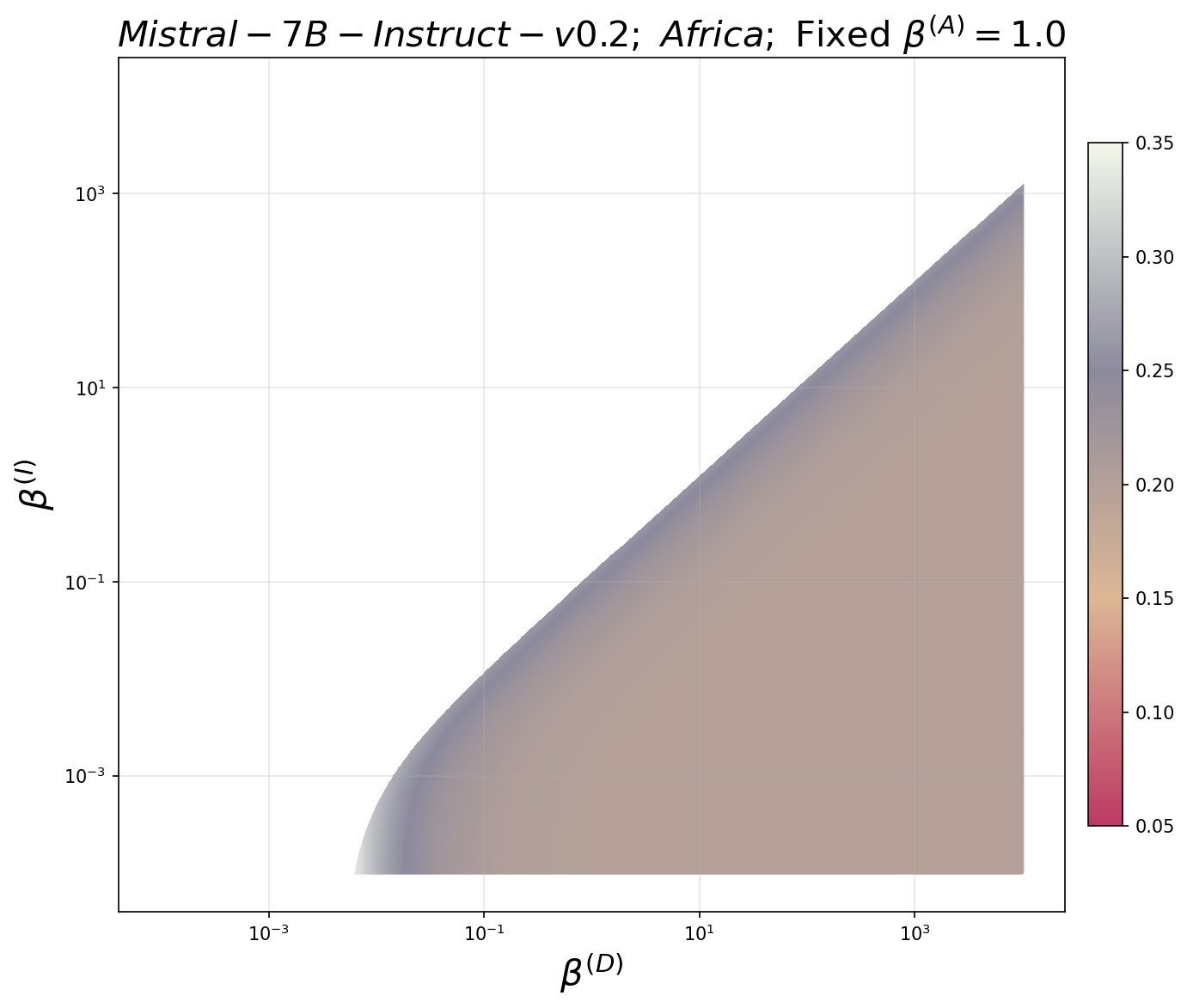}
    \includegraphics[width=0.3\linewidth]{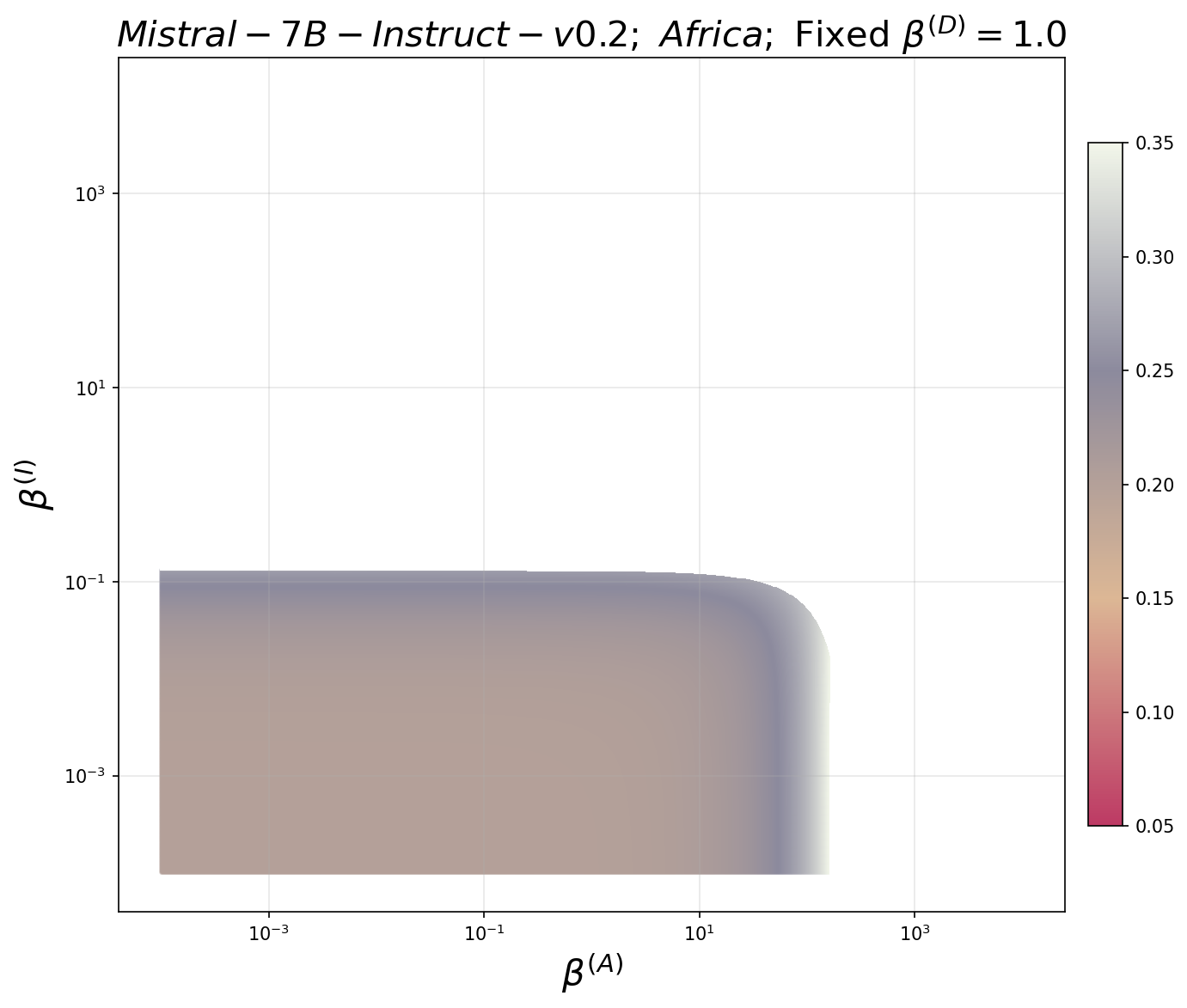}
    \includegraphics[width=0.3\linewidth]{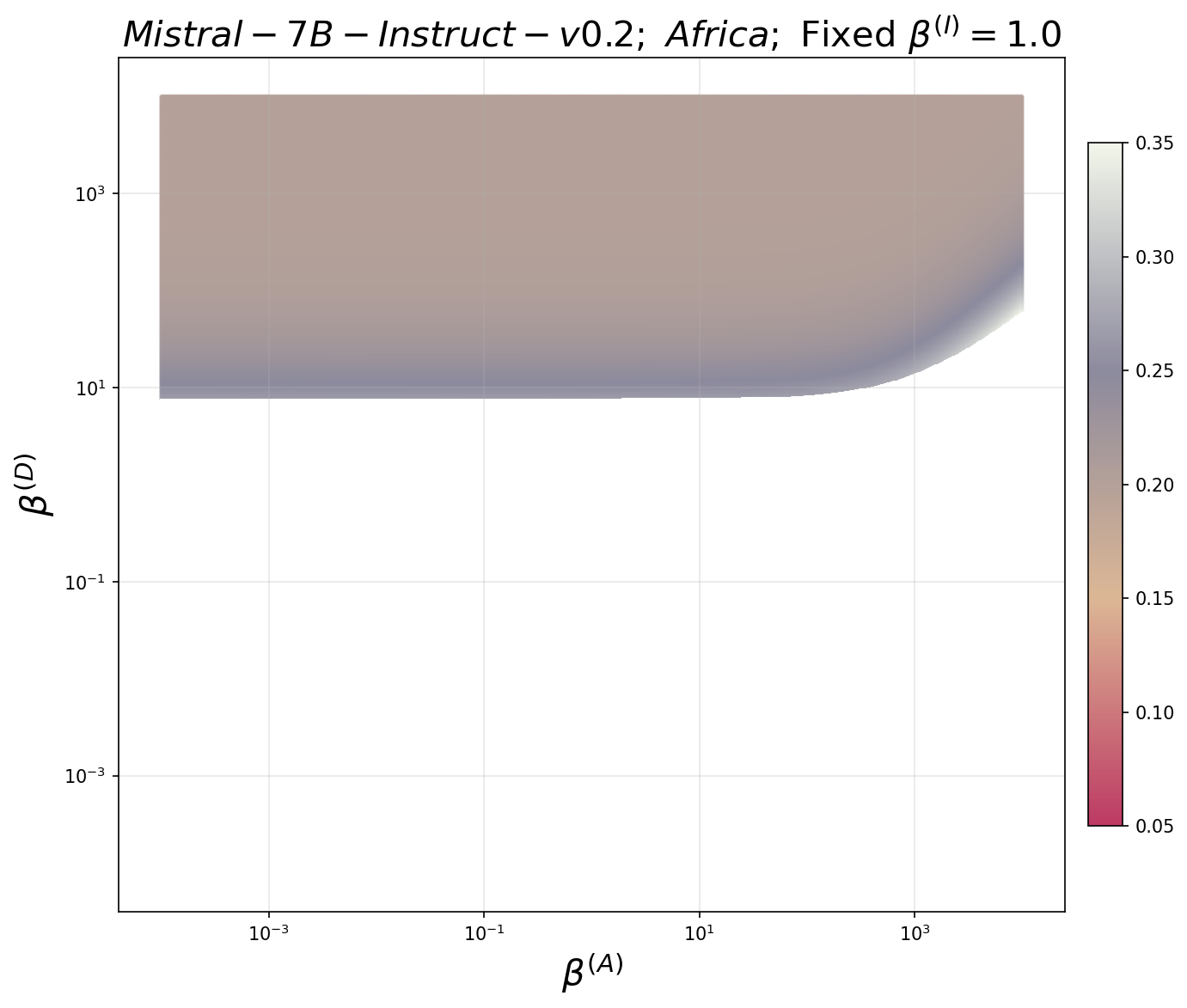}

    \includegraphics[width=0.3\linewidth]{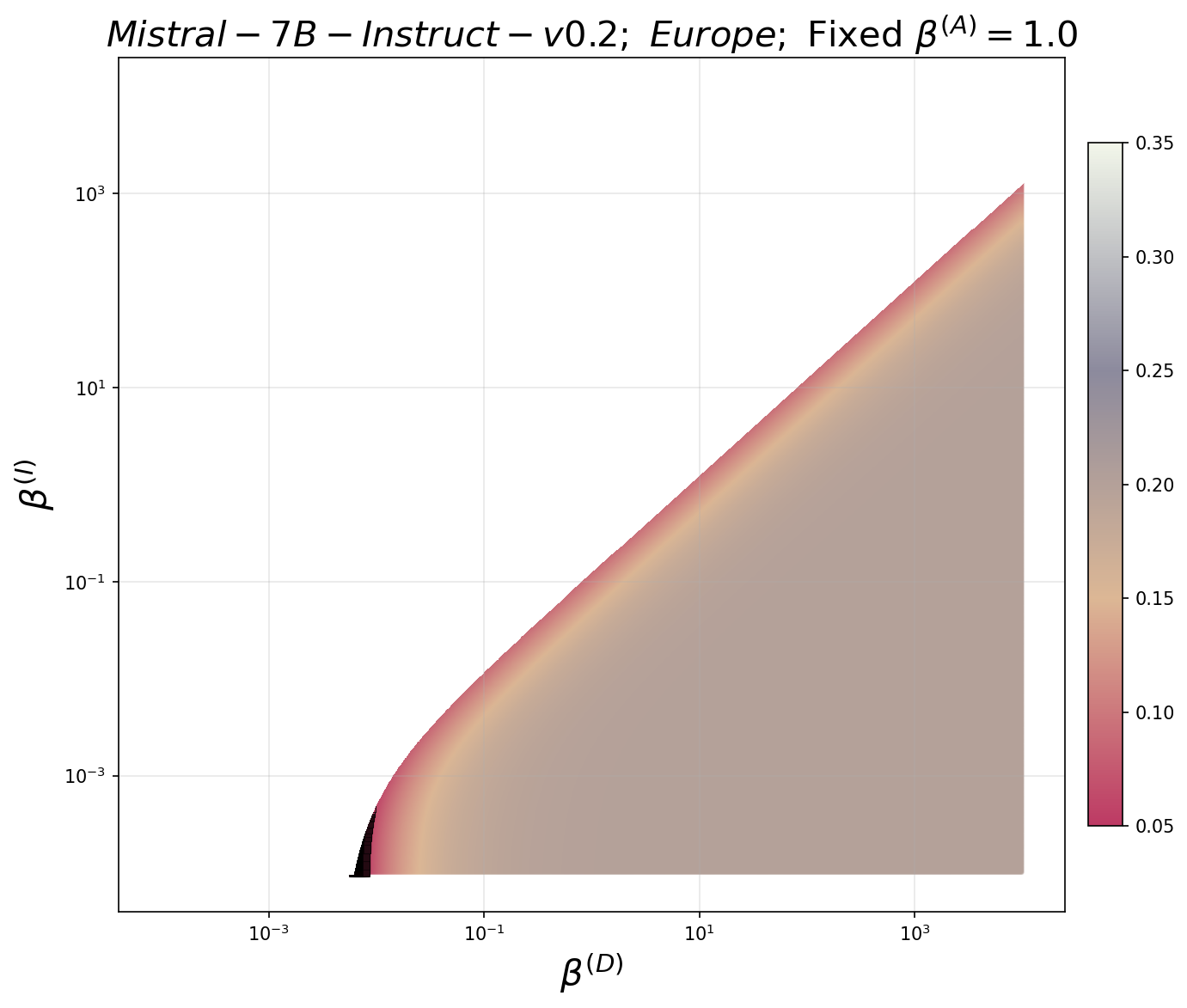}
    \includegraphics[width=0.3\linewidth]{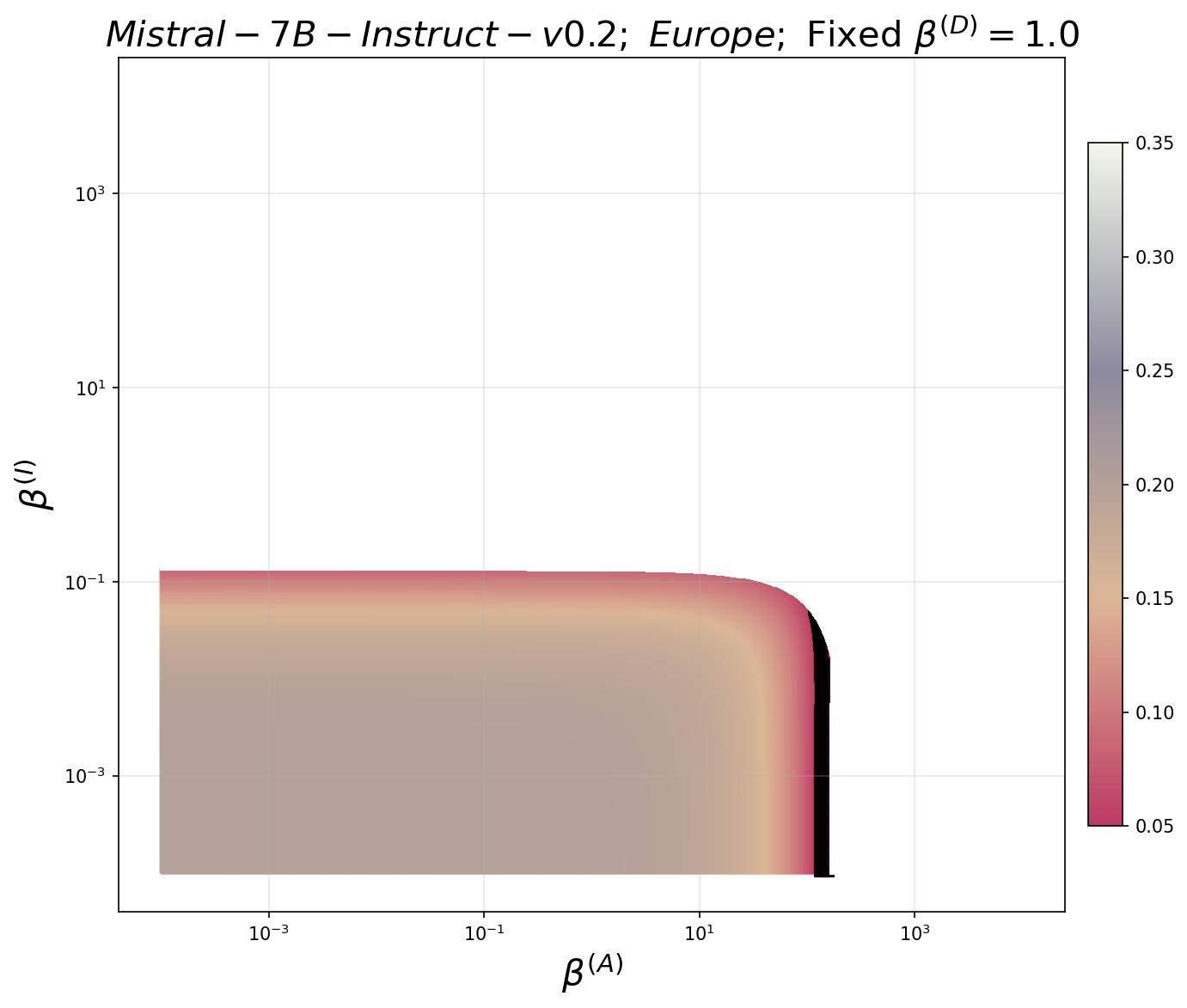}
    \includegraphics[width=0.3\linewidth]{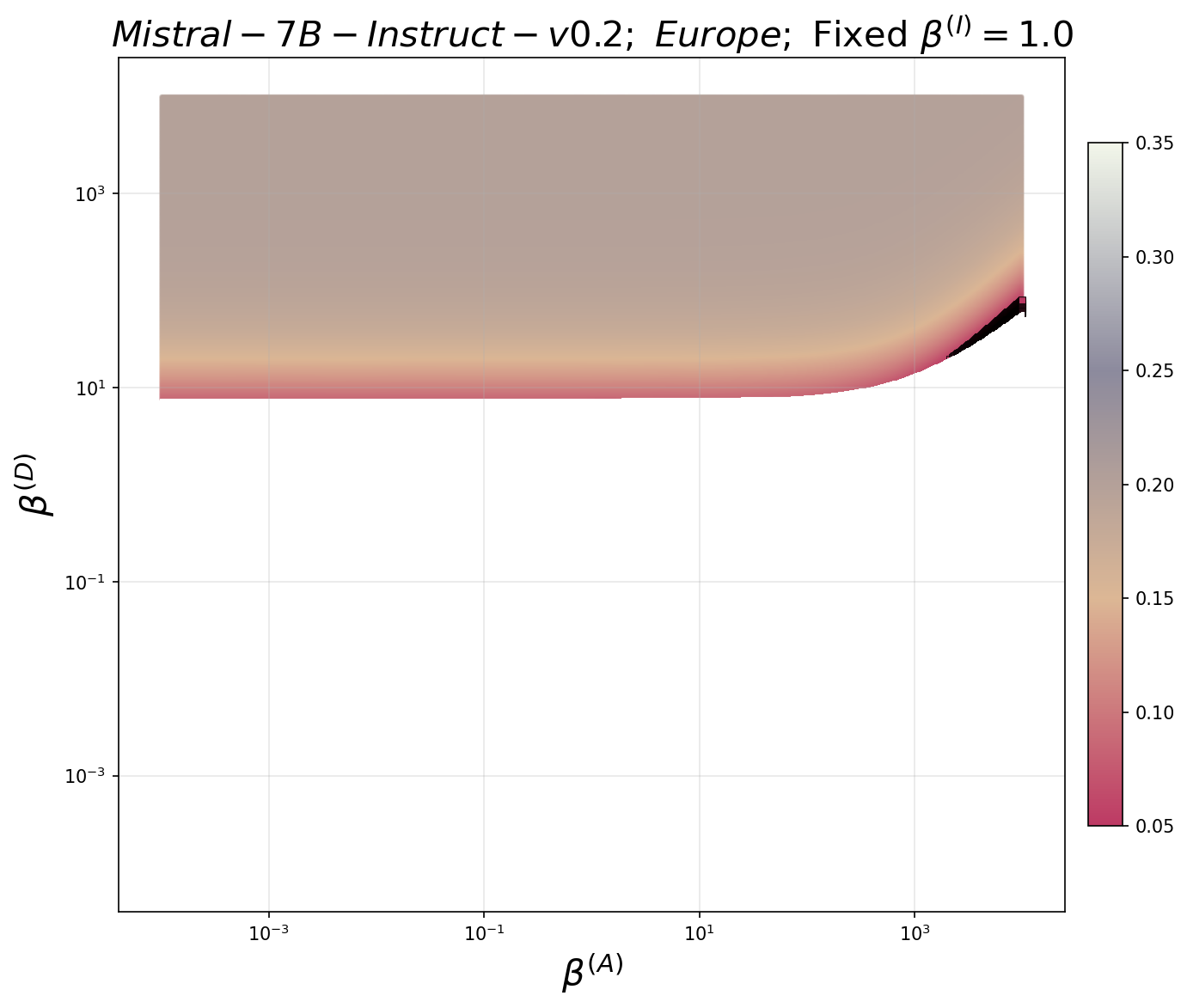}

    \includegraphics[width=0.3\linewidth]{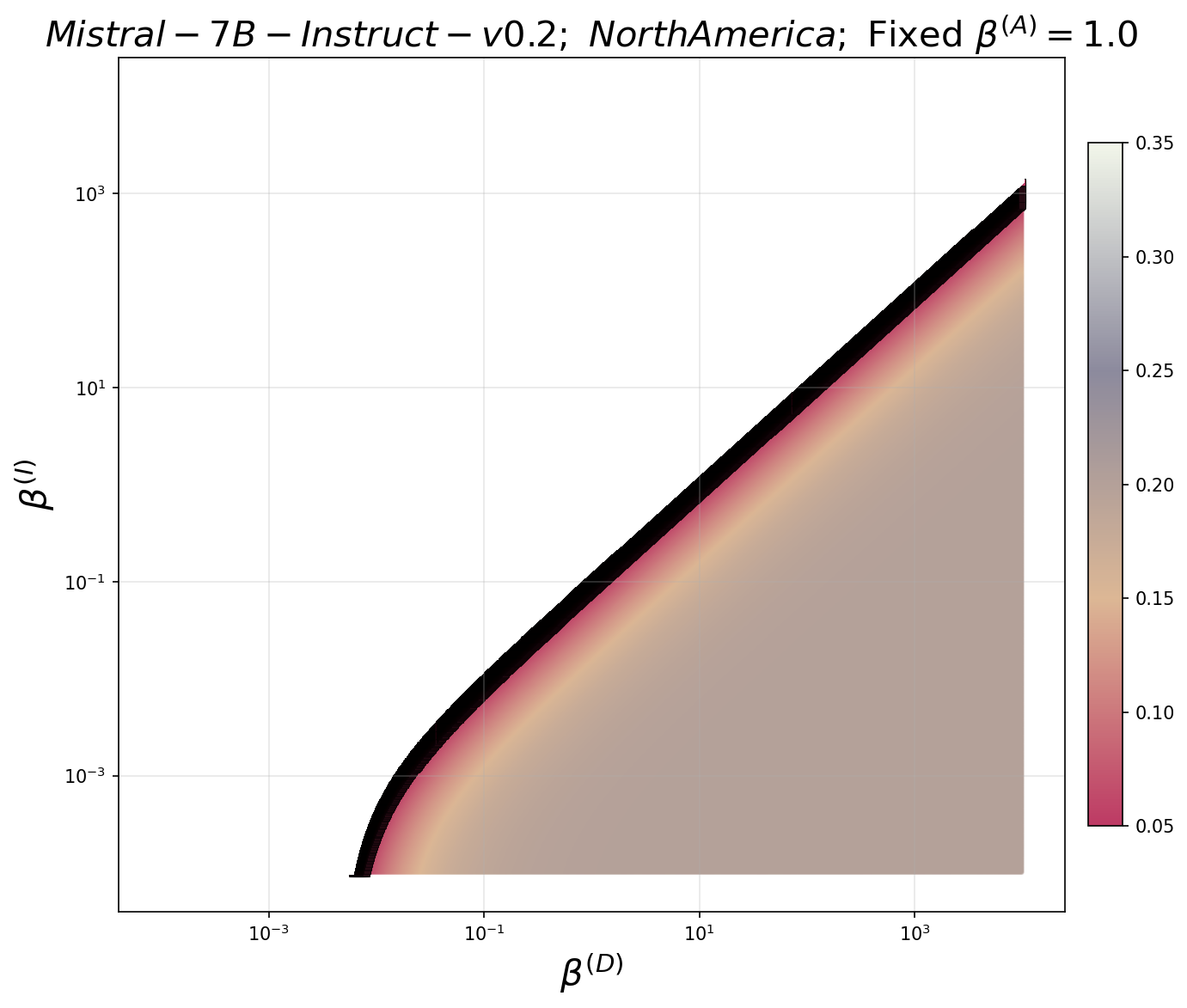}
    \includegraphics[width=0.3\linewidth]{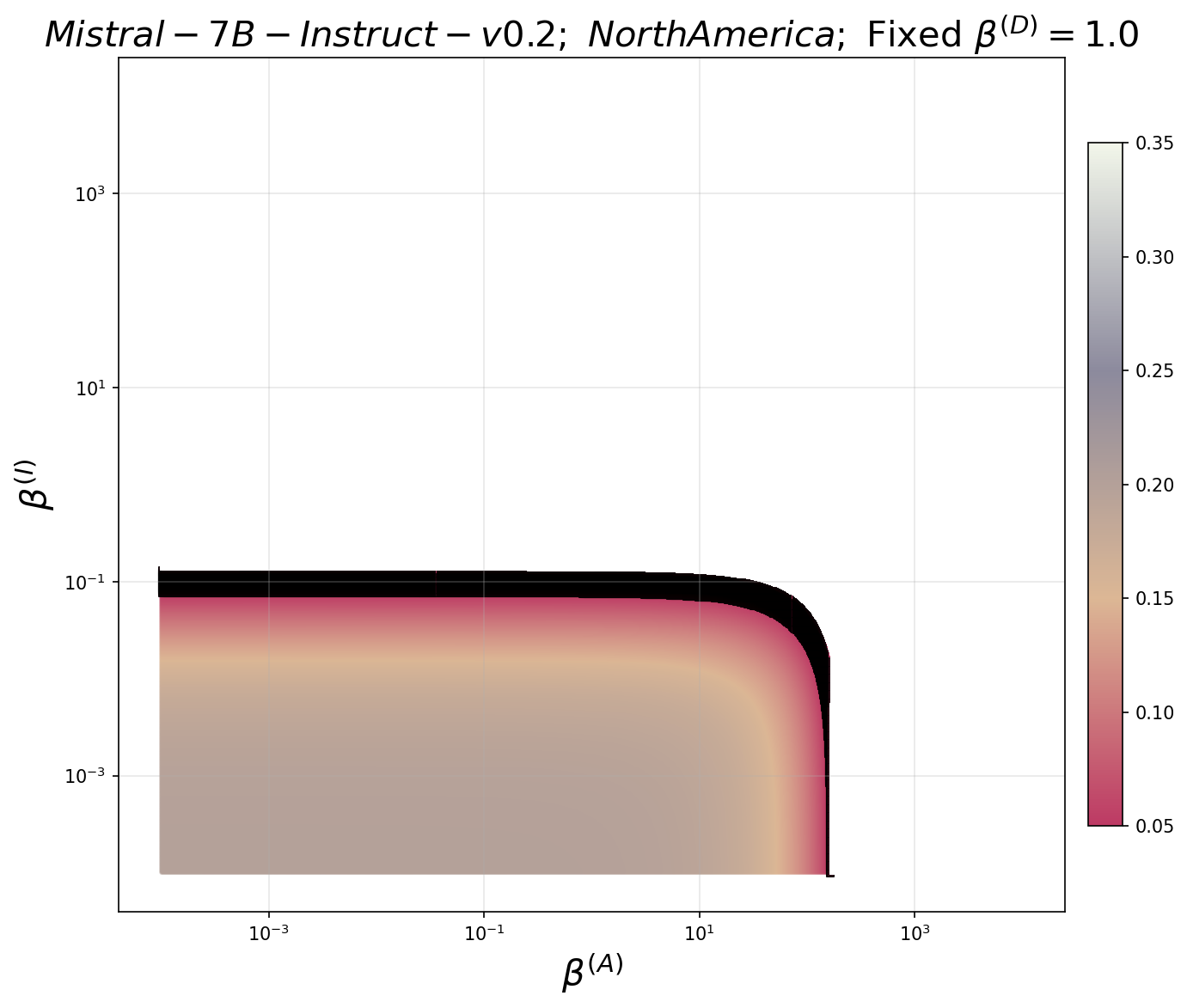}
    \includegraphics[width=0.3\linewidth]{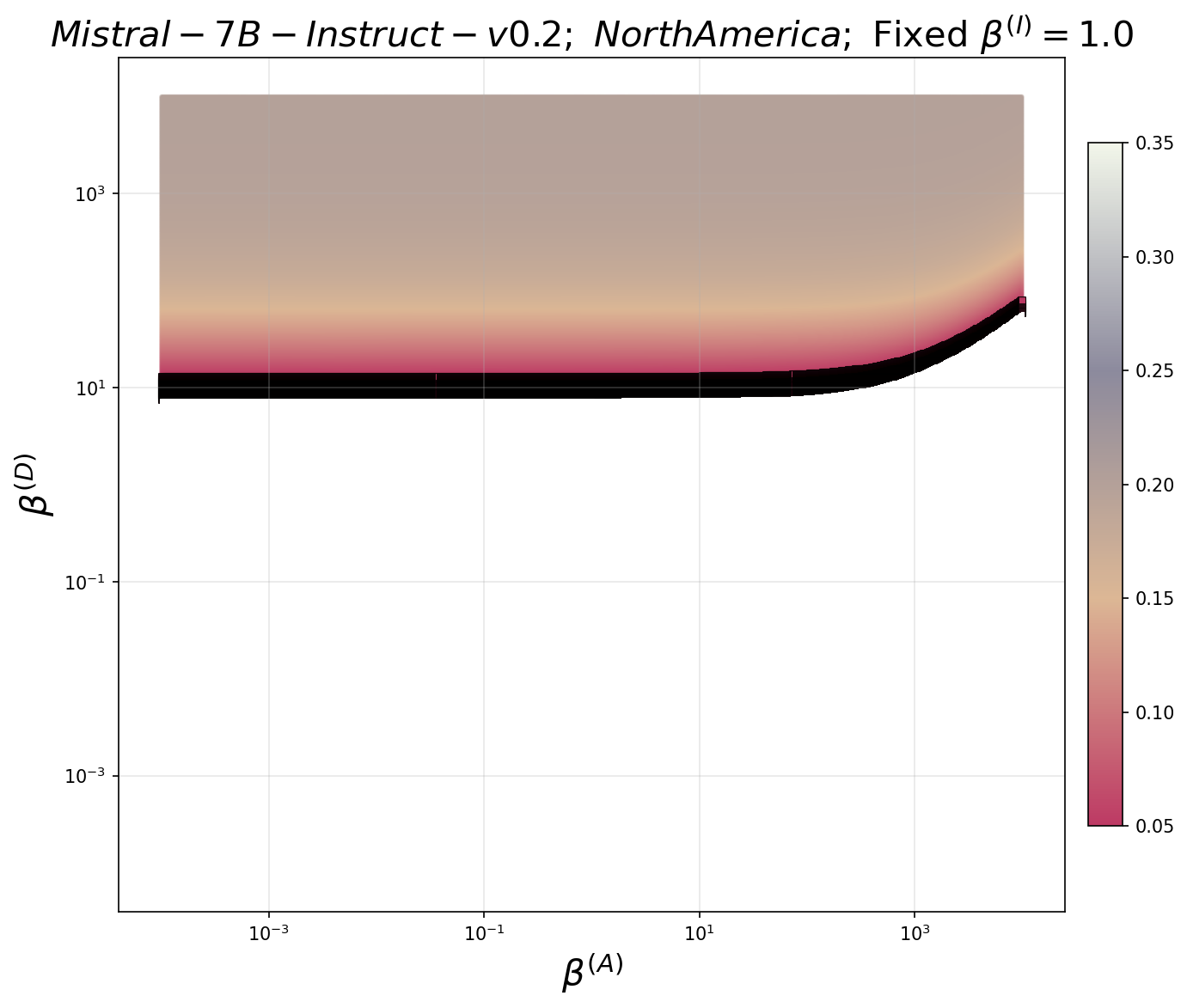}

    \caption{Political exclusion of \emph{Mistral-7B-Instruct-v0.2} on the $\mathtt{CultureBank}$ dataset.}
    \label{fig:Region}
\end{figure}

\begin{figure}
    \centering

    \includegraphics[width=0.25\linewidth]{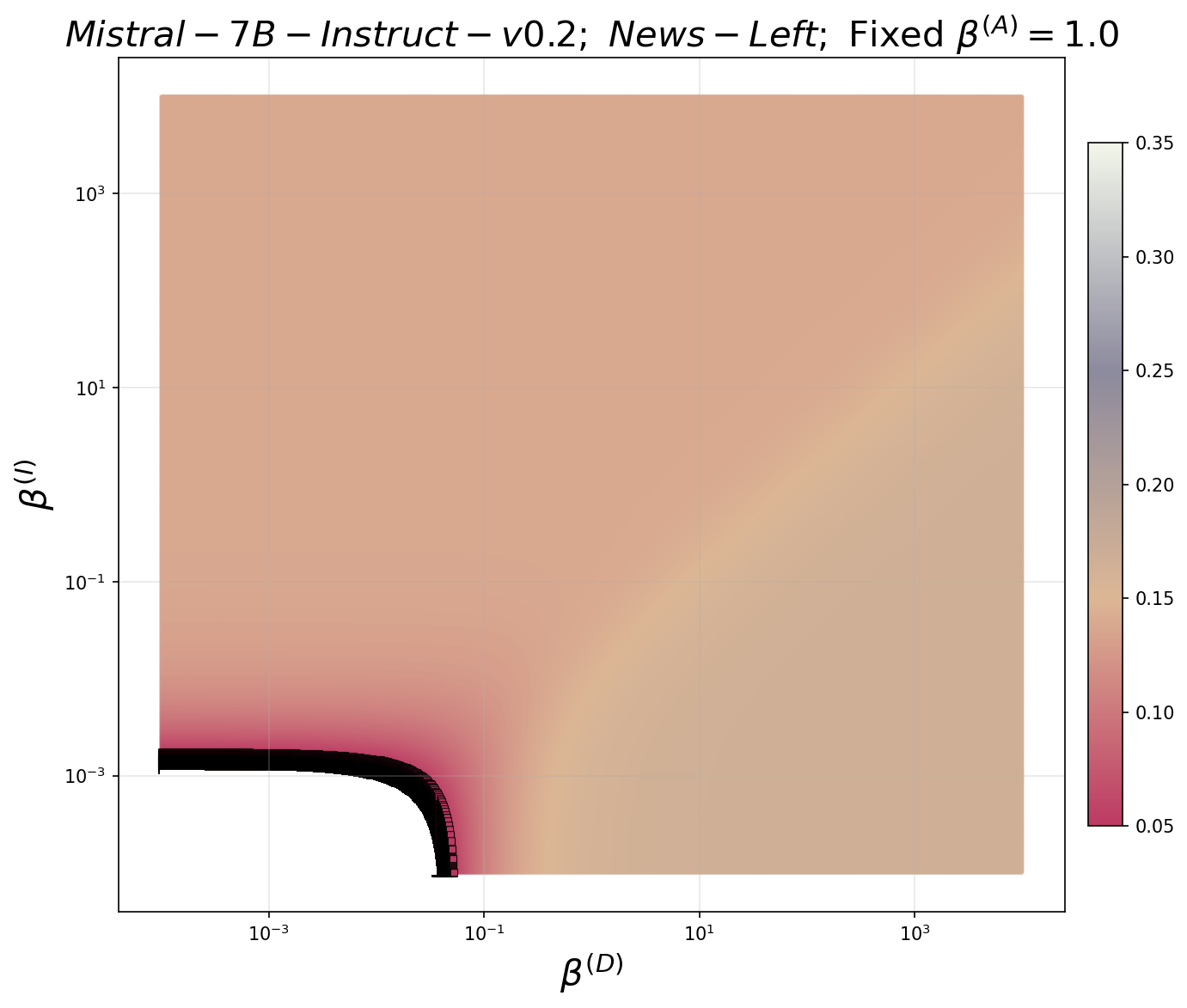}
    \includegraphics[width=0.25\linewidth]{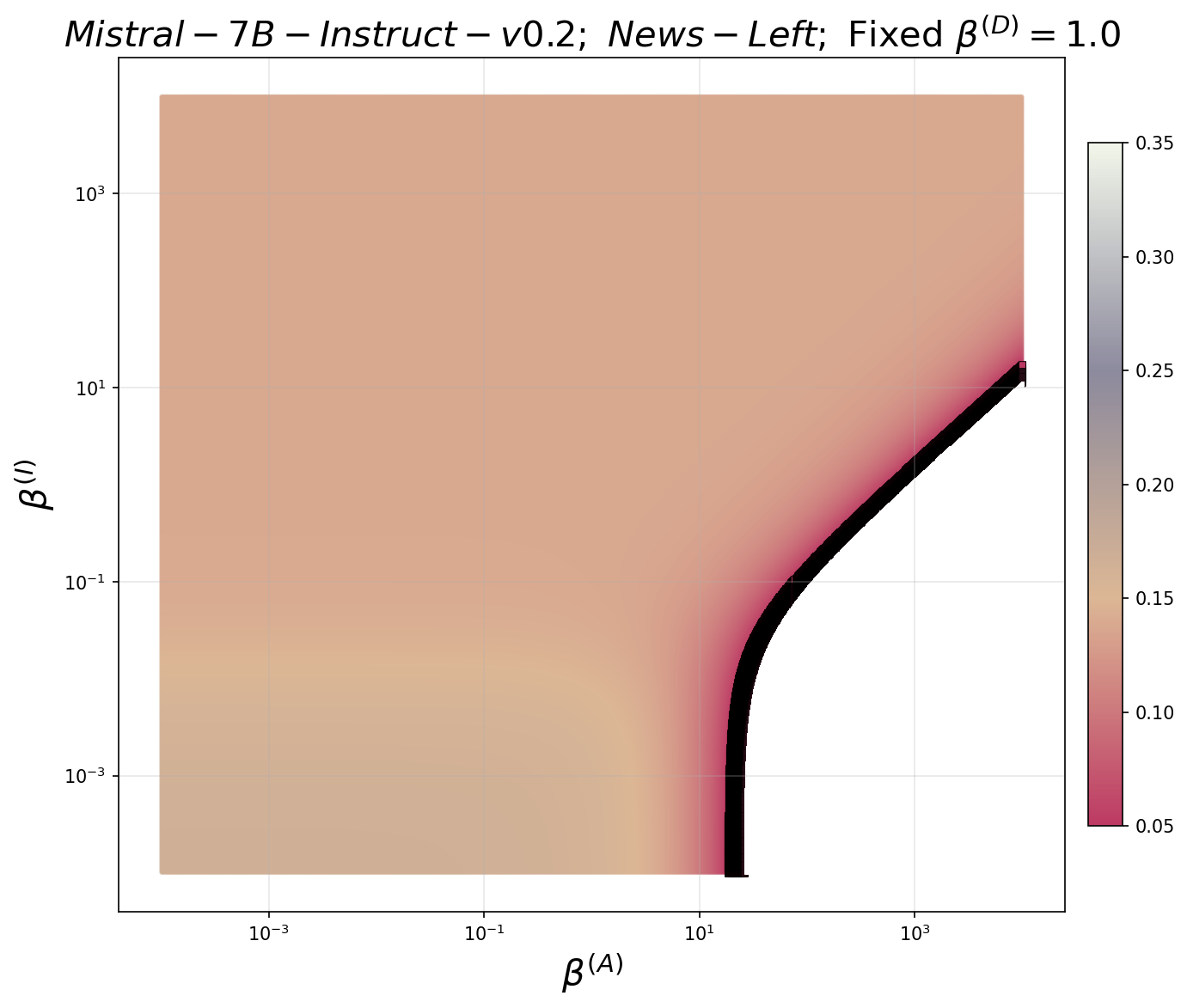}
    \includegraphics[width=0.25\linewidth]{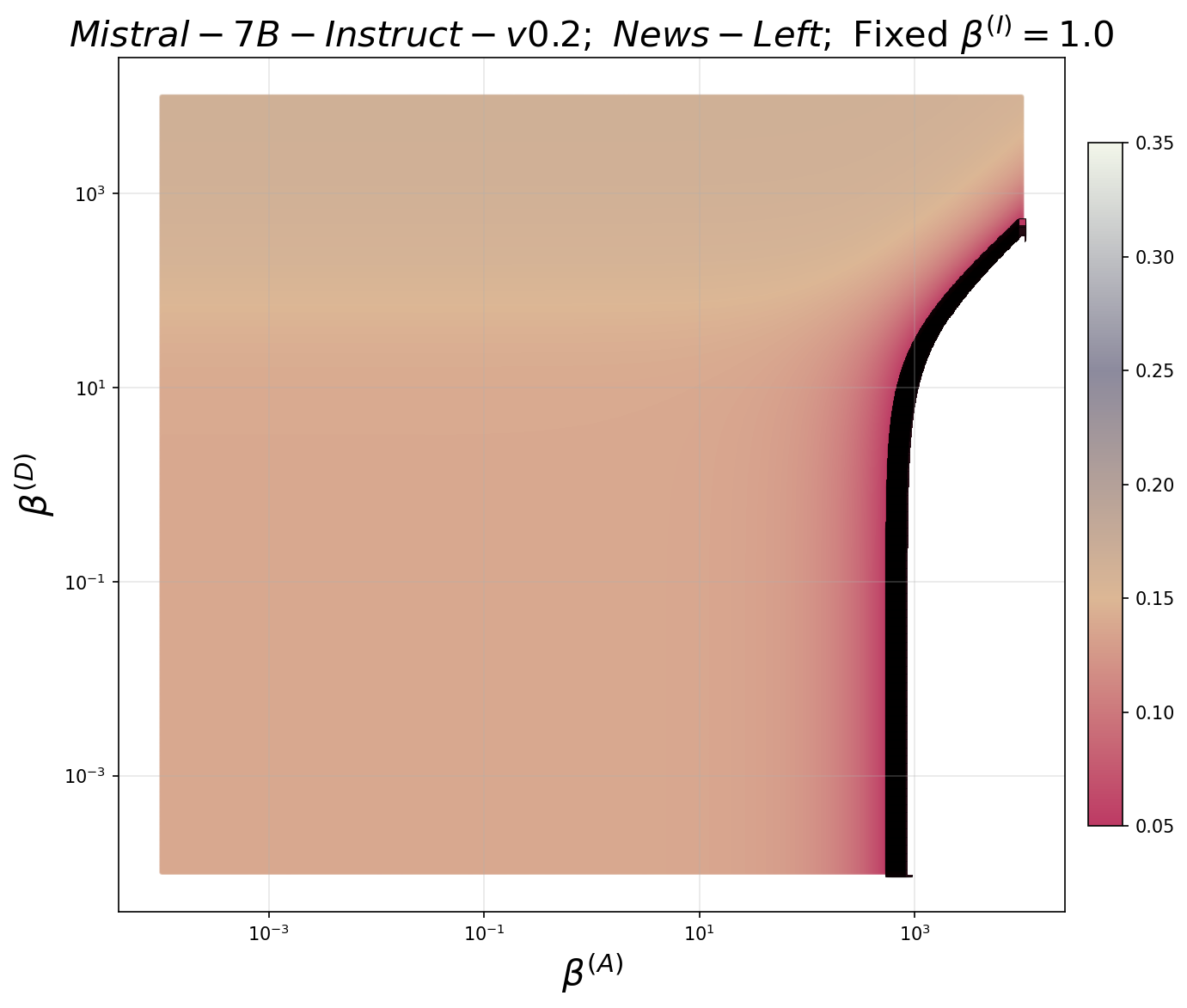}

    \includegraphics[width=0.25\linewidth]{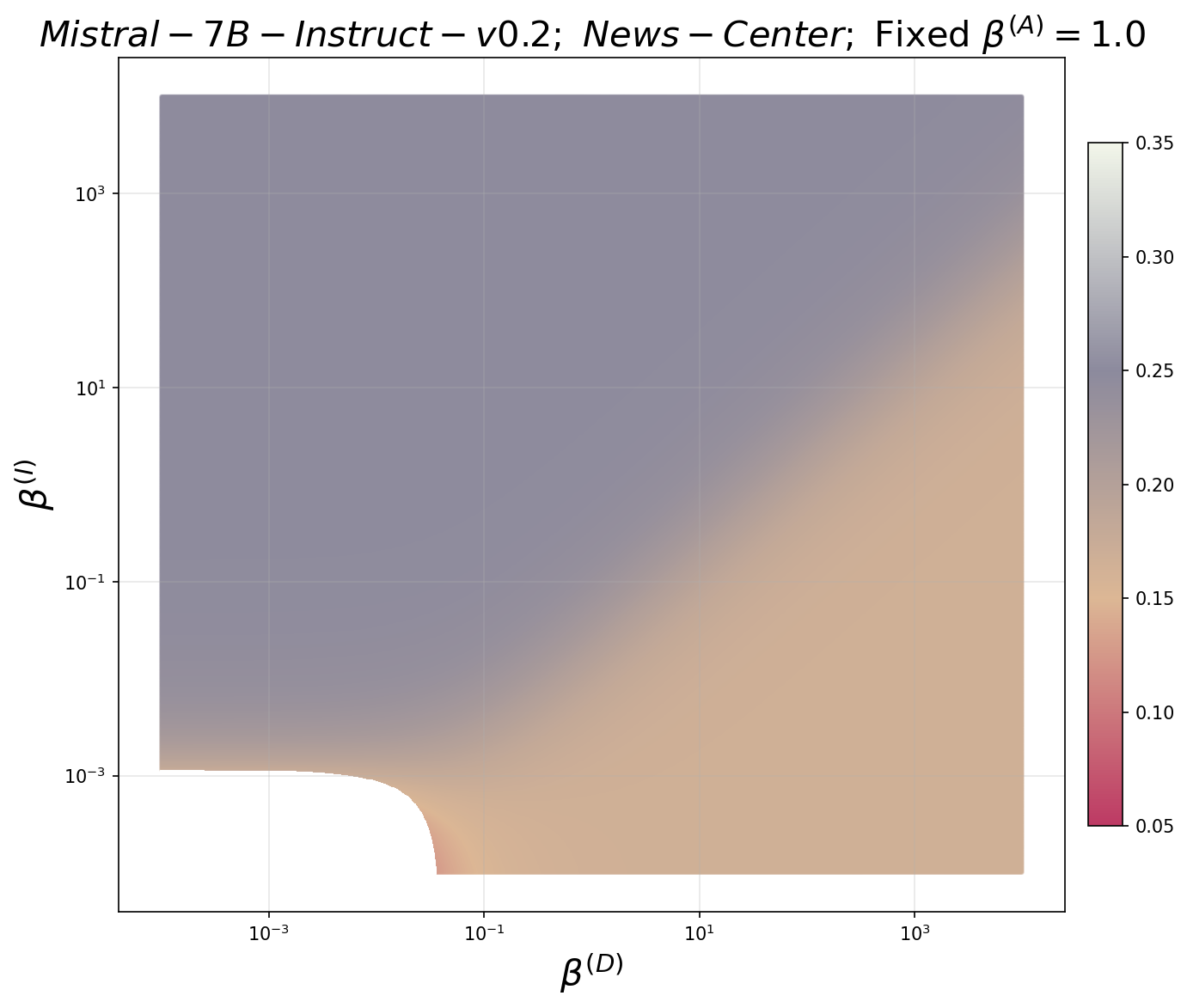}
    \includegraphics[width=0.25\linewidth]{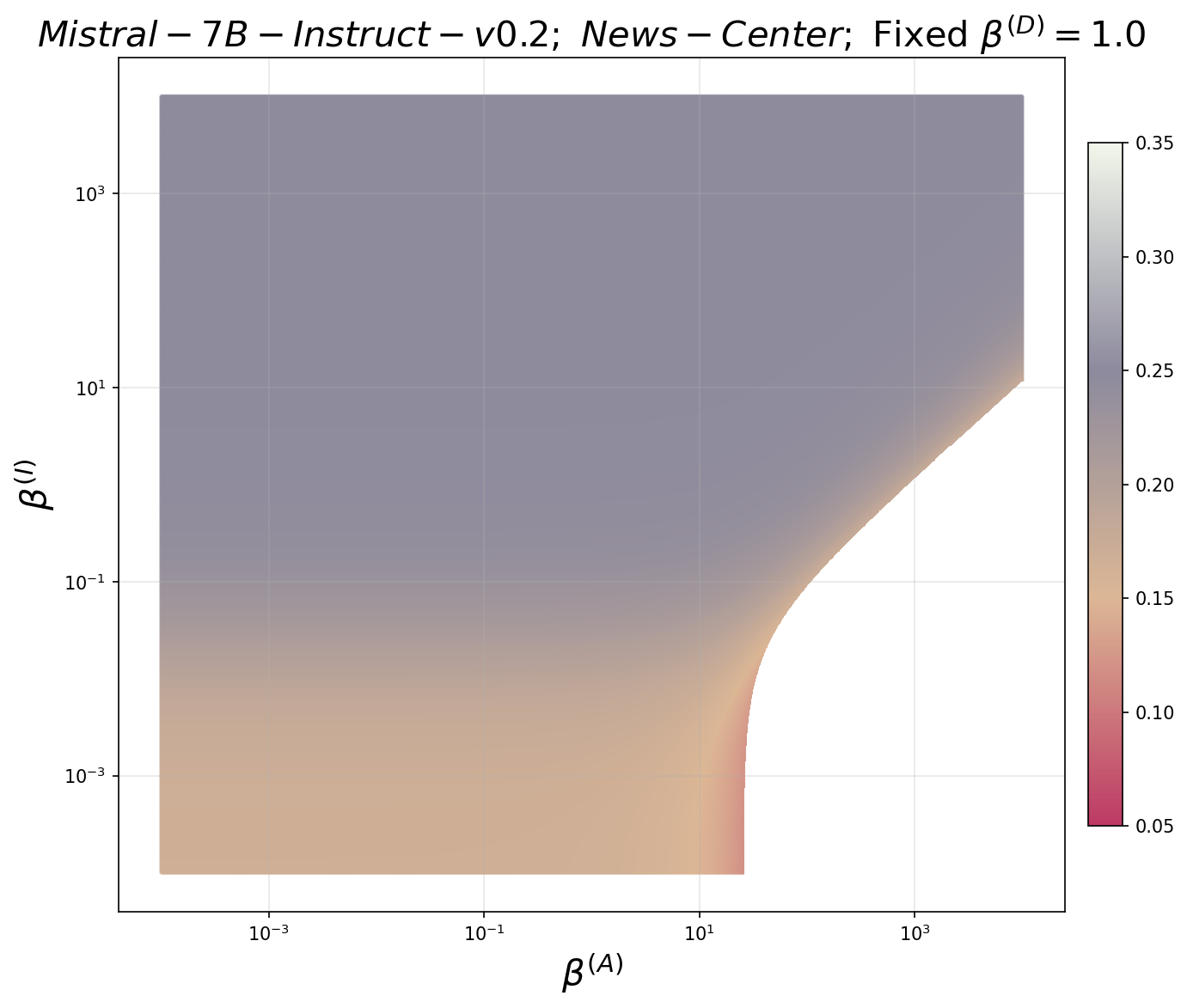}
    \includegraphics[width=0.25\linewidth]{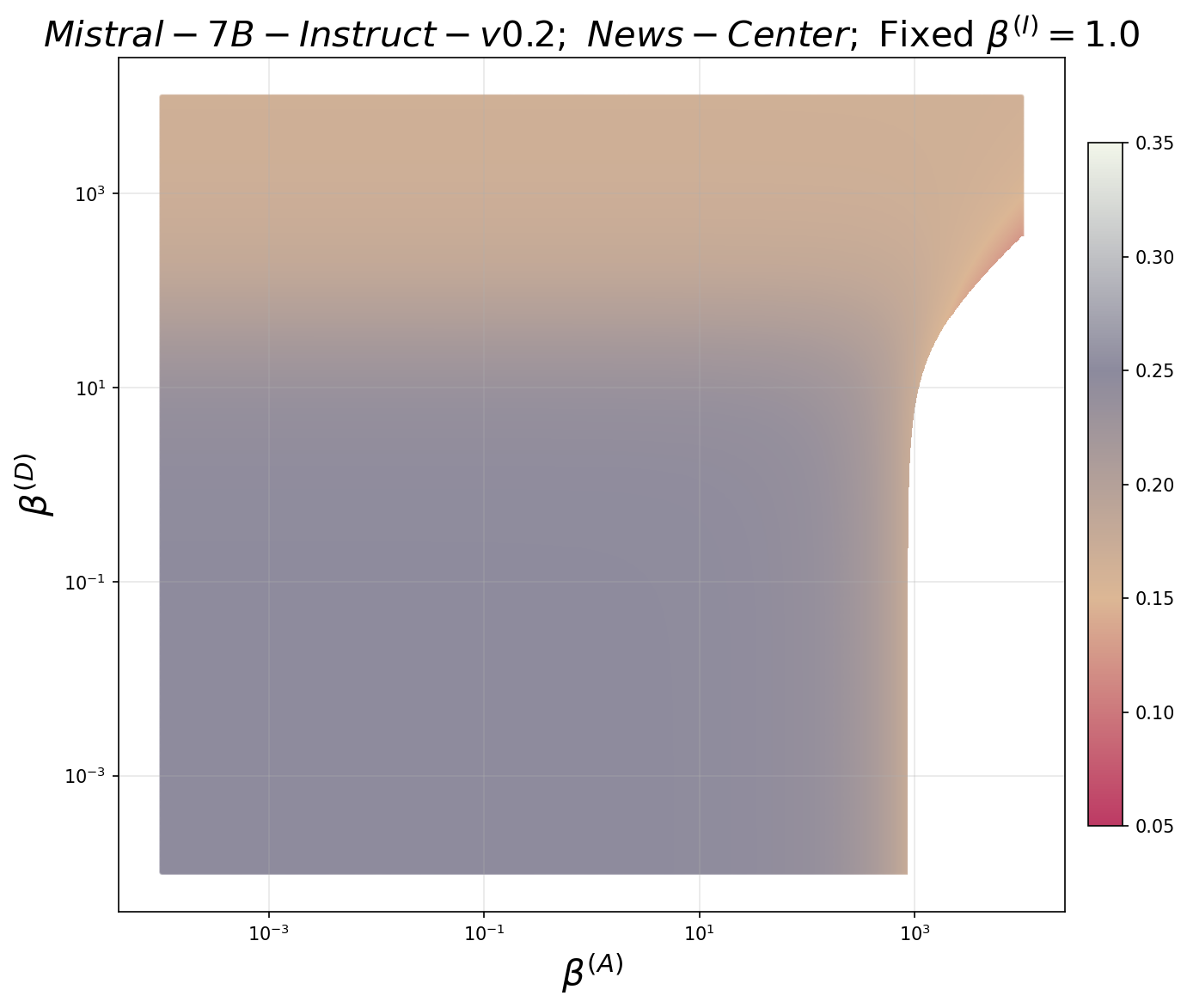}

    \includegraphics[width=0.25\linewidth]{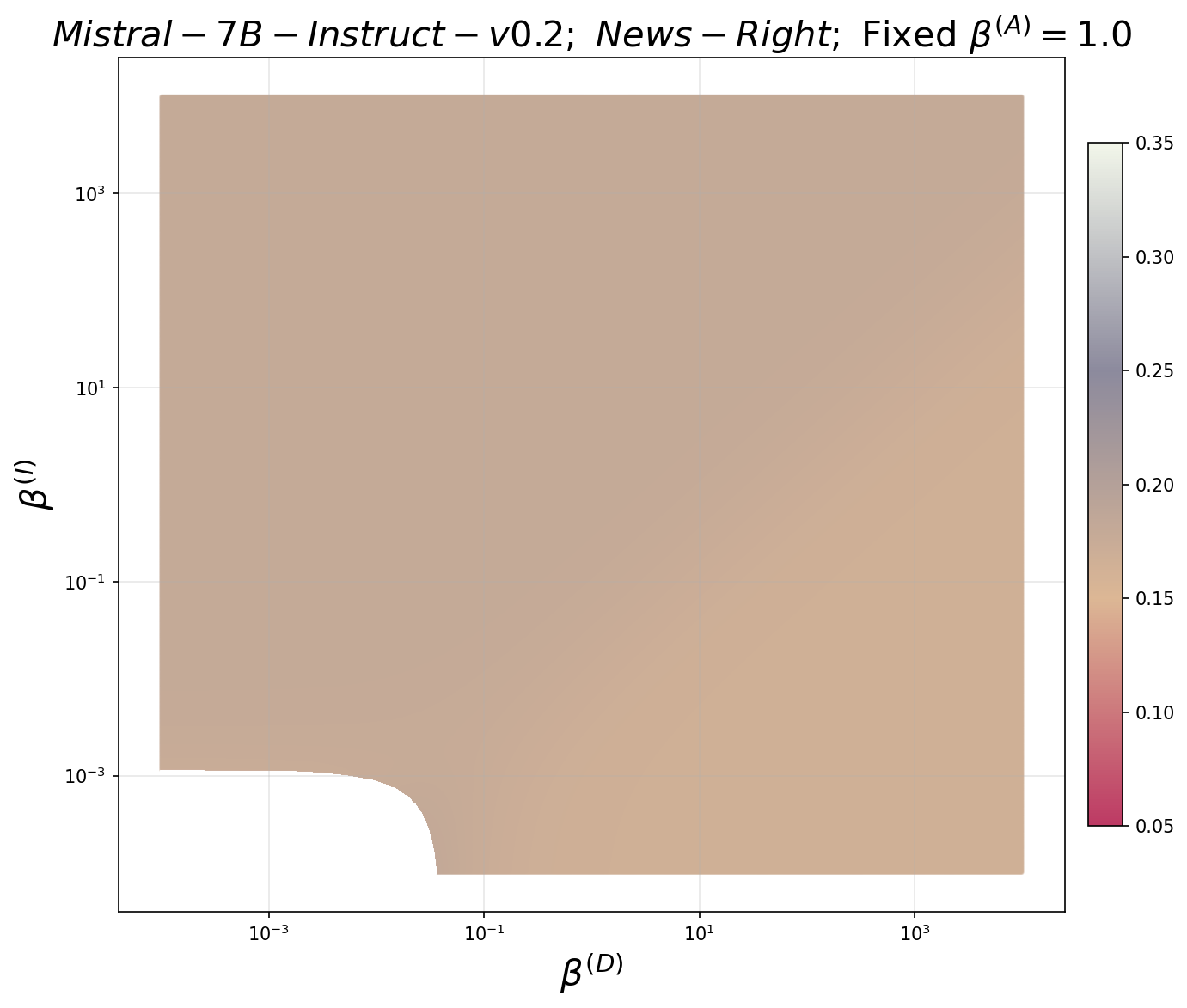}
    \includegraphics[width=0.25\linewidth]{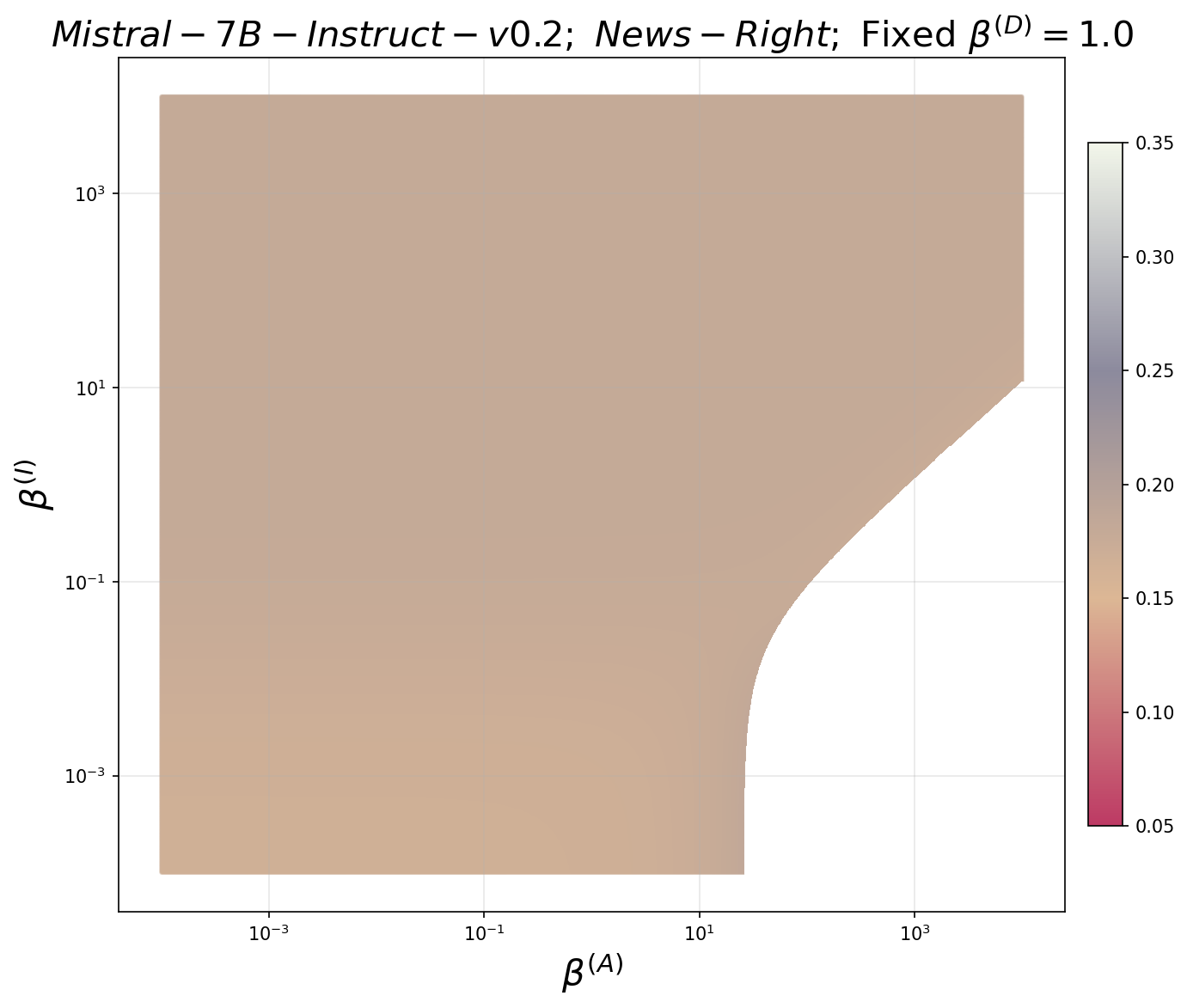}
    \includegraphics[width=0.25\linewidth]{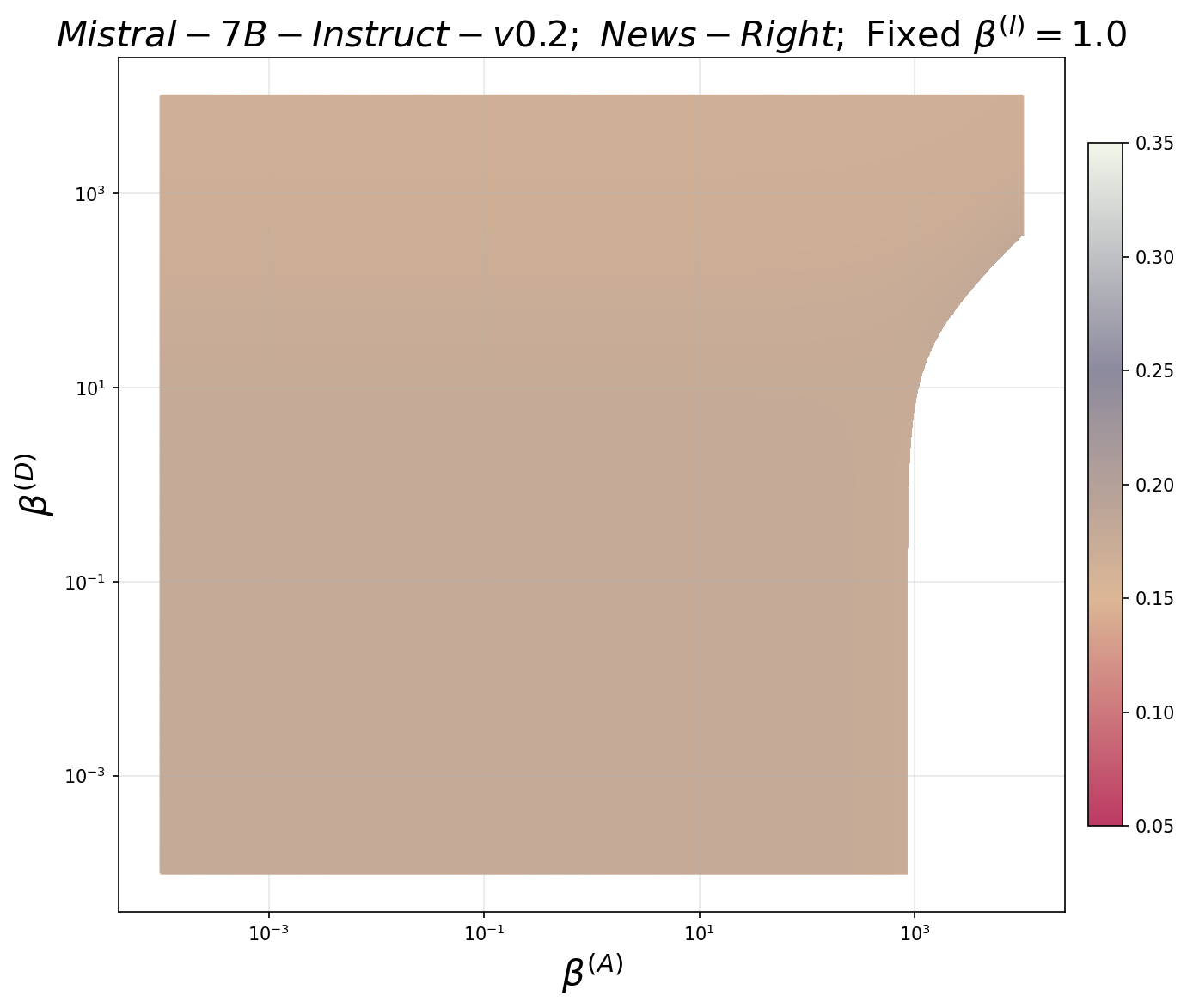}
    
    \centering

    \includegraphics[width=0.25\linewidth]{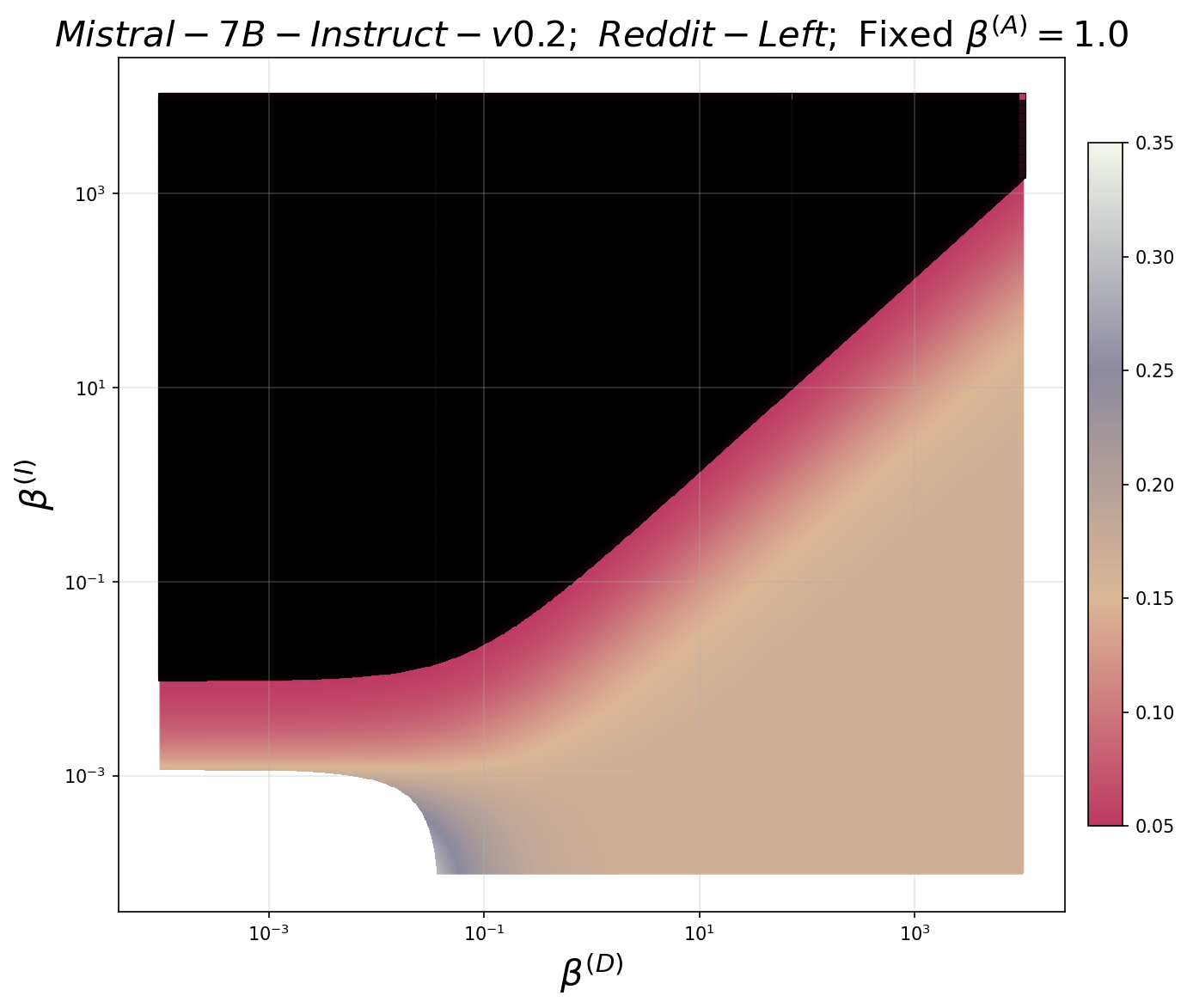}
    \includegraphics[width=0.25\linewidth]{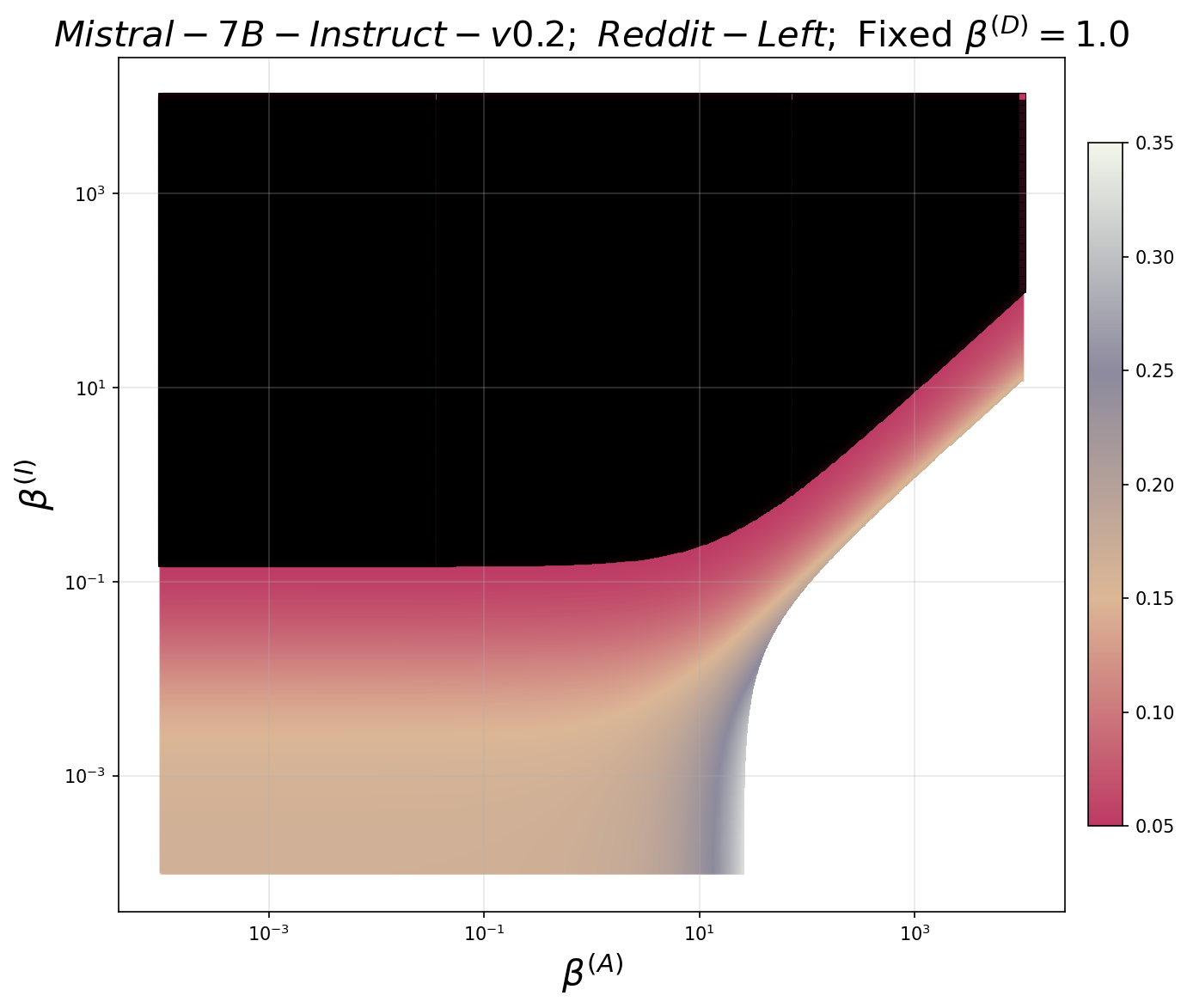}
    \includegraphics[width=0.25\linewidth]{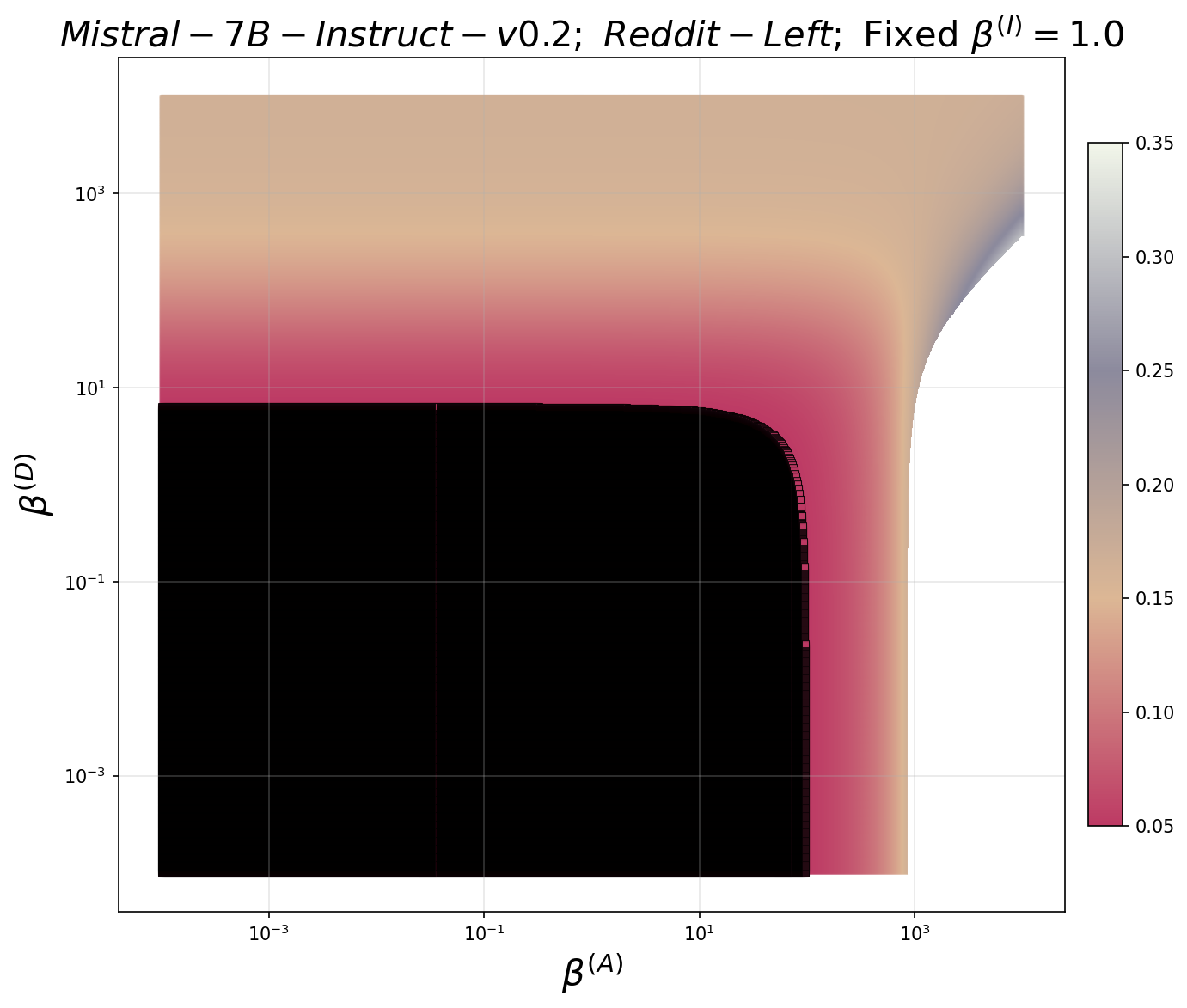}

    \includegraphics[width=0.25\linewidth]{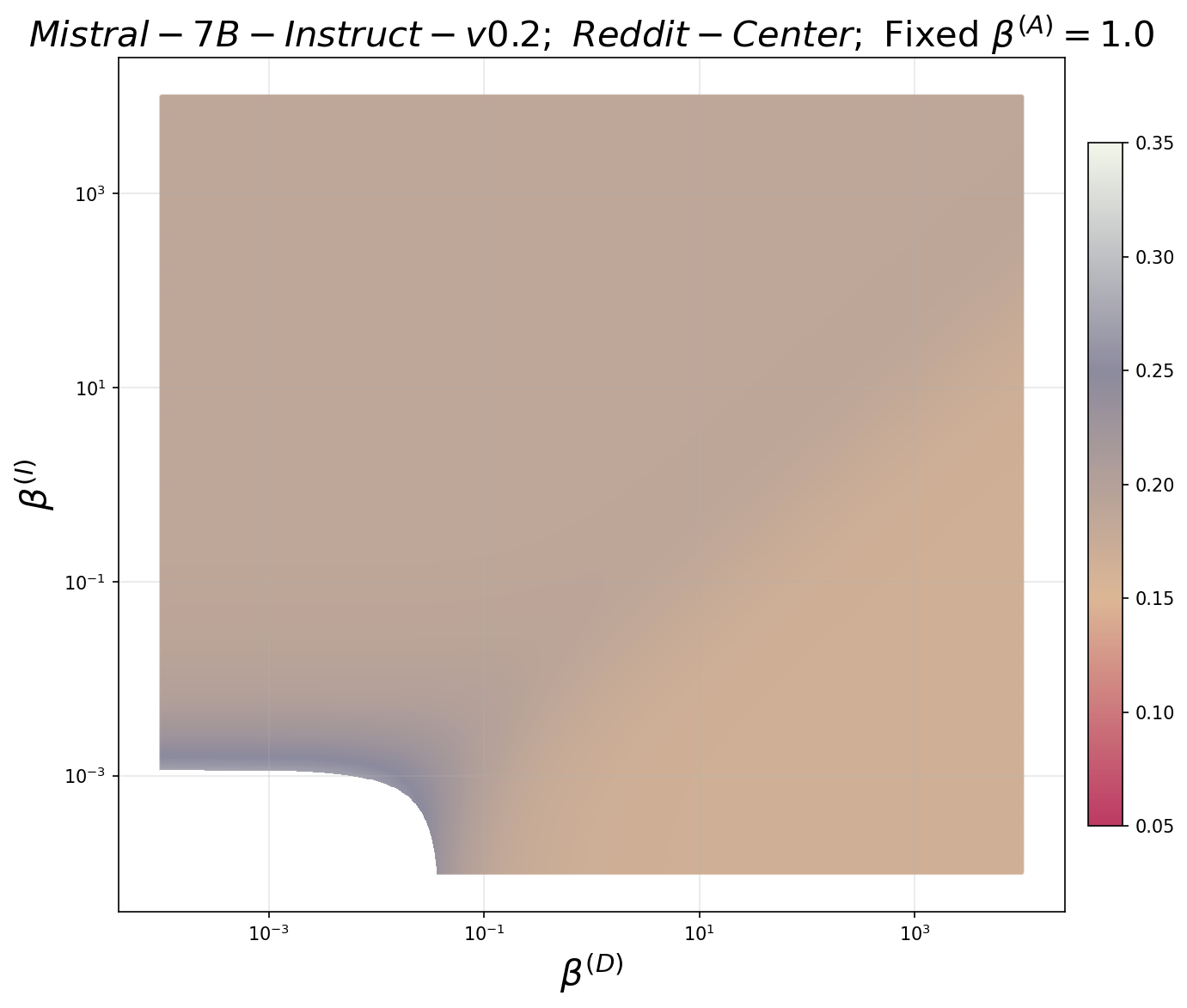}
    \includegraphics[width=0.25\linewidth]{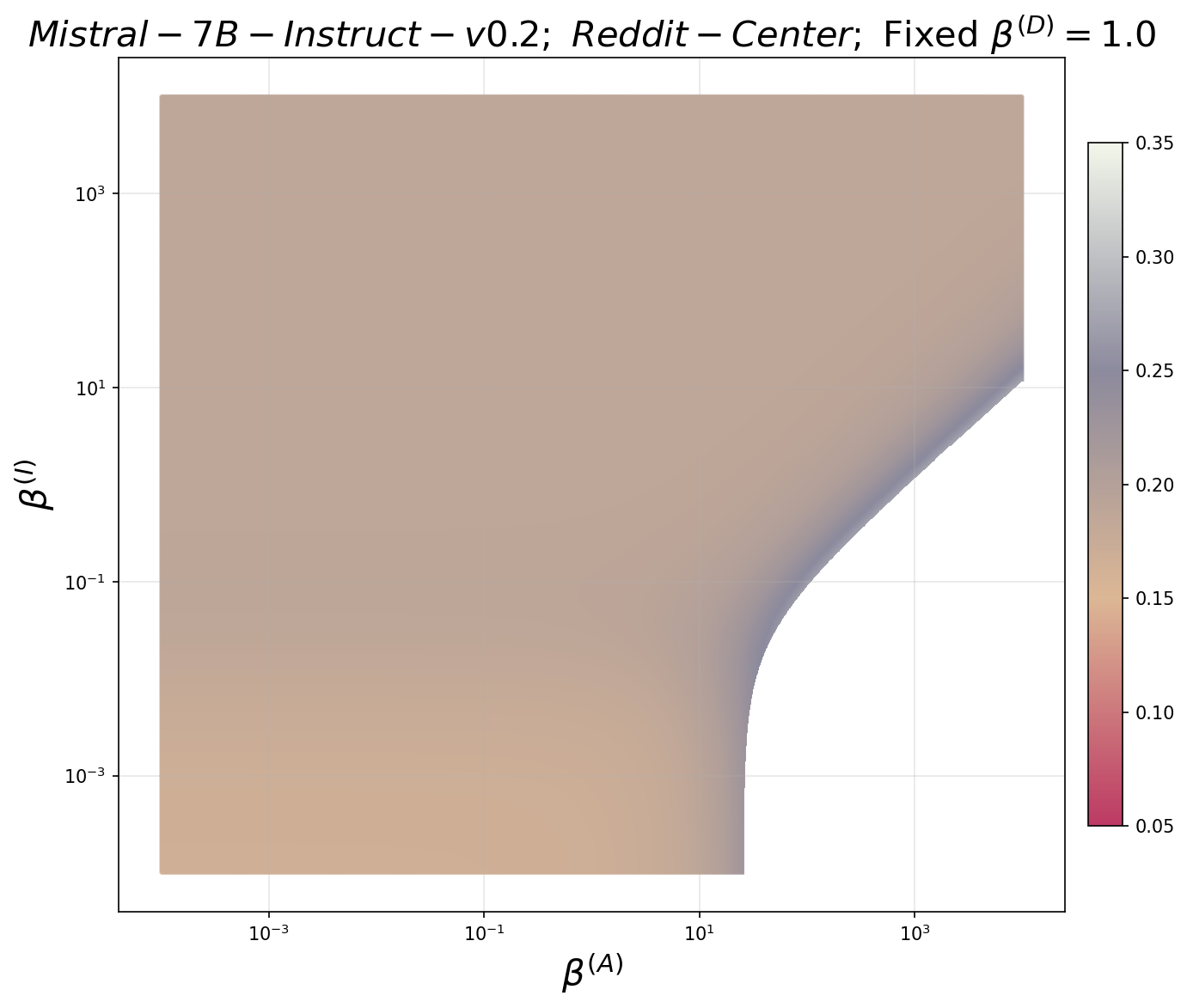}
    \includegraphics[width=0.25\linewidth]{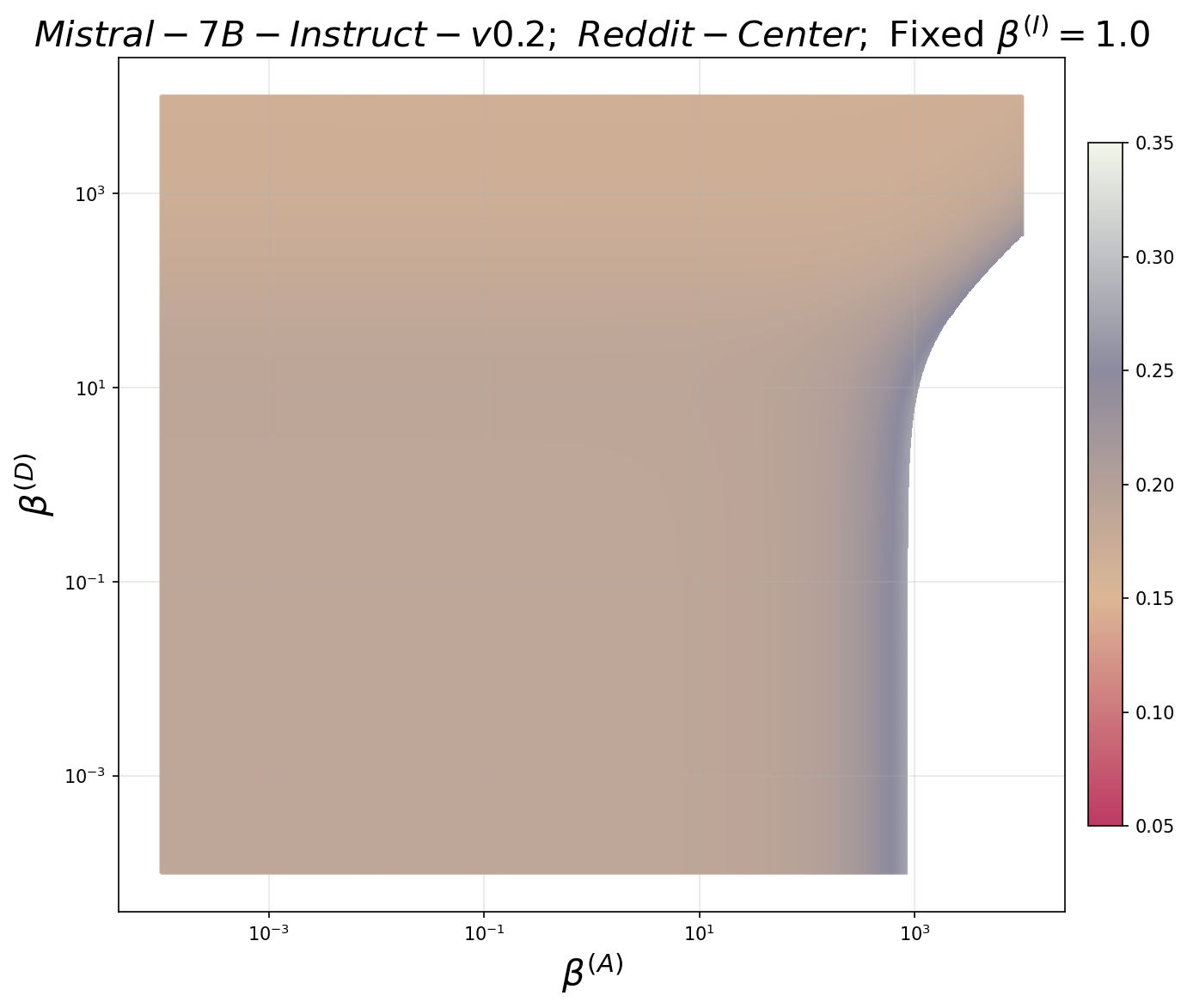}

    \includegraphics[width=0.25\linewidth]{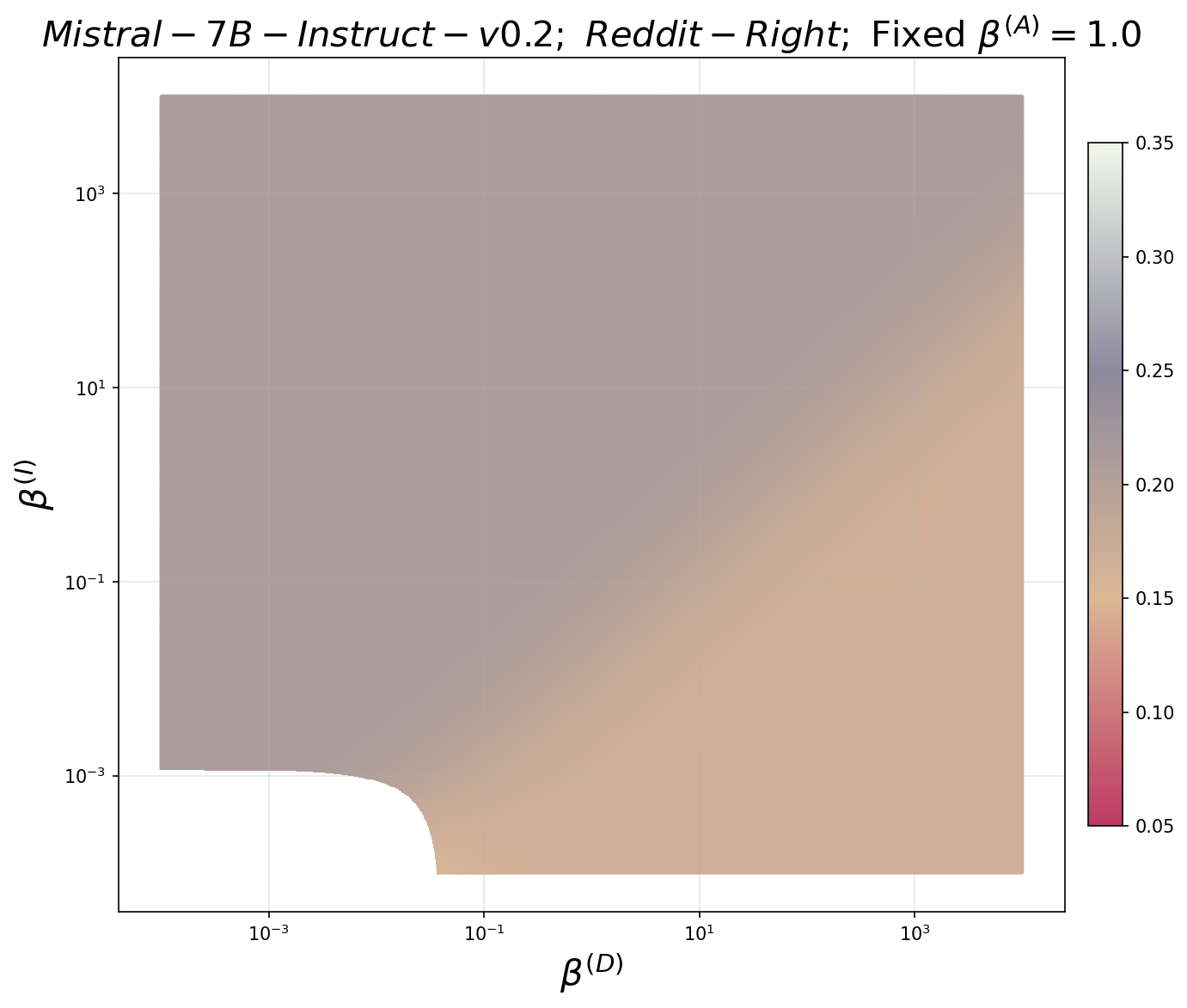}
    \includegraphics[width=0.25\linewidth]{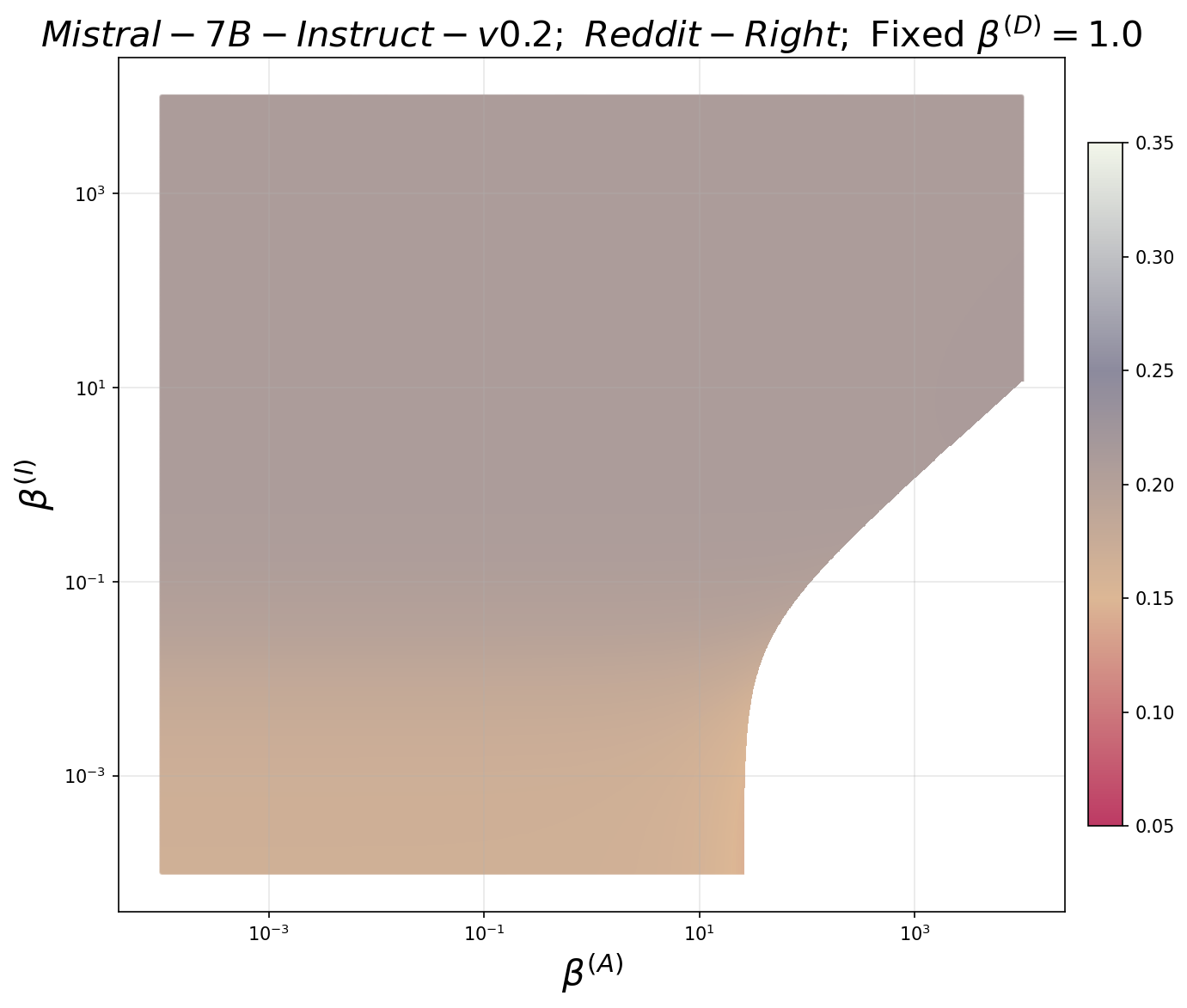}
    \includegraphics[width=0.25\linewidth]{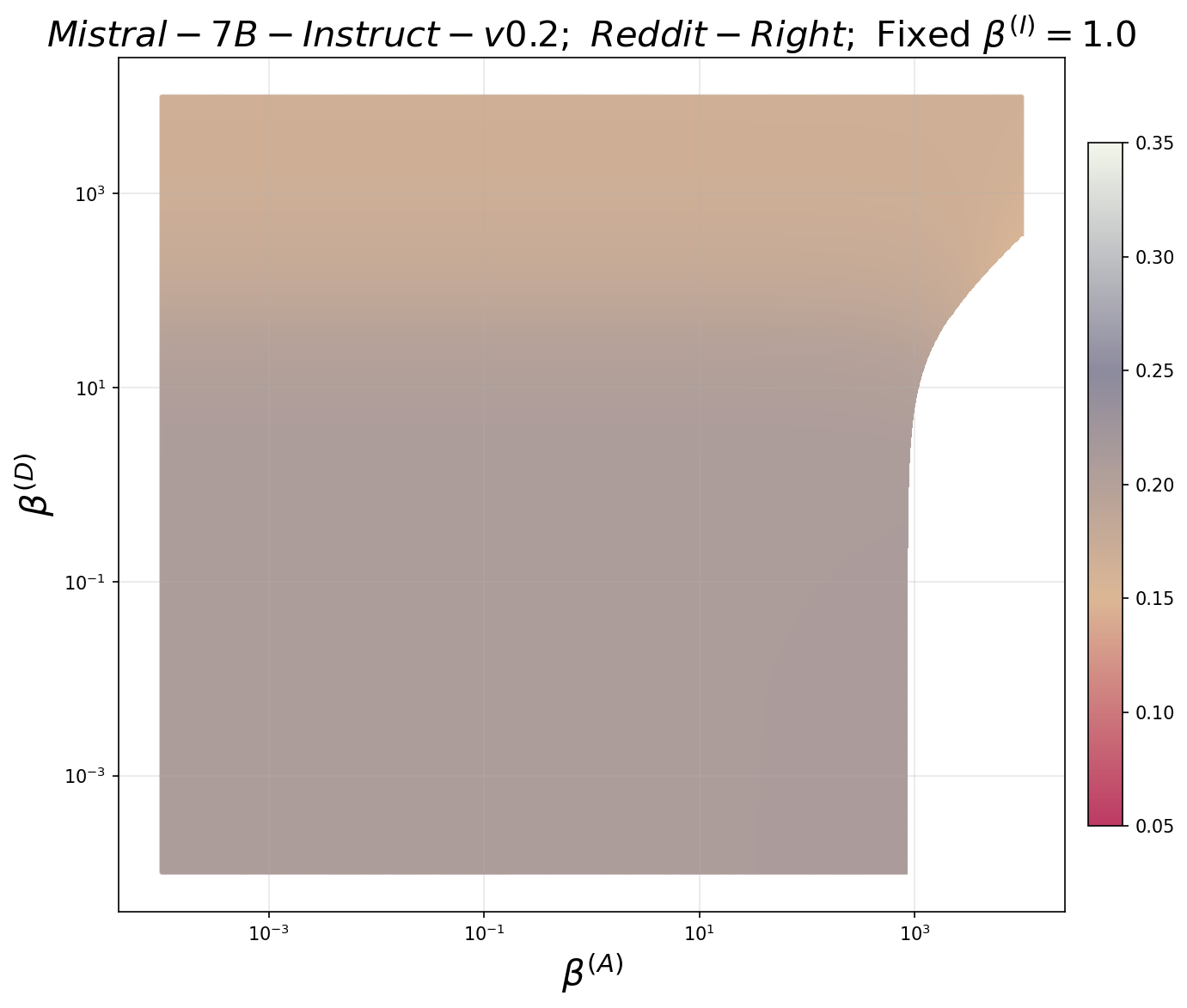}
    
    \caption{Political exclusion of \emph{Mistral-7B-Instruct-v0.2} on the $\mathtt{POLITICS}$ dataset.}
    \label{fig:Reddit}
\end{figure}


\begin{figure}
    \centering
    \includegraphics[width=0.3\linewidth]{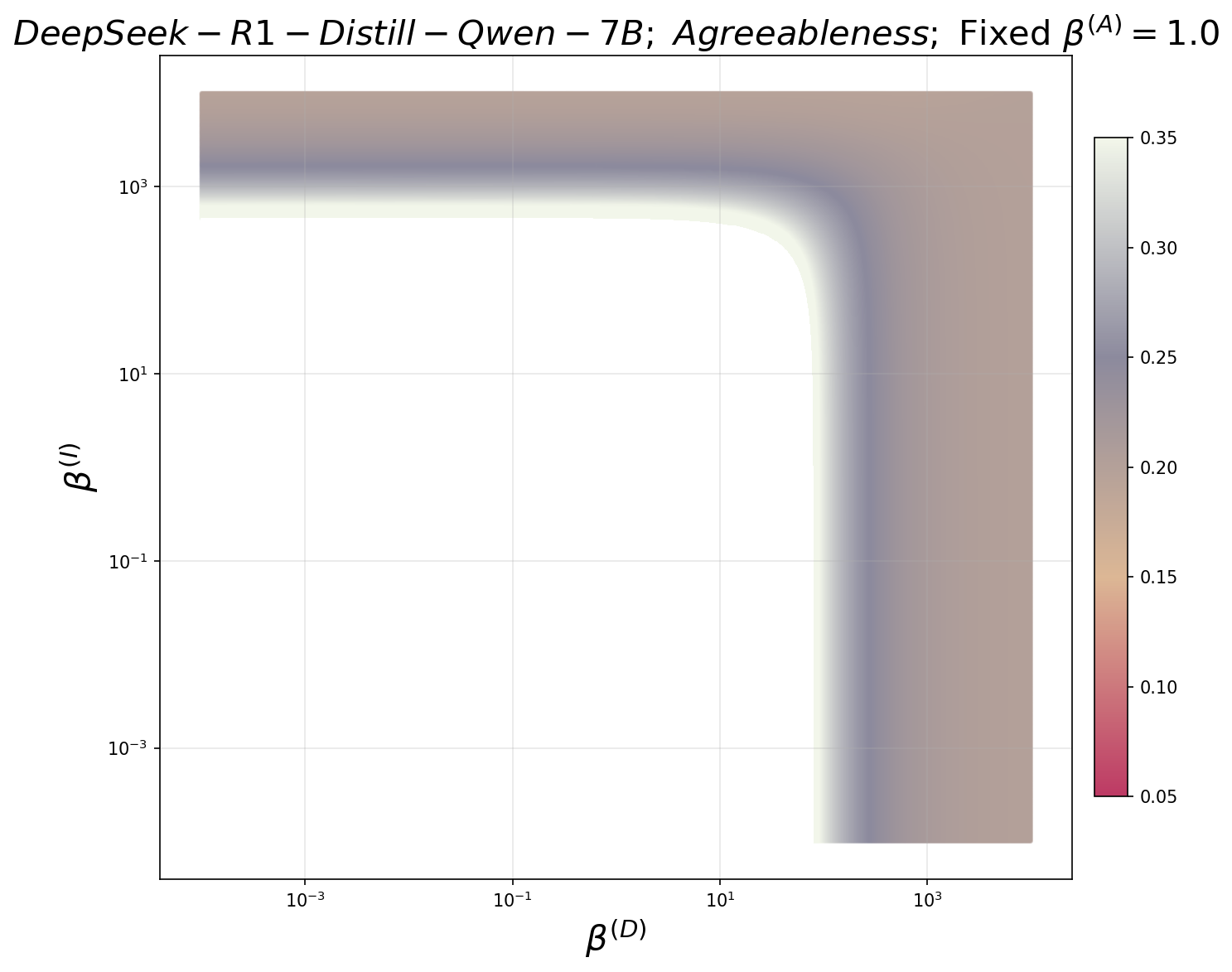}
    \includegraphics[width=0.3\linewidth]{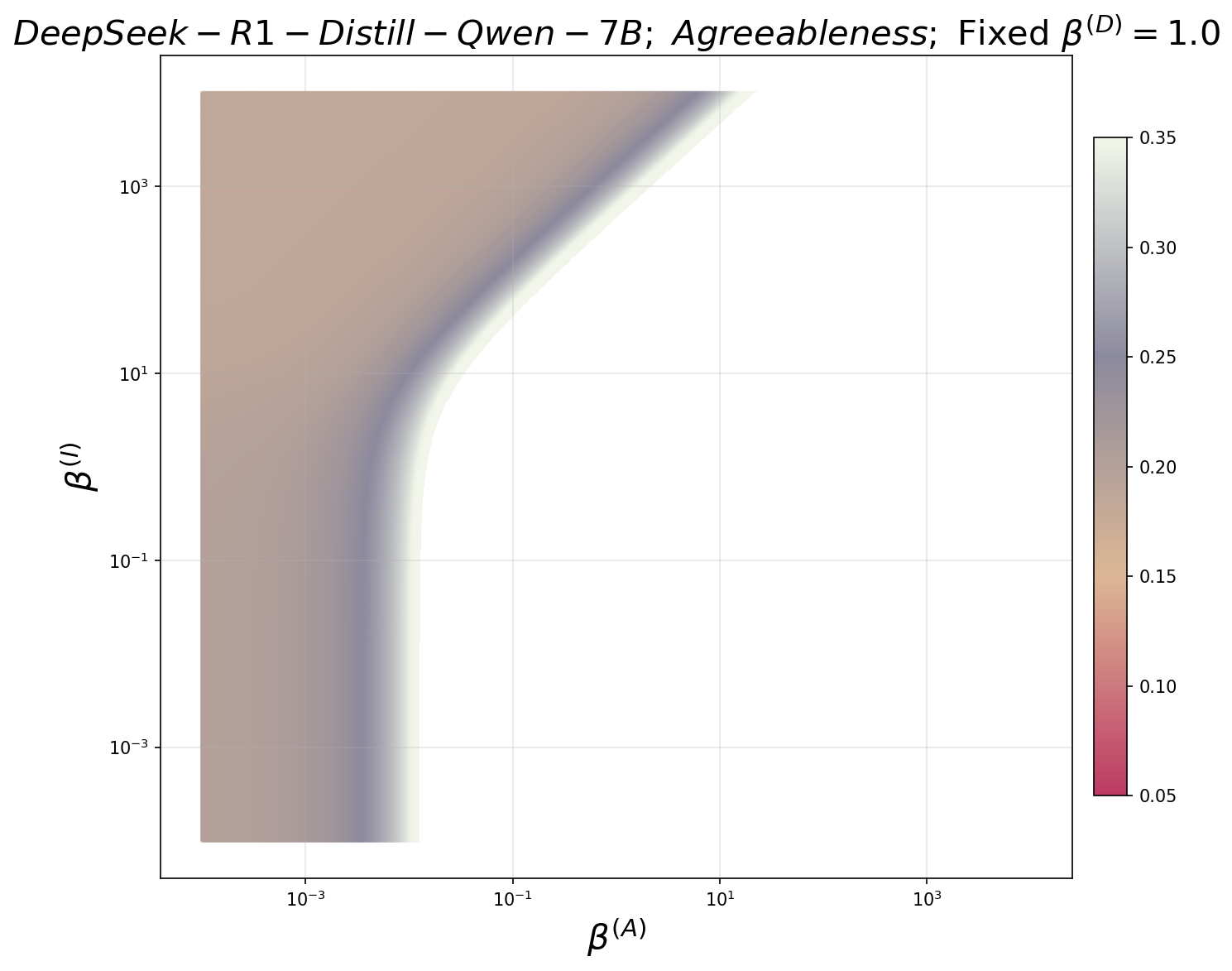}
    \includegraphics[width=0.3\linewidth]{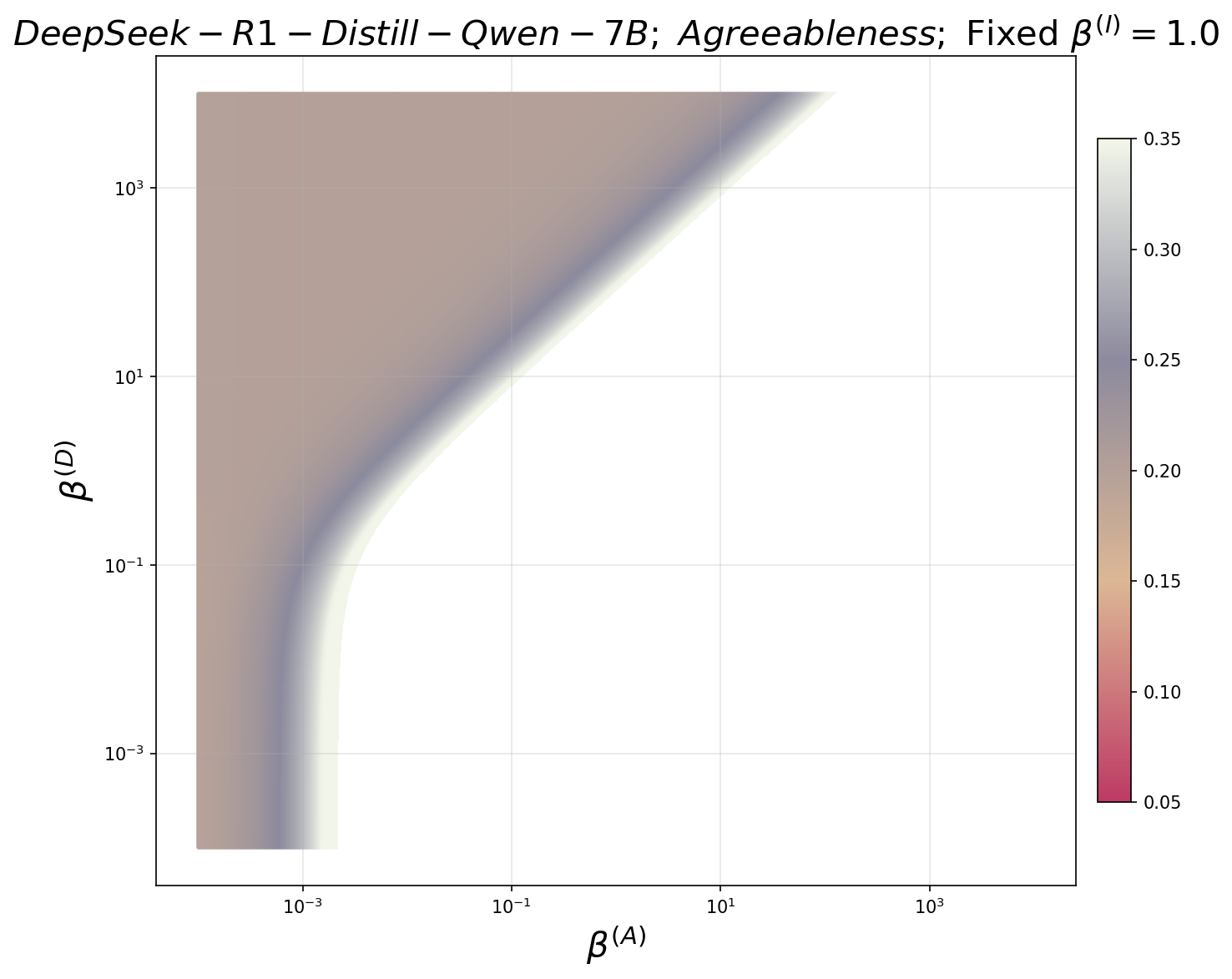}

    \includegraphics[width=0.3\linewidth]{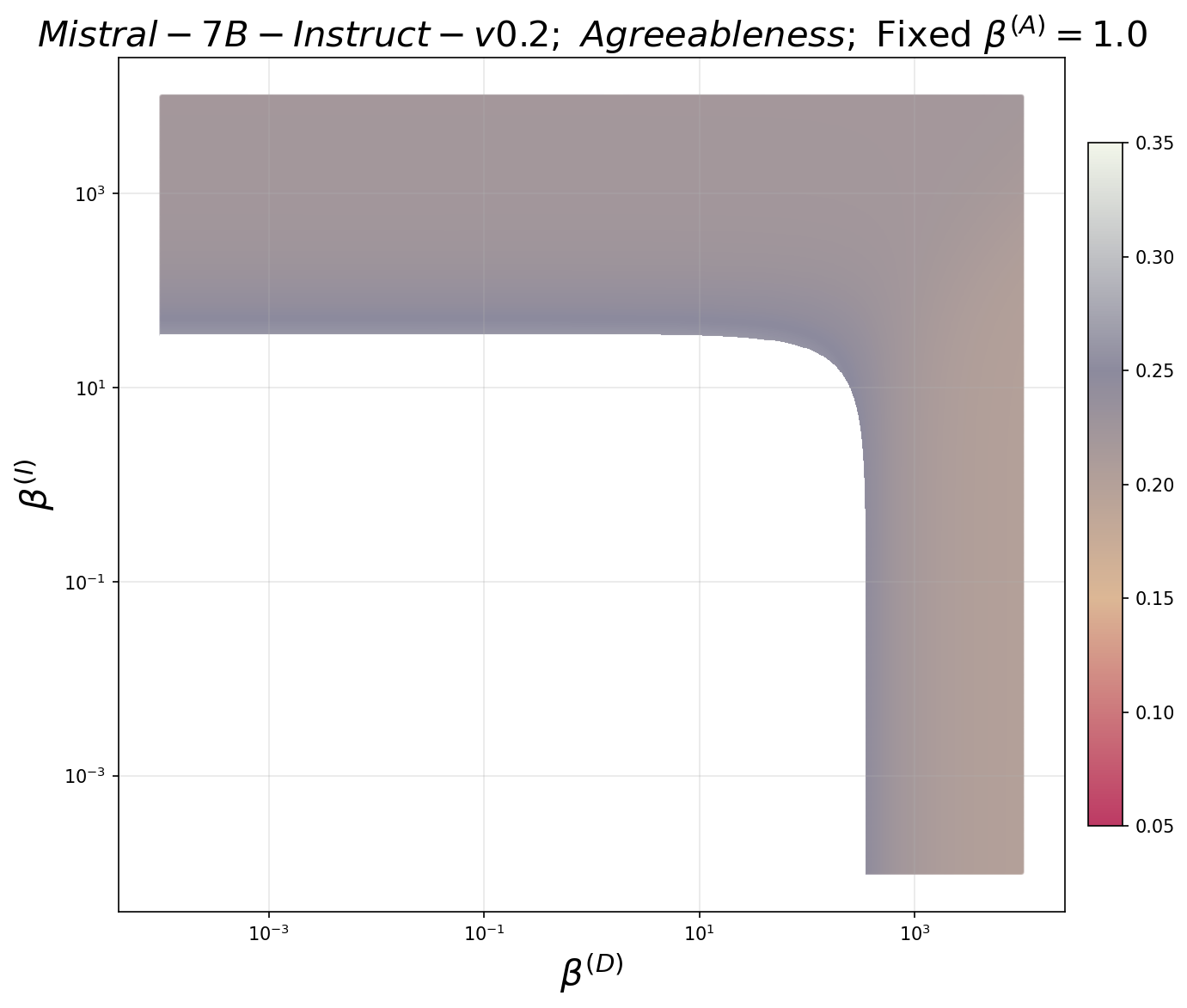}
    \includegraphics[width=0.3\linewidth]{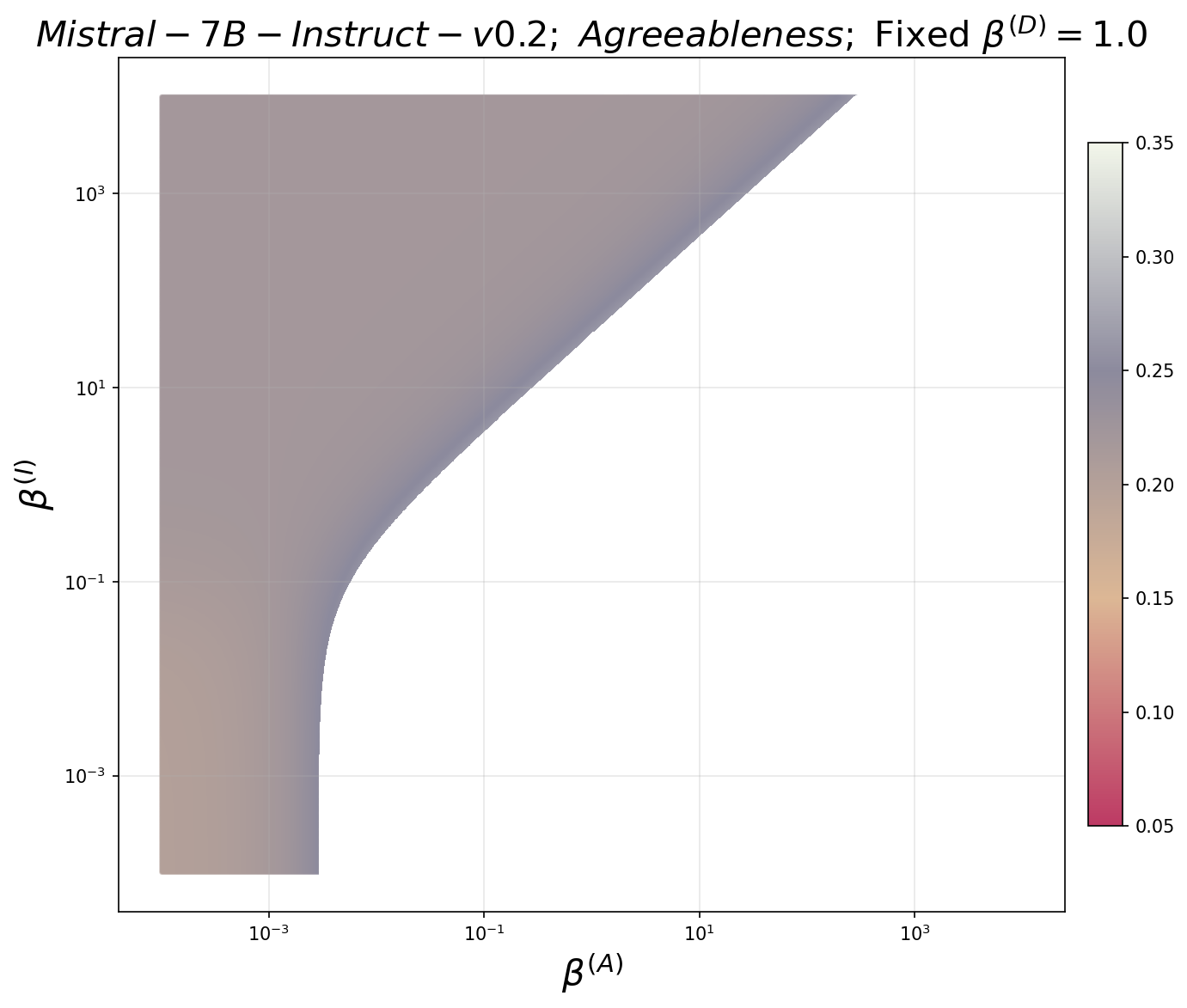}
    \includegraphics[width=0.3\linewidth]{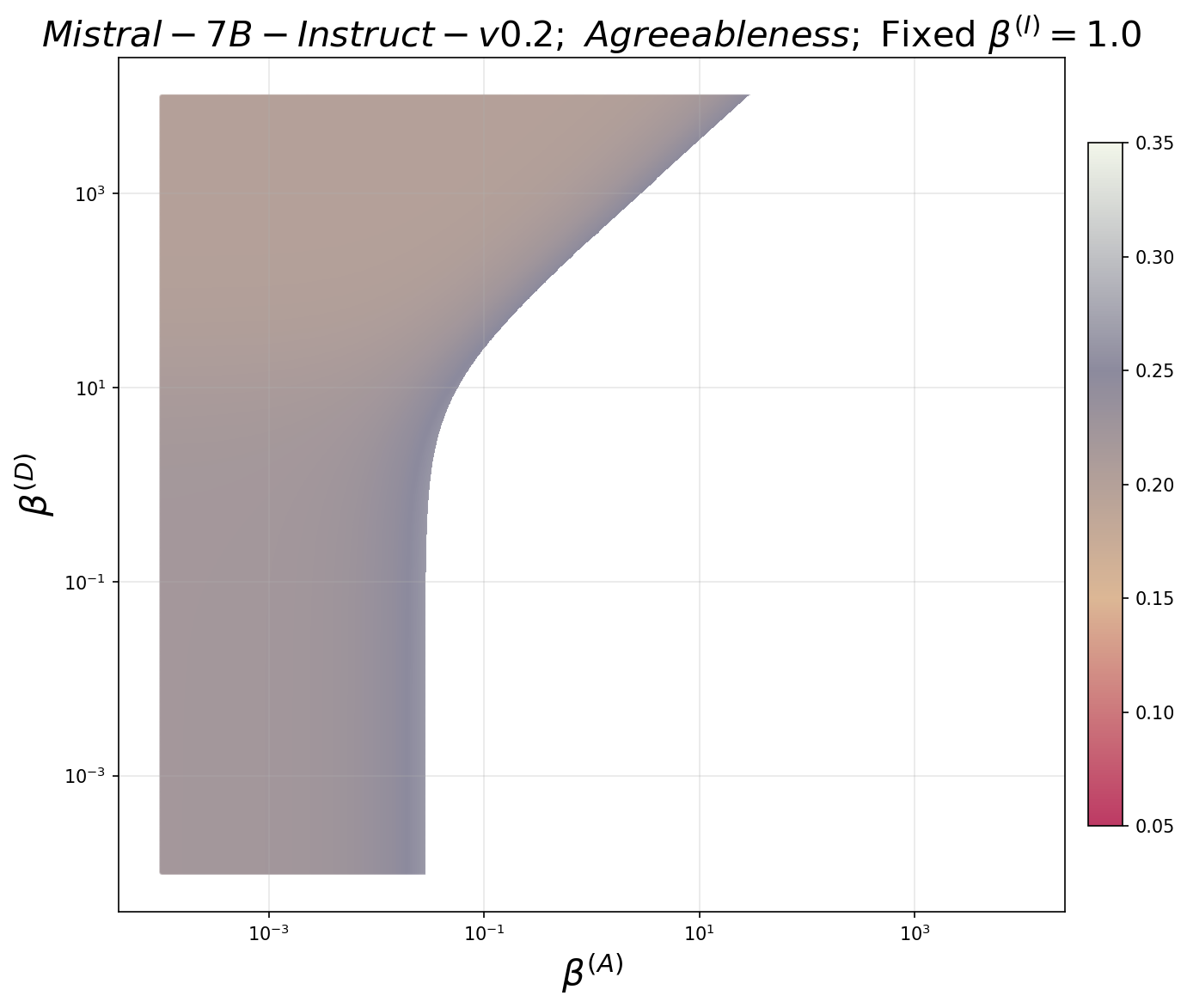}

    \includegraphics[width=0.3\linewidth]{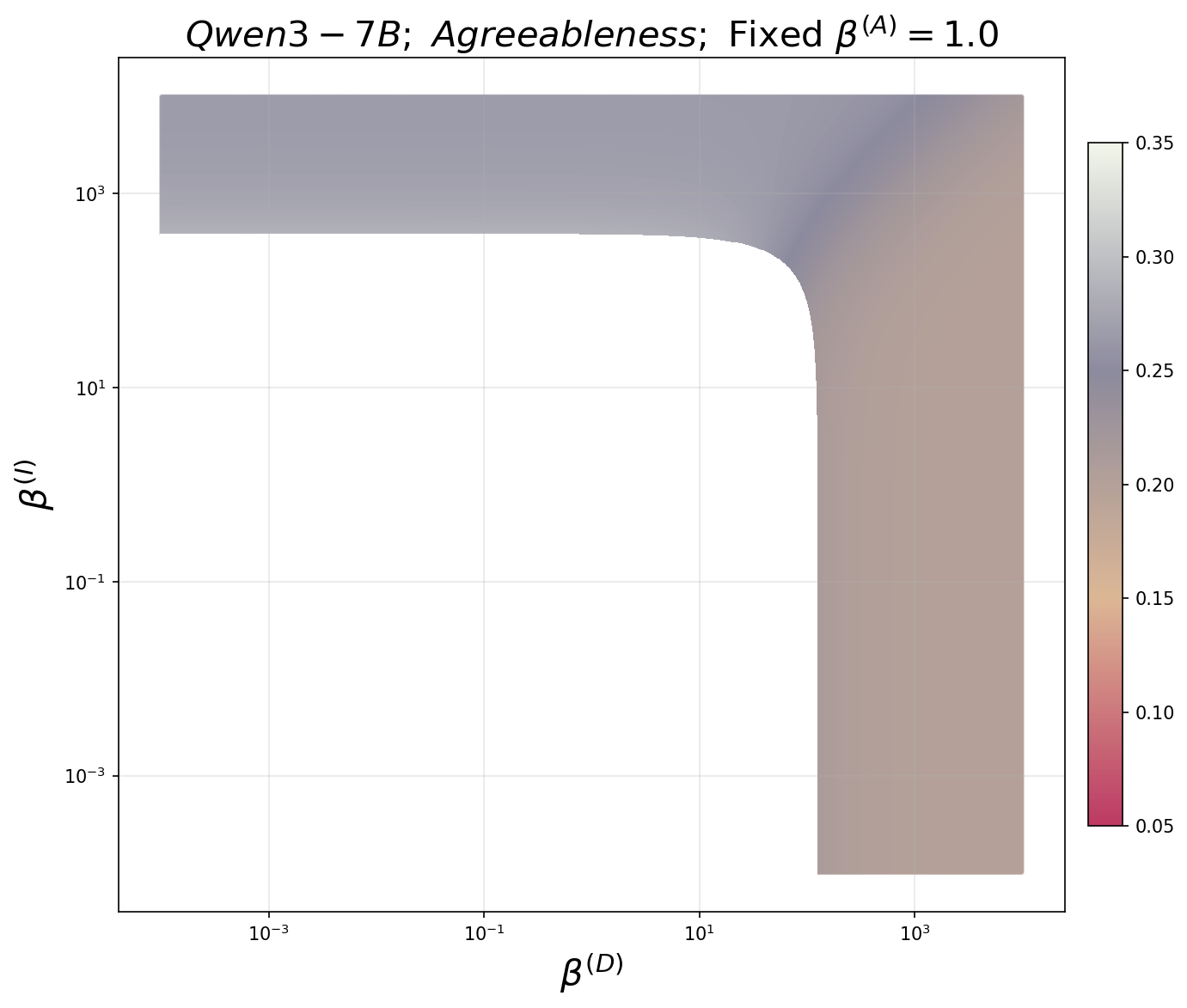}
    \includegraphics[width=0.3\linewidth]{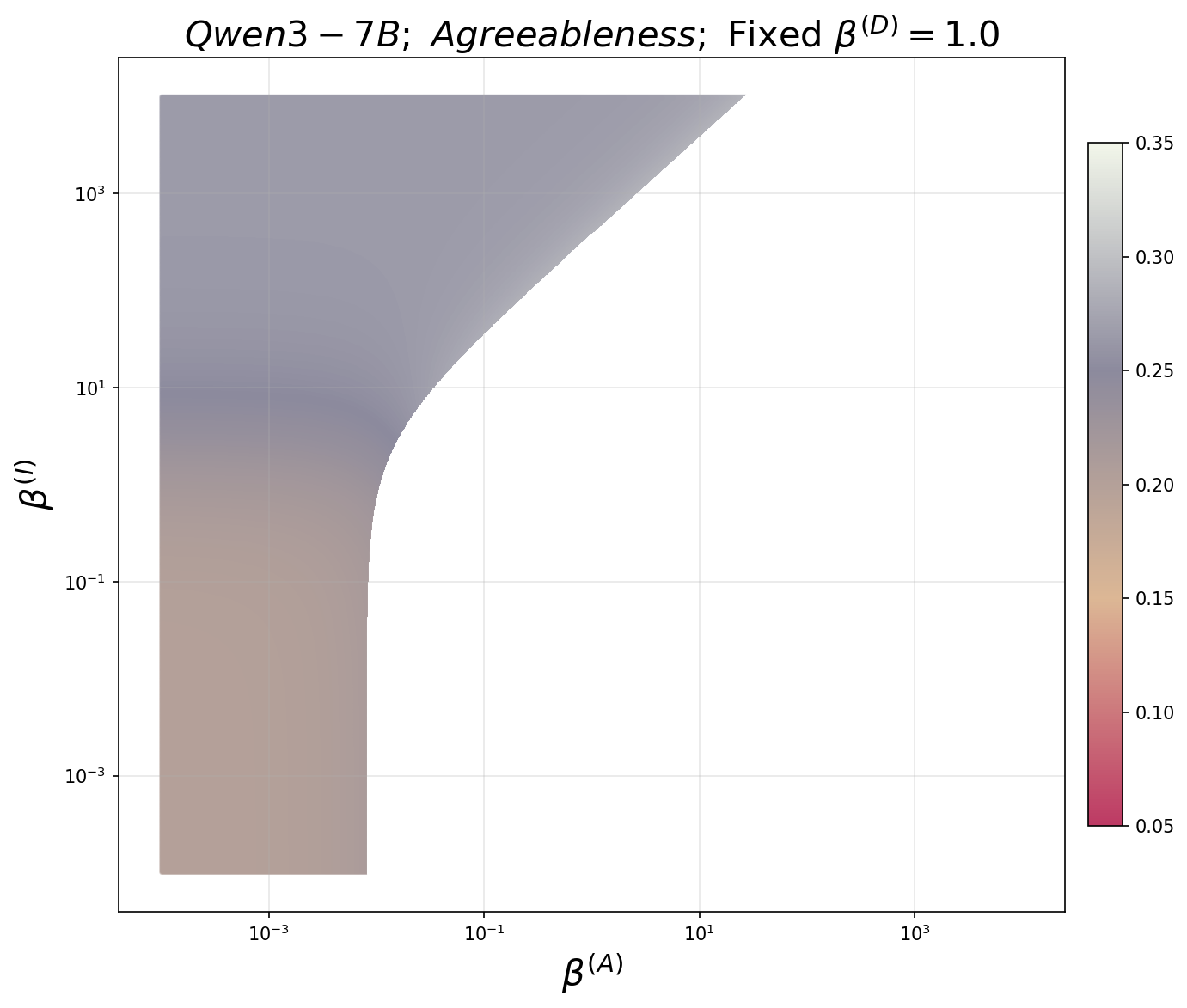}
    \includegraphics[width=0.3\linewidth]{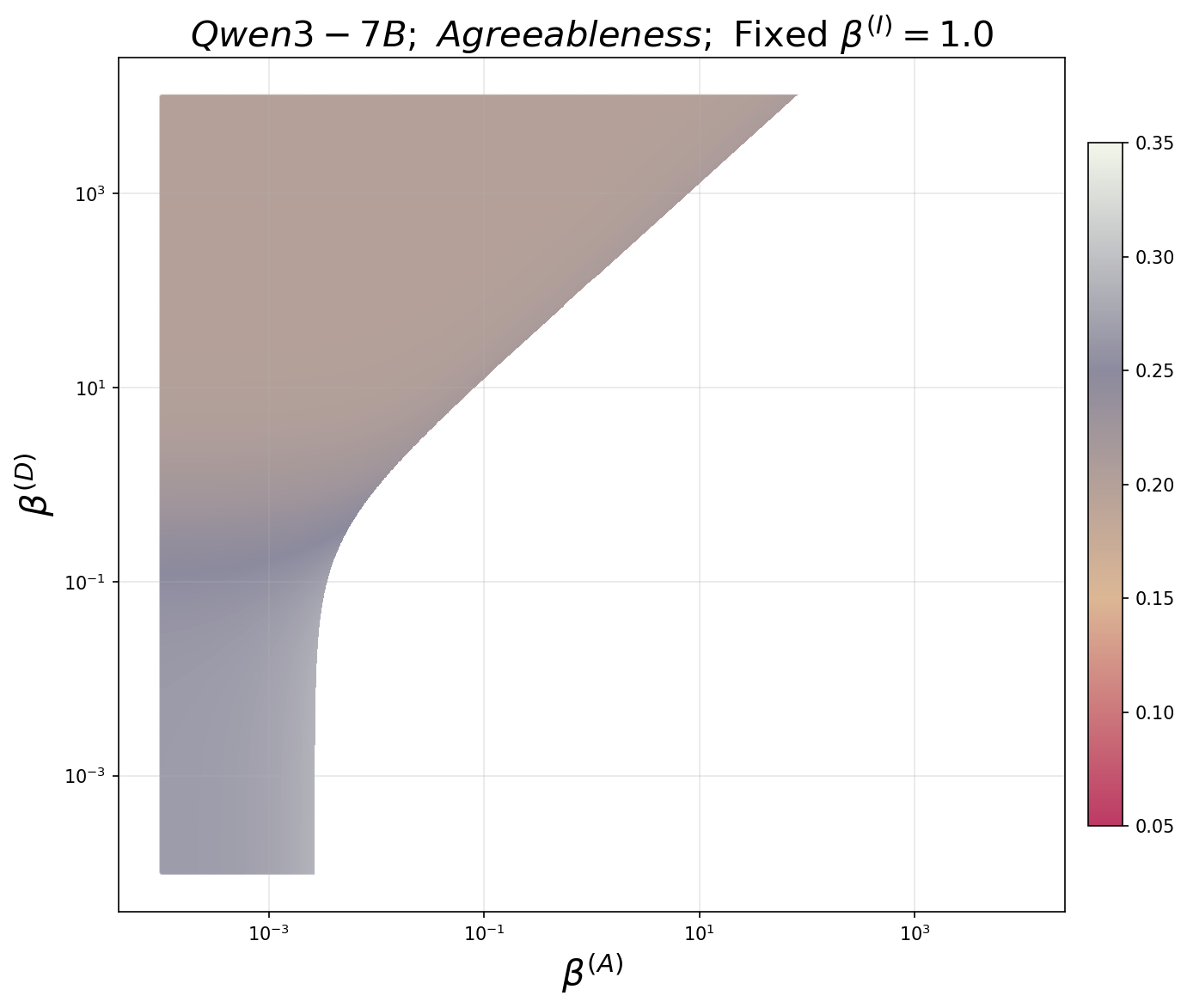}

    \includegraphics[width=0.3\linewidth]{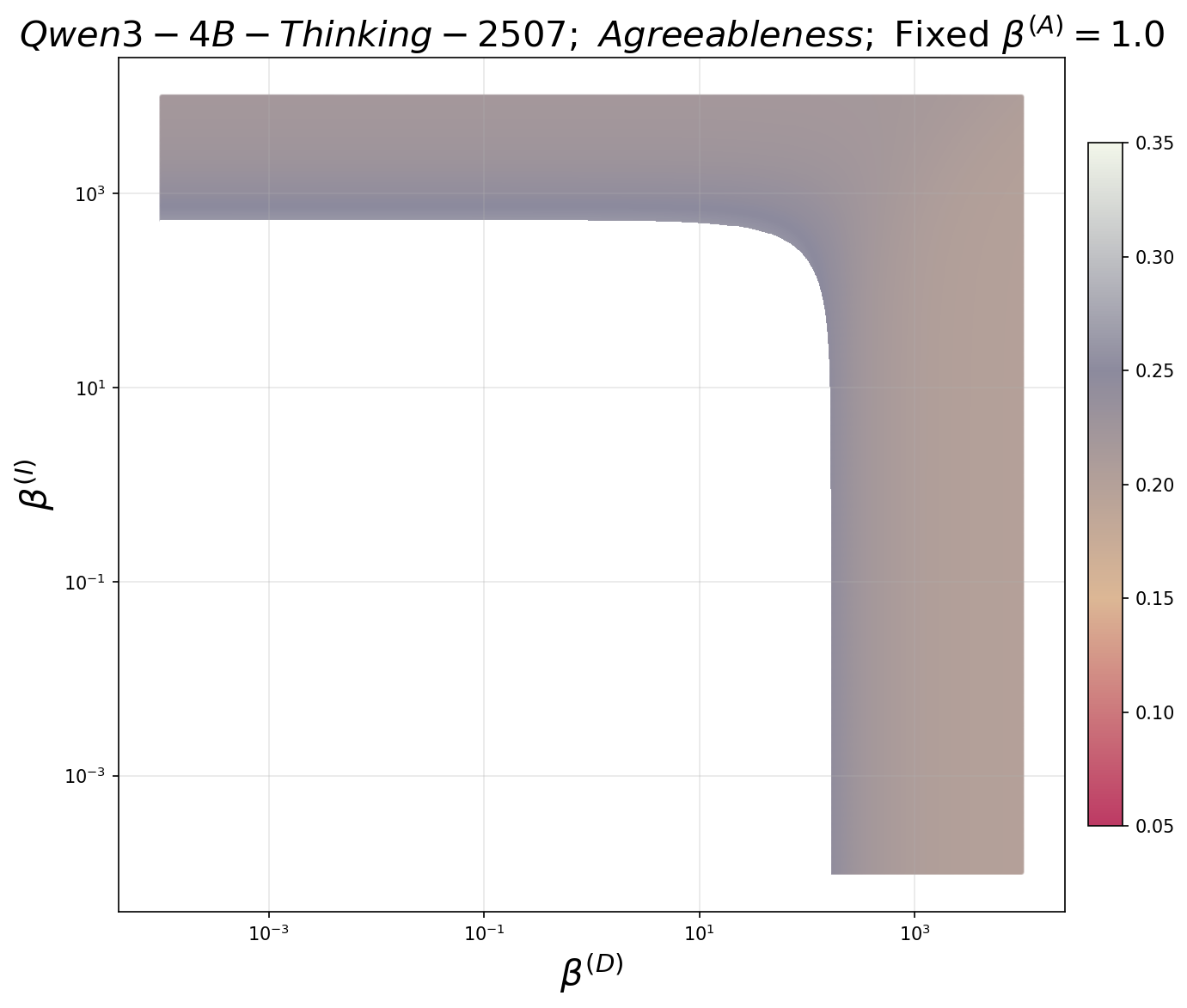}
    \includegraphics[width=0.3\linewidth]{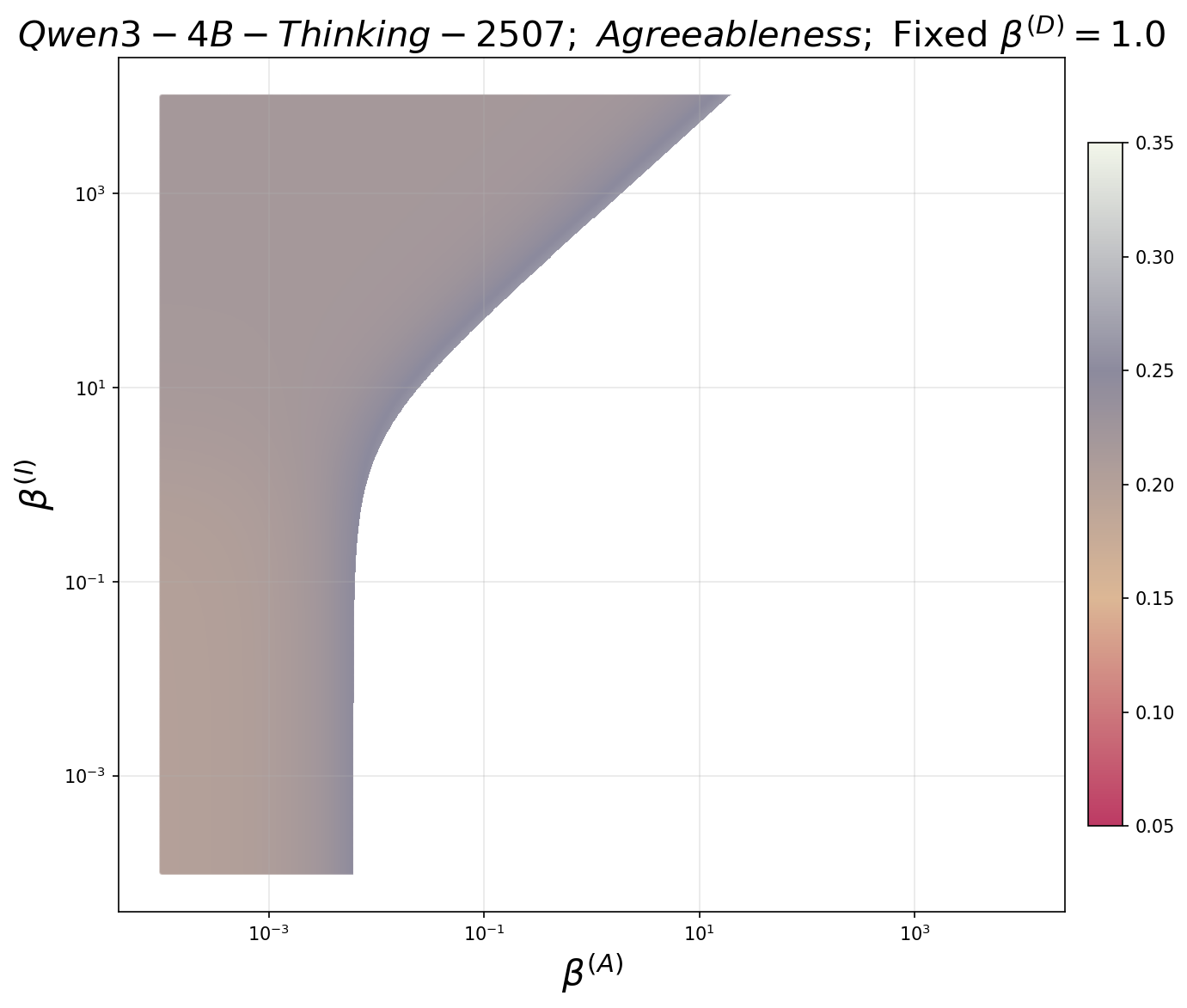}
    \includegraphics[width=0.3\linewidth]{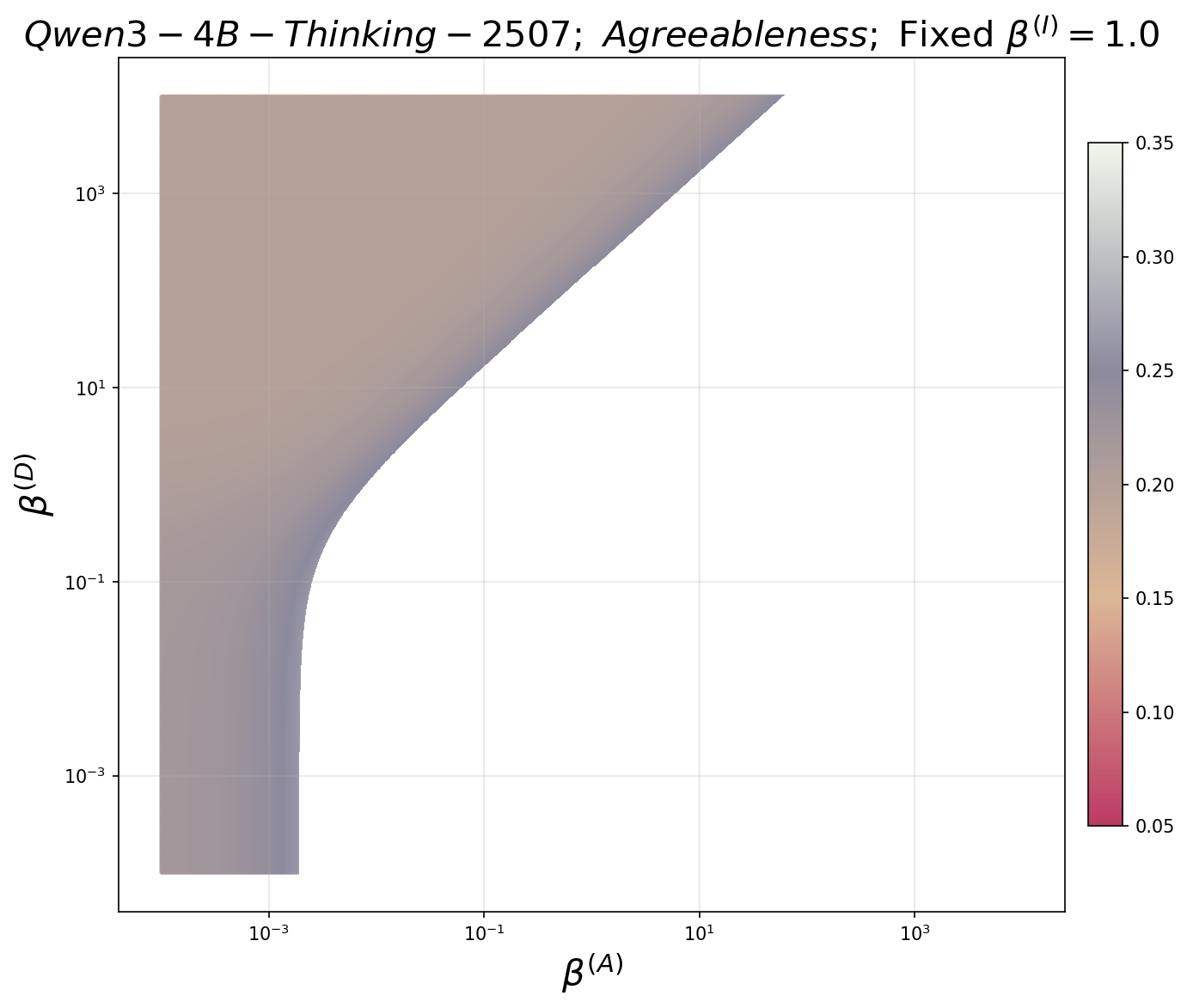}

    \includegraphics[width=0.3\linewidth]{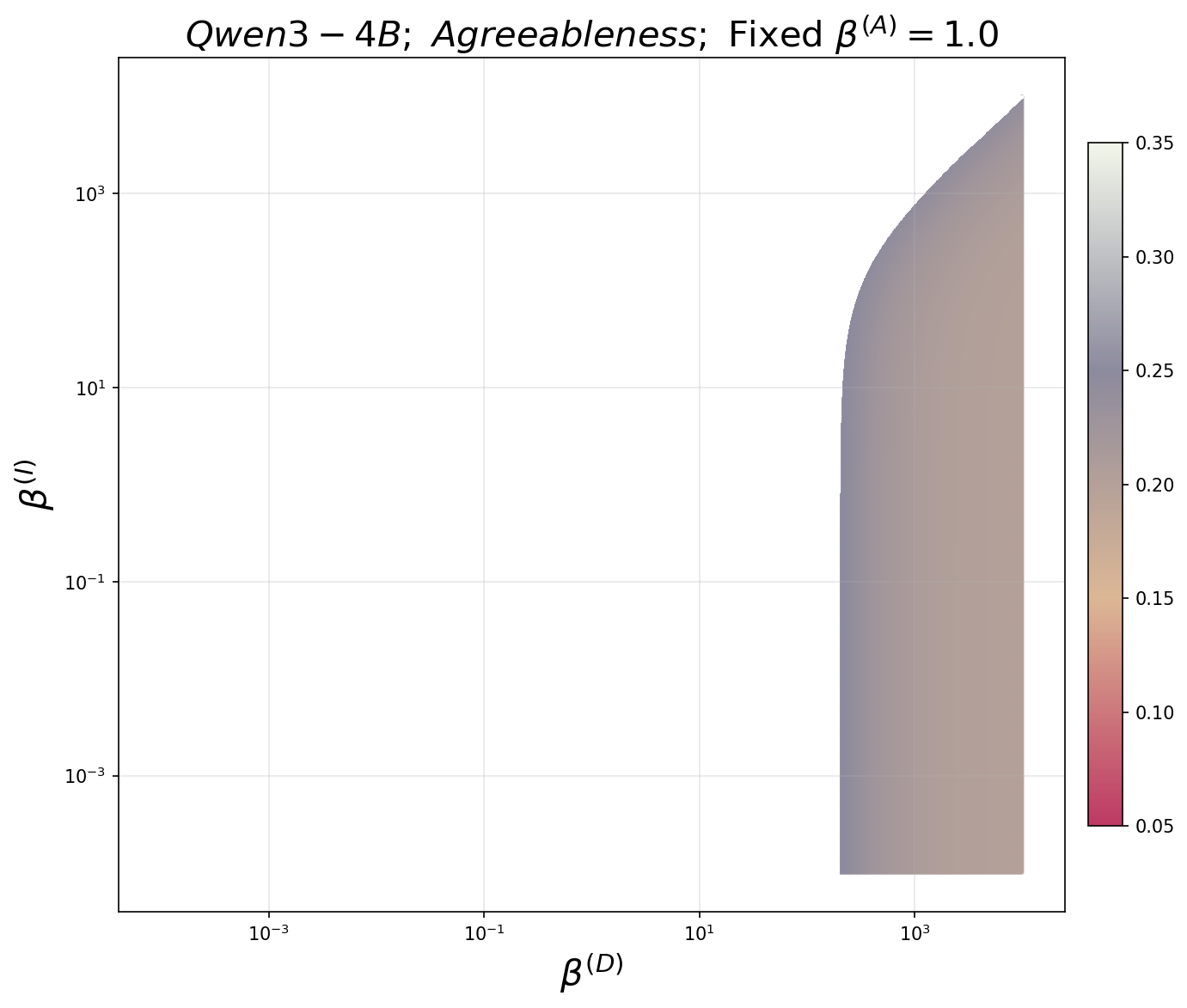}
    \includegraphics[width=0.3\linewidth]{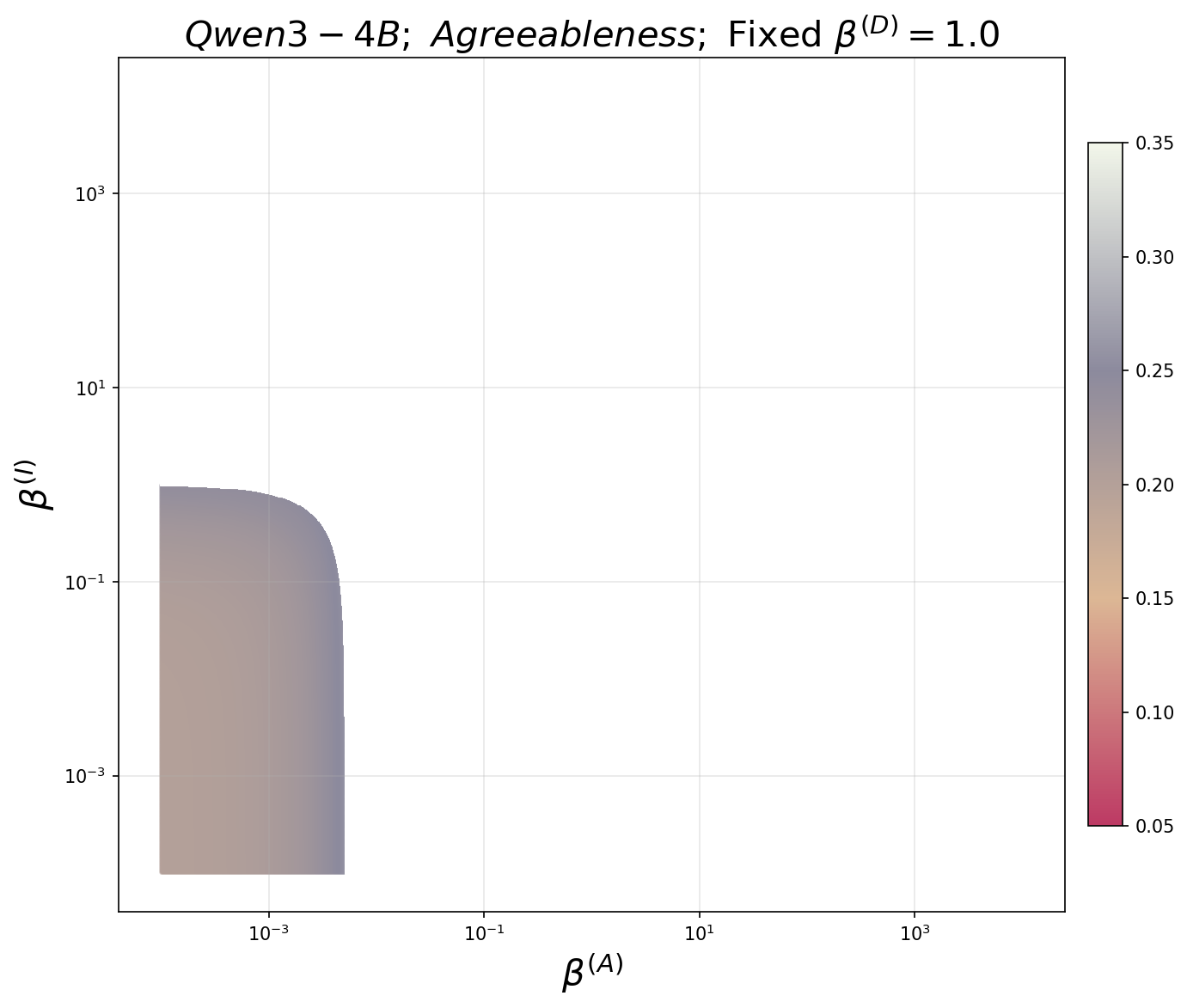}
    \includegraphics[width=0.3\linewidth]{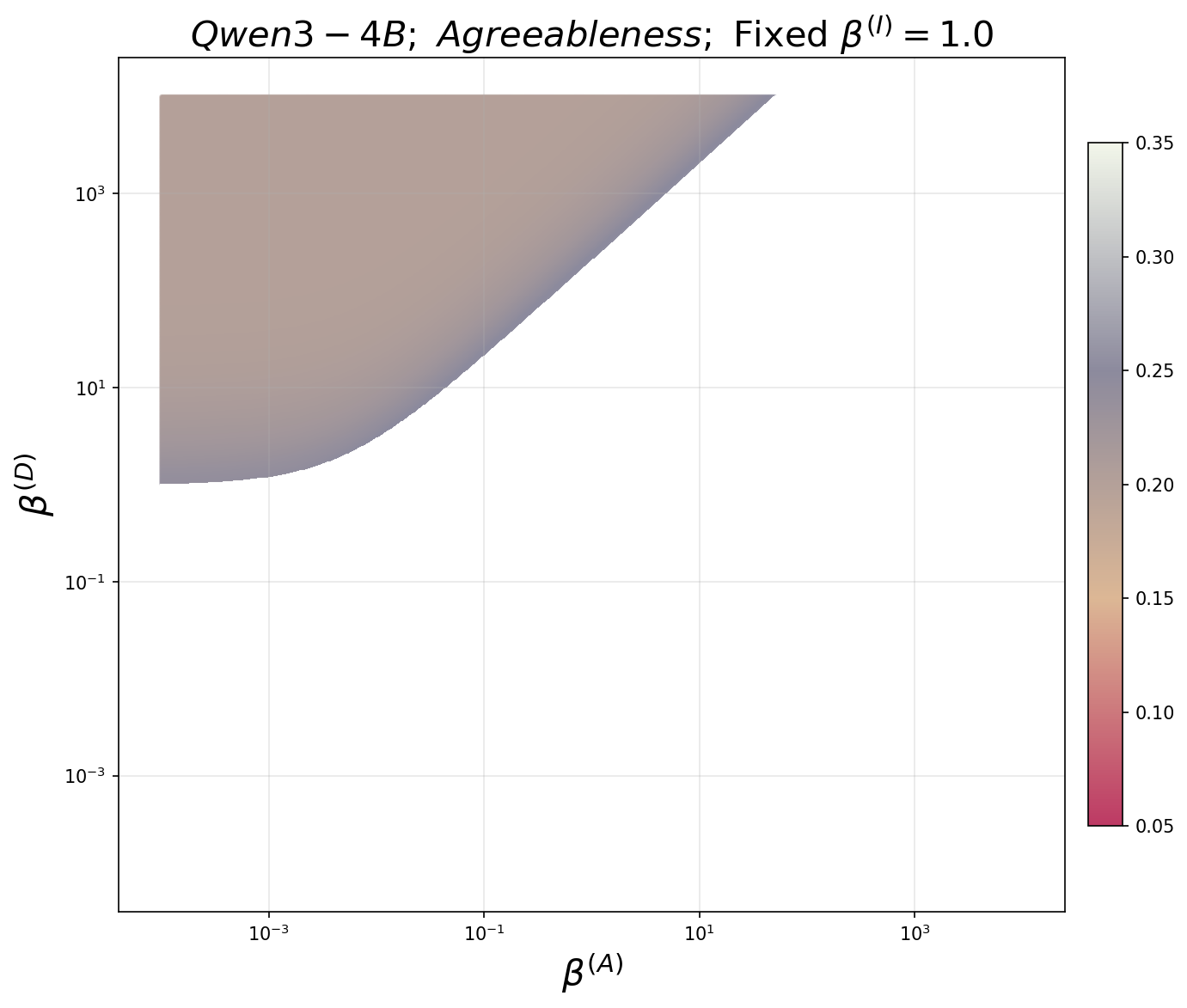}
    
    \caption{Political exclusion regarding the subpopulation \emph{Agreeableness} on the $\mathtt{Big\ Five}$ dataset.}
    \label{fig:Agreeableness}
\end{figure}

\begin{figure}
    \centering

    \includegraphics[width=0.3\linewidth]{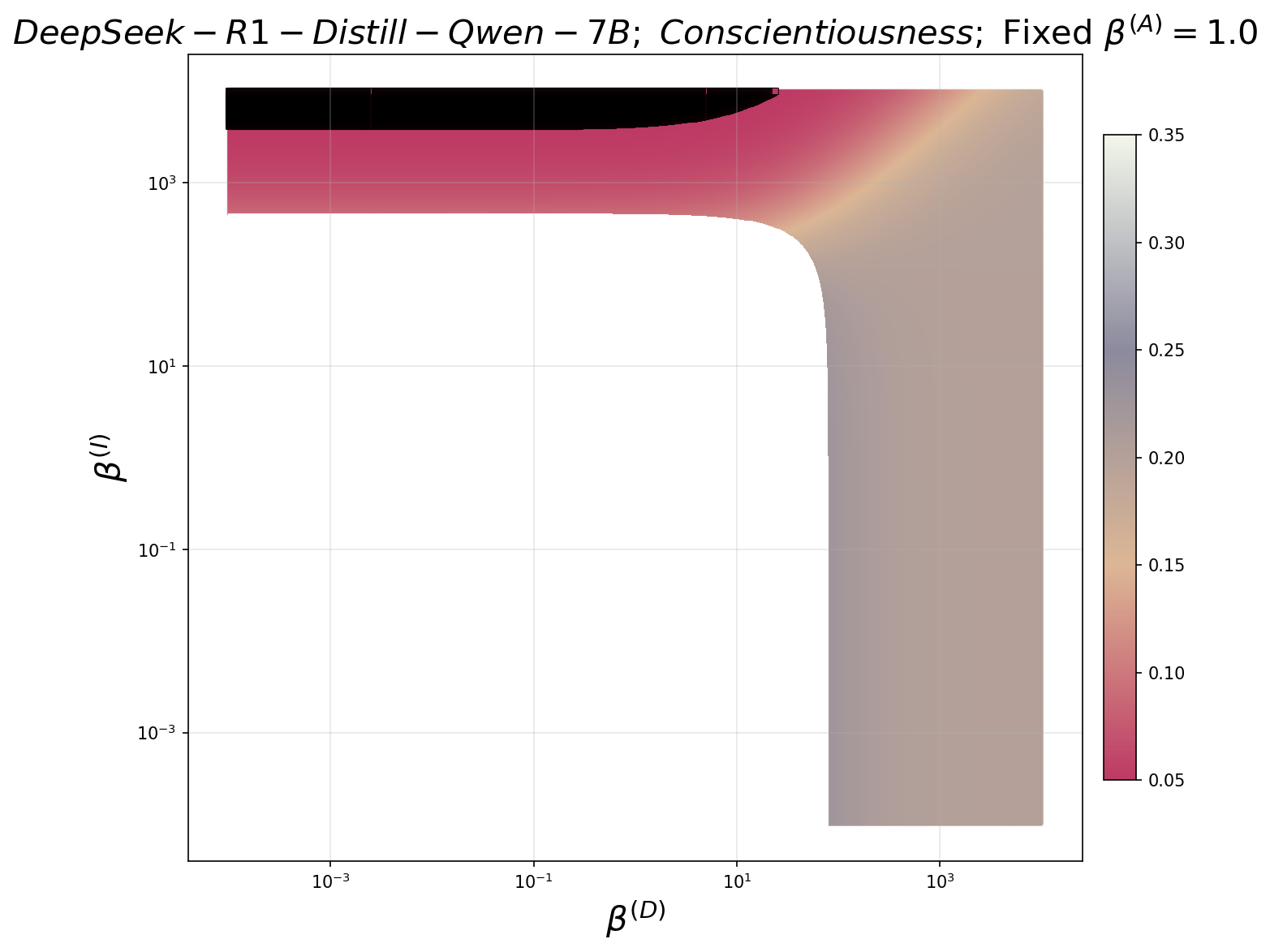}
    \includegraphics[width=0.3\linewidth]{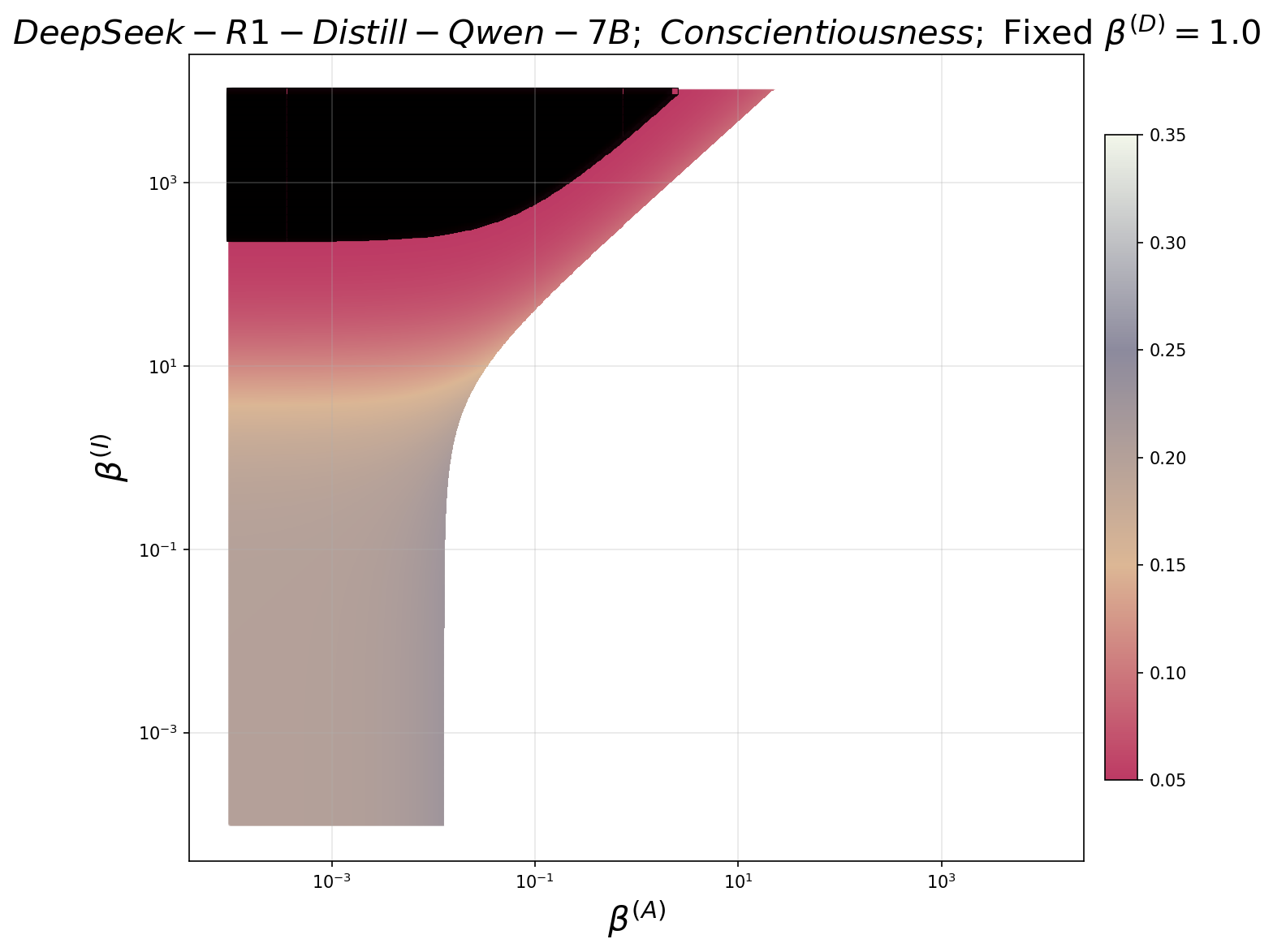}
    \includegraphics[width=0.3\linewidth]{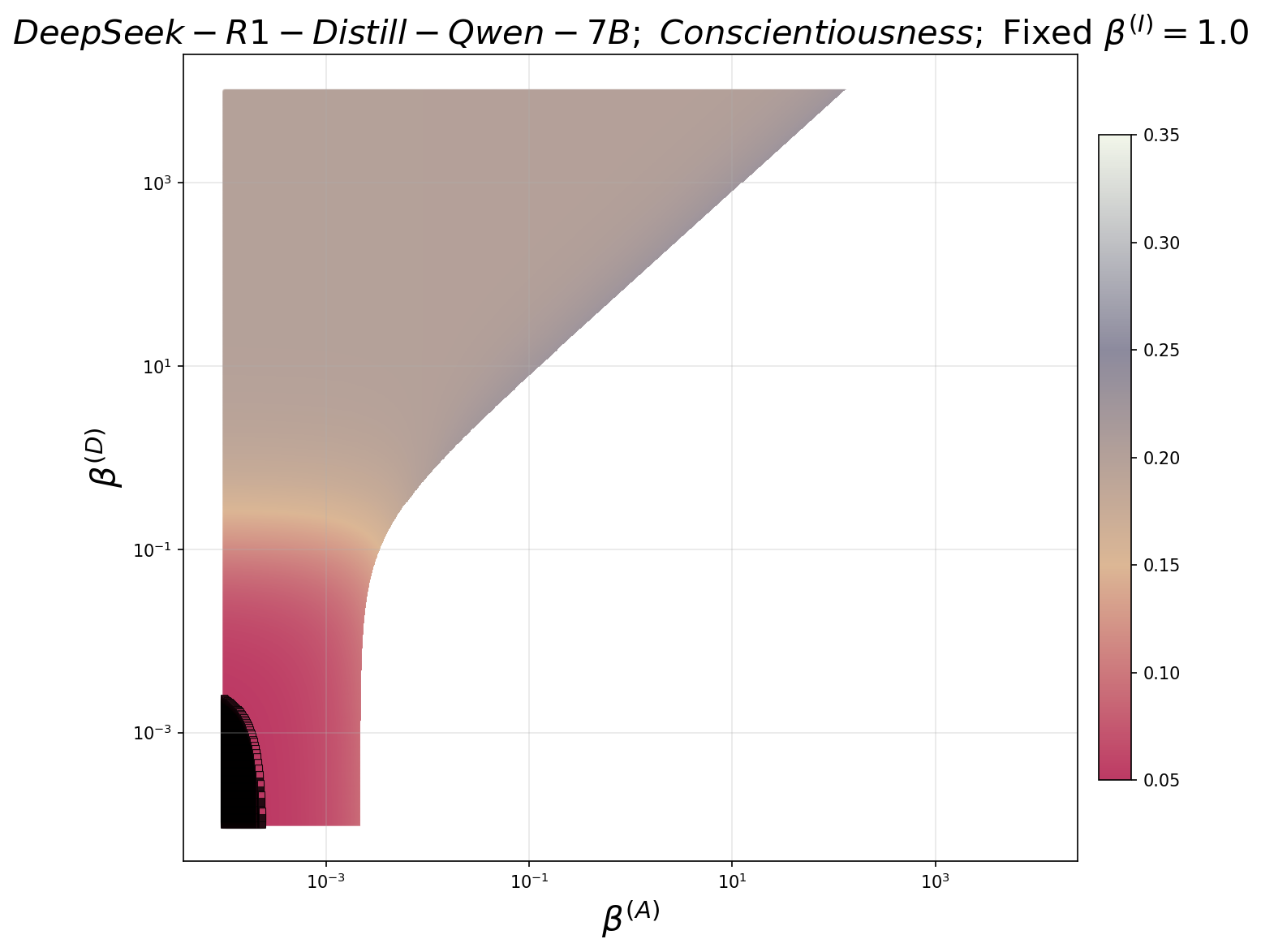}

    \includegraphics[width=0.3\linewidth]{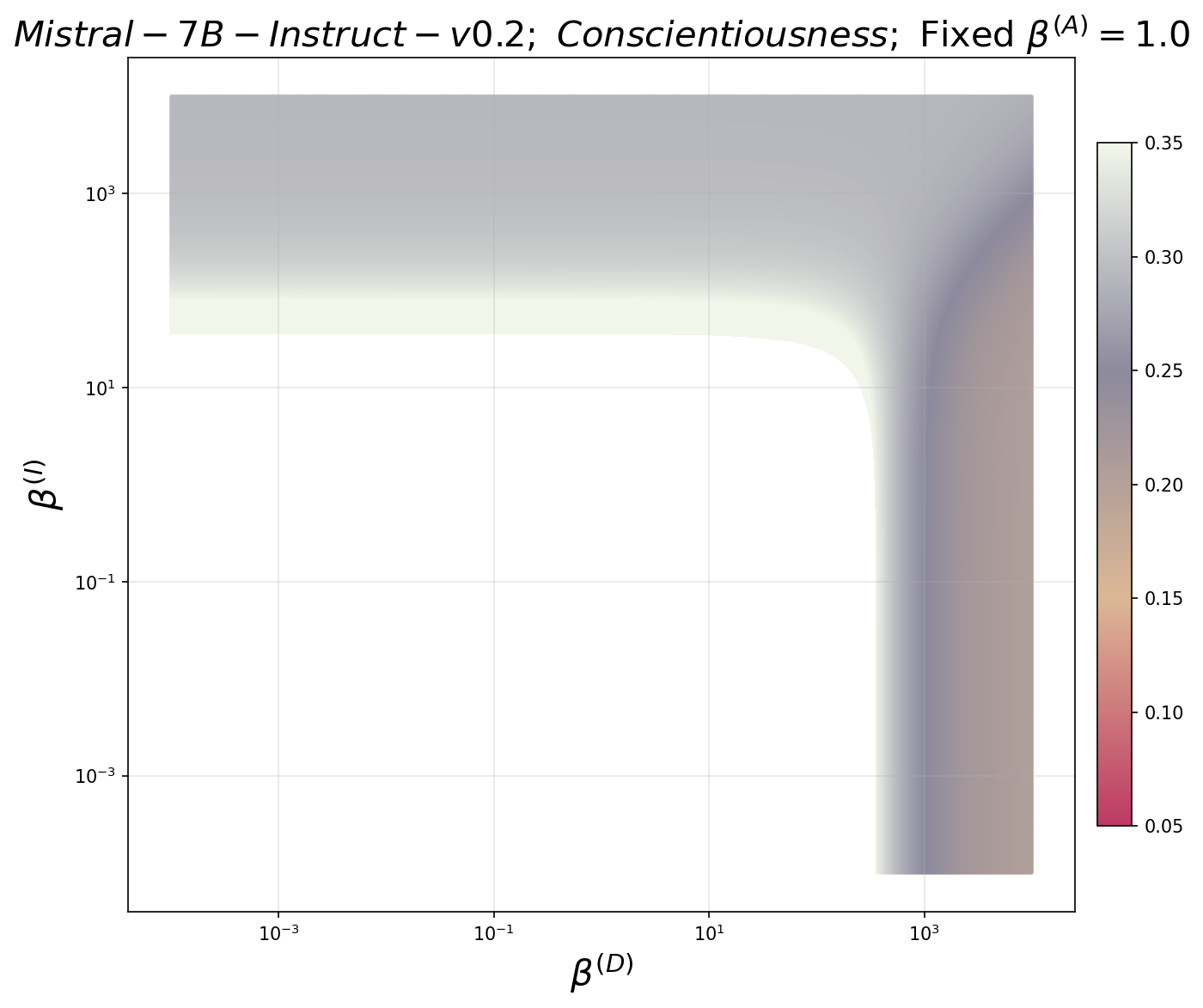}
    \includegraphics[width=0.3\linewidth]{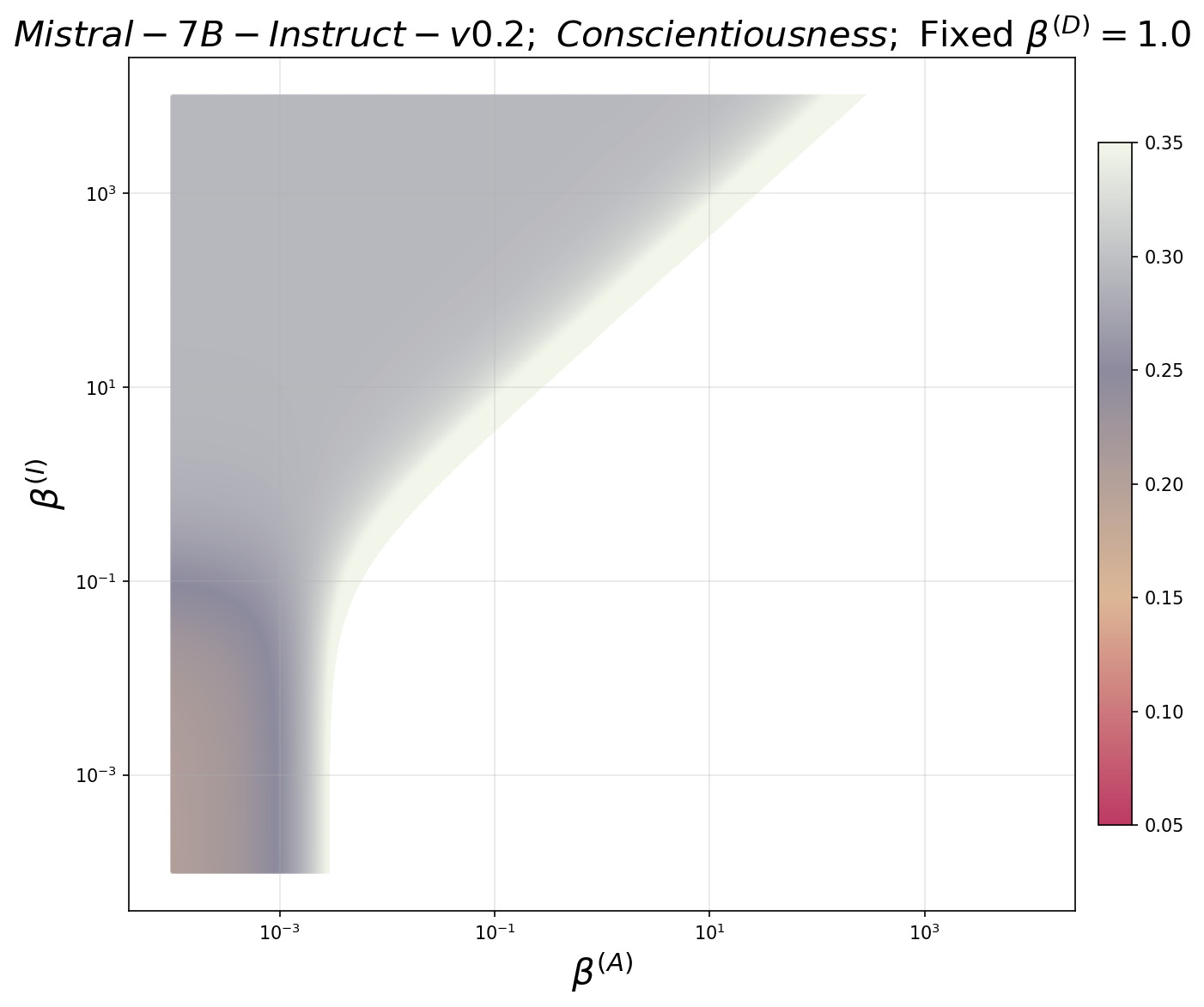}
    \includegraphics[width=0.3\linewidth]{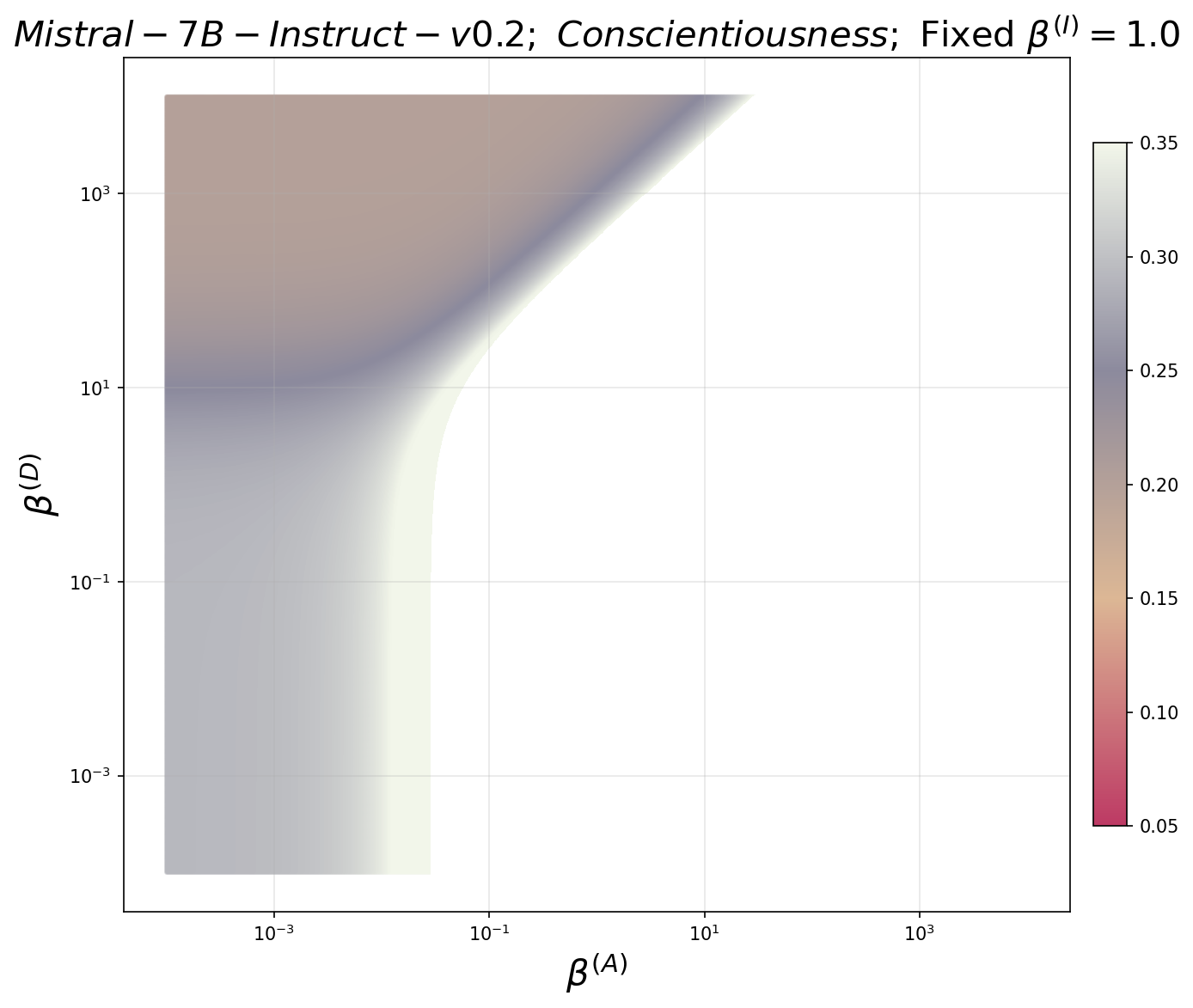}

    \includegraphics[width=0.3\linewidth]{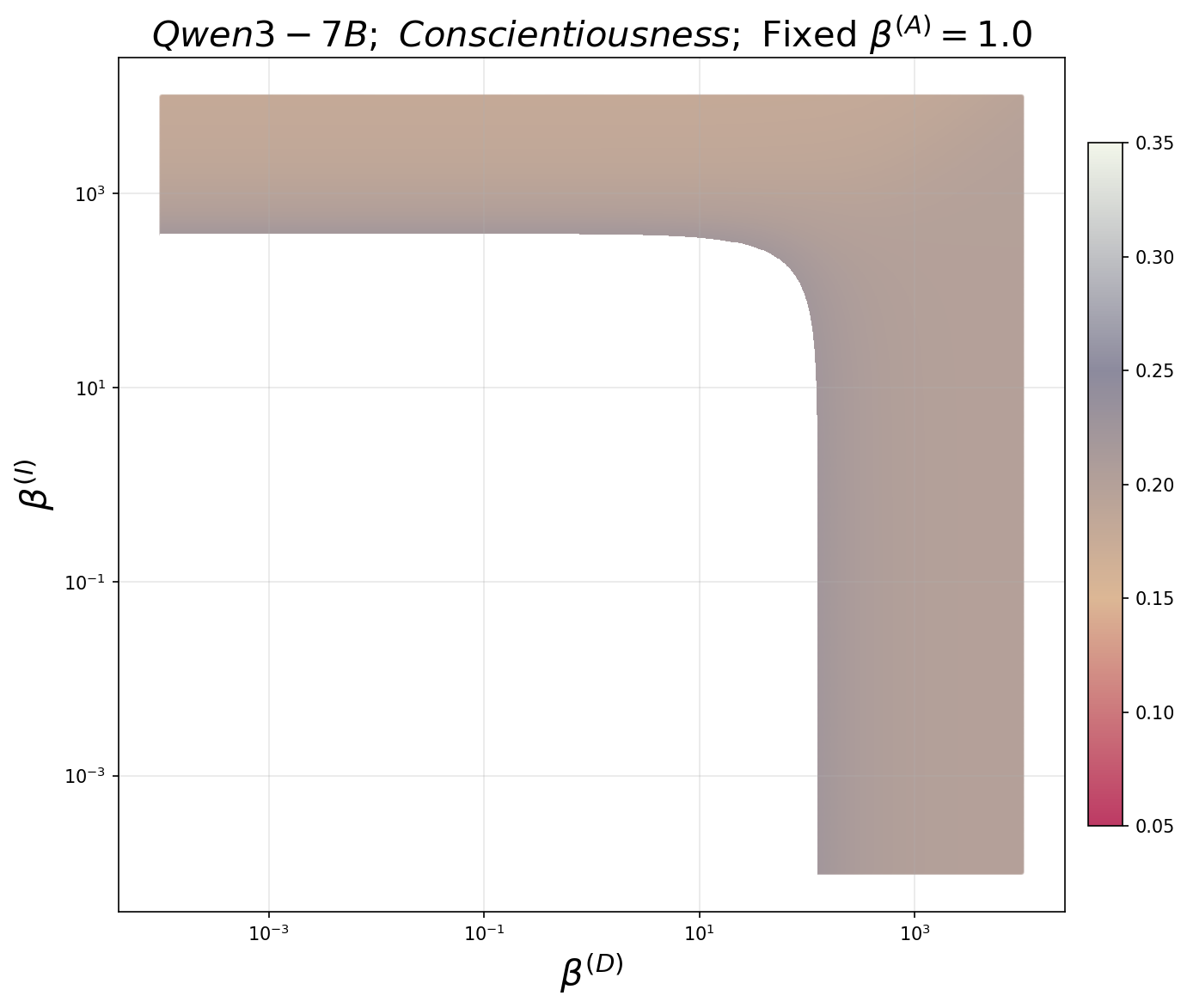}
    \includegraphics[width=0.3\linewidth]{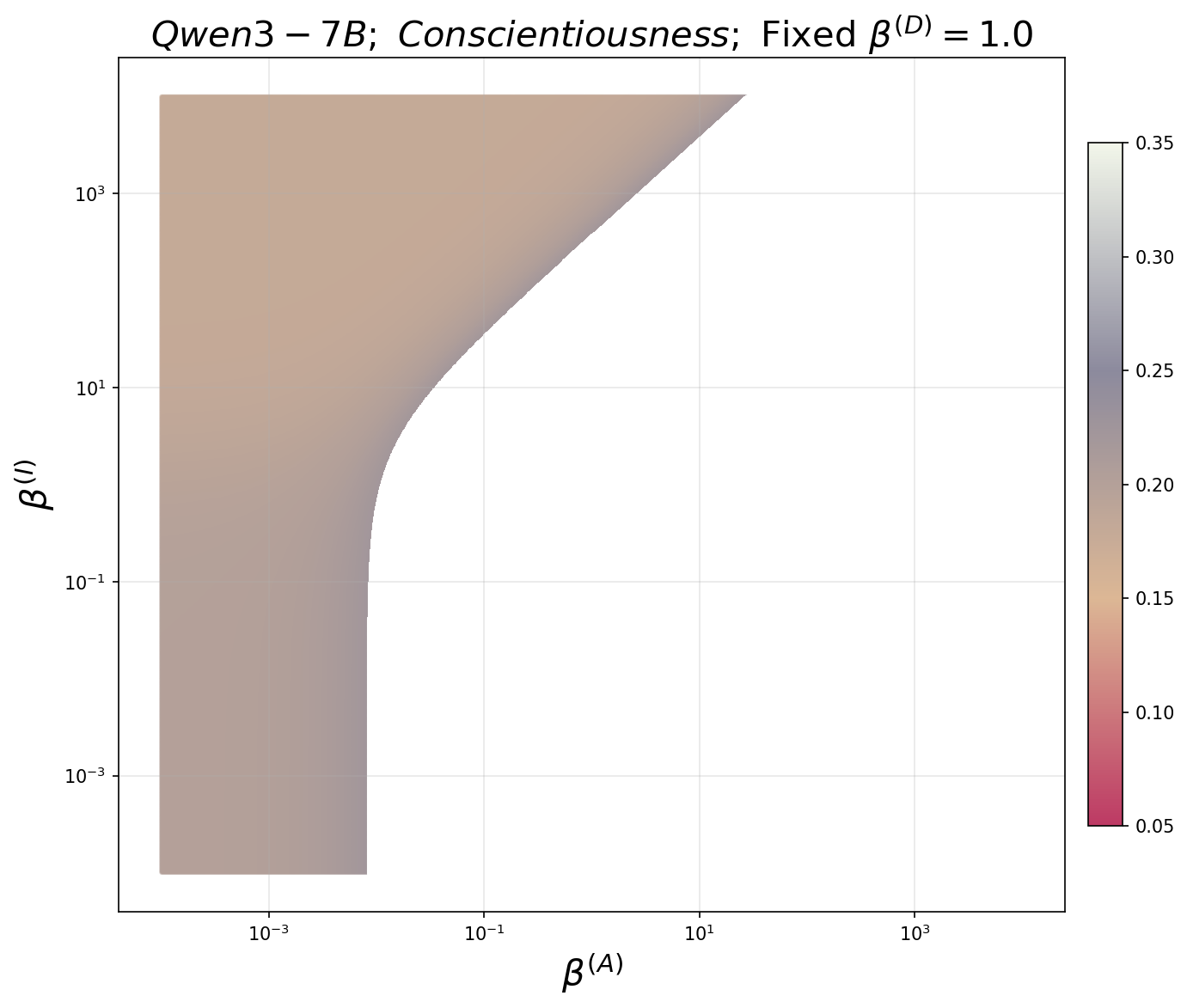}
    \includegraphics[width=0.3\linewidth]{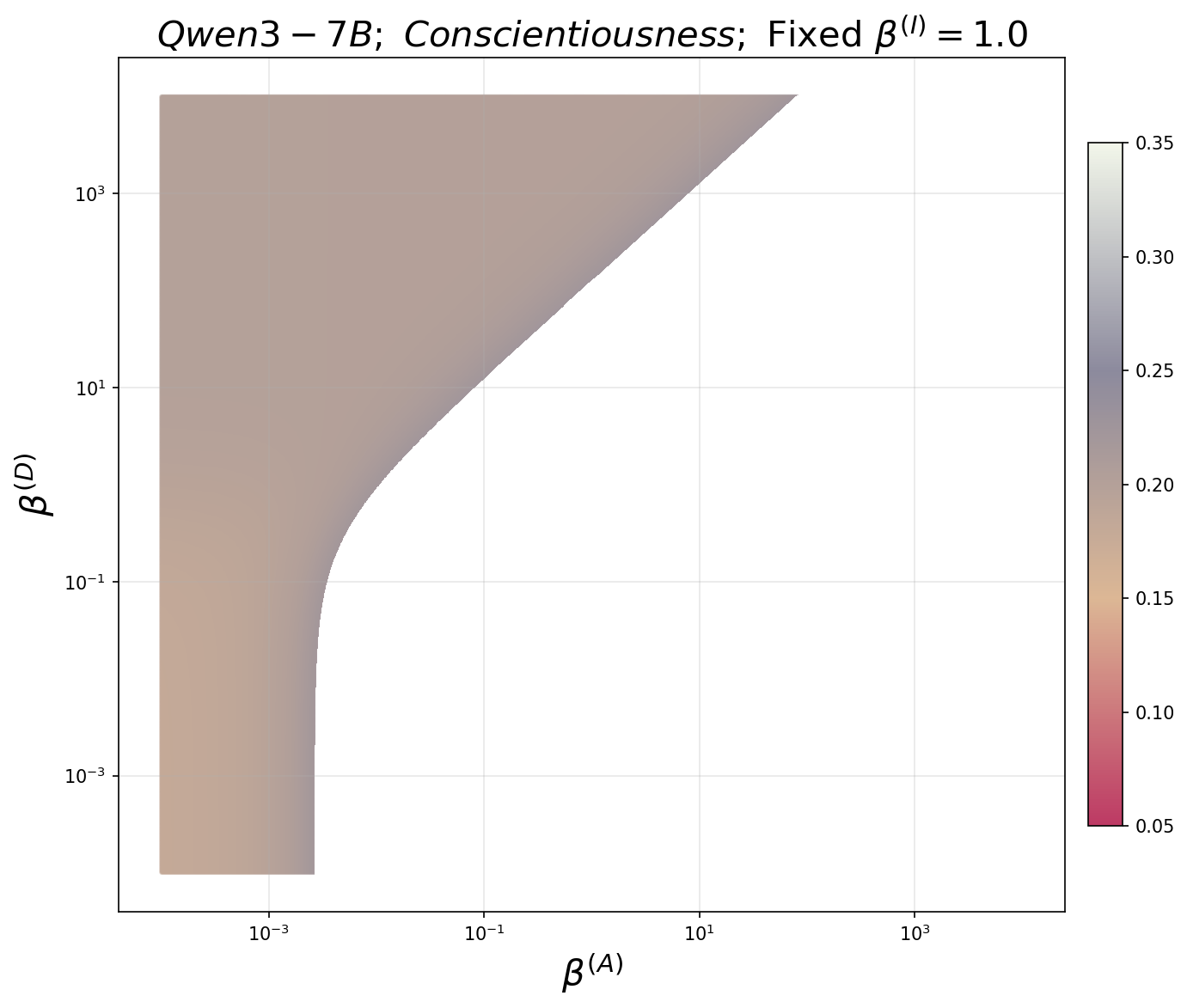}

    \includegraphics[width=0.3\linewidth]{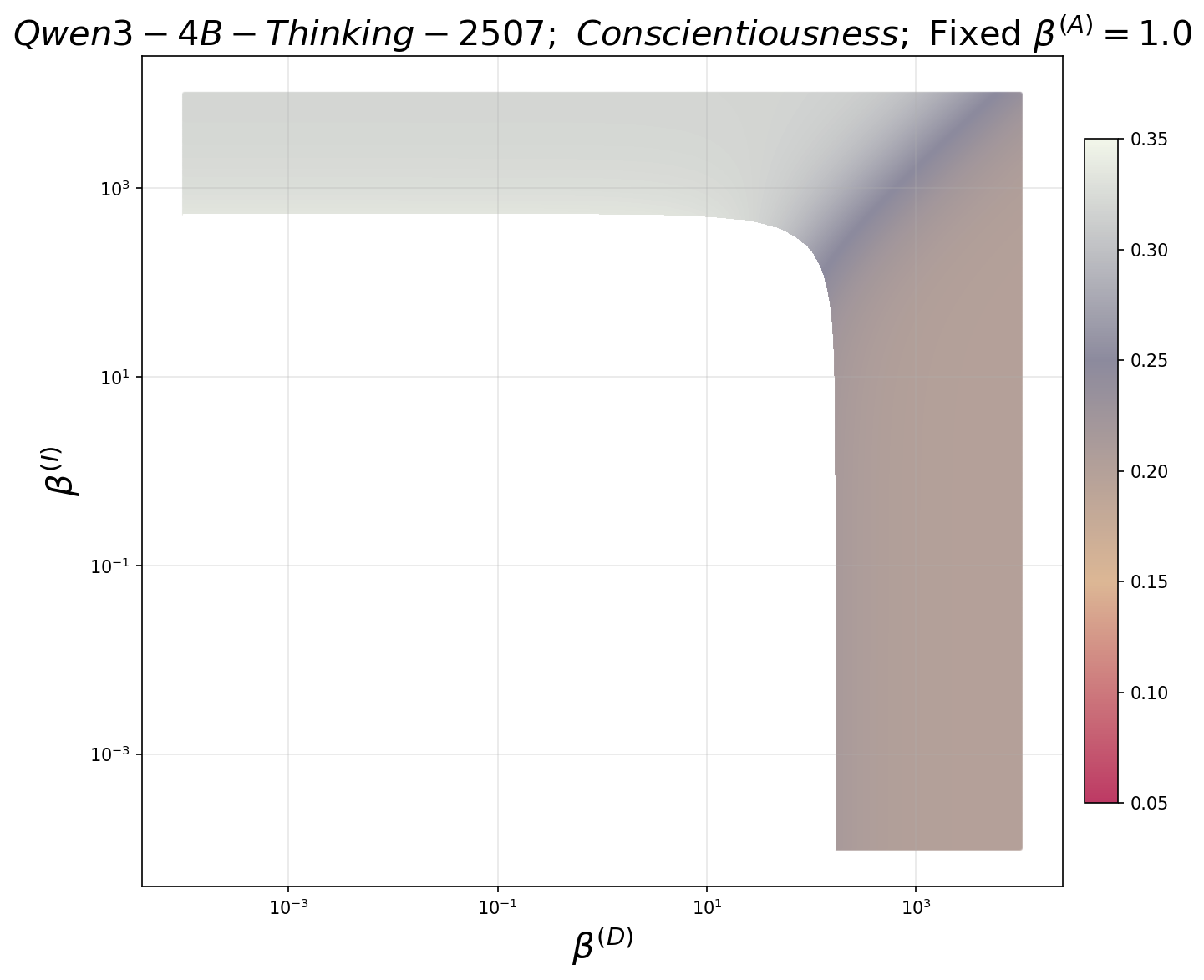}
    \includegraphics[width=0.3\linewidth]{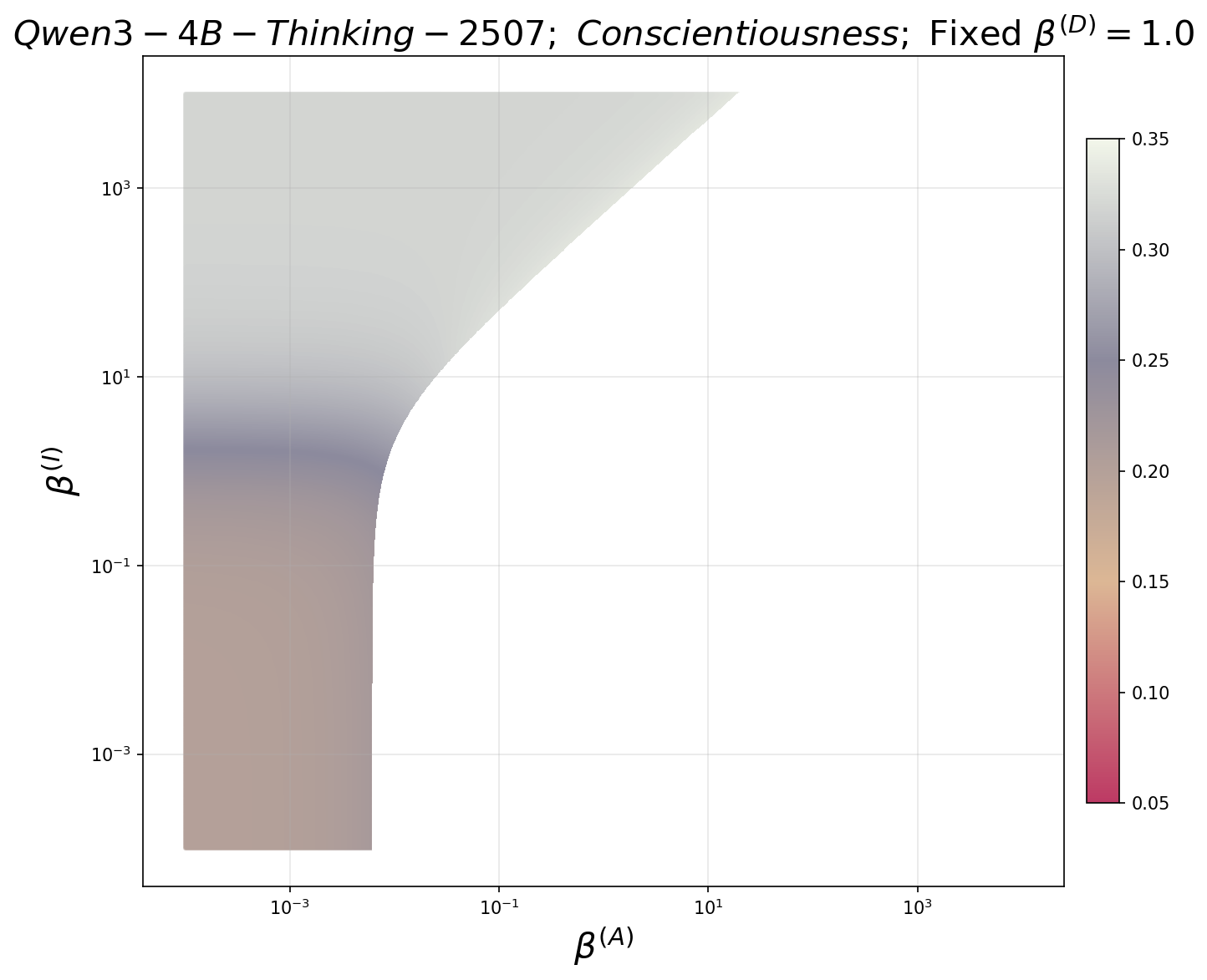}
    \includegraphics[width=0.3\linewidth]{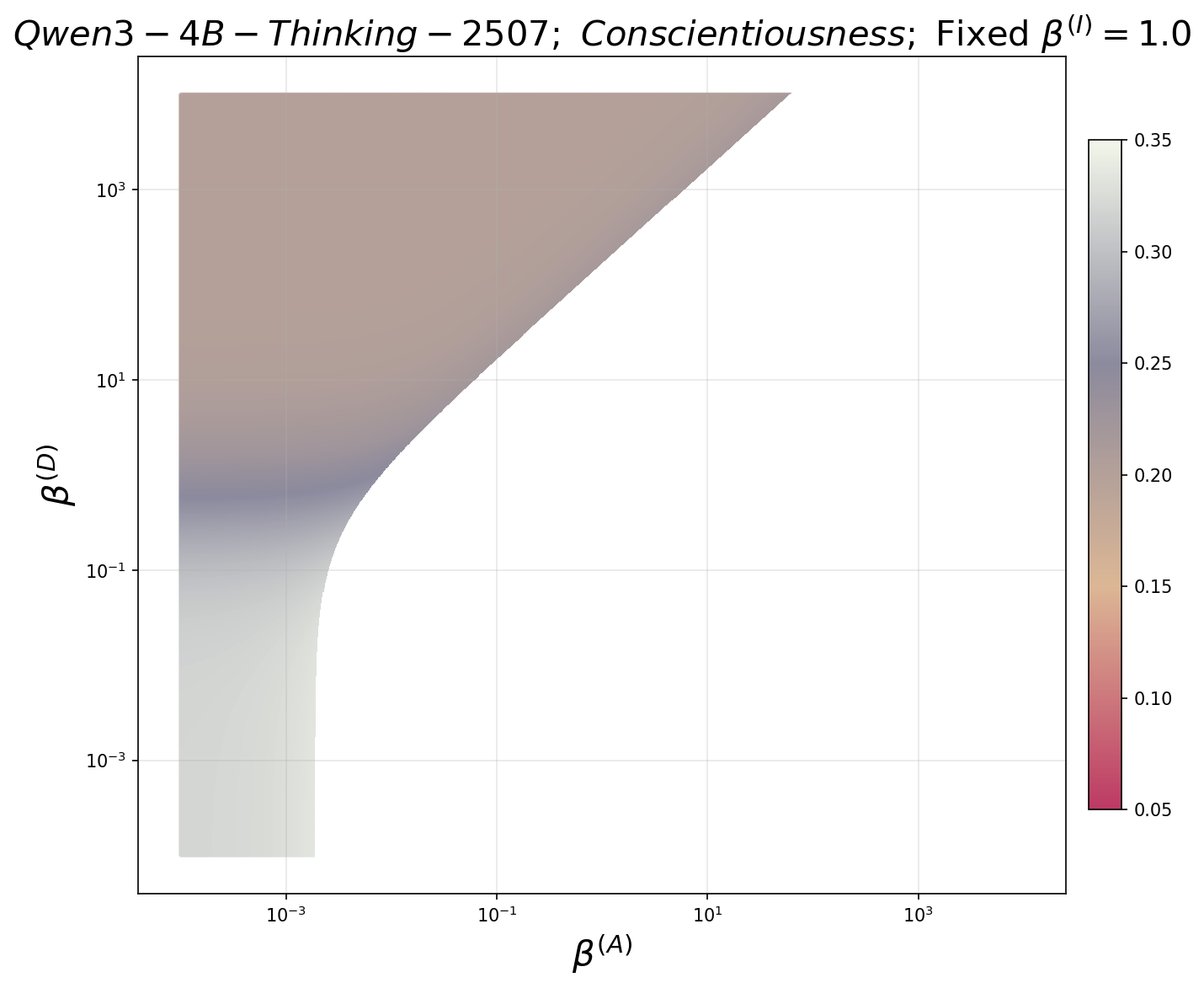}

    \includegraphics[width=0.3\linewidth]{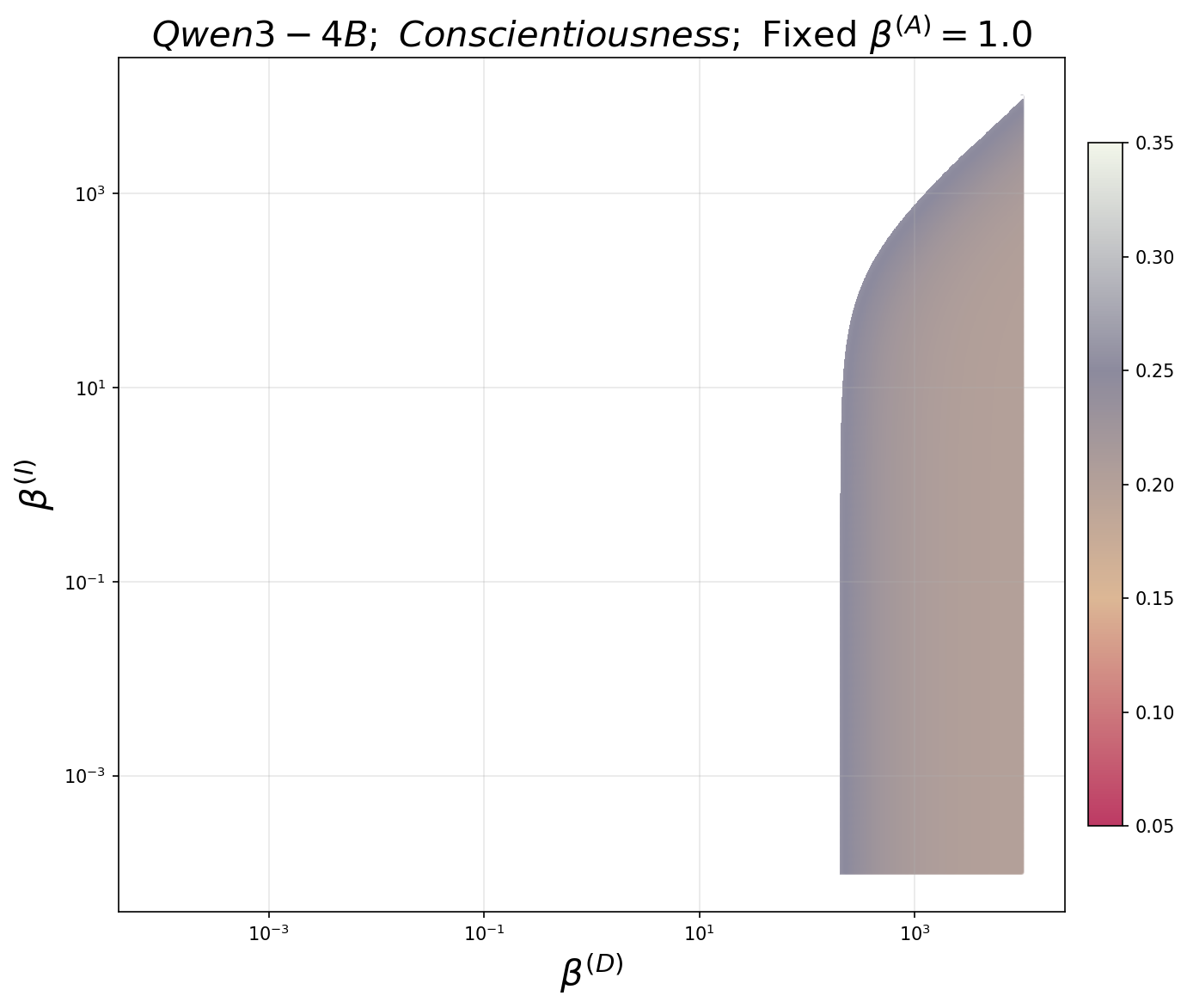}
    \includegraphics[width=0.3\linewidth]{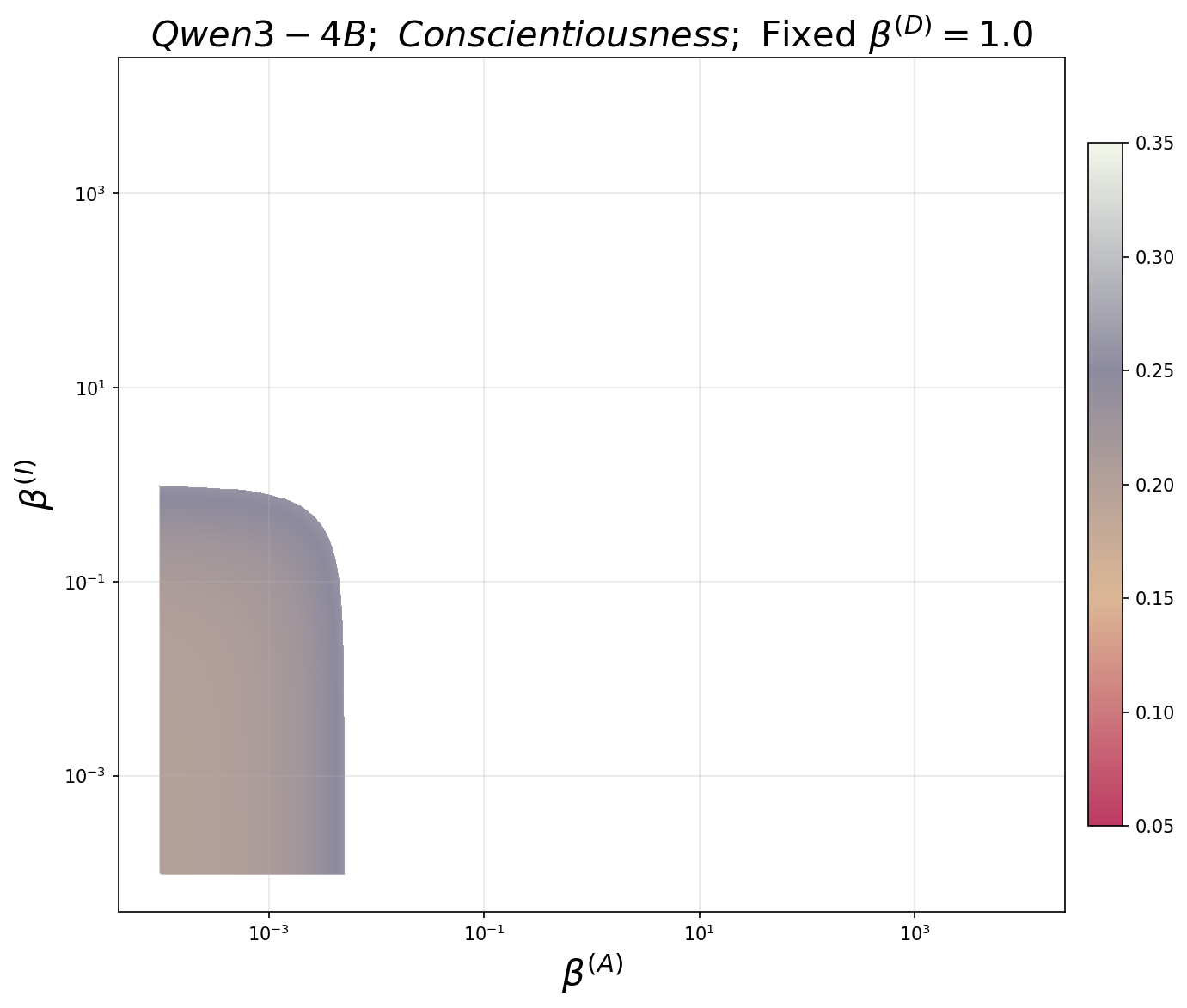}
    \includegraphics[width=0.3\linewidth]{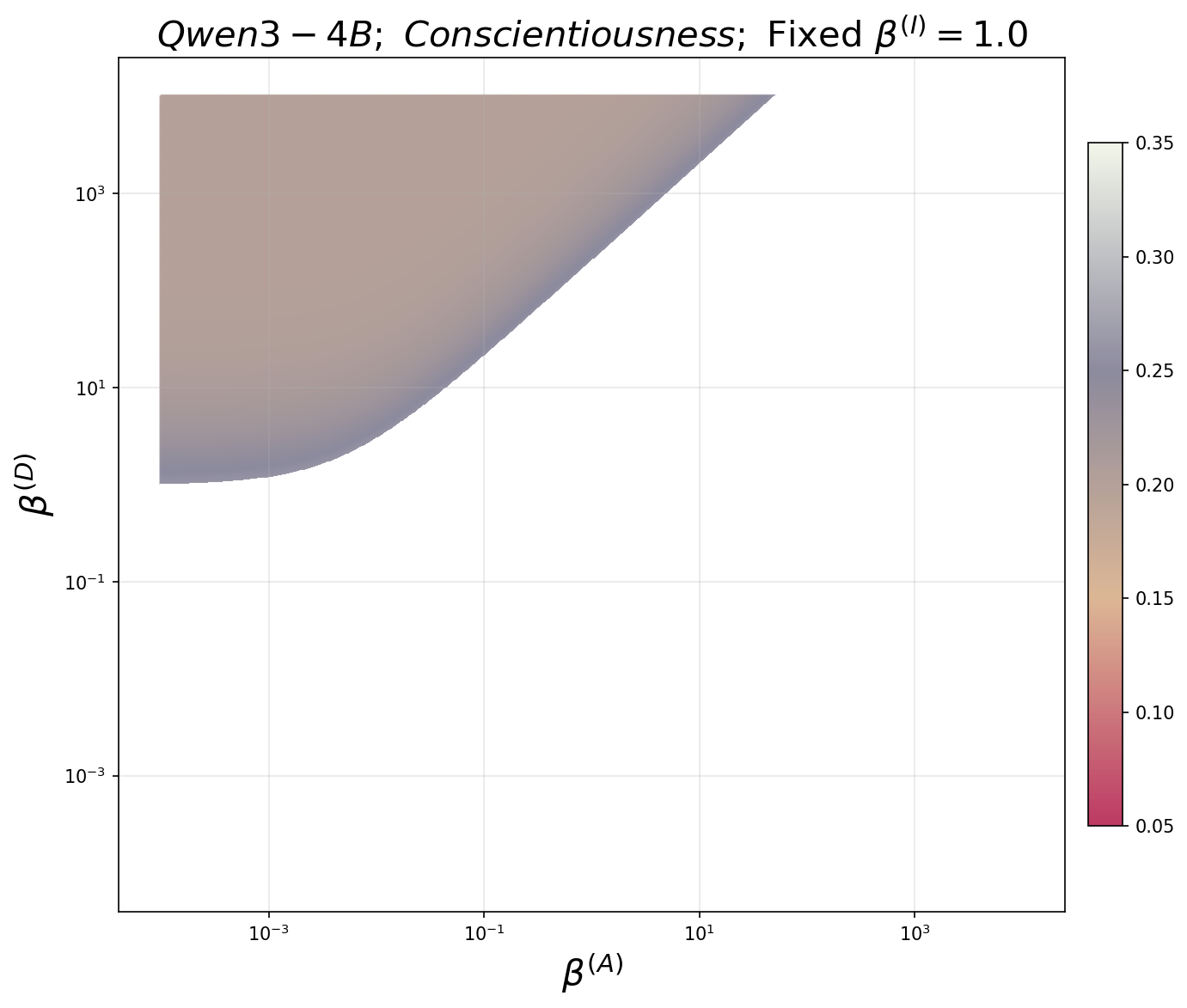}
    
    \caption{Political exclusion regarding the subpopulation \emph{Conscientiousness} on the $\mathtt{Big\ Five}$ dataset.}
    \label{fig:Conscientiousness}
\end{figure}

\begin{figure}
    \centering

    \includegraphics[width=0.3\linewidth]{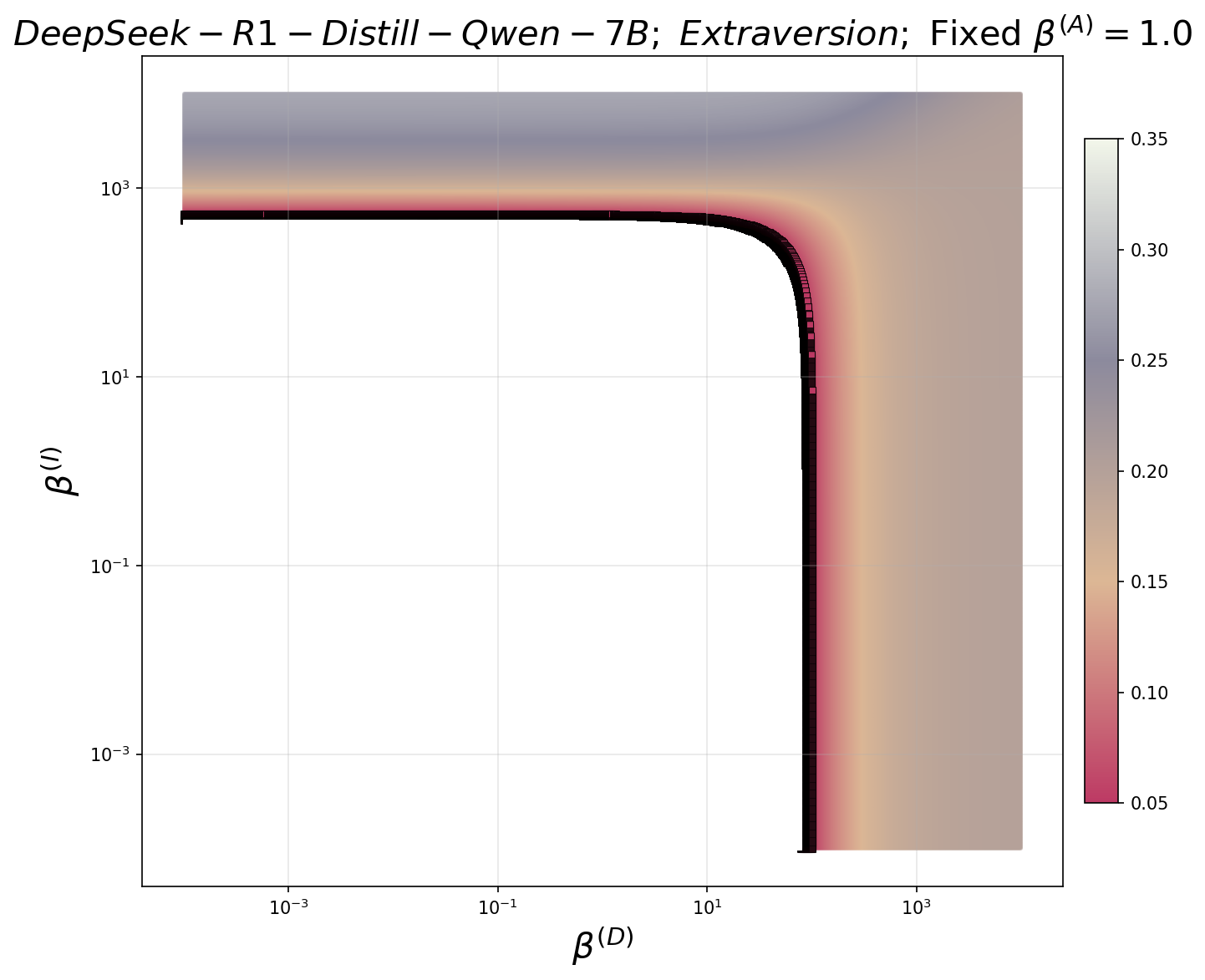}
    \includegraphics[width=0.3\linewidth]{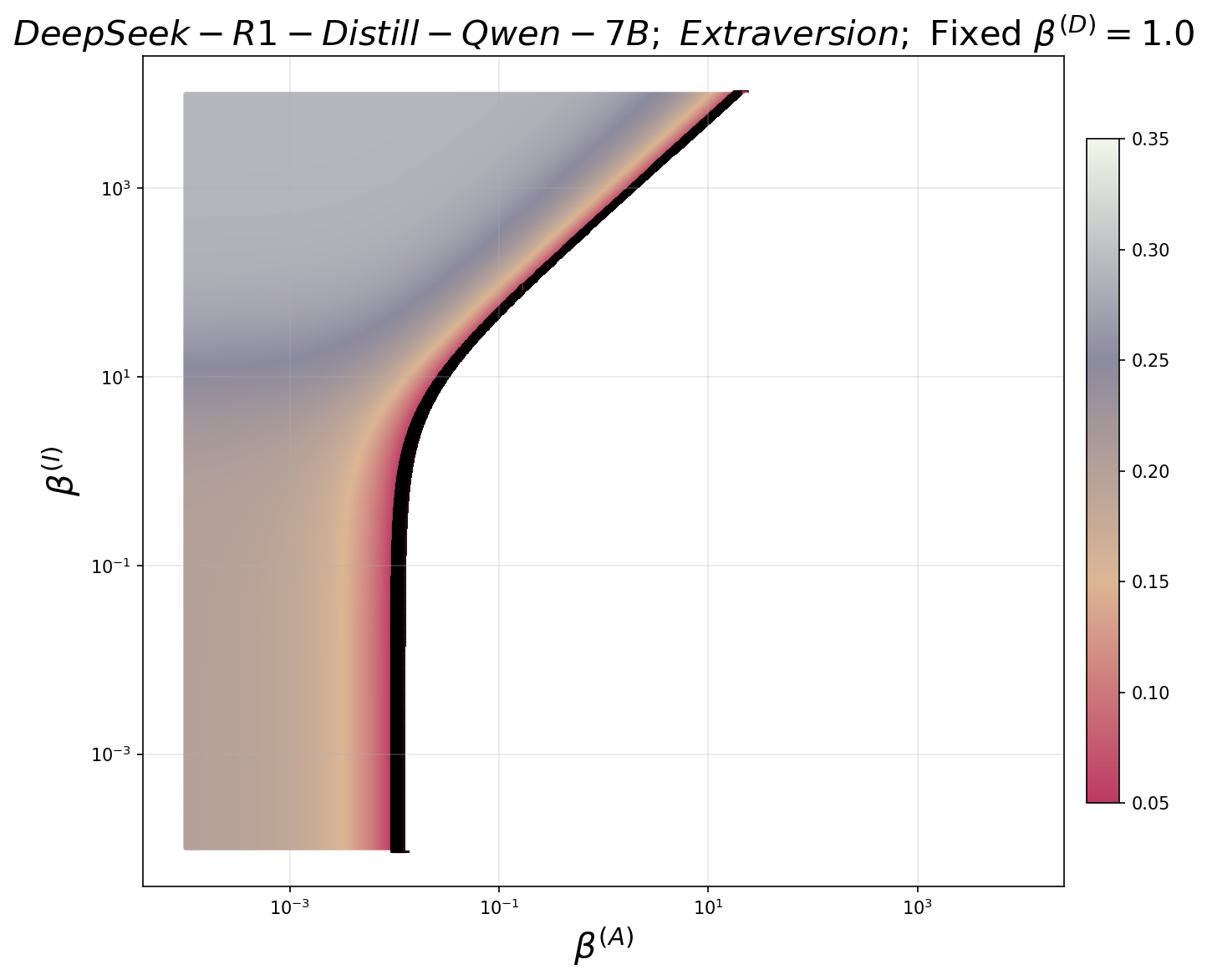}
    \includegraphics[width=0.3\linewidth]{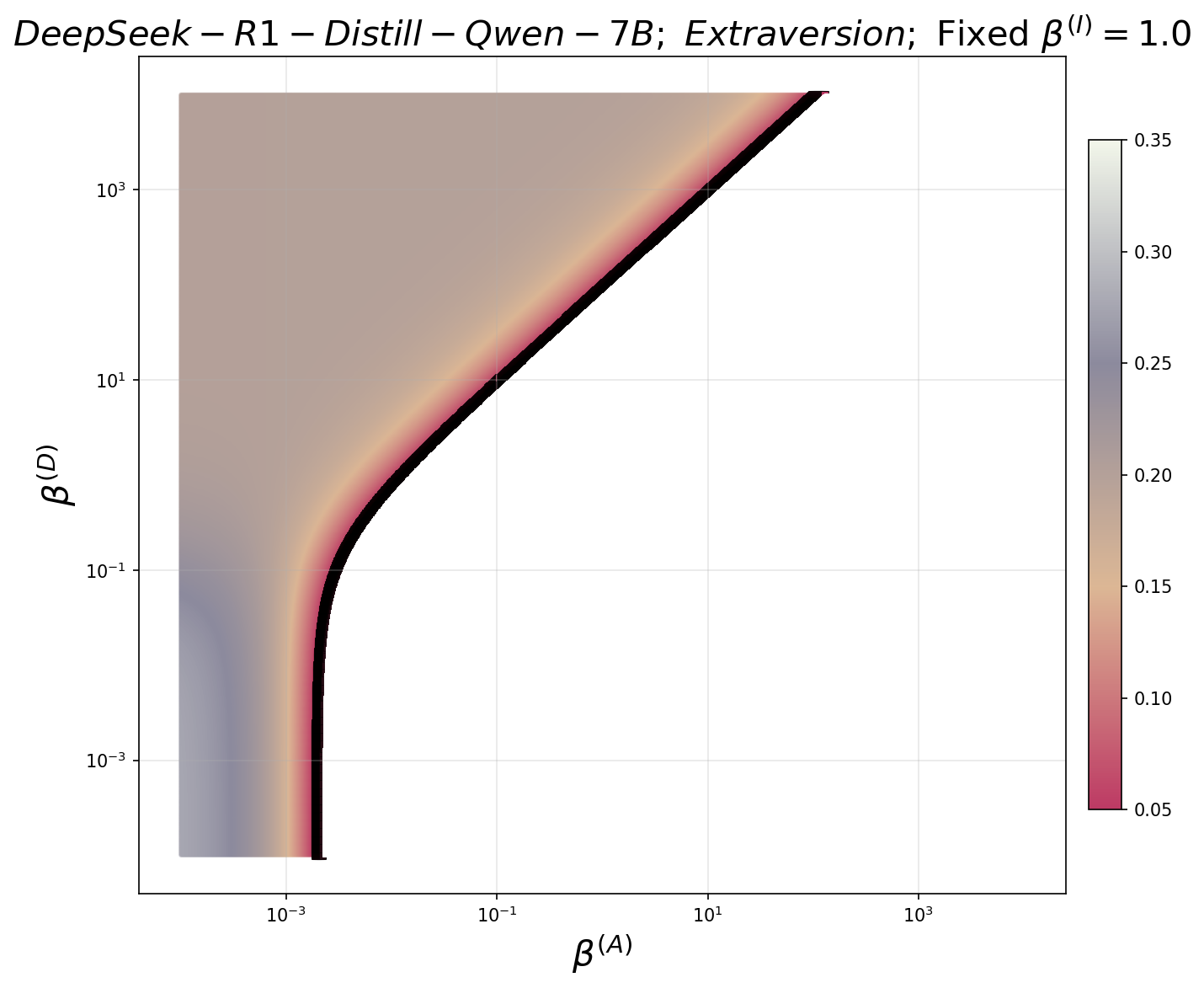}

    \includegraphics[width=0.3\linewidth]{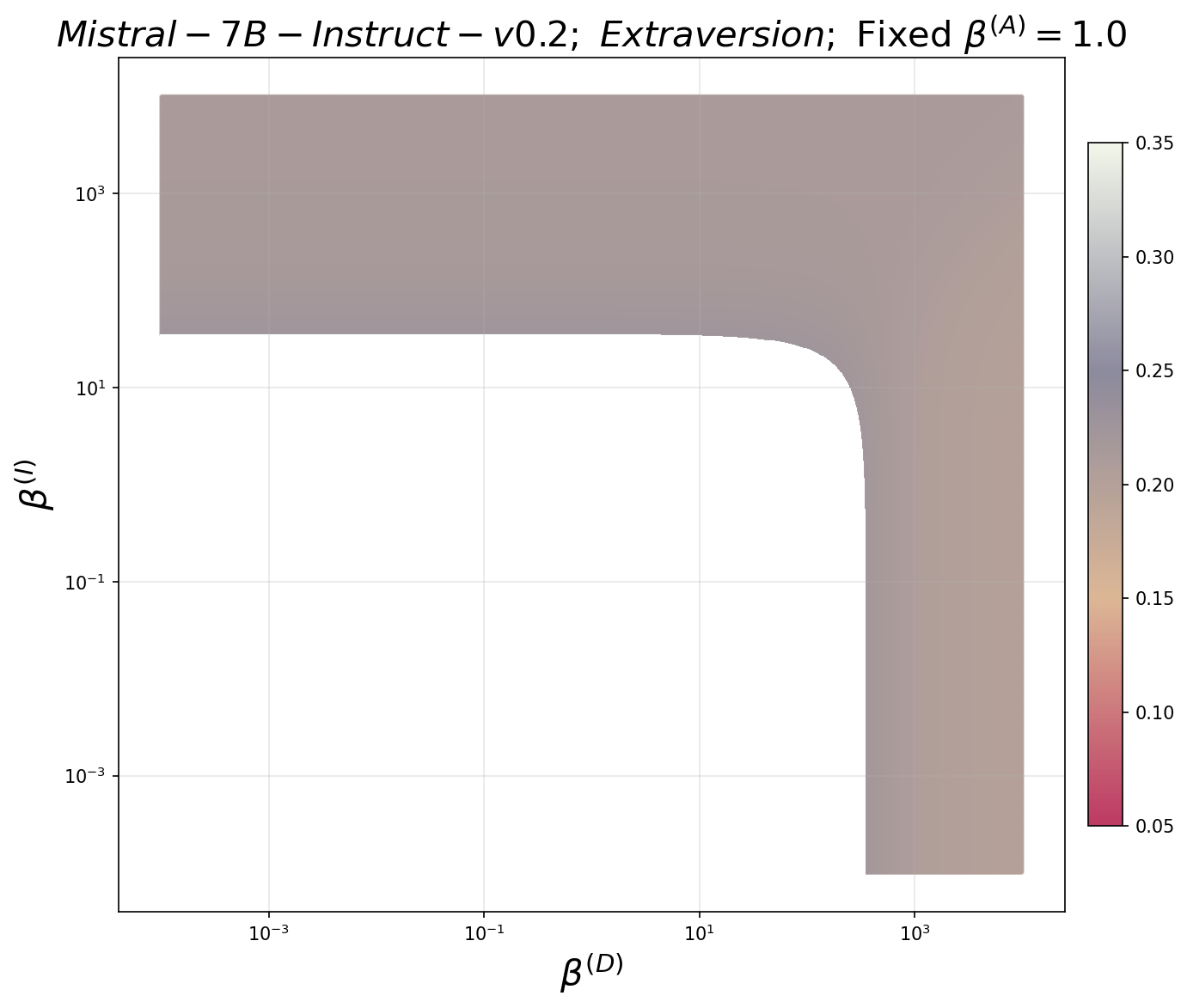}
    \includegraphics[width=0.3\linewidth]{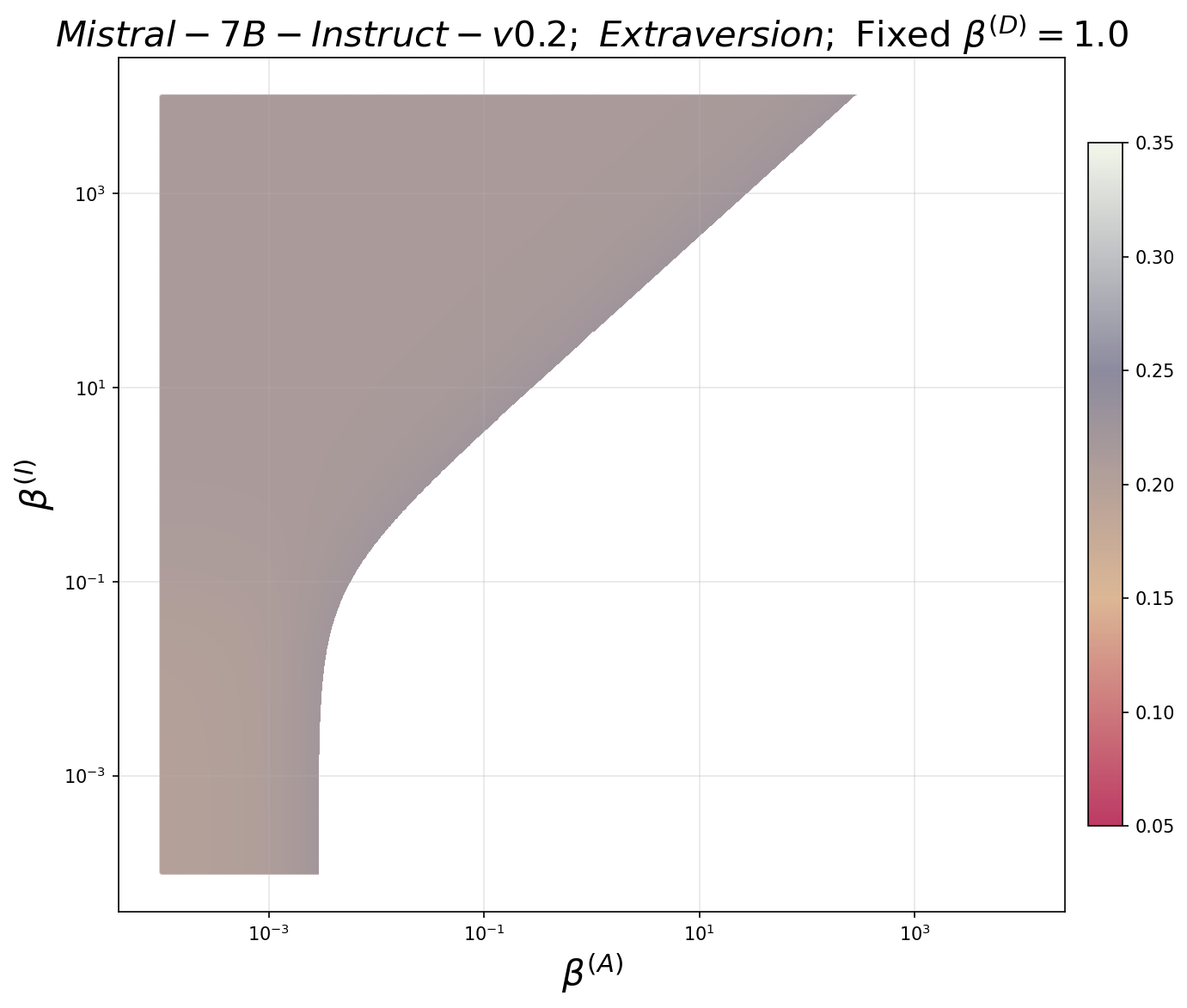}
    \includegraphics[width=0.3\linewidth]{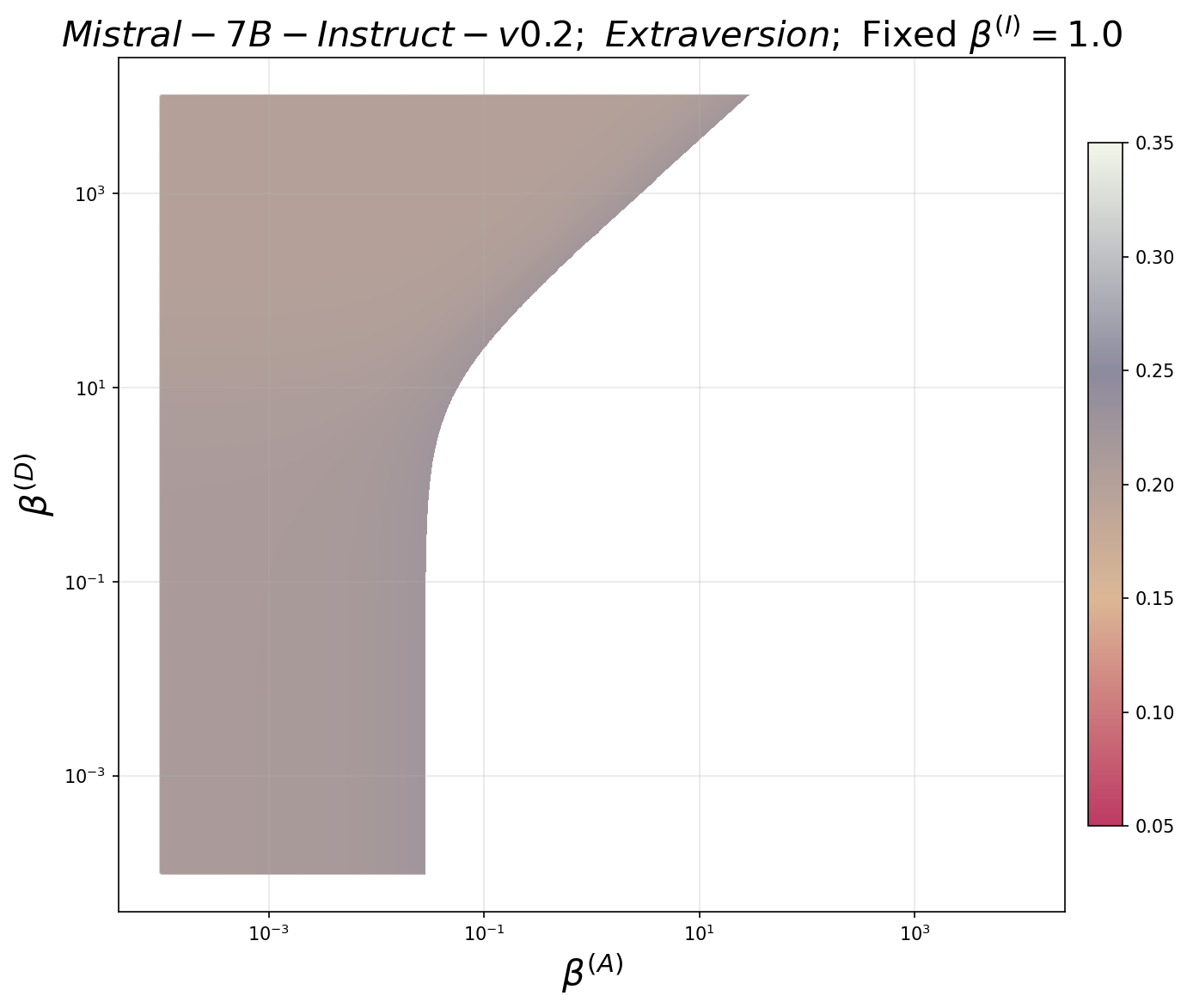}

    \includegraphics[width=0.3\linewidth]{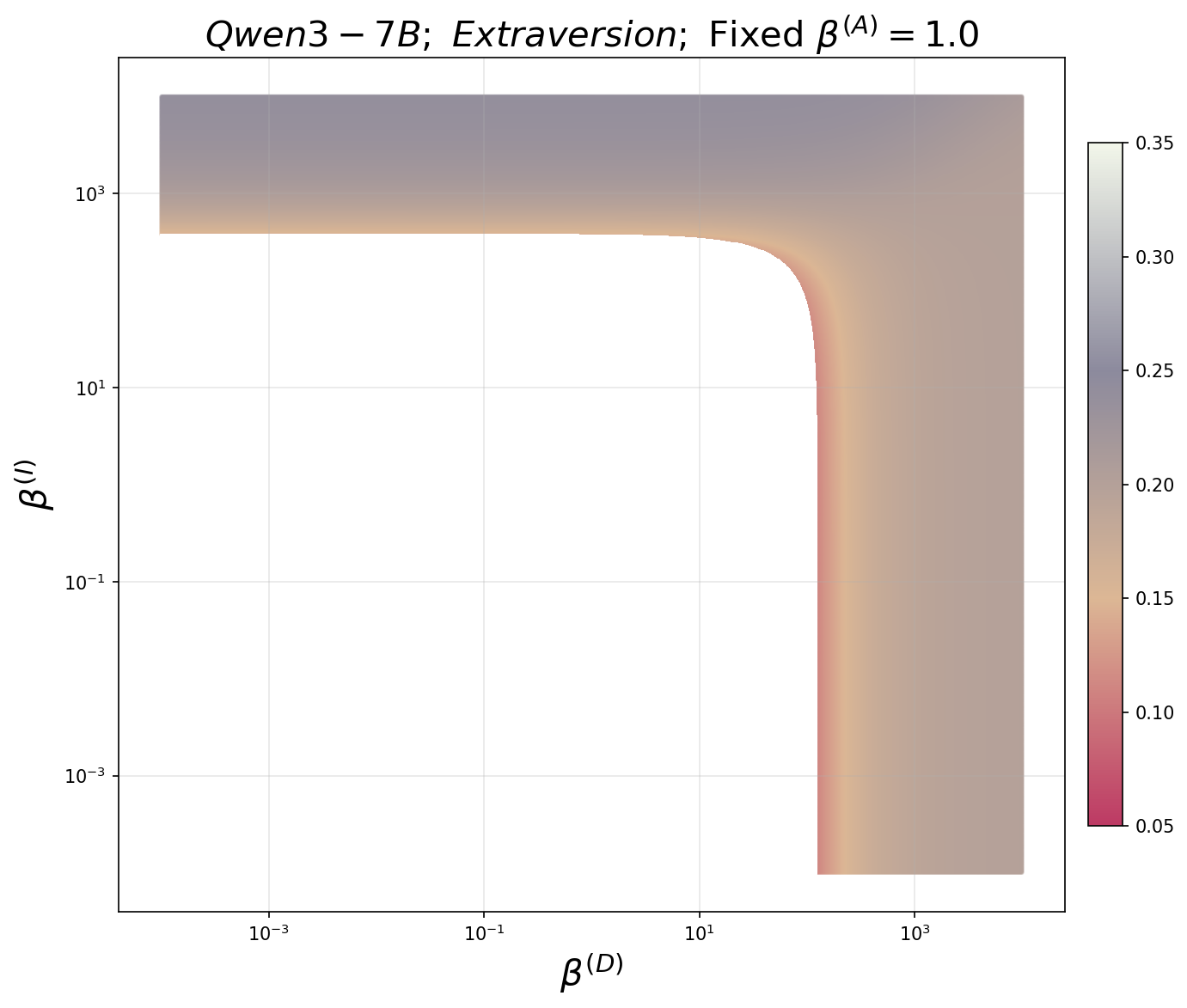}
    \includegraphics[width=0.3\linewidth]{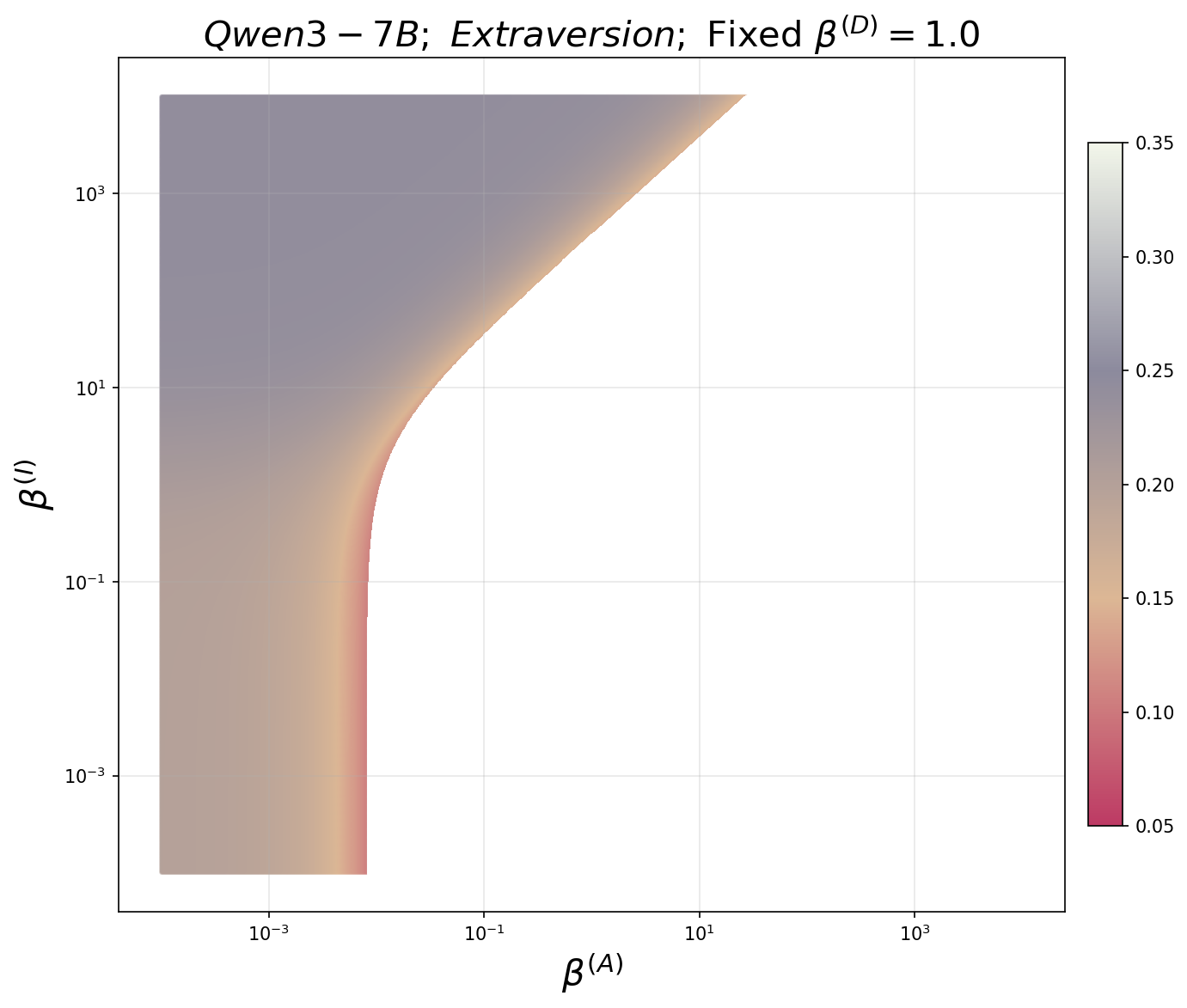}
    \includegraphics[width=0.3\linewidth]{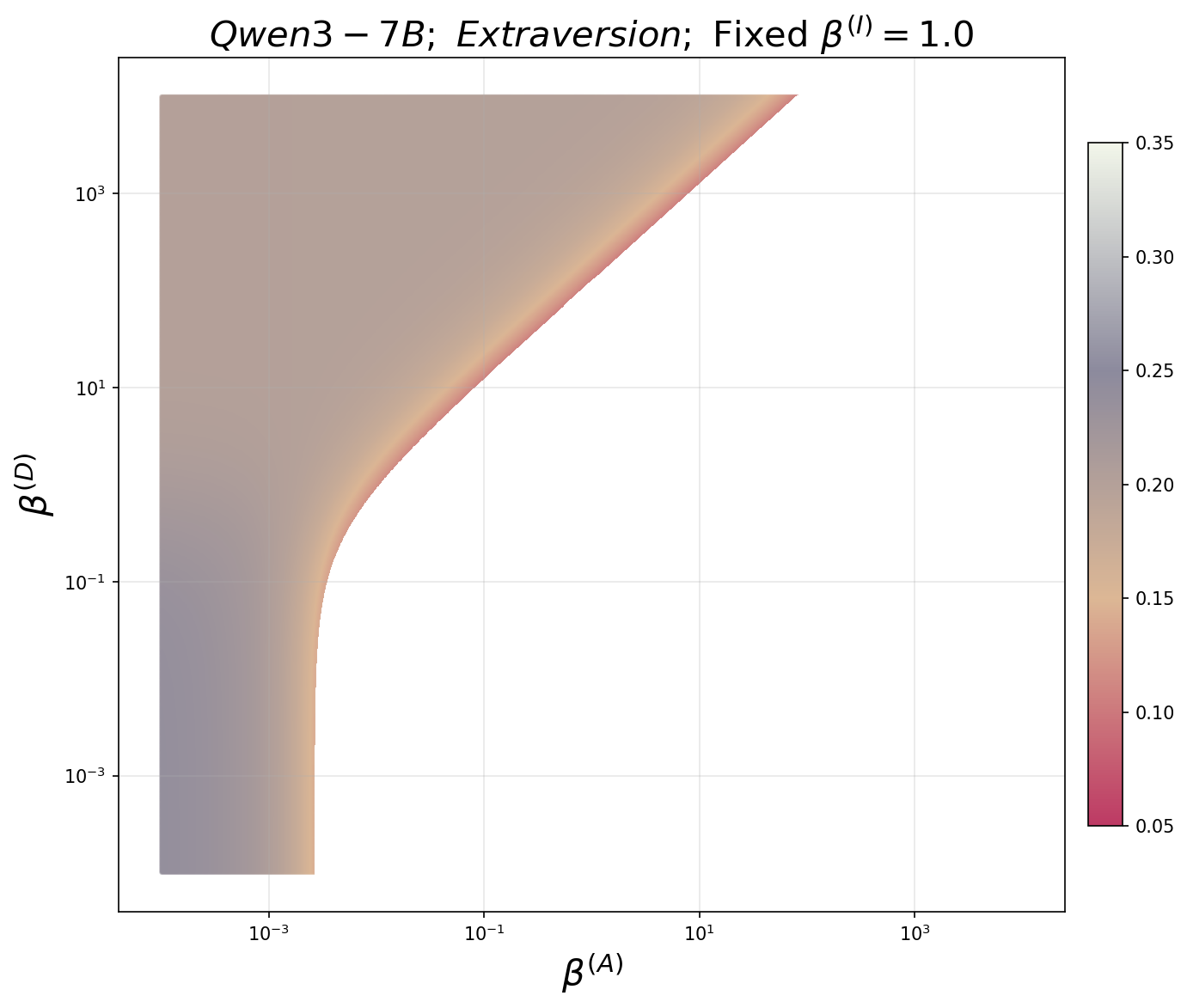}

    \includegraphics[width=0.3\linewidth]{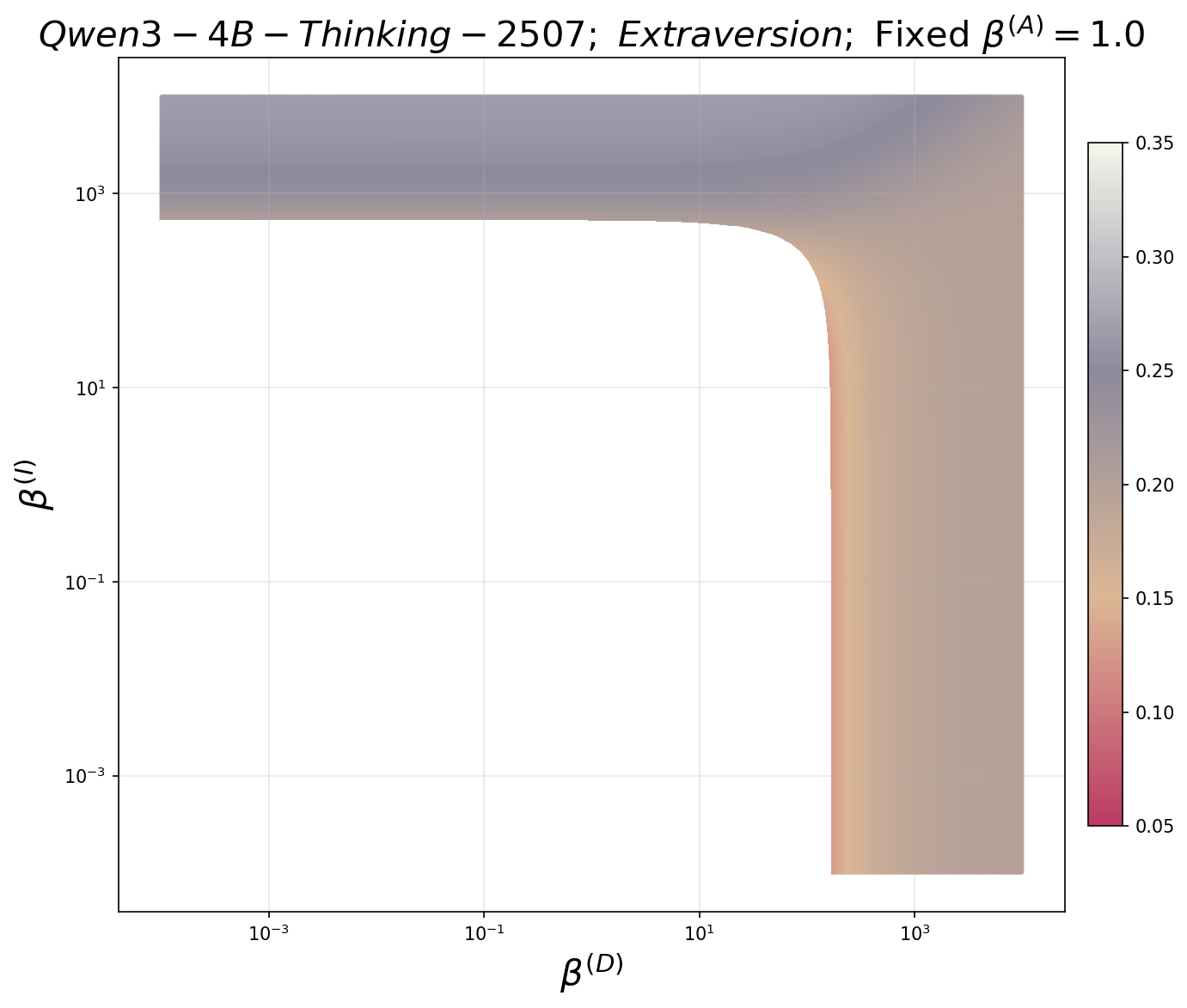}
    \includegraphics[width=0.3\linewidth]{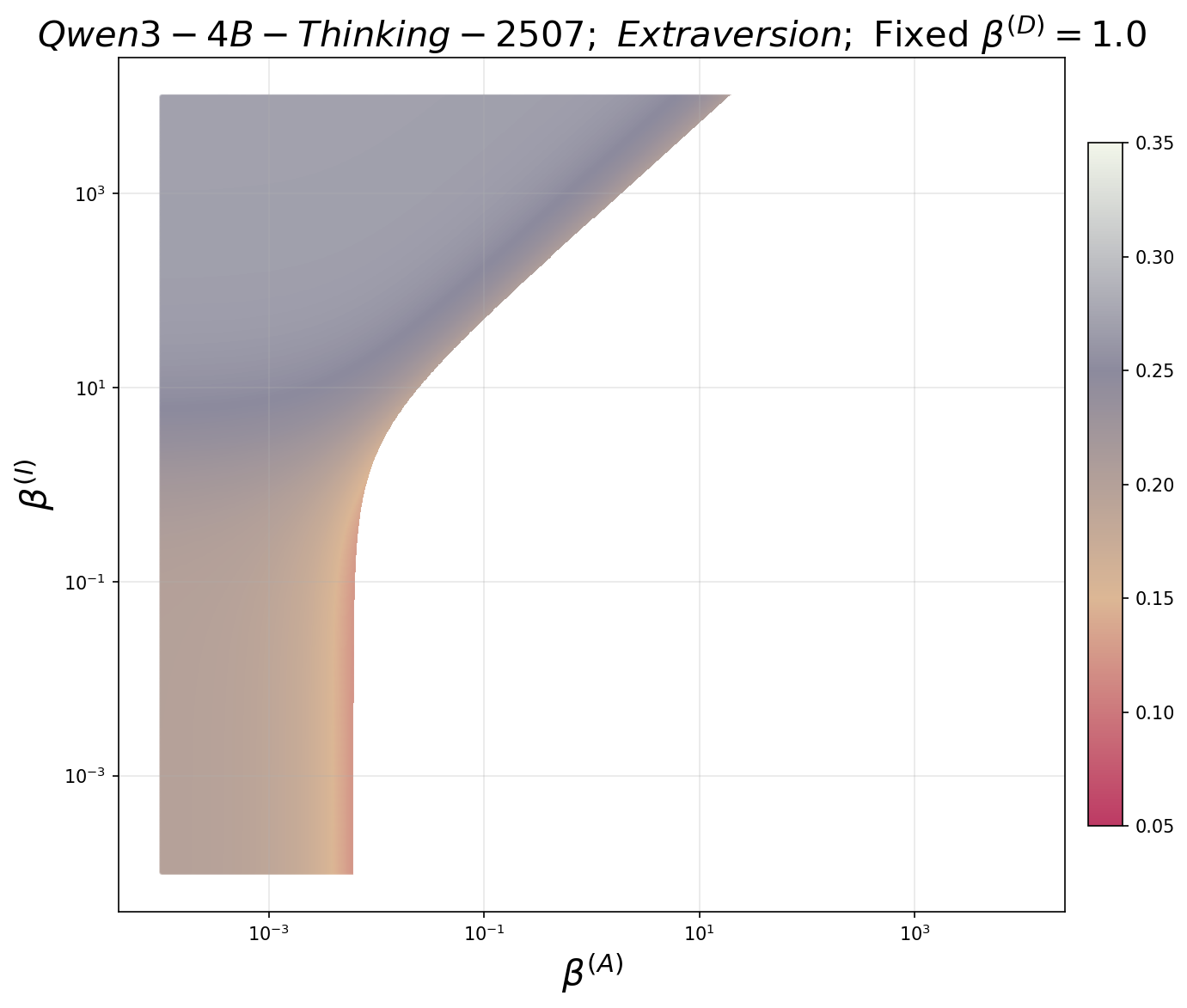}
    \includegraphics[width=0.3\linewidth]{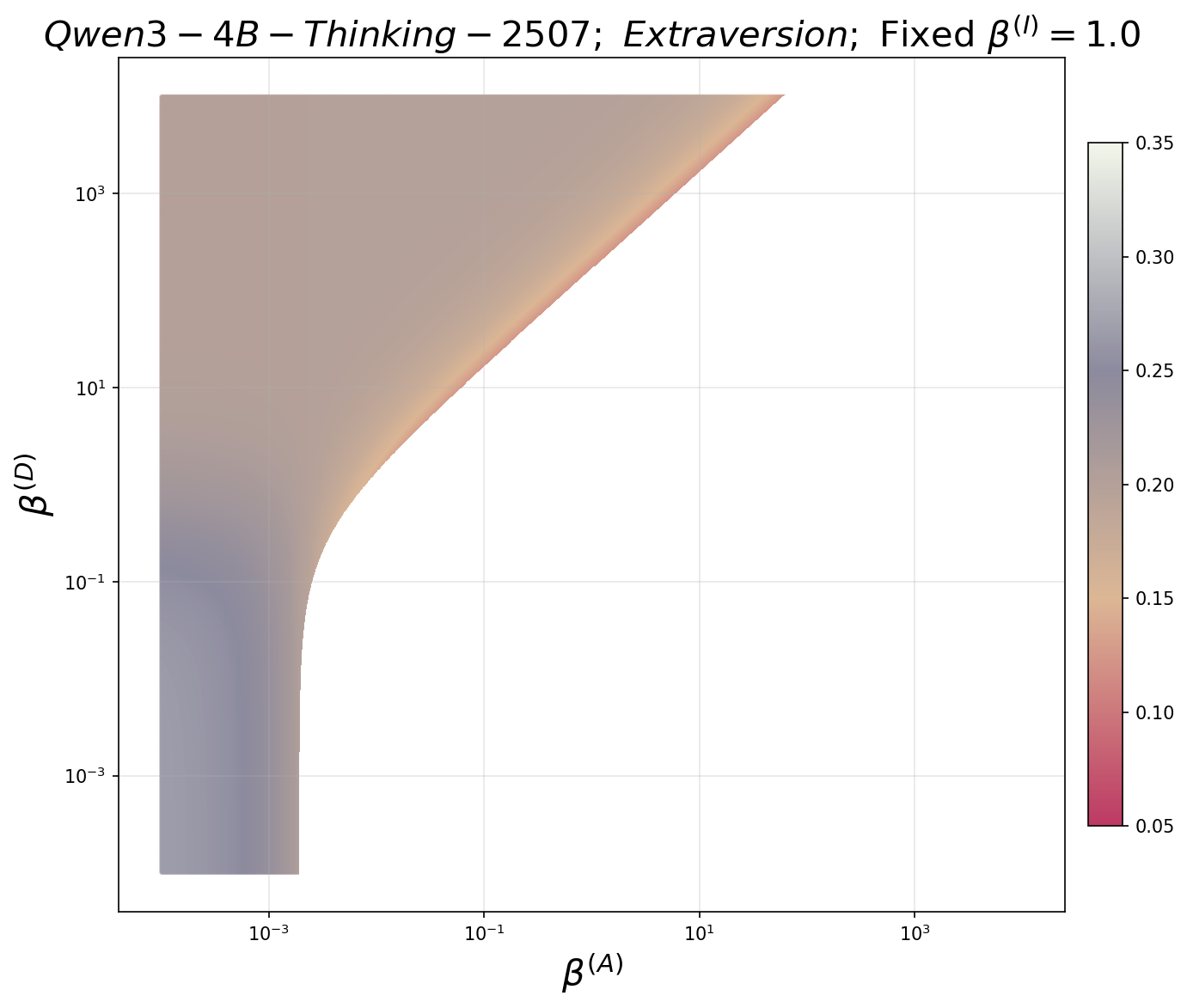}

    \includegraphics[width=0.3\linewidth]{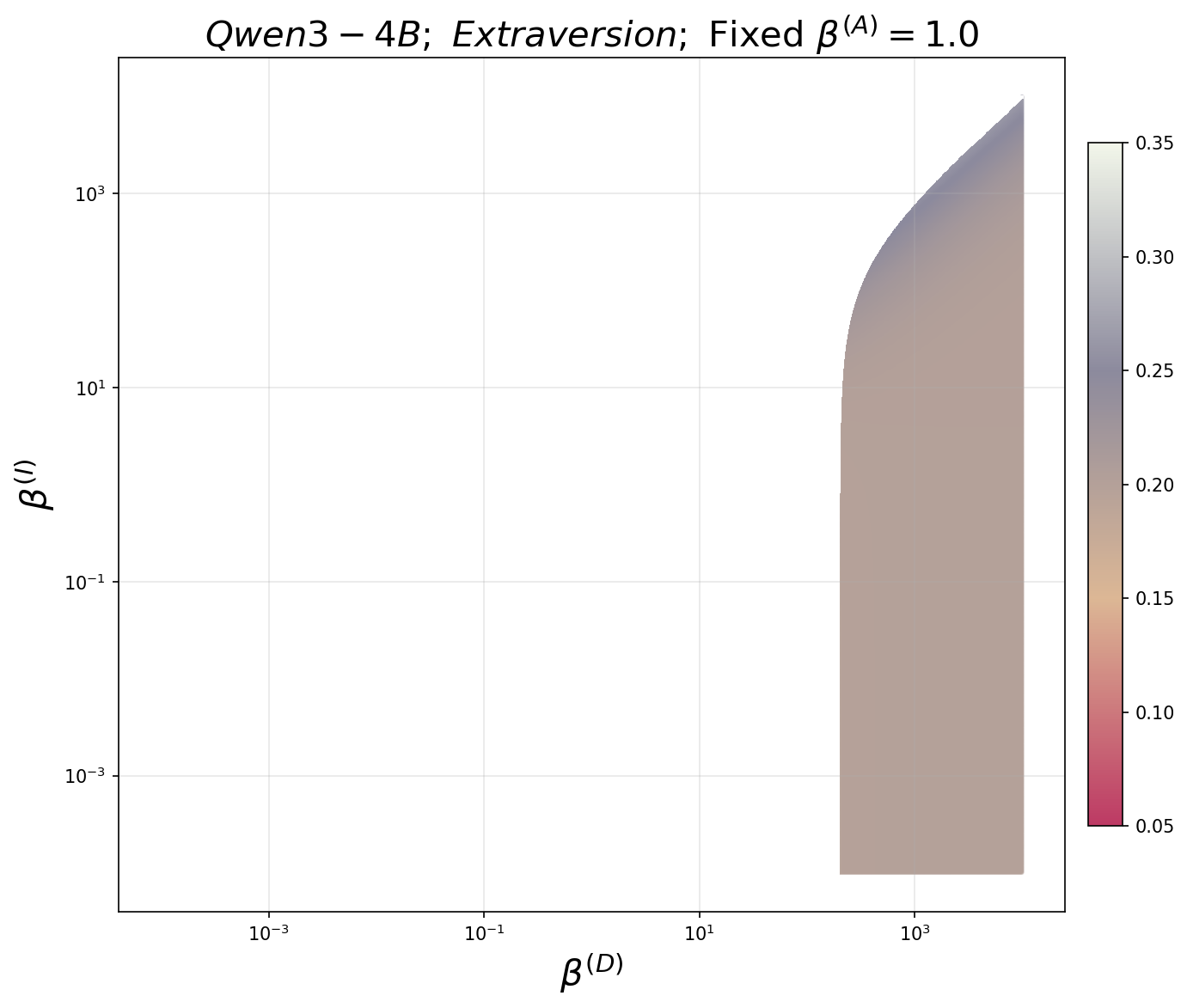}
    \includegraphics[width=0.3\linewidth]{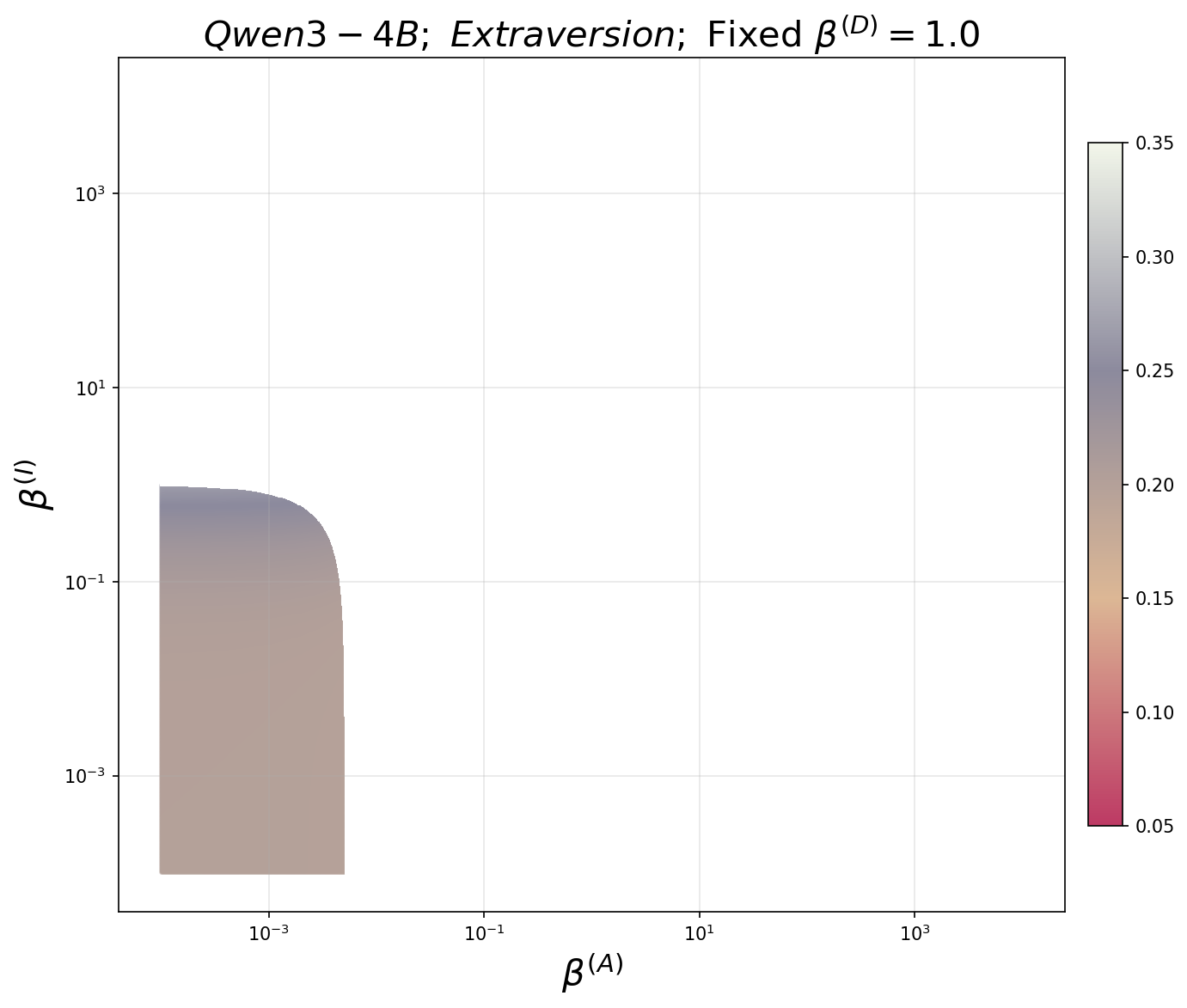}
    \includegraphics[width=0.3\linewidth]{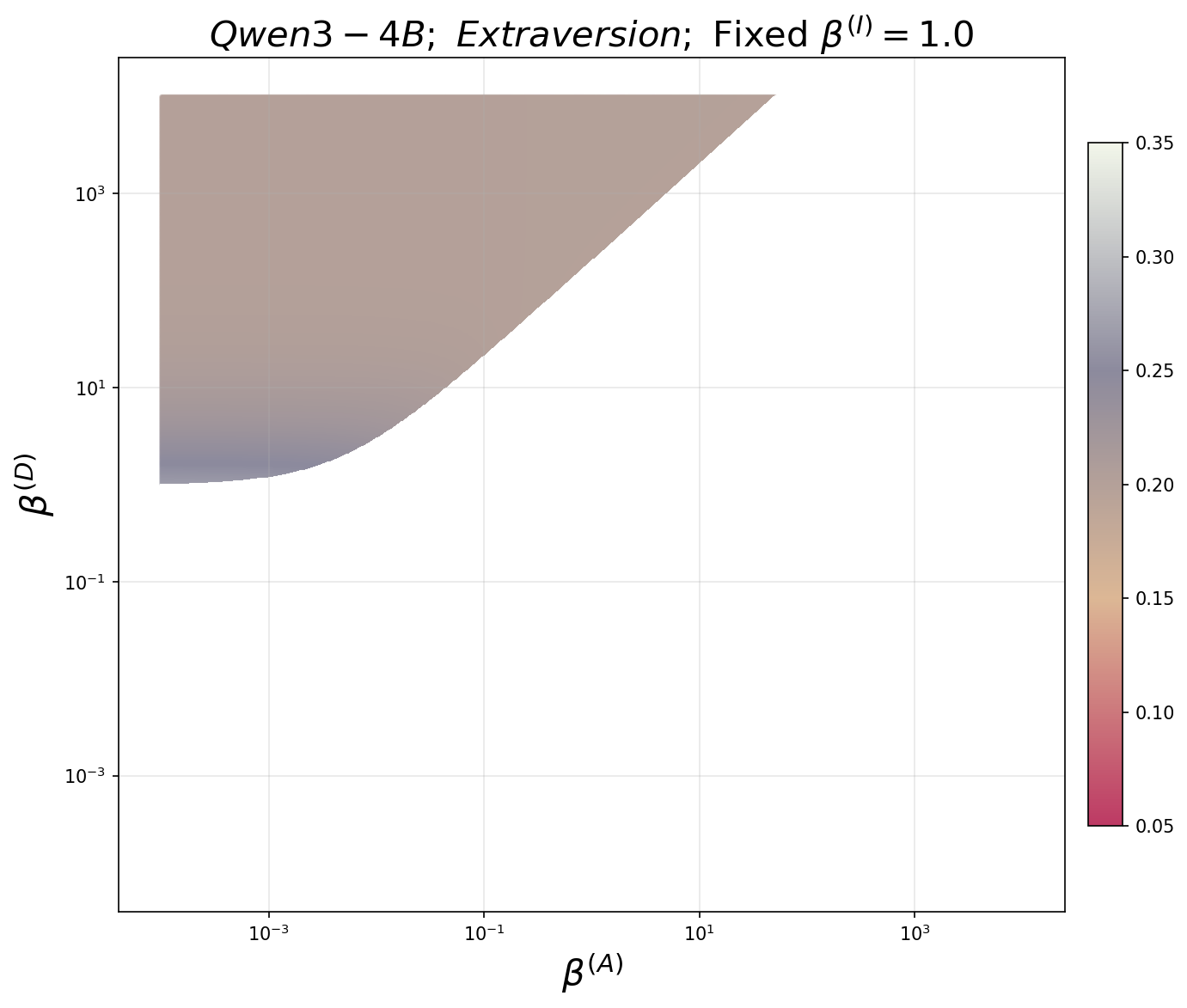}
    
    \caption{Political exclusion regarding the subpopulation \emph{Extraversion} on the $\mathtt{Big\ Five}$ dataset.}
    \label{fig:Extraversion}
\end{figure}

\begin{figure}
    \centering

    \includegraphics[width=0.3\linewidth]{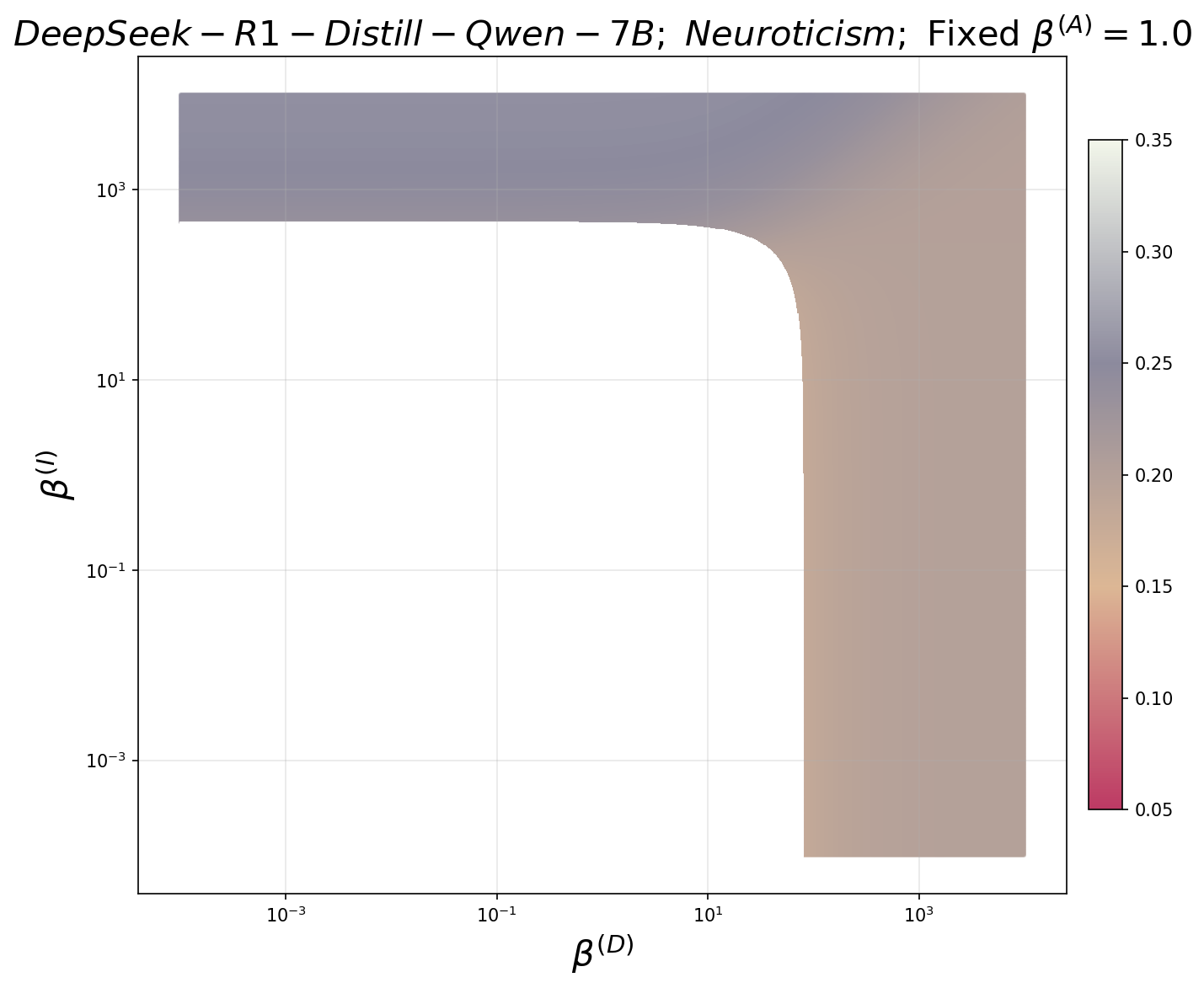}
    \includegraphics[width=0.3\linewidth]{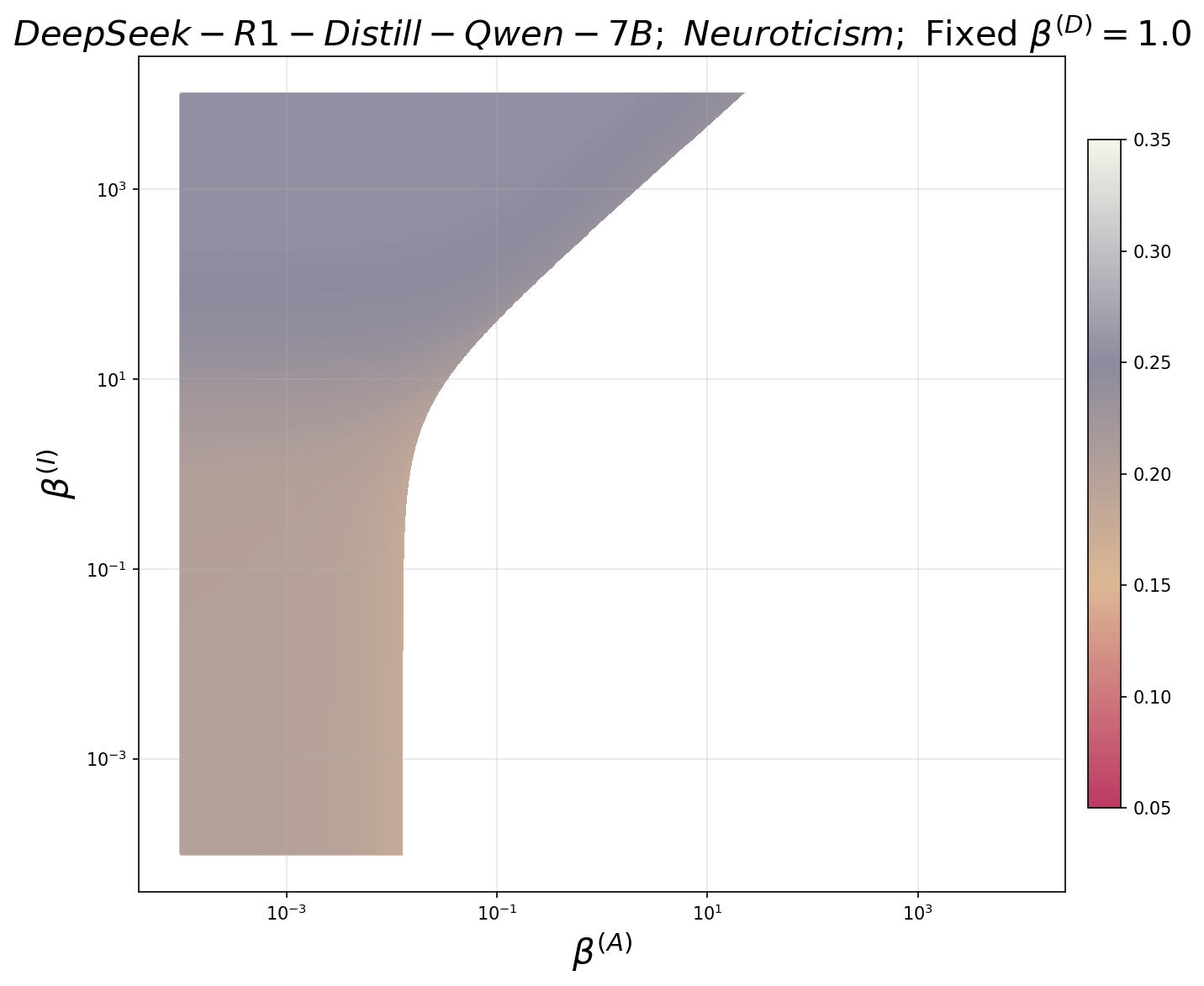}
    \includegraphics[width=0.3\linewidth]{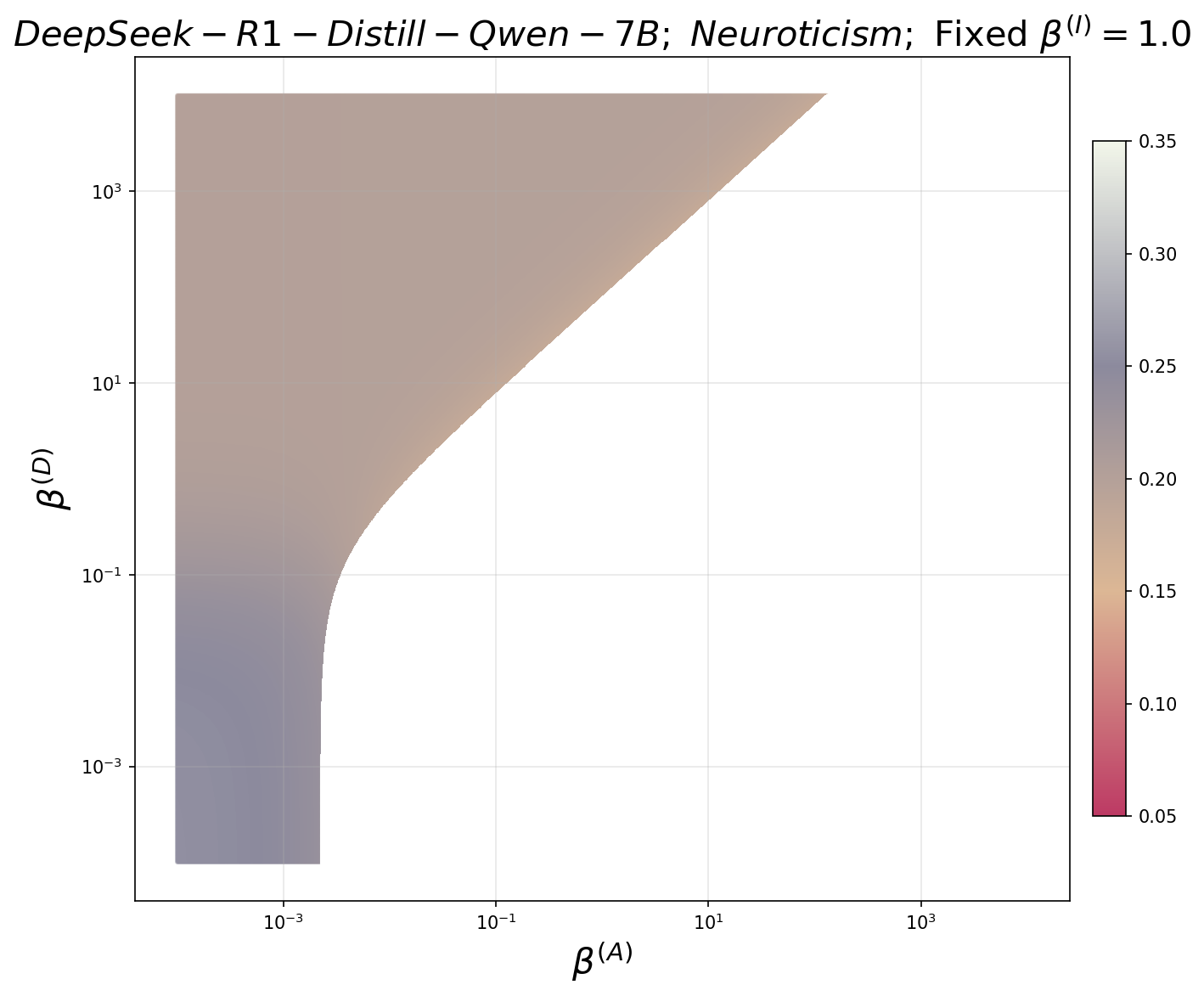}

    \includegraphics[width=0.3\linewidth]{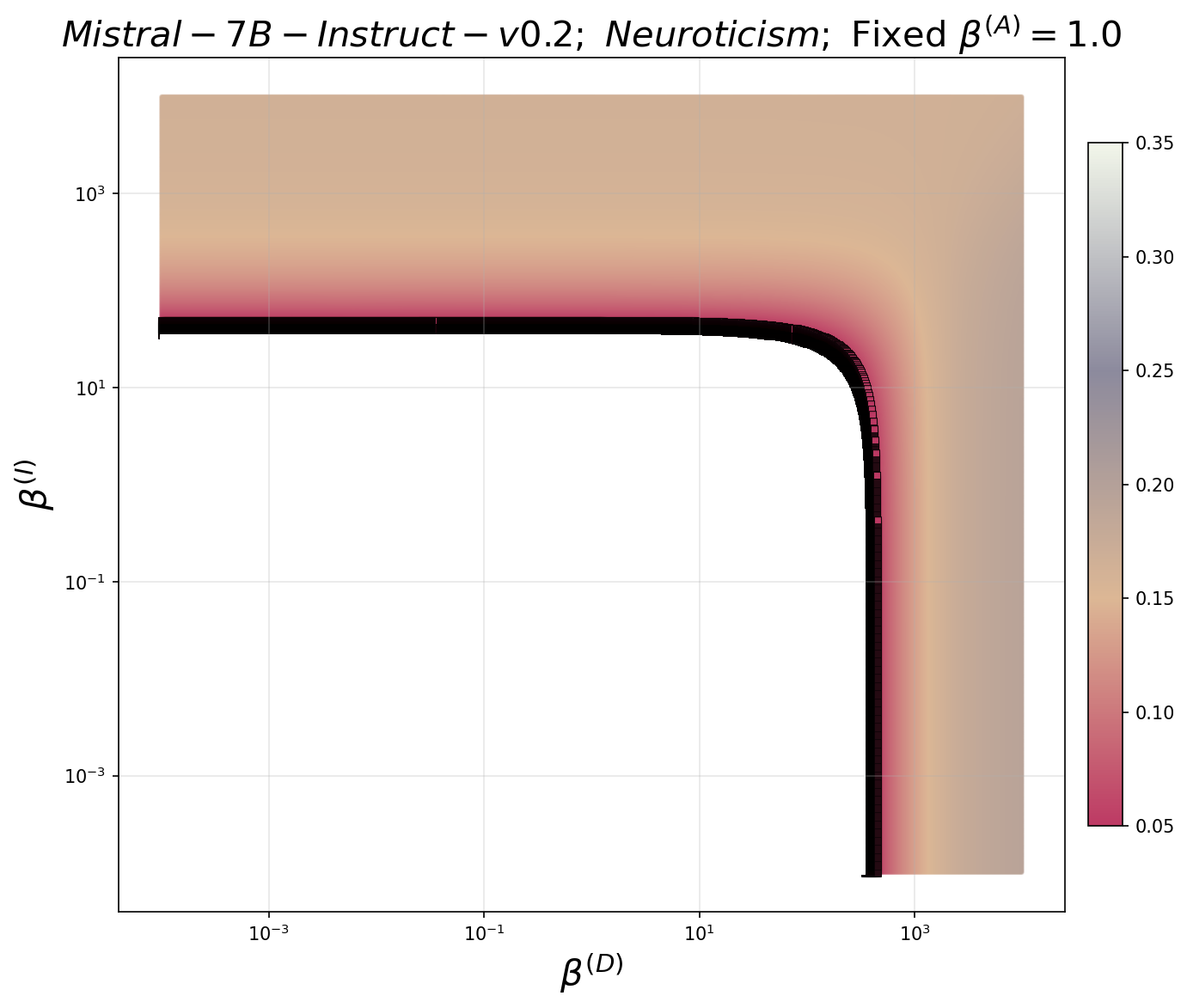}
    \includegraphics[width=0.3\linewidth]{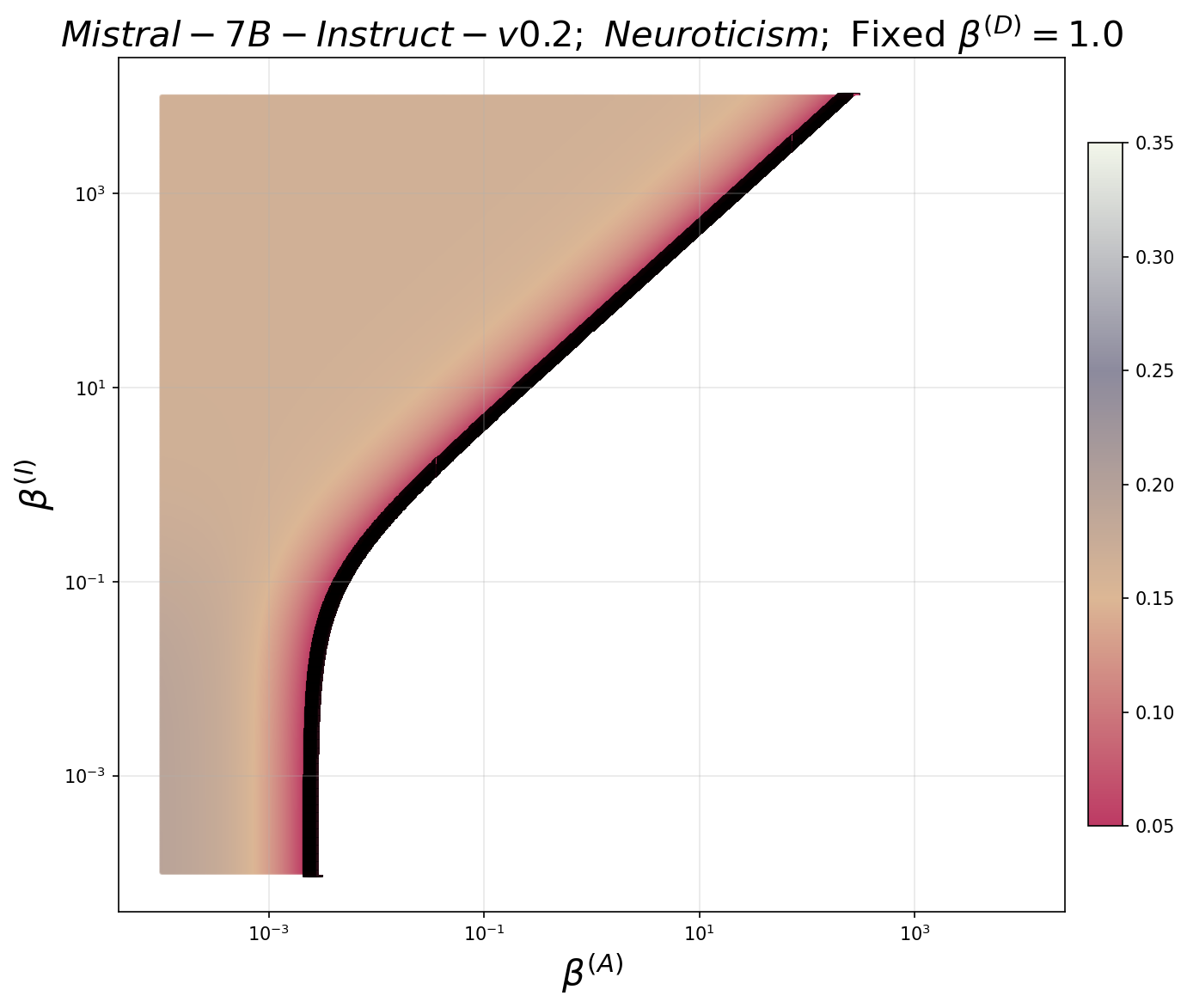}
    \includegraphics[width=0.3\linewidth]{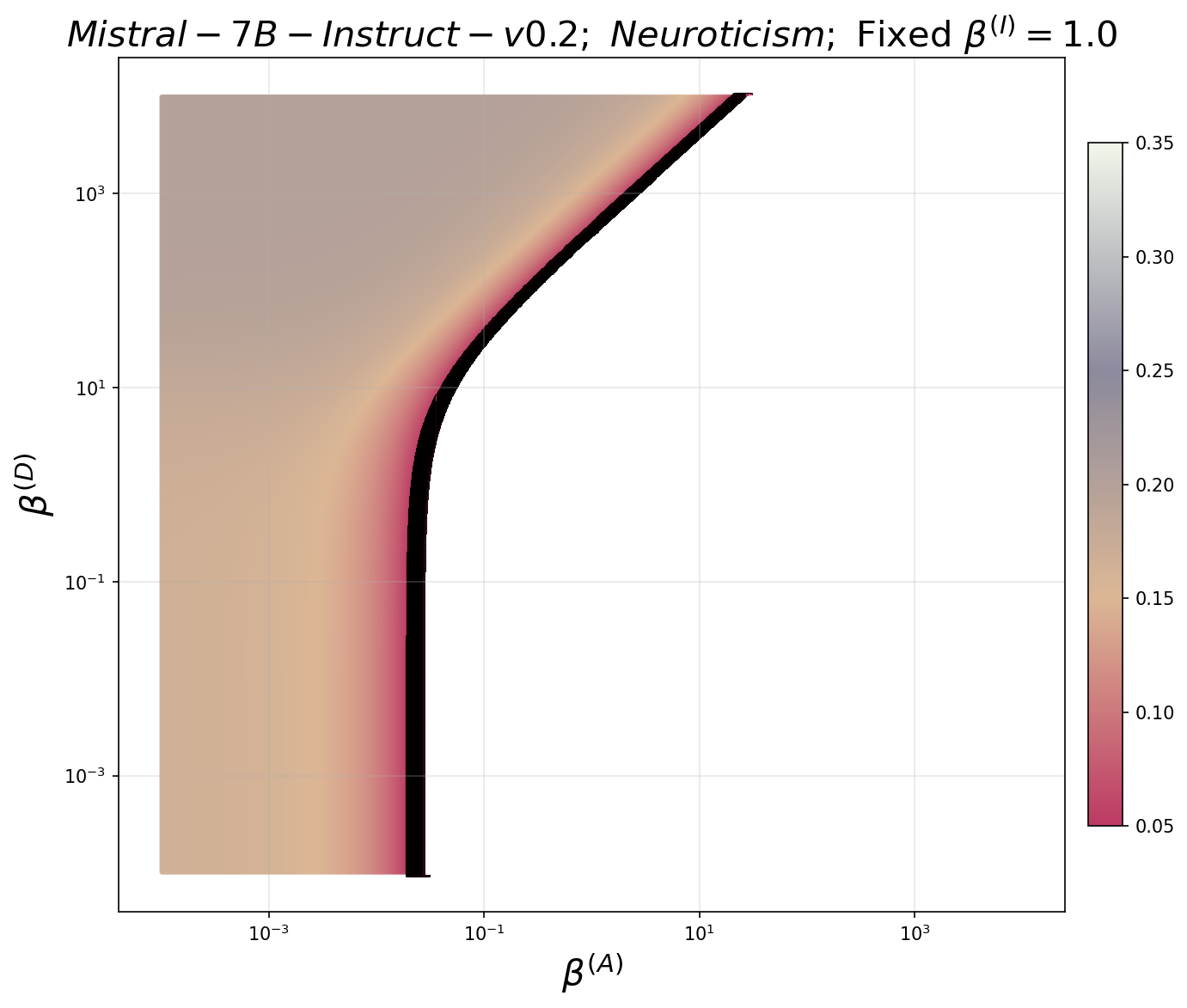}

    \includegraphics[width=0.3\linewidth]{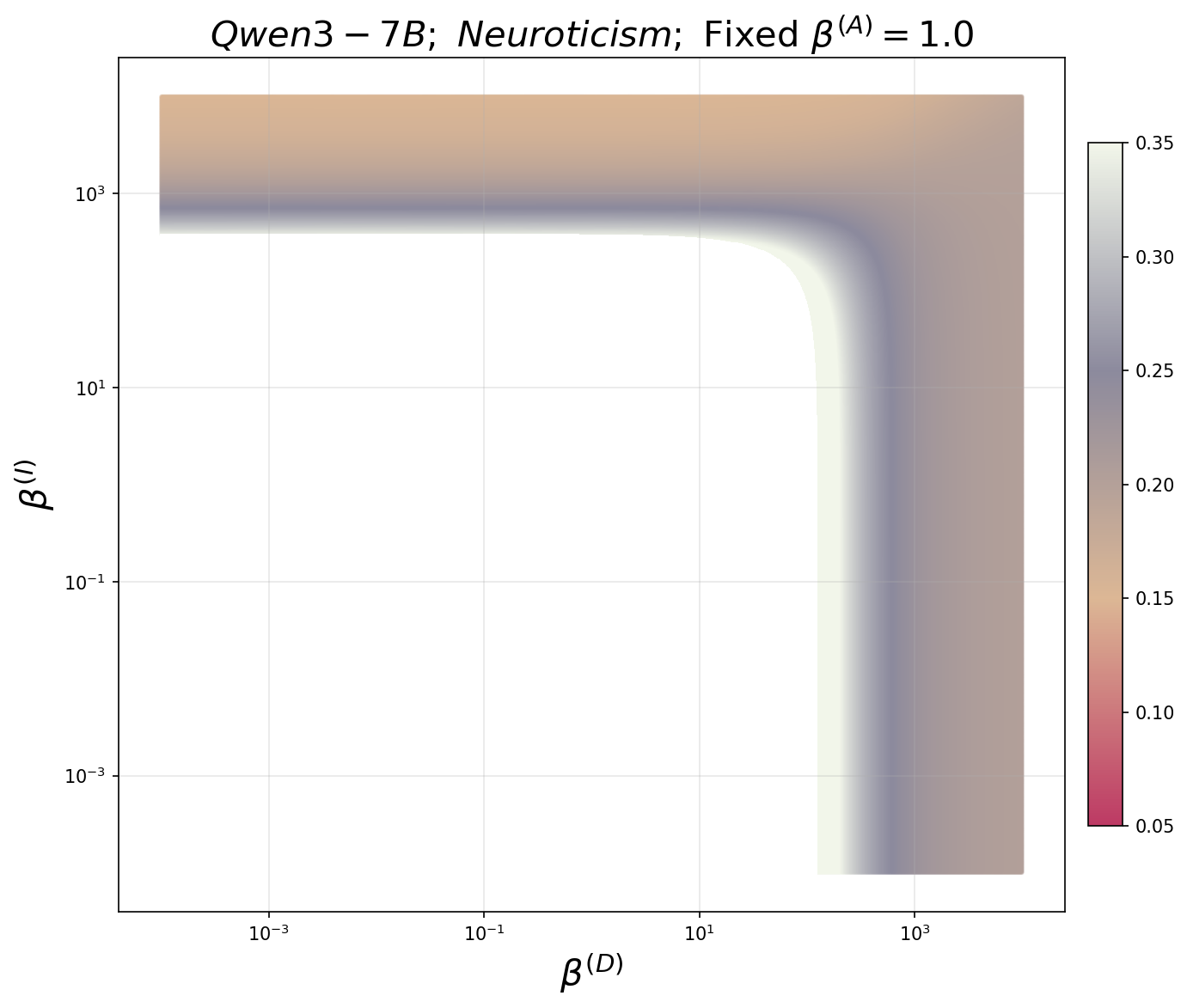}
    \includegraphics[width=0.3\linewidth]{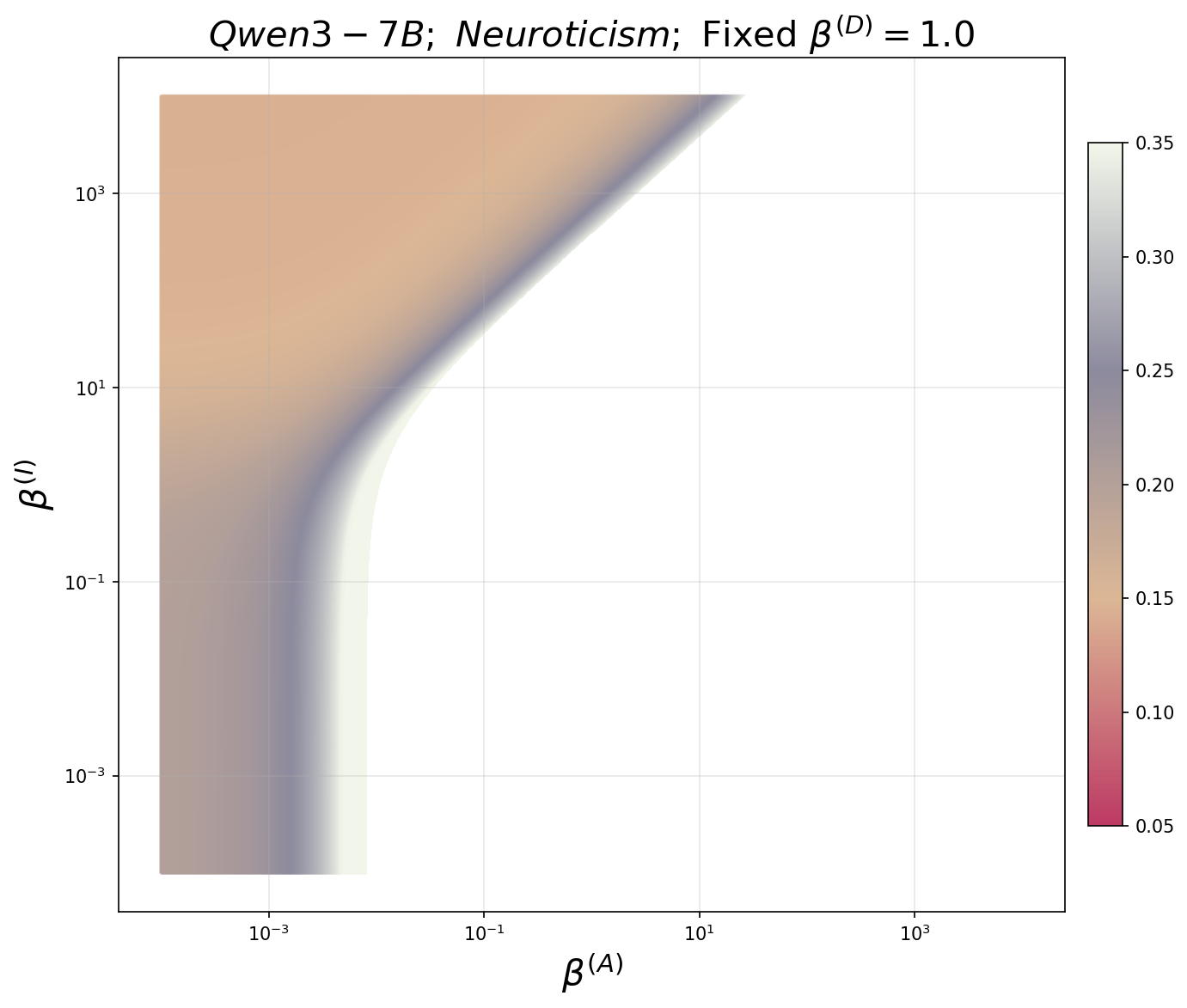}
    \includegraphics[width=0.3\linewidth]{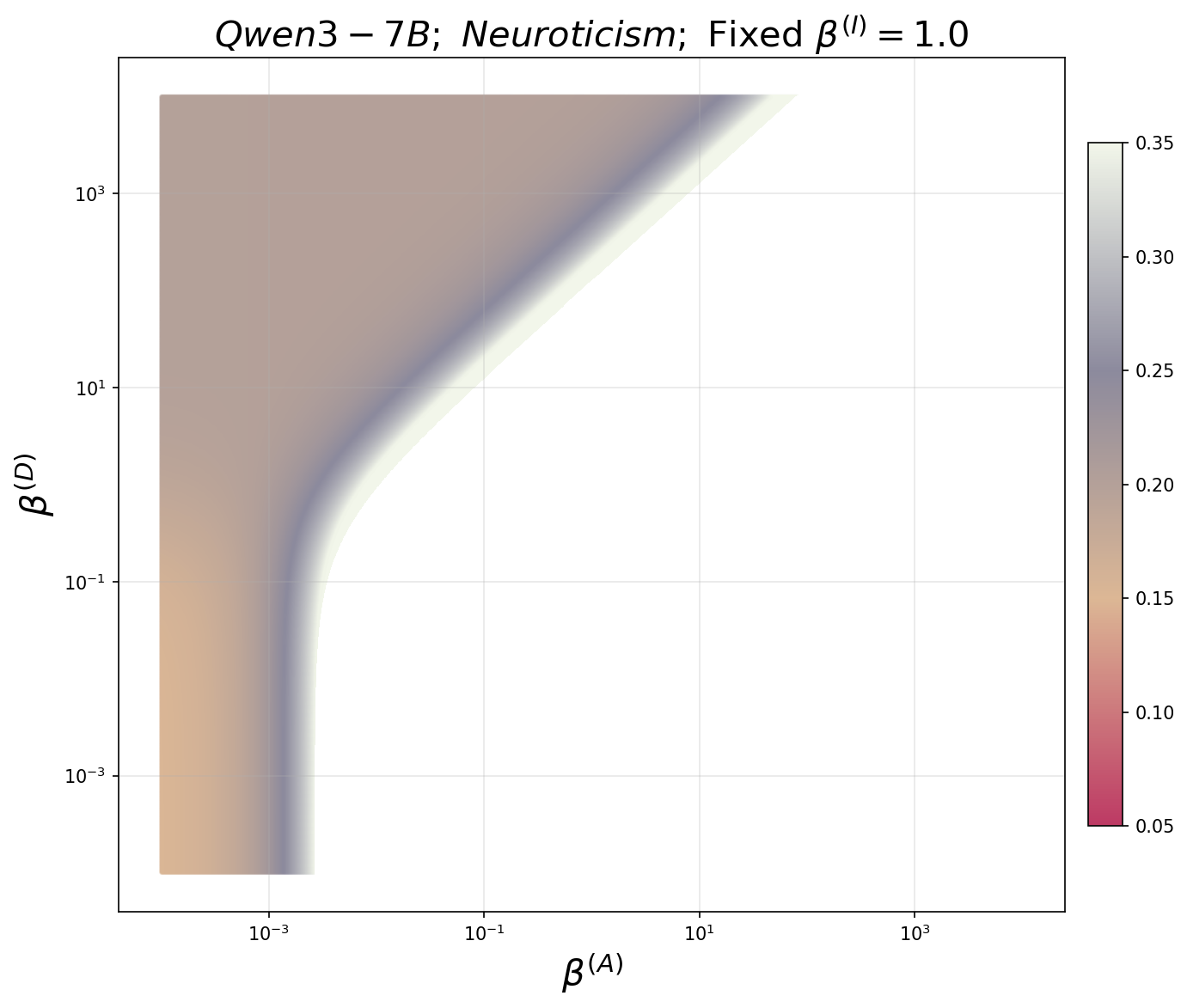}

    \includegraphics[width=0.3\linewidth]{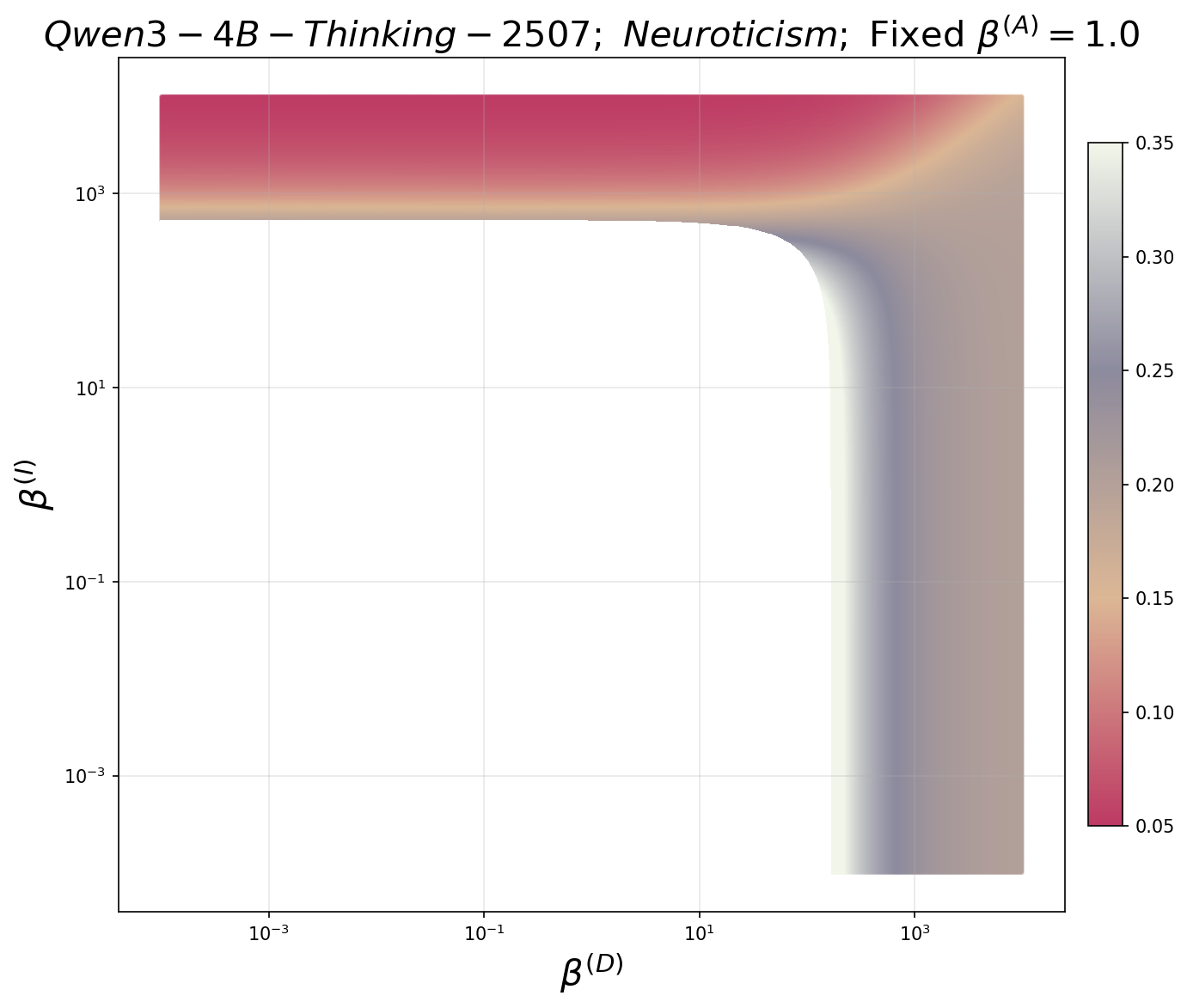}
    \includegraphics[width=0.3\linewidth]{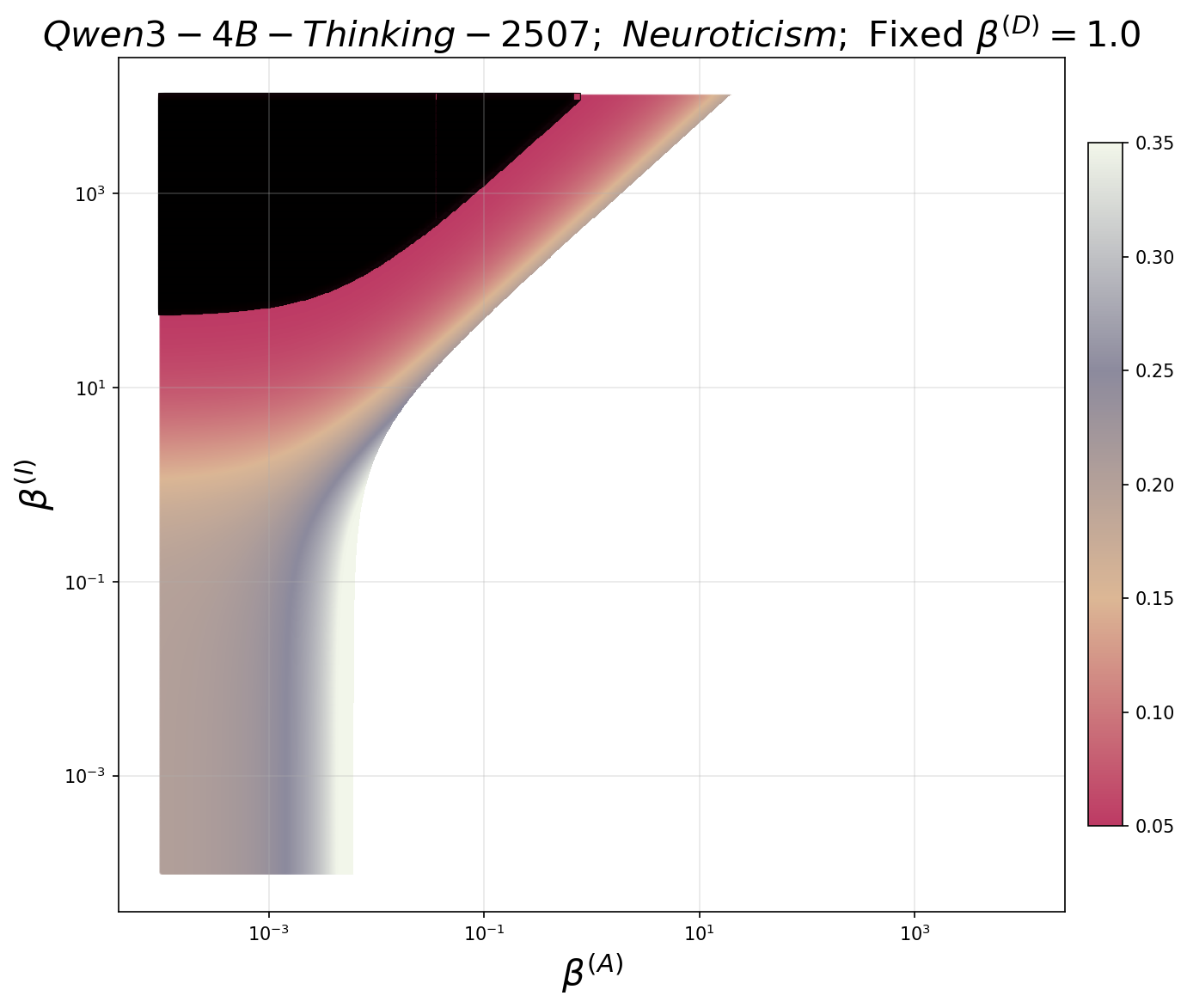}
    \includegraphics[width=0.3\linewidth]{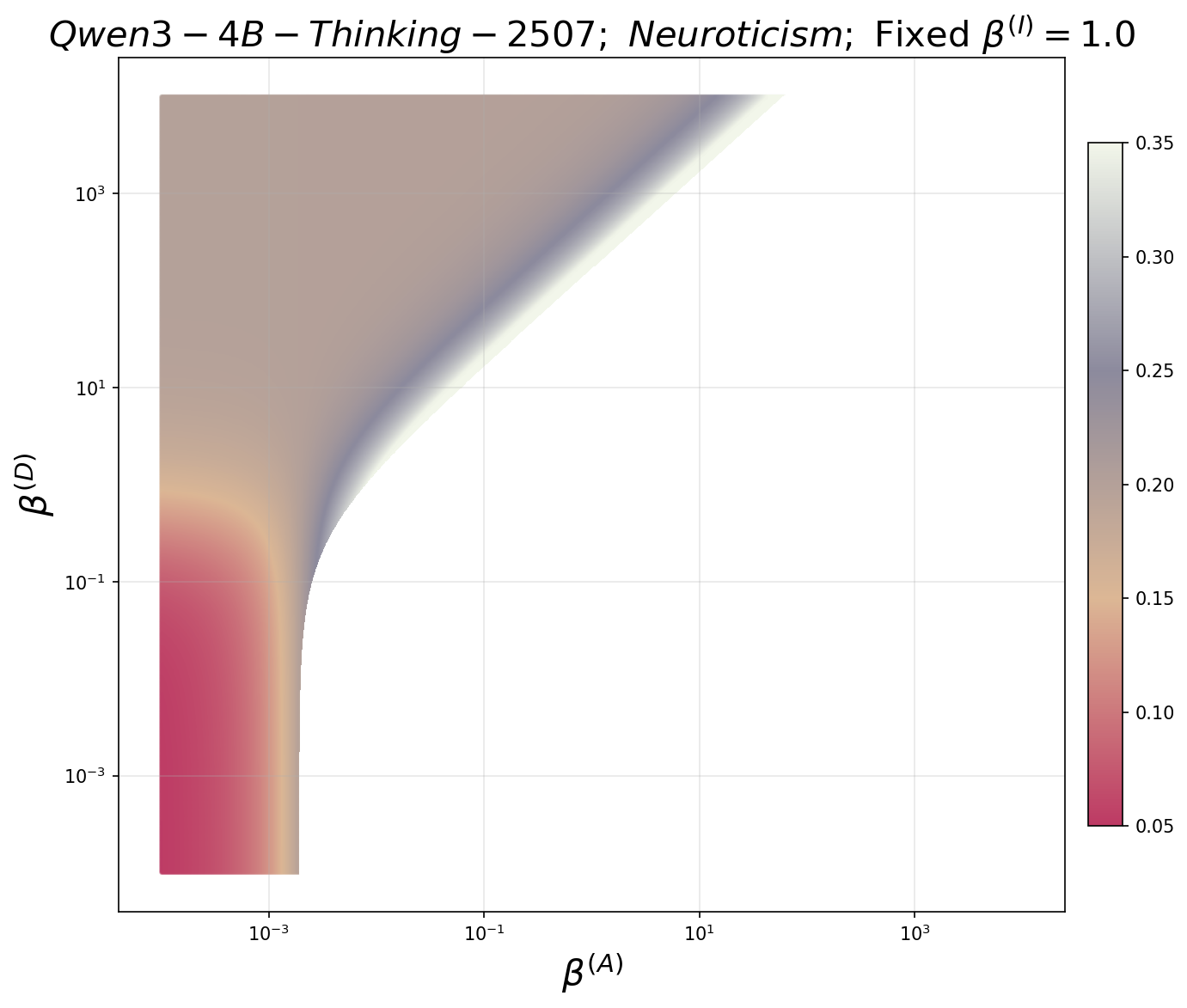}

    \includegraphics[width=0.3\linewidth]{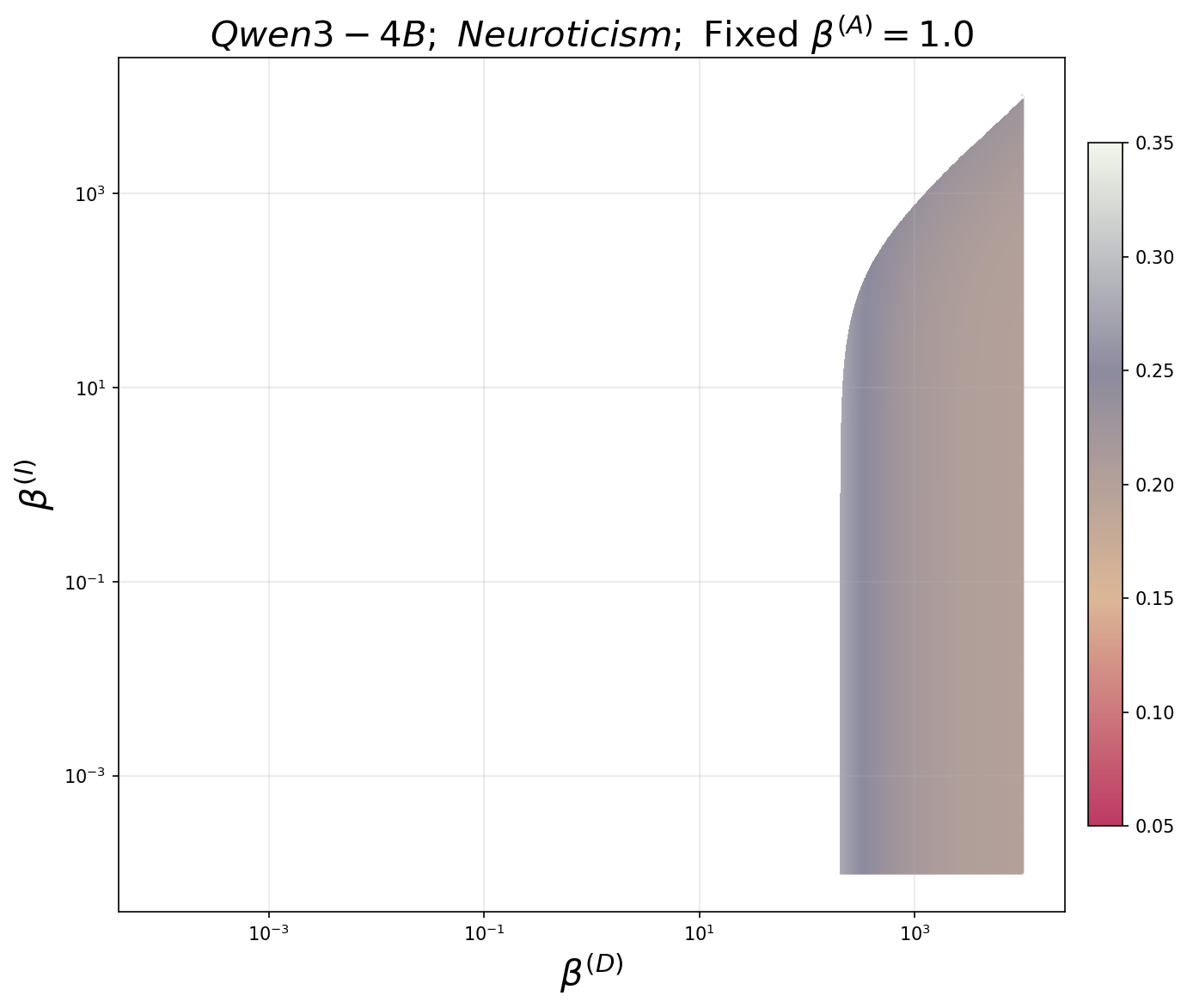}
    \includegraphics[width=0.3\linewidth]{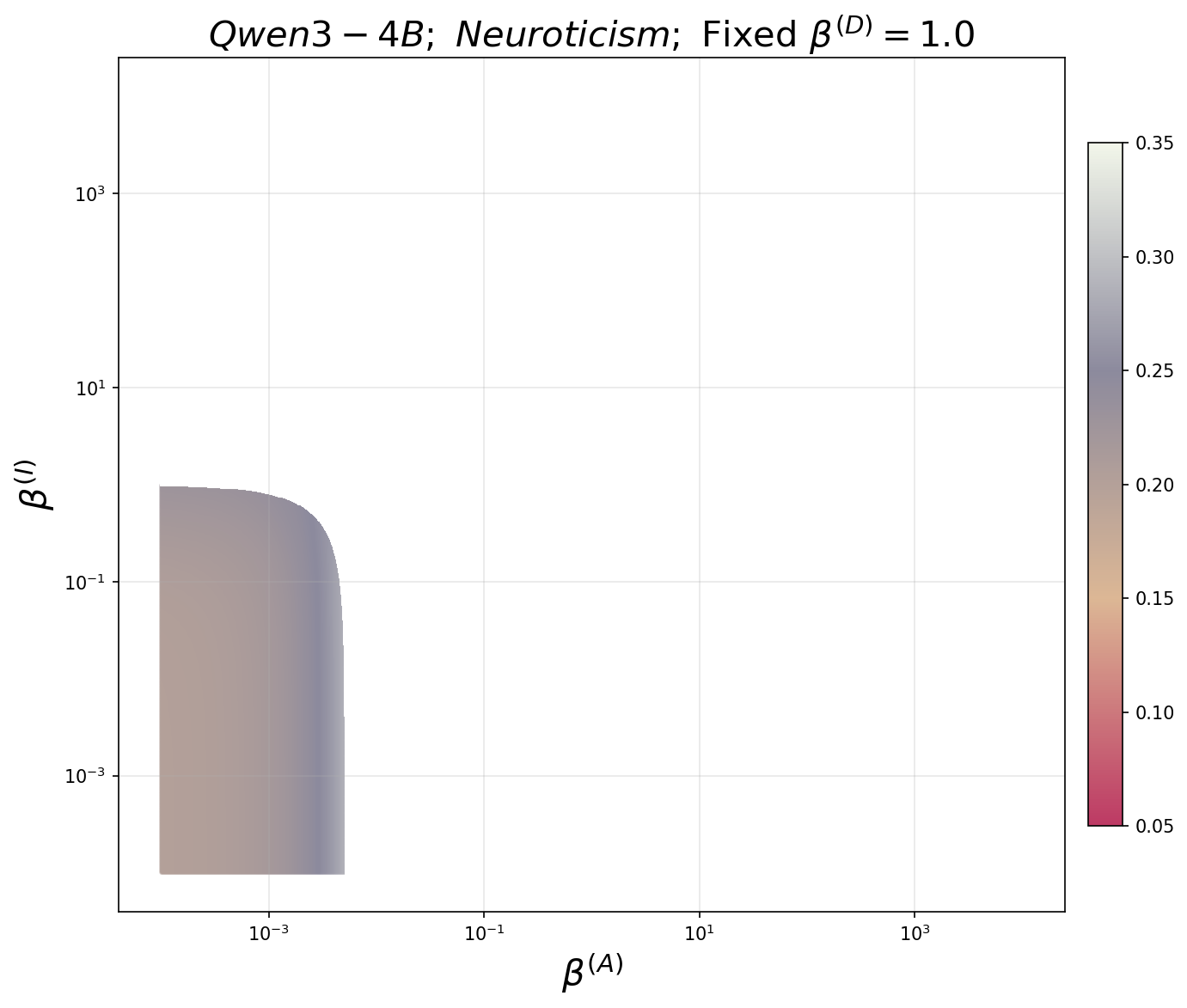}
    \includegraphics[width=0.3\linewidth]{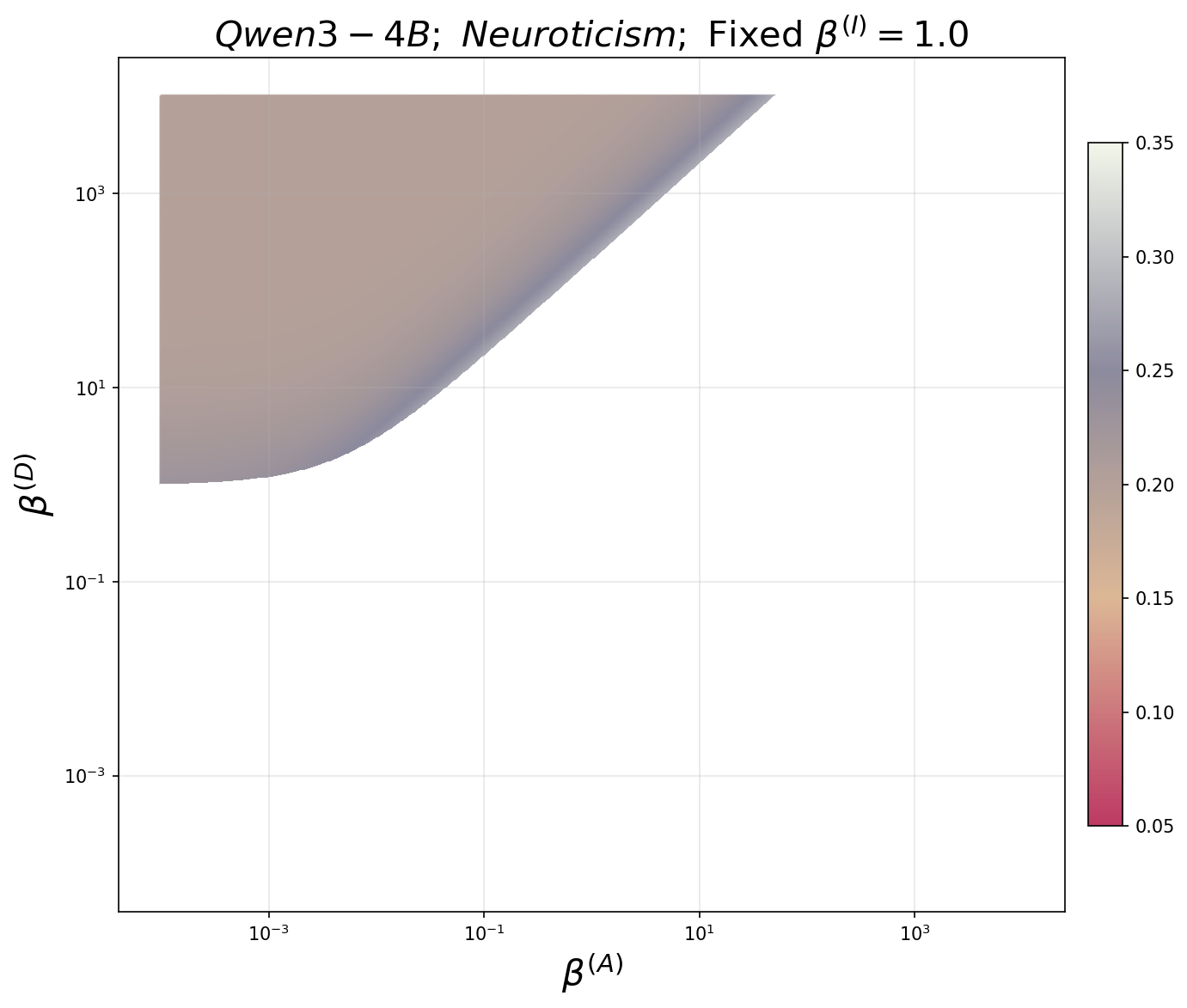}
    
    \caption{Political exclusion regarding the subpopulation \emph{Neuroticism} on the $\mathtt{Big\ Five}$ dataset.}
    \label{fig:Neuroticism}
\end{figure}

\begin{figure}
    \centering

    \includegraphics[width=0.3\linewidth]{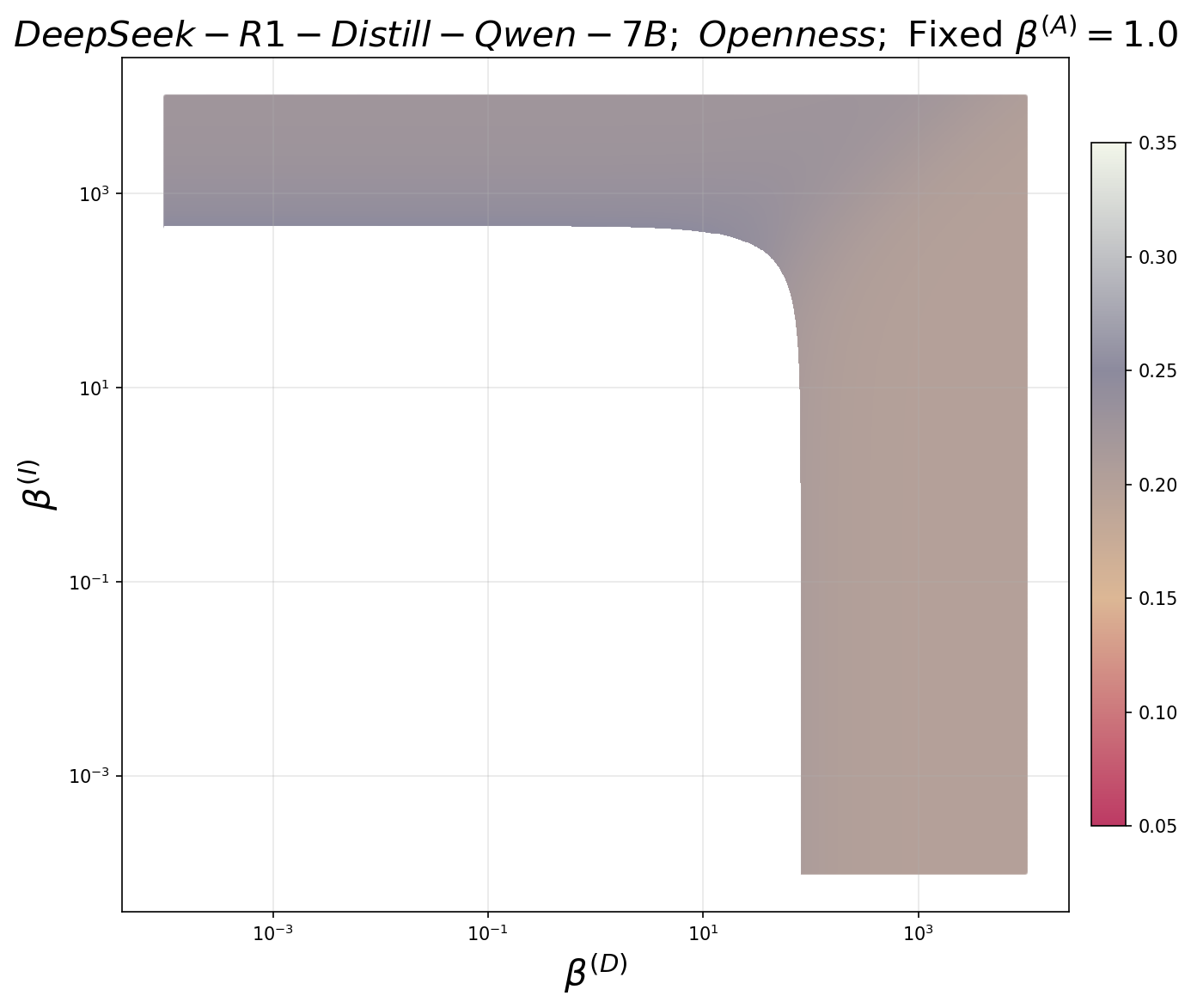}
    \includegraphics[width=0.3\linewidth]{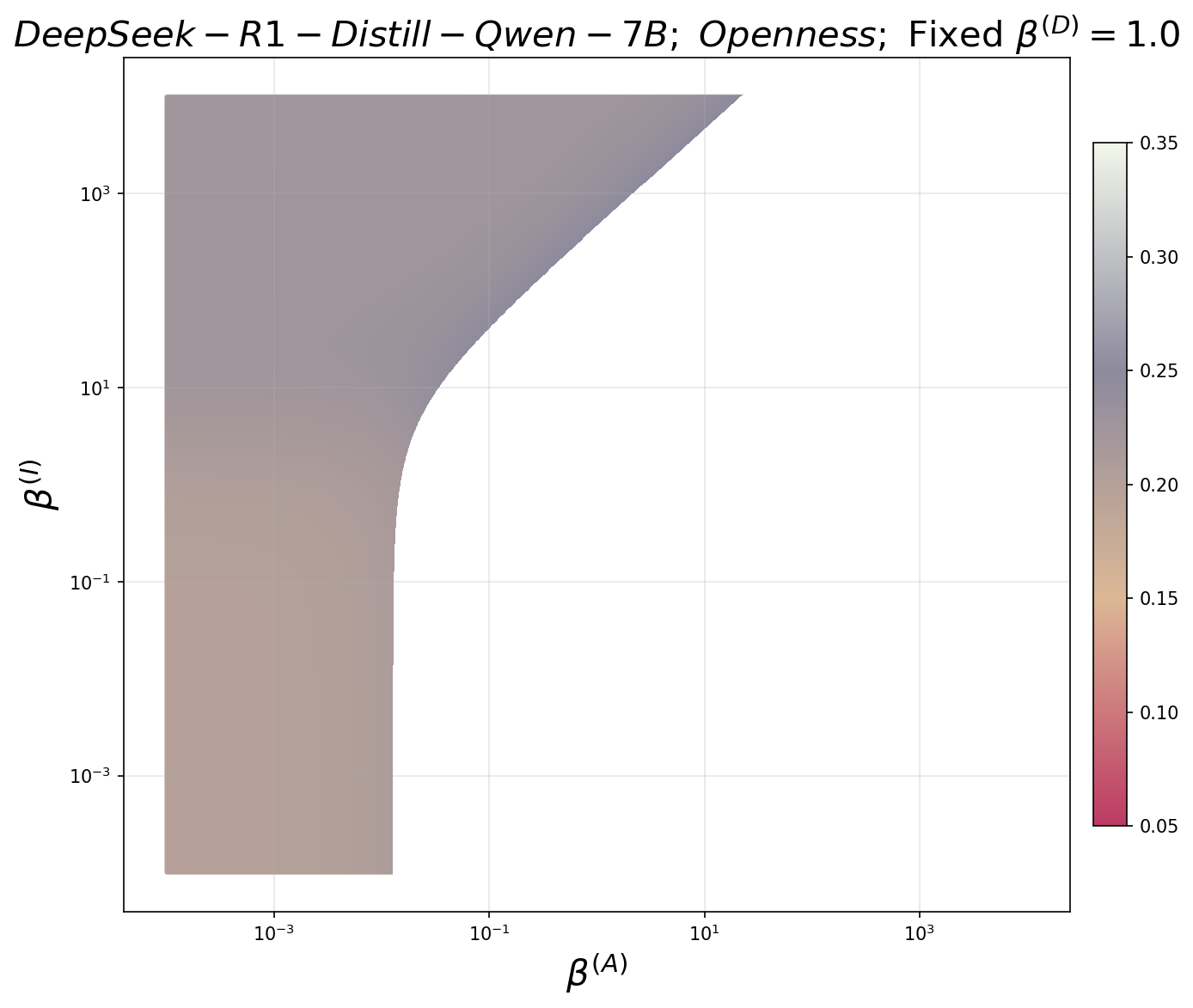}
    \includegraphics[width=0.3\linewidth]{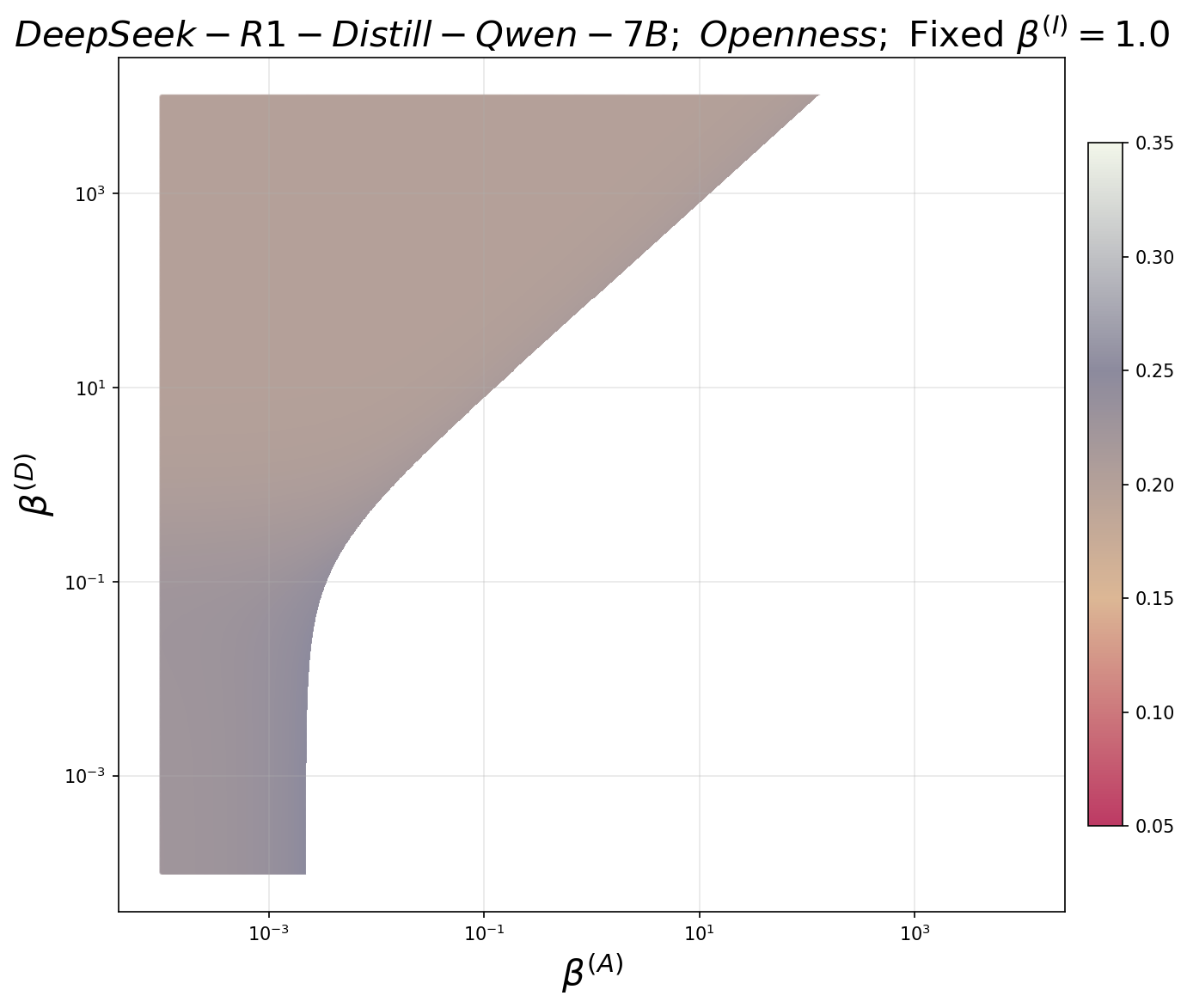}

    \includegraphics[width=0.3\linewidth]{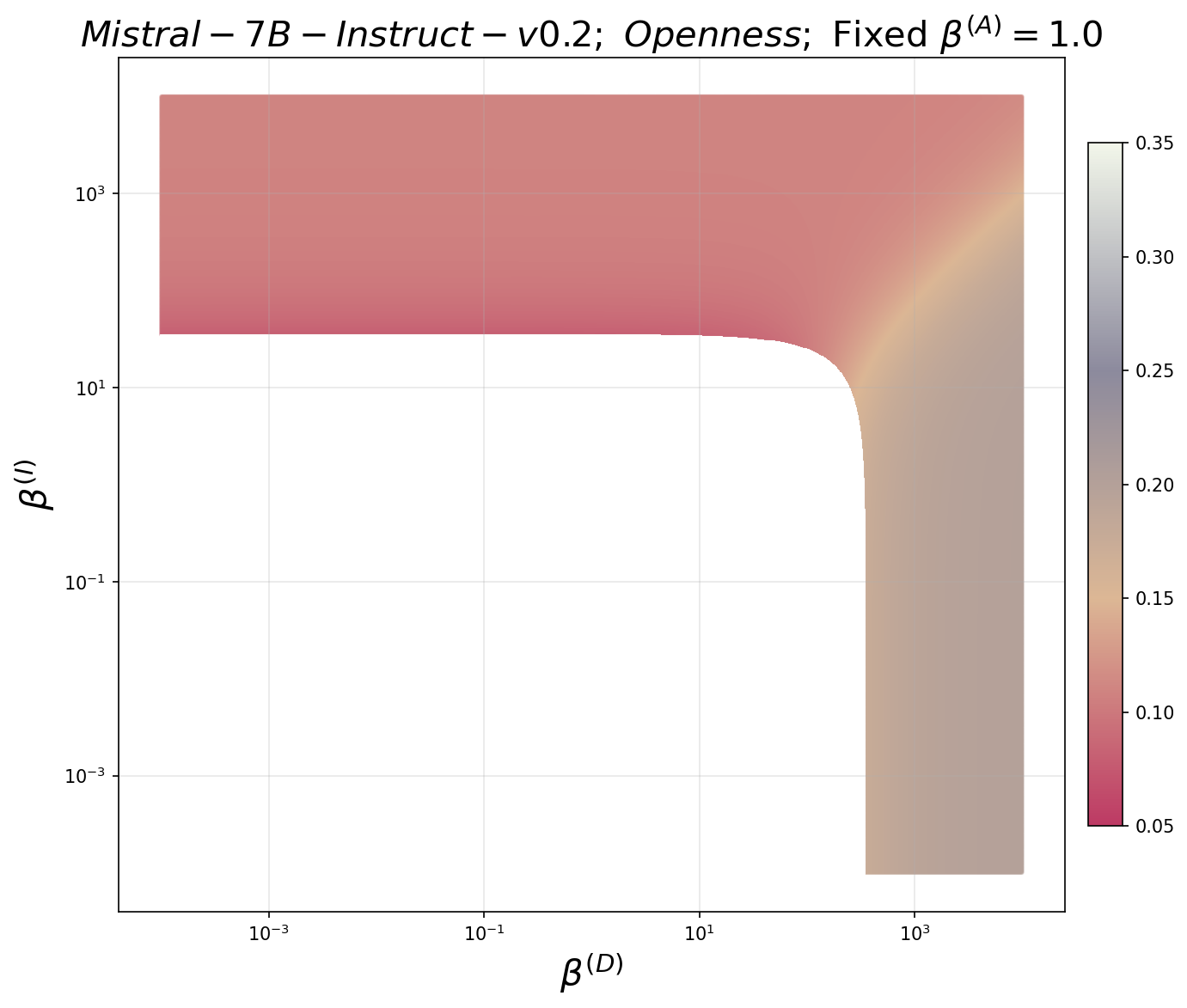}
    \includegraphics[width=0.3\linewidth]{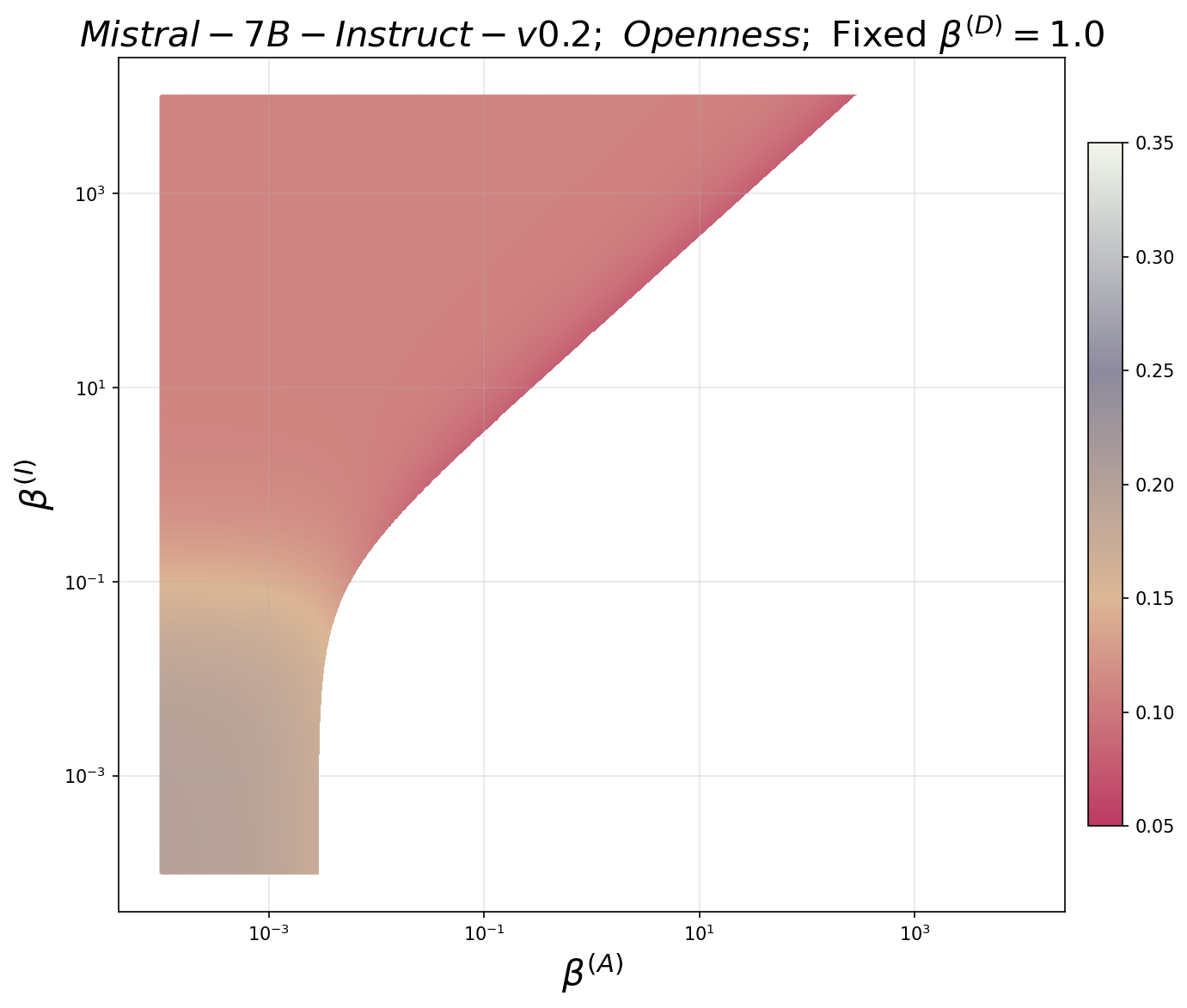}
    \includegraphics[width=0.3\linewidth]{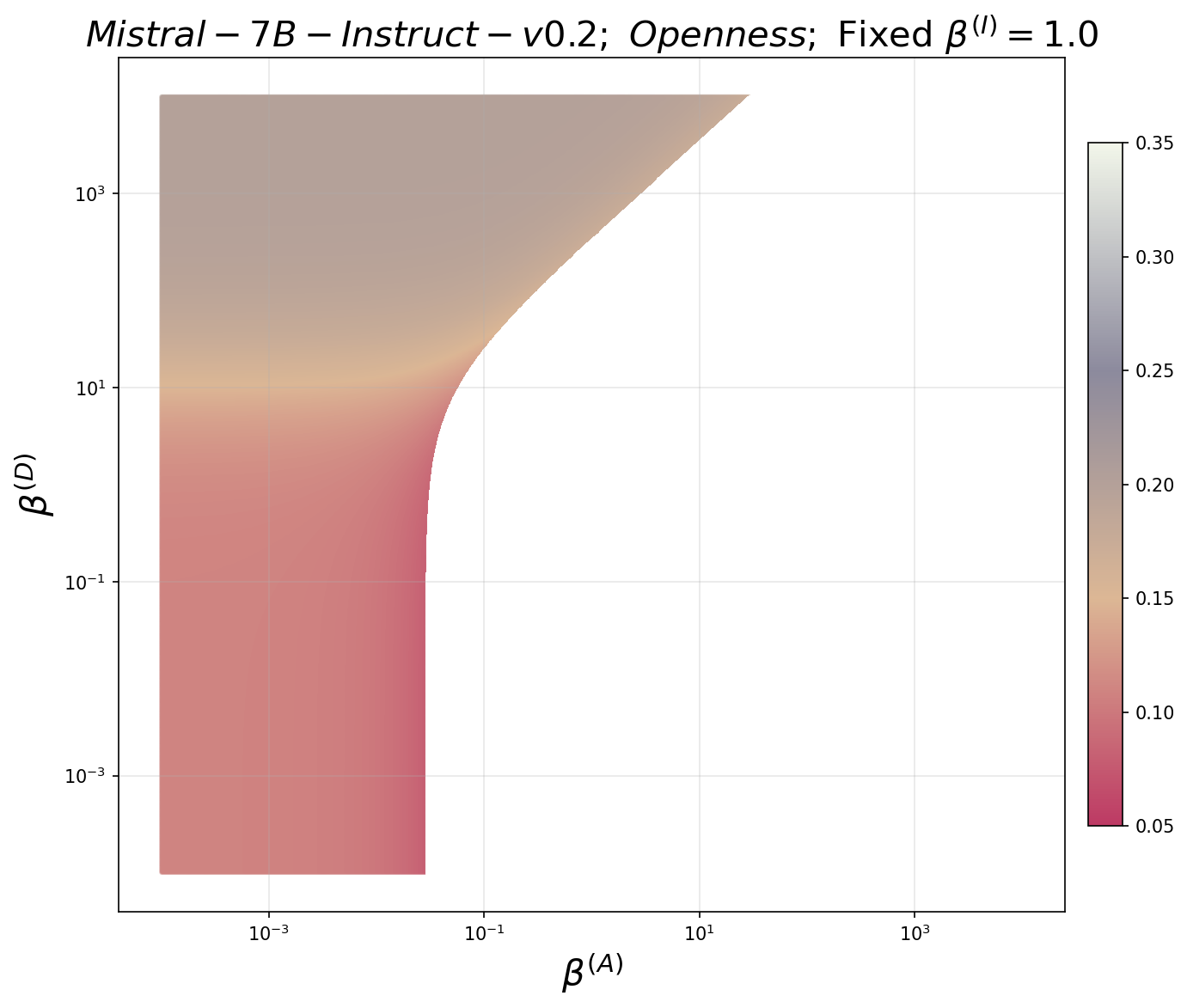}

    \includegraphics[width=0.3\linewidth]{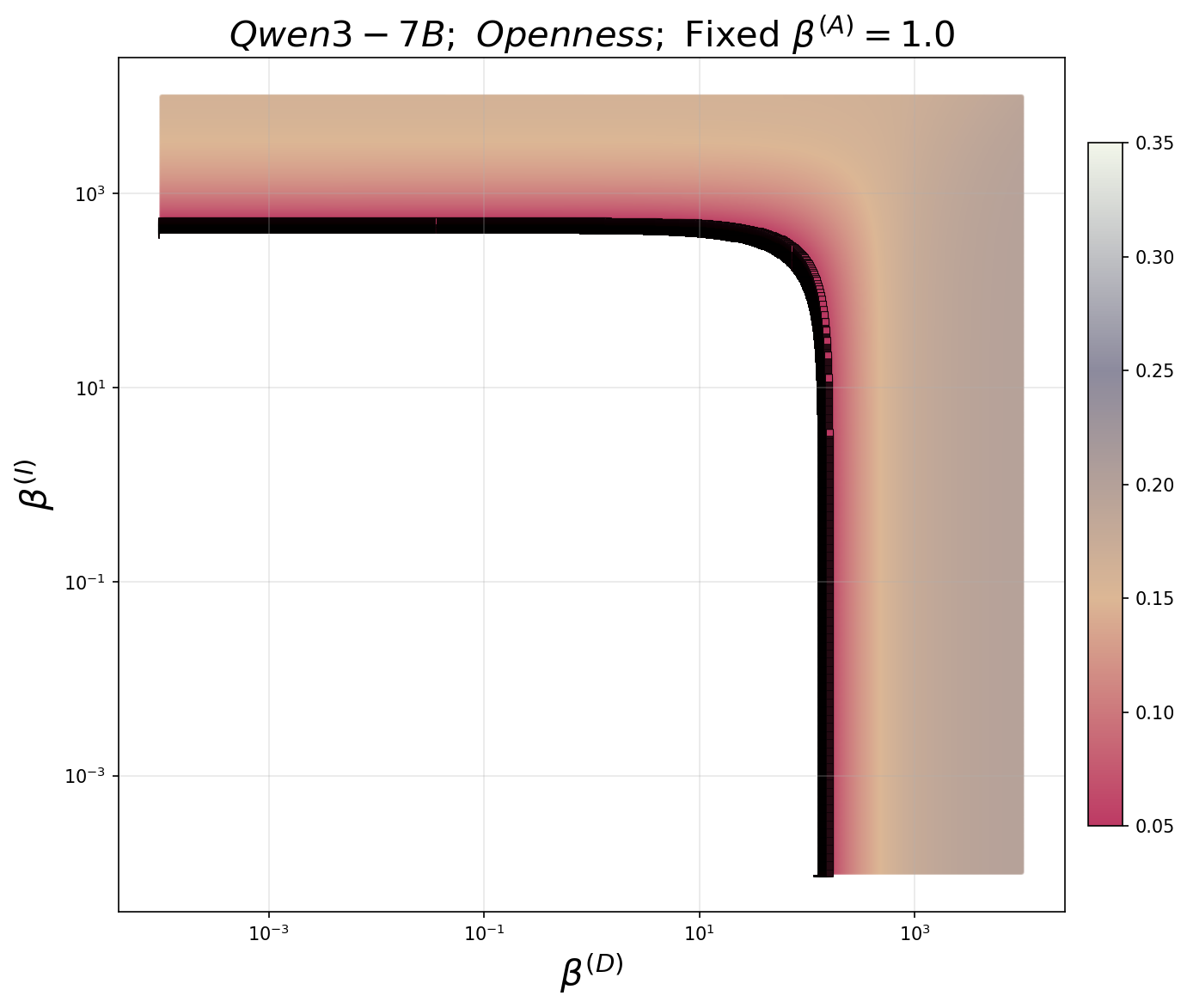}
    \includegraphics[width=0.3\linewidth]{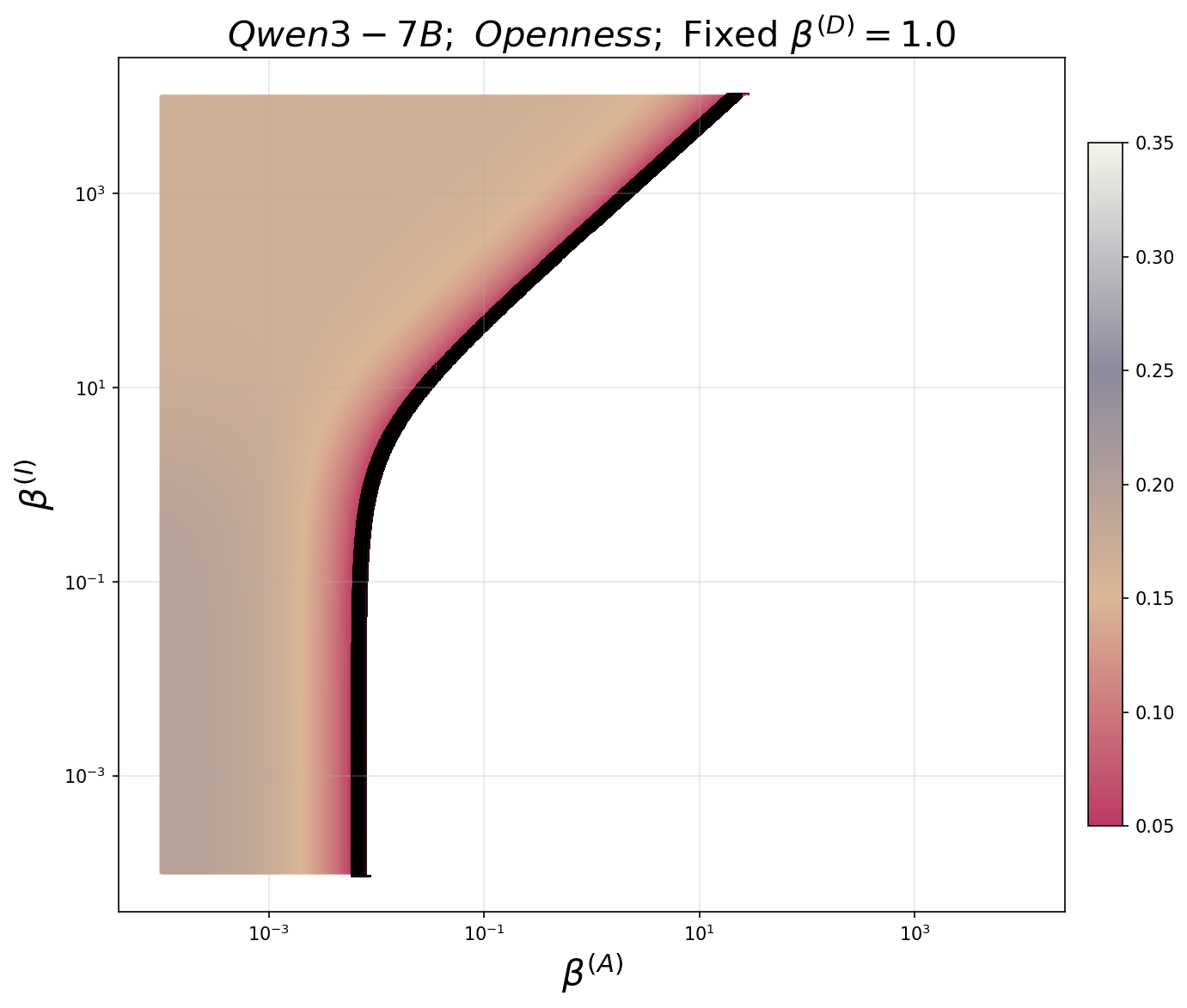}
    \includegraphics[width=0.3\linewidth]{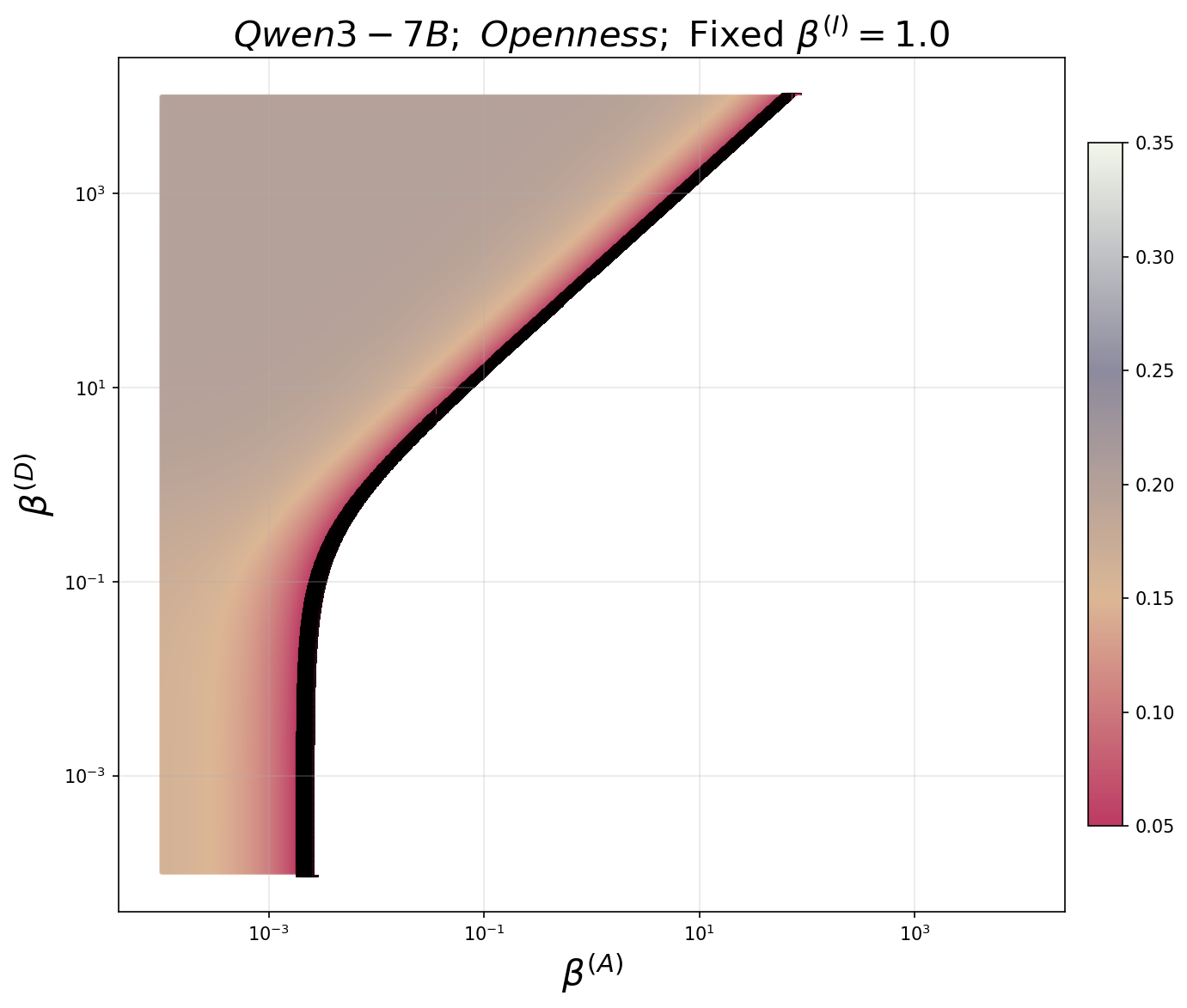}

    \includegraphics[width=0.3\linewidth]{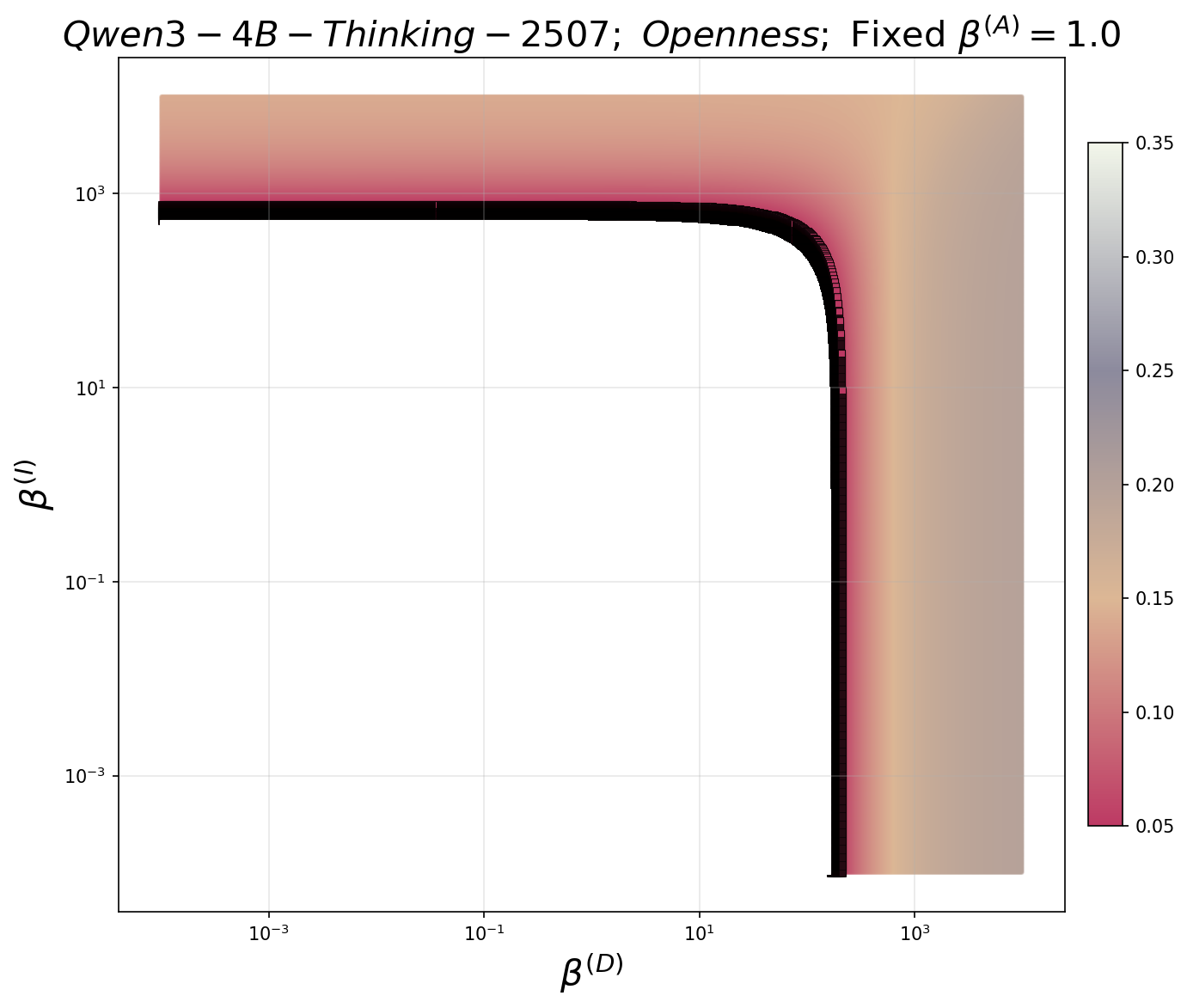}
    \includegraphics[width=0.3\linewidth]{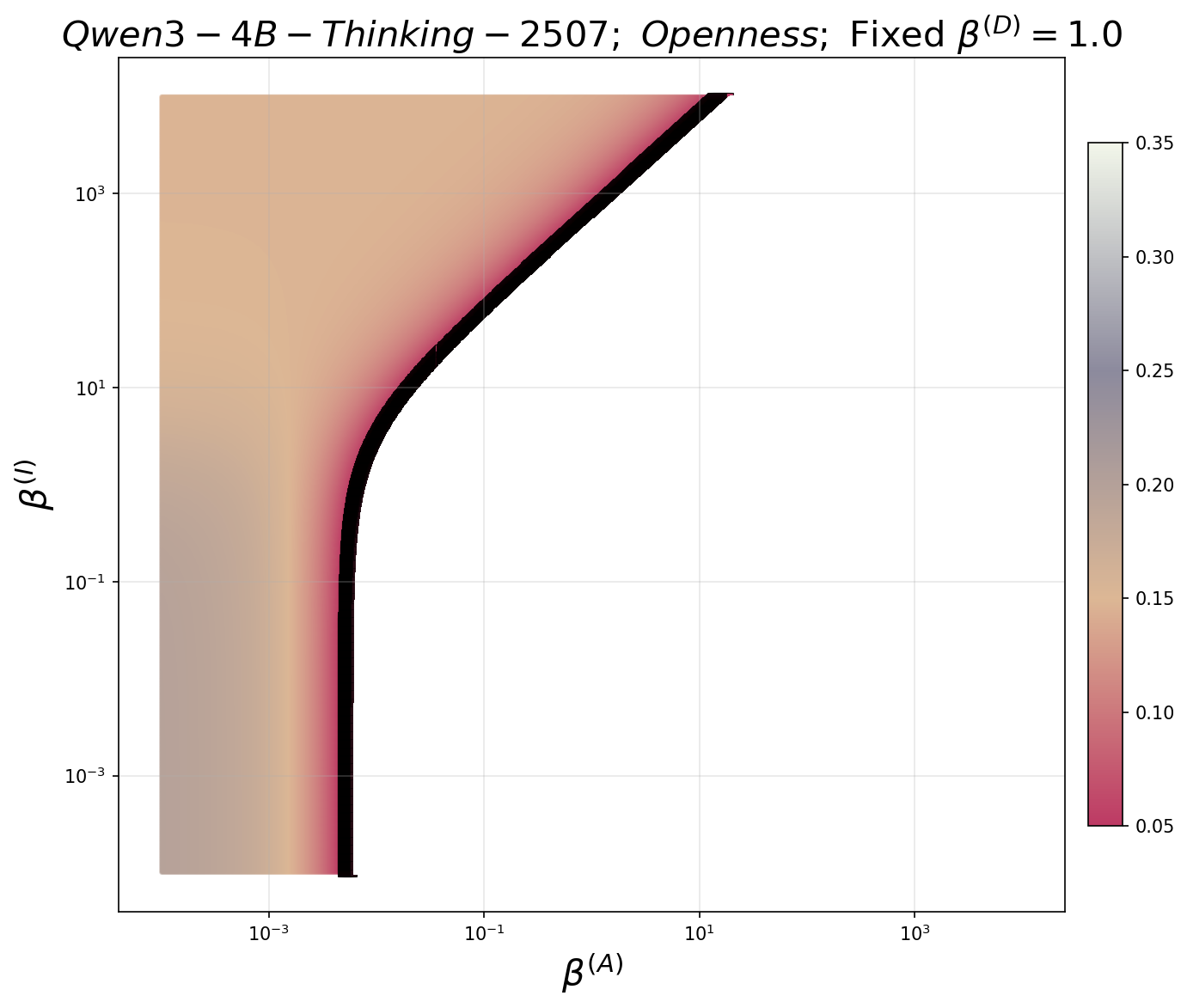}
    \includegraphics[width=0.3\linewidth]{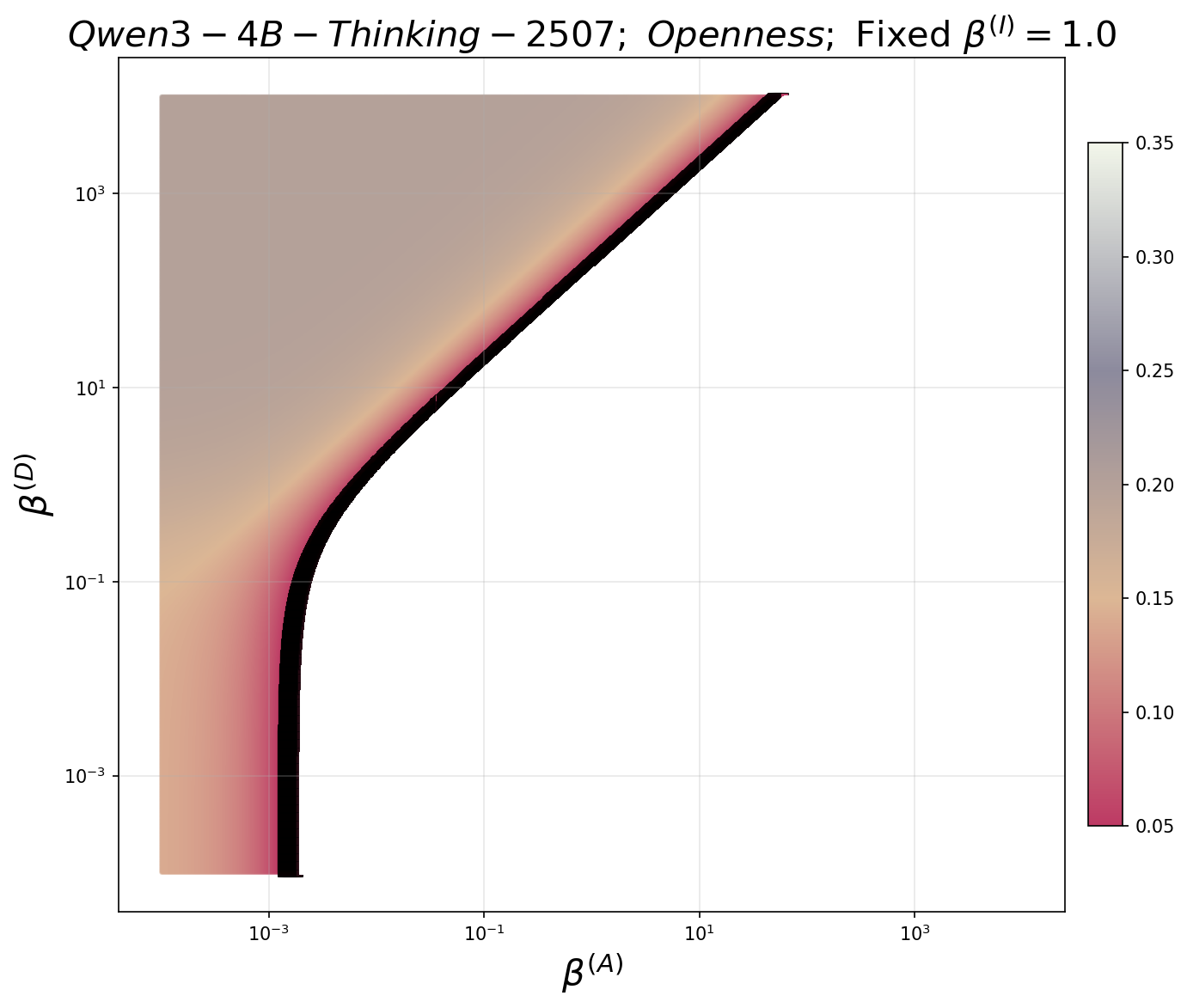}

    \includegraphics[width=0.3\linewidth]{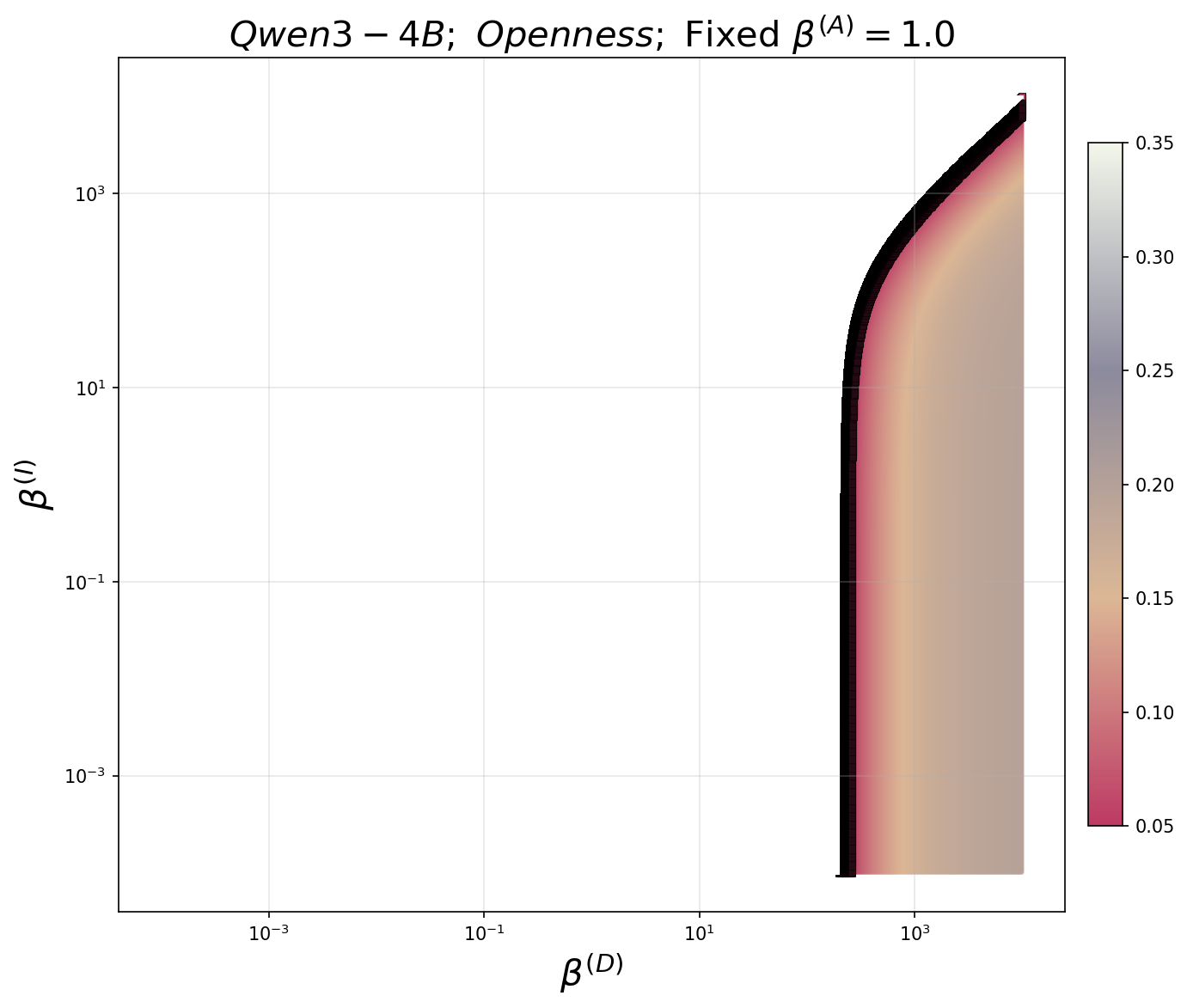}
    \includegraphics[width=0.3\linewidth]{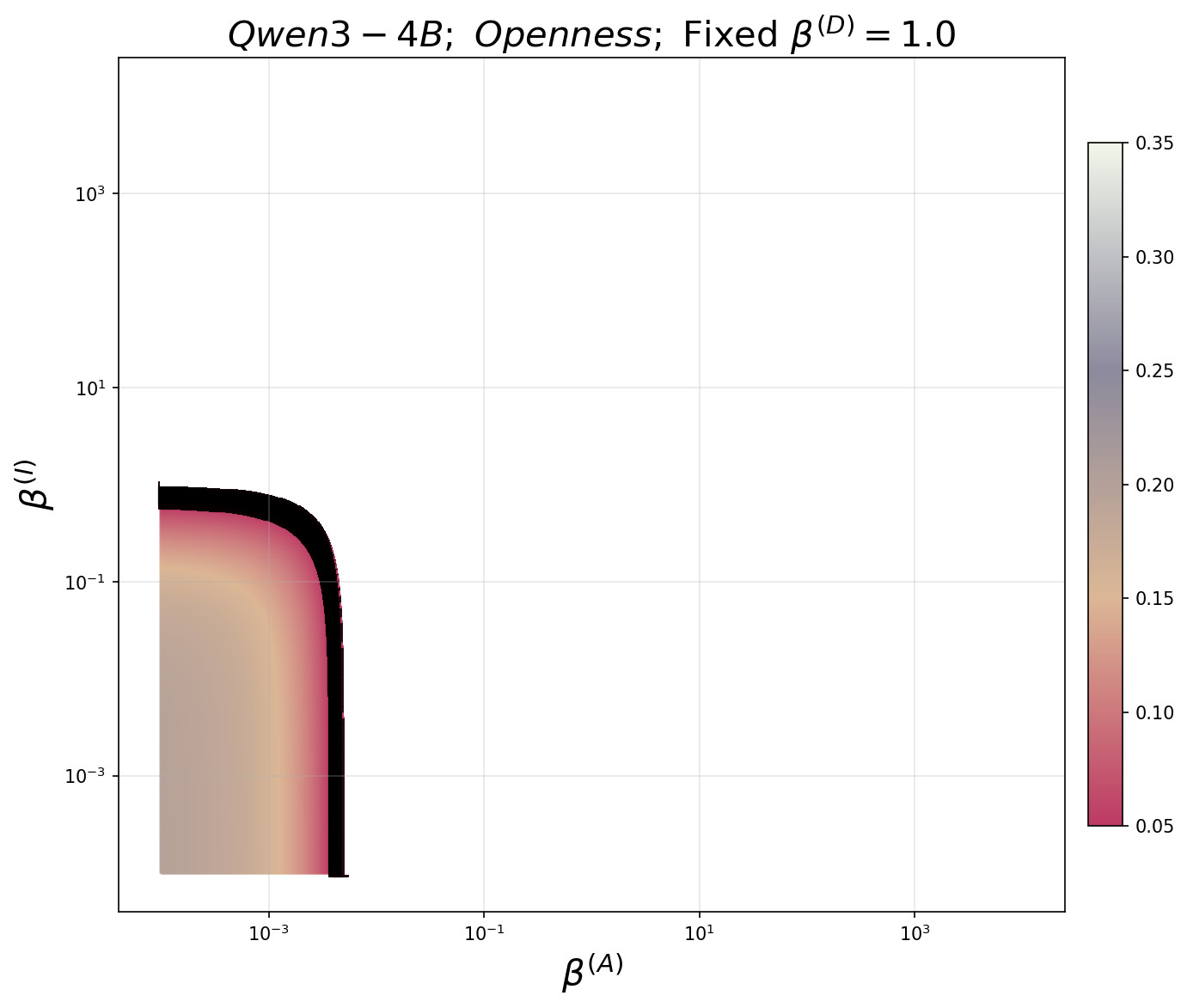}
    \includegraphics[width=0.3\linewidth]{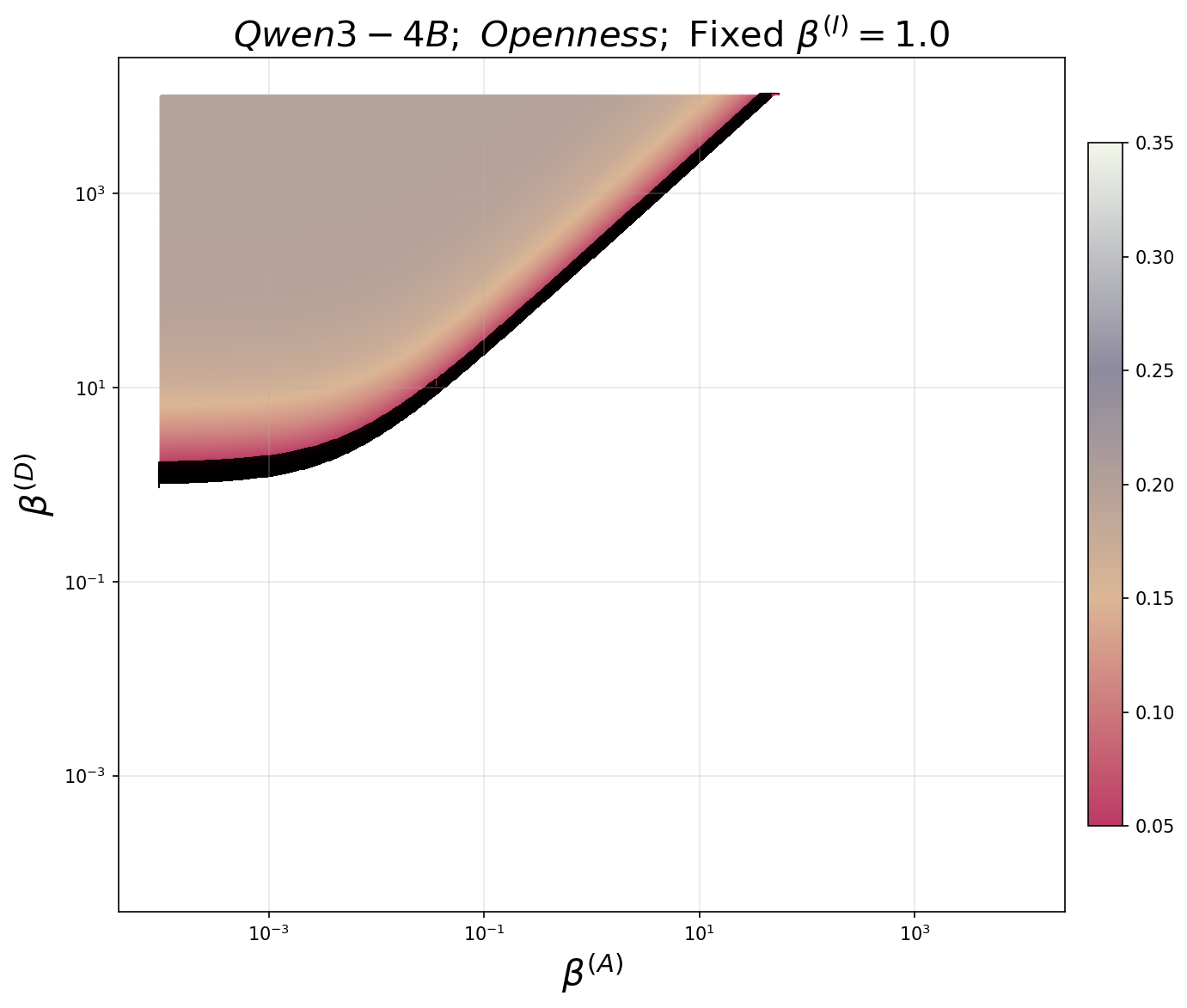}
    
    \caption{Political exclusion regarding the subpopulation \emph{Openness} on the $\mathtt{Big\ Five}$ dataset.}
    \label{fig:Openness}
\end{figure}


\section{More Related Work}

In our paper, we adopt the equilibrium behavior of LLM agents as a governance lever that guides the design of system incentive. This is related to the literature of \emph{differentiable economics}~\cite{zheng2022ai,wang2024gemnet,finocchiaro2021bridging,wang2023deep,ivanov2024principal,zhang2024position,hossain2024multi,rahme2020auction,ivanov2022optimal,curry2022differentiable,duan2023scalable,dutting2021optimal,curry2025automated,wang2025bundleflow} and environment design~\citep{zhang2024position}. 
At a high level, this line of work replaces hand-derived mechanism structure with parameterized architectures, and then searches over designs using optimization methods such as gradient descent. However, it is largely understudied how these automated mechanism design methods can be extended to address problems arising in settings with LLM players.

In recent ML-driven settings, Nash equilibrium computation is often approached with first-order, gradient-based dynamics, applying stabilized gradient schemes such as extra-gradient or extrapolation~\citep{gidel2019vi_gans,mokhtari2020unified,jelassi2020extragradient}. Beyond extrapolation, several works modify the raw gradient dynamics to reduce rotation and improve stability in differentiable games—for example optimism in adversarial learning~\citep{daskalakis2018training}, symplectic gradient adjustment based on a decomposition of game Jacobians~\citep{balduzzi2018mechanics,letcher2019differentiable}, local symplectic surgery~\citep{mazumdar2025finding}, and game-aware local solvers like Competitive Gradient Descent~\citep{schafer2019competitive}. For general-sum games, \citet{song2019convergence} study convergence of a finite step-size gradient-based multi-agent learning dynamic (GA-SPP), providing Nash convergence guarantees without requiring infinitesimal learning rates. \citet{gemp2024approximating} define a loss based on the norm of the projected gradient of players’ expected payoffs onto the product of simplices, so driving this projected-gradient norm to (near) zero via stochastic optimization yields an (approximate) NE. In principle, these methods can be used in our setting to find boundary NE.
